\documentclass{amsart}

\usepackage{stackrel}
\newcommand{\expect}{{\rm I\!E}}
\newcommand{\real}{{\rm I\!R}}

\newcommand{\integer}{{\rm I\!N}}
\newcommand{\prob}{{\rm I\!P}}
\newcommand{\ball}{{\rm I\!B}}

\usepackage{hyperref}
\usepackage{xcolor,colortbl}
\usepackage{hyperref}
\hypersetup{
	colorlinks=true,
	linkcolor=blue,
	filecolor=magenta,      
	urlcolor=cyan,
}

\definecolor{Gray}{gray}{0.85}
\definecolor{LightCyan}{rgb}{0.88,1,1}

\usepackage{algpseudocode}
\usepackage{mathtools}
\usepackage{listings}
\usepackage{color} 
\usepackage{graphicx}
\usepackage{easybmat}
\usepackage{bm}
\usepackage{subfigure}
\usepackage{hyperref}
\usepackage{algorithm}
\usepackage{dsfont}
\usepackage{enumitem}
\usepackage{algpseudocode}
\usepackage{bbm}
\usepackage{dsfont}
\usepackage{epstopdf}
\usepackage{graphicx,subfigure}
\usepackage[labelfont=bf]{caption}
\usepackage[export]{adjustbox}
\usepackage{romannum}
\usepackage{tikz}
\usepackage{lettrine}

\newtheorem{theorem}{Theorem}[section]

\newtheorem{lemma}[theorem]{Lemma}
\newtheorem{corollary}{Corollary}[theorem]
\theoremstyle{definition}
\newtheorem{definition}[theorem]{Definition}

\usepackage[margin=1.5in]{geometry} 
\usepackage{stackrel}

\newcommand{\df}{\stackrel{\text{def}}{=}}

\usepackage{filecontents}
\usepackage{mathrsfs}

\usepackage{hyperref}
\hypersetup{
	colorlinks=true,
	linkcolor=[RGB]{110,25,25},
	citecolor=[RGB]{25,25,110},
	filecolor=magenta,      
	urlcolor=black,
}

\begin{document}
	
	\title[A Mean-Field Theory for Kernel Alignment with Random Features]{\hspace{-3mm}A Mean-Field Theory for Kernel Alignment with Random Features in Generative and Discriminative Models}

	
	

	\author[Masoud Badiei Khuzani, et al.]{Masoud Badiei Khuzani$^{\dagger,\ddagger}$, Liyue Shen$^{\ddagger}$, Shahin Shahrampour$^{\S}$, Lei Xing$^{\ddagger}$}

	\address{Stanford University, 450 Serra Mall, Stanford, CA 94305\\
	}
	\curraddr{}
	\email{mbadieik,liyues,lei@stanford.edu,shahin@tamu.edu.}

	{\let\thefootnote\relax\footnotetext{       
			\small{
				\vspace{-4mm}
				\\$^{\dagger}$Department of Management Science,	Stanford University, CA 94043, USA\\
				$^{\dagger}$Department of Electrical Engineering, Stanford University, CA 94043, USA\\
				$^{\ddagger}$Department of Radiation Oncology,	Stanford University, CA 94043, USA \\
				$^{\S}$Department of Industrial Engineering, Texas A\&M University, TX 77843, USA}}}

\begin{abstract}
We propose a novel supervised learning method to optimize the kernel in the maximum mean discrepancy generative adversarial networks (MMD GANs), and the kernel support vector machines (SVMs). Specifically, we characterize a distributionally robust optimization problem to compute a good distribution for the random feature model of Rahimi and Recht. Due to the fact that the distributional optimization is infinite dimensional, we consider a Monte-Carlo sample average approximation (SAA) to obtain a more tractable finite dimensional optimization problem. We subsequently leverage a particle stochastic gradient descent (SGD) method to solve the derived finite dimensional optimization problem. Based on a mean-field analysis, we then prove that the empirical distribution of the interactive particles system at each iteration of the SGD follows the path of the gradient descent flow on the Wasserstein manifold. We also establish the non-asymptotic consistency of the finite sample estimator. We evaluate our kernel learning method for the hypothesis testing problem by evaluating the kernel MMD statistics, and show that our learning method indeed attains better power of the test for larger threshold values compared to an untrained kernel. Moreover, our empirical evaluation on benchmark data-sets shows the advantage of our kernel learning approach compared to alternative kernel learning methods.
\end{abstract}
	\maketitle
\pagestyle{plain}

\section{Introduction}

\lettrine{\textbf{A}} fundamental step in devising efficient learning algorithms is the design of good feature maps. Such feature maps embed data from an input space to some potentially high-dimensional non-linear feature space, where it is often easier to discriminate data.

In the kernel methods of machine learning, such feature maps and thus the underlying kernel function, must be specified by the user. However, in many practical machine learning applications, it is unclear how to select an appropriate kernel function. In particular, choosing a good kernel often entails some domain knowledge about the underlying learning applications. 
In practice, it can be difficult for the practitioner to find prior justification for the use of one kernel instead of another to generate random feature maps. Moreover, a bad kernel model yields a poor classification accuracy.

Although such kernel model selection issues can be addressed via the cross-validation, jackknife, or their approximate surrogates \cite{beirami2017optimal,wang2018approximate,giordano2018return}, those statistical approaches often slow down the training process as they repeatedly re-fit the model.  Thus, to facilitate the kernel model selection problem in MMD GANs \cite{li2017mmd} and kernel SVMs \cite{friedman2001elements}, in this paper we put forth a novel framework to learn a good kernel function from the training data-set. The proposed kernel learning approach is based on a distributional optimization problem to learn a good distribution for  the random feature model of Rahimi and Recht \cite{rahimi2008random,rahimi2009weighted}. The main contributions of this paper are summarized below:
\begin{itemize}
	\item Based on the notion of \textit{the kernel-target alignment}, we characterize a distributional optimization problem to learn a good distribution for Rahimi and Recht's random feature model of a kernel function \cite{rahimi2008random,rahimi2009weighted}.
	
	\item To obtain a tractable optimization problem, we leverage the sample average approximation (SAA) method to transform the infinite dimensional distributional optimization problem to a finite dimensional optimization problem.
	
	\item We establish the consistency of the finite sample average approximation. In particular, we show that as the number of samples in SAA tends to infinity, the optimal value of the finite sample estimates tend to their population values. 
	
	\item We propose a particle stochastic gradient descent method to solve the finite dimensional optimization problem associated with the Monte-Carlo sampling method. Using a mean-field analysis, we then show that the interactive particle system with SGD dynamics follows the gradient decent path on the Wasserstein manifold to minimize the distributional optimization problem. In this sense, we establish the consistency of the proposed particle SGD method for solving the distributional optimization problem.
	
	\item We validate our kernel learning approach for generative modeling and discriminative analysis on benchmark data-sets. 
\end{itemize}

\subsection{Related works}

The mean-field description of SGD dynamics has been studied in several prior works for different information processing tasks. Wang \textit{et al.} \cite{wang2017scaling} consider the problem of online learning for the principal component analysis (PCA), and analyze the scaling limits of different online learning algorithms based on the notion of \textit{finite exchangeability}.  In their seminal papers, Montanari and co-authors  \cite{mei2018mean,javanmard2019analysis, mei2019mean} consider the scaling limits of SGD for training a two-layer neural network, and characterize the related Mckean-Vlasov PDE for the limiting distribution of the empirical measure associated with the weights of the input layer. They also establish the uniqueness and existence of the solution for the PDE using the connection  between Mckean-Vlasov type PDEs and the gradient flows on the Wasserstein manifolds established by Otto \cite{otto2001geometry}, and Jordan, Kinderlehrer, and Otto \cite{jordan1998variational}. Similar mean-field type results for two-layer neural networks are also studied recently in \cite{rotskoff2018neural,sirignano2018mean}. 

The present application of the mean-field theory to the kernel optimization is partly motivated by the work of Sinha and Duchi \cite{sinha2016learning}, which studied a distributional optimization for kernel learning with random features in the context of classification problems. Therein, the authors have proposed a robust optimization framework for the importance sampling of the random features. In contrast to the work of \cite{sinha2016learning} that assign a weight (importance) to each sample and optimizes the weights, in this paper we directly optimize the samples.

For generative modeling using MMD GANs, kernel learning has been studied in several prior works \cite{li2019implicit,wang2018improving,gretton2012optimal}. Nevertheless, those methods are either based on heuristics, or are difficult to characterize theoretically. Our work is also related to the unpublished work of Wang, \textit{et al.} \cite{wangsolvable}, which proposes a solvable model of GAN and analyzes the scaling limits. However, our GAN model is significantly different from \cite{wangsolvable} and is based on the notion of the kernel MMD.  Our work for kernel learning in generative modeling is also closely related to the recent work of Li, \textit{et al} \cite{li2019implicit} which proposes an implicit kernel learning method based on the following kernel definition
\begin{align}
K_{h}(\iota(\bm{x}),\iota(\bm{y}))=\expect_{\bm{\xi}\sim \mu_{0}}\left[e^{(ih(\bm{\xi})(\iota(\bm{x})-\iota(\bm{y})))}\right],
\end{align}
where $\mu_{0}$ is a user defined base distribution, and $h\in \mathcal{H}$ is a functions that transforms the base distribution $\mu_{0}$ into a distribution $\mu$ that provides a better kernel. Therefore, the work of Li, \textit{et al} \cite{li2019implicit} \textit{implicitly} optimizes the distribution of random features through a function. In contrast, the proposed distributional optimization framework in this paper optimizes the distribution of random feature \textit{explicitly}, via optimizing their empirical measures. Perhaps more importantly from a practical perspective is the fact that our kernel learning approach does not require the user-defined function class $\mathcal{H}$. Moreover, our particle SGD method in \eqref{Eq:SGD} obviates tuning hyper-parameters related to the implicit kernel learning method such as the gradient penalty factor and the variance constraint factor (denoted  by $\lambda_{GP}$ and $\lambda_{h}$, respectively, in Algorithm 1 of \cite{li2019implicit}).

For classification problems using kernel SVMs, Cortes, \textit{et al}. studied a kernel learning procedure from the class of mixture of base kernels.  They have also studied the generalization bounds of the proposed methods. The same authors have also studied a two-stage kernel learning in \cite{cortes2010two} based on a notion of alignment. The first stage of this technique consists of learning a kernel that is a convex combination of a set of base kernels. The second stage consists of using the learned kernel with a standard kernel-based learning algorithm such as SVMs to select a prediction hypothesis.  In \cite{lanckriet2004learning}, the authors have proposed a semi-definite programming for the kernel-target alignment problem. An alternative approach to kernel learning is to sample random features from an arbitrary distribution (without tuning its hyper-parameters) and then apply a supervised feature screening method to distill random features with high discriminative power from redundant random features. This approach has been adopted by Shahrampour \textit{et al}. \cite{shahrampour2017data}, where an energy-based exploration of random features is proposed. The feature selection algorithm in \cite{shahrampour2017data} employs a score function based on the kernel polarization techniques of \cite{baram2005learning} to explore the domain of random features and retains samples with the highest score.

The rest of this paper is organized as follows. In Section \ref{Section:Generative Moment Matching Network}, we propose the kernel learning approach for the generative modeling. In Section \ref{Section:Kernel learning in discriminative models}, we study the kernel learning approach for the discriminative modeling. In Section \ref{Section:Theoretical Results}, we present the main theoretical results, while deferring their proofs to Appendix. In Section \ref{Section:Experiments_on_Synthetic_Data_Set}, we evaluate our kernel learning approach for generative modeling using synthetic data-set and CIFAR-10 and MNIST benchmark data-sets. In Section \ref{Section:Empirical Results for Classification Tasks}, we provide empirical evaluations of our proposed kernel learning approach for classification via the kernel SVMs, where we evaluate our method on some benchmark data-sets. Lastly, in Section \ref{Section:Conclusion}, we conclude this paper.

\section{Kernel learning in generative models}
\label{Section:Generative Moment Matching Network}

Assume we are given data $\{\bm{v}_{i}\}_{i=1}^{n}$ that are sampled from an unknown distribution $P_{\bm{V}}$ with the support $\mathcal{V}$. In many unsupervised tasks, we wish to attain new samples from the distribution $P_{\bm{V}}$ without directly estimating it. Generative Adversarial Network (GAN) \cite{goodfellow2014generative} provides such a framework. In vanilla GAN, a deep network $\mathcal{G}(\cdot;\bm{W})$ parameterized by $\bm{W}\in \mathcal{W}$ is trained as a generator to transform the samples $\bm{Z}\sim P_{\bm{Z}},\bm{Z}\in \mathcal{Z}$ from a user-defined distribution $P_{\bm{Z}}$ (\textit{e.g.} Gaussian distribution) into a new sample $\mathcal{G}(\bm{Z};\bm{W}) \sim P_{\bm{W}}$, such that the distributions $P_{\bm{W}}$ and $P_{\bm{V}}$ are close under some specified metric. In addition, a discriminator network $\mathcal{D}(\cdot;\bm{\delta})$ parameterized by $\bm{\delta}\in \Delta$ is also trained to reject or accept the generated samples as a realization of the data distribution. The training of the generator and discriminator networks is then accomplished via solving a minimax optimization problem as below
\begin{align}
\label{Eq:MinMax_Problem}
\min_{\bm{W}\in \mathcal{W}}\max_{\bm{\delta}\in \Delta}\expect_{ P_{\bm{V}}}[\mathcal{D}(\bm{X};\bm{\delta})]+\expect_{P_{\bm{Z}}}[\log(1-\mathcal{D}(\mathcal{G}(\bm{Z};\bm{W});\bm{\delta}))].
\end{align}
In the high dimensional regimes, the generator trained via the minimax program of Eq. \eqref{Eq:MinMax_Problem} can potentially collapse to a single mode of distribution where it always emits the same point \cite{che2016mode}. To overcome this shortcoming, other adverserial generative models are developed in the literature, which propose to modify or replace the discriminator network by a statistical two-sample test based on the notion of the maximum mean discrepancy (MMD). Below, we formalize the notion of MMD:
\begin{definition}\textsc{(Maximum Mean Discrepancy, \cite{gretton2007kernel})}
	\label{Def:MMD}
	Let $(\mathcal{X},d)$ be a metric space, $\mathcal{F}$ be a class of functions $f:\mathcal{X}\rightarrow \real$, and $P,Q\in \mathcal{B}(\mathcal{X})$ be two probability measures from the set of all Borel probability measures $\mathcal{B}(\mathcal{X})$ on $\mathcal{X}$. The maximum mean discrepancy (MMD) between the distributions $P$ and $Q$ with respect to the function class $\mathcal{F}$ is defined below
	\begin{align}
	\label{Eq:different_choices}
	\mathrm{MMD}_{\mathcal{F}}[P,Q]\df \sup_{f\in \mathcal{F}}\int_{\mathcal{X}} f(\bm{x})(P-Q)(\mathrm{d}\bm{x}).
	\end{align}
\end{definition}
Adapting different function classes $\mathcal{F}$ in Eq. \eqref{Eq:different_choices} of Definition \eqref{Def:MMD} yield different adversarial models such as Wasserstein GANs (WGAN) \cite{arjovsky2017wasserstein}, $f$-GANs \cite{nowozin2016f}, GMMN \cite{li2017mmd}, and MMD GAN \cite{li2015generative}. In the latter two cases, the function class $\mathcal{F}$ corresponds to a  reproducing kernel Hilbert space (RKHS) of functions with a kernel $K:\mathcal{X}\times \mathcal{X}\rightarrow \real$, denoted by $(\mathcal{H}_{\mathcal{X}},K)$. For RKHS function class, the \textit{squared} MMD in Eq. \eqref{Eq:different_choices} between the distributions $P=P_{\bm{V}}$ and $Q=P_{\bm{W}}$ has the following expression\footnote{We note that the notation $\mathrm{MMD}_{K}[P_{\bm{V}},P_{\bm{W}}]$ is conventionally used to denote the square root of the expression on the right hand side of Eq. \eqref{Eq:kernel_MMD}. However, to simplify the arguments in subsequent results and their proofs, we adapt the notation of Eq. \eqref{Eq:kernel_MMD}.}
\begin{align}
\nonumber
\mathrm{MMD}_{K}[P_{\bm{V}},P_{\bm{W}}]&\df \sup_{f\in \mathcal{H}_{\mathcal{X}}}\left\|\int_{\mathcal{X}} K(\cdot,\bm{x})(\mathrm{d}P_{\bm{V}}(\bm{x})-\mathrm{d}P_{\bm{W}}(\bm{x}))\right\|^{2}_{\mathcal{H}_{\mathcal{X}}}\\ \label{Eq:kernel_MMD}
&=\expect_{P^{\otimes 2}_{\bm{V}}}[K(\bm{V};\bm{V}')]+\expect_{ P^{\otimes 2}_{\bm{W}}}[K(\bm{W},\bm{W}')]-2\expect_{P_{\bm{V}}, P_{\bm{W}}}[K(\bm{V};\bm{W})], 
\end{align}
where $\mathcal{X}=\mathcal{V}\cup \mathcal{W}$. 

Instead of training the generator via solving the Min-Max optimization in Eq. \eqref{Eq:MinMax_Problem}, the MMD GAN model of \cite{li2015generative} proposes to optimize the discrepancy between two distributions via optimization of an embedding function $\iota:\real^{d}\mapsto \real^{d_{0}},d_{0}\ll d$, \textit{i.e.},
\begin{align}
\label{Eq:Optimization_Framework}
\min_{\bm{W}\in \mathcal{W}}\max_{\iota\in \mathcal{Q}} {\mathrm{MMD}}_{k\circ \iota}[P_{\bm{V}},P_{\bm{W}}],
\end{align}
where  $k:\real^{d_{0}}\times \real^{d_{0}}\rightarrow \real$ is a user-defined fixed kernel. In \cite{li2015generative}, the proposal for the kernel $k:\real^{d_{0}}\times \real^{d_{0}}\rightarrow \real$ is a mixture of the Gaussians,
\begin{align}
\label{Eq:Gaussian_Mixture}
k\circ \iota (\bm{x},\bm{y})=k(\iota(\bm{x}),\iota(\bm{y}))=\sum_{i=1}^{m}\exp\left(\dfrac{\|\iota(\bm{x})-\iota(\bm{y})\|_{2}^{2}}{\sigma^{2}_{i}}\right),
\end{align}
where the bandwidth parameters $\sigma_{1},\cdots,\sigma_{m}>0$. However, there are two challenges associated with the optimization of the embedding map with a per-defined Gaussian mixture  kernel:
\begin{itemize}
	\item[(\textit{i})] The Gaussian mixture kernel is not expected to be optimal. Even if the Gaussian mixture kernel  yields a good performance on a given data-set, there are no theoretical guarantees to suggest that it performs well across all data-sets that are encountered in practice.
	
	\item[(\textit{ii})] When the Gaussian mixture kernel model is admissible,  the choice of the Gaussian bandwidths $\sigma_{1},\cdots,\sigma_{m}$ poses a statistical model selection problem. Although such model selection issues can be addressed using the cross-validation or jackknife, in practice such methods slow down the training process due to repeatedly re-fitting the model.
\end{itemize}
To address the aforementioned issues, in this paper we propose to optimize the MMD loss with respect to the underlying kernel 
\begin{align}
\label{Eq:Minimax_MMD}
\min_{\bm{W}\in \mathcal{W}}\max_{K\in \mathcal{K}} {\mathrm{MMD}}_{K}[P_{\bm{V}},P_{\bm{W}}],
\end{align}
over a suitable kernel class $\mathcal{K}$. Notice that the optimization of the embedding function in Eq. \eqref{Eq:Optimization_Framework} is a special case of the general kernel learning problem formulated in Eq. \eqref{Eq:Minimax_MMD} when the kernel class is selected as $\mathcal{K}\df \{K: K(\bm{x},\bm{y})=k(\iota(\bm{x}),\iota(\bm{y})),\iota\in \mathcal{Q}\}$.

\subsection{Robust distributional optimization for kernel learning}
\label{subection:Kernel-Target Alignment as an Unbiased Estimator of MMD Loss}

To address the kernel model selection issue in MMD GAN \cite{li2017mmd}, we consider a kernel optimization scheme with random features \cite{rahimi2008random,rahimi2009weighted}. Let $\varphi: \real^{d}\times \real^{D} \rightarrow [-1, 1]$ denotes the explicit feature maps and $\mu\in \mathcal{M}(\real^{D})$ denotes a probability measure from the space of probability measures $\mathcal{M}(\real^{D})$ on $\real^{D}$. The kernel function is characterized via the explicit feature maps using the following integral equation
\begin{align}
\label{Eq:Generative_Distribution}
K_{\mu}(\bm{x},\bm{y})=\expect_{\mu}[\varphi(\bm{x};\bm{\xi})\varphi(\bm{y};\bm{\xi})]=\int_{\Xi} \varphi(\bm{x};\bm{\xi})\varphi(\bm{y};\bm{\xi})\mu(\mathrm{d}\bm{\xi}).
\end{align}
Let $\mathrm{MMD}_{\mu}[P_{\bm{V}},P_{\bm{W}}]\df \mathrm{MMD}_{K_{\mu}}[P_{\bm{V}},P_{\bm{W}}]$. Then, the kernel optimization problem in can be formulated as a distribution optimization for random features, \textit{i.e},
\begin{align}
\label{Eq:Population_MMD_Loss}
\min_{\bm{W}\in \mathcal{W}}\sup_{\mu\in \mathcal{P}} \mathrm{MMD}_{\mu}[P_{\bm{V}},P_{\bm{W}}].
\end{align}
Here,  $\mathcal{P}$ is the set of probability distributions corresponding to a kernel class $\mathcal{K}$. In the sequel, we consider $\mathcal{P}$ to be the distribution ball of radius $R$ as below
\begin{align}
\label{Eq:distributional_ball}
\mathcal{P}\df \ball^{p}_{R}(\mu_{0})\df \{\mu\in \mathcal{M}(\real^{D}):W_{p}({\mu,\mu_{0}})\leq R\},
\end{align}
where $\mu_{0}$ is a user-defined base distribution, and 
To establish the proof, we consider the 
$W_{p}(\cdot,\cdot)$ is the $p$-Wasserstein distance defined as below
\begin{align}
\label{Eq:Wasserstein}
W_{p}(\mu_{1},\mu_{2})\df  \left(\inf_{\pi \in \Pi(\mu_{1},\mu_{2})} \int_{\real^{D}\times \real^{D}}\|\bm{\xi}_{1}-\bm{\xi}_{2}\|^{p}_{2} \mathrm{d}\pi(\bm{\xi}_{1},\bm{\xi}_{2})\right)^{1\over p},
\end{align}
where the  infimum is taken with respect to all couplings $\pi$ of the measures $\mu,\mu_{0}\in \mathcal{M}(\real^{D})$, and $\Pi(\mu,\mu_{0})$ is the set of all such couplings with the marginals $\mu$ and $\mu_{0}$.

The kernel MMD loss function in Eq. \eqref{Eq:Population_MMD_Loss} is defined with respect to the unknown distributions of the data-set $P_{\bm{V}}$ and the model $P_{\bm{W}}$. Therefore, we construct an unbiased estimator for the MMD loss function in Eq. \eqref{Eq:Population_MMD_Loss} based on  the training samples. To describe the estimator, sample the labels from a uniform distribution $y_{1},\cdots,y_{n}\sim_{\text{i.i.d.}}\mathrm{Uniform}\{-1,+1\}$, where we assume that the number of positive and negative labels are balanced. In particular, consider the set of positive labels $\mathcal{I}=\{i\in \{1,2,\cdots,n\}: y_{i}=+1\}$, and negative labels $\mathcal{J}=\{1,2,\cdots,n\}/\mathcal{I}$, where their cardinality is $|\mathcal{I}|=|\mathcal{J}|={n\over 2}$. We consider the following assignment of labels:
\begin{itemize}[leftmargin=*]
	\item \textit{Positive class labels}: If $y_{i}=+1$, sample the corresponding feature map from data-distribution $\bm{x}_{i}=\bm{v}_{i}\sim  P_{\bm{V}}$. 
	\item \textit{Negative class labels}: If $y_{i}=-1$, sample from the corresponding feature map from the generated distribution $\bm{x}_{i}=\mathcal{G}(\bm{Z}_{i},\bm{W})\sim P_{\bm{W}},\bm{Z}_{i}\sim P_{\bm{Z}}$. 
\end{itemize}
By this construction, the joint distribution of features and labels $P_{Y,\bm{X}}$ has the marginals $P_{\bm{X}|Y=+1}=P_{\bm{V}}$, and $P_{\bm{X}|Y=-1}=P_{\bm{W}}$. Moreover, the following statistic, known as the \textit{kernel alignment}  in the literature (see, \textit{e.g.}, \cite{sinha2016learning,cortes2012algorithms}), is an unbiased estimator of the MMD loss in  Eq. \eqref{Eq:Population_MMD_Loss},
\begin{align}
\label{Eq:standard_kernel_target}
\min_{\bm{W}\in \mathcal{W}}\sup_{\mu\in \mathcal{P}}\widehat{{\mathrm{MMD}}}_{\mu}\Big[P_{\bm{V}},P_{\bm{W}}\Big]\df {8\over n(n-1)}\sum_{1\leq i<j\leq n} y_{i}y_{j}K_{\mu}(\bm{x}_{i},\bm{x}_{j}).
\end{align}
See Appendix \ref{App:Proof of Theorem_1} for the related proof. The kernel alignment in  \eqref{Eq:standard_kernel_target} can also be viewed through the lens of the risk minimization 
\begin{subequations}
	\label{Eq:Population_Objective_Function_0}
	\begin{align}
	\min_{\bm{W}\in \mathcal{W}}\inf_{\mu\in \mathcal{P}}  \widehat{\mathrm{MMD}}^{\alpha}_{\mu}\Big[P_{\bm{V}},P_{\bm{W}}\Big]&\df {8\over n(n-1)\alpha}\sum_{1\leq i<j\leq n}\left(\alpha y_{i}y_{j}-K_{\mu}(\bm{x}_{i},\bm{x}_{j})\right)^{2}\\
	&={8\over n(n-1)\alpha}\sum_{1\leq i<j\leq n}\left(\alpha y_{i}y_{j}-\expect_{\mu}[\varphi(\bm{x}_{i};\bm{\xi})\varphi(\bm{x}_{j};\bm{\xi})]\right)^{2}.
	\end{align}
\end{subequations}
Here, $\alpha>0$ is a scaling factor that determines the separation between feature vectors, and  $\bm{K}_{\ast}\df \alpha \bm{y}\bm{y}^{T}$ is the ideal kernel that provides the maximal separation between the feature vectors over the training data-set, \textit{i.e.}, $K_{\ast}(\bm{x}_{i},\bm{x}_{j})=\alpha$ when features have identical labels $y_{i}=y_{j}$, and $K_{\ast}(\bm{x}_{i},\bm{x}_{j})=-\alpha$ otherwise. Upon expansion of the risk function in \eqref{Eq:Population_Objective_Function_0}, it can be easily shown that it reduces to the kernel alignment in \eqref{Eq:standard_kernel_target} when $\alpha\rightarrow +\infty$. Intuitively, the risk minimization in \eqref{Eq:Population_Objective_Function_0} gives a feature space in which pairwise distances are similar to those in the output space $\mathcal{Y}=\{-1,+1\}$.

\subsection{SAA for distributional optimization}

The distributional optimization problem in Eq. \eqref{Eq:Population_MMD_Loss} is infinite dimensional, and thus cannot be solved directly. To obtain a tractable optimization problem, instead of optimizing with respect to the distribution $\mu$ of random features, we optimize the i.i.d. samples (particles) $\bm{\xi}^{1},\cdots,\bm{\xi}^{N}\sim_{\text{i.i.d.}}\mu$ generated from the distribution. The empirical distribution of these particles is accordingly defined as follows
\begin{align}
\widehat{\mu}^{N}(\bm{\xi})\df {1\over N}\sum_{k=1}^{N} \delta(\bm{\xi}-\bm{\xi}^{k}),
\end{align}
where $\delta(\cdot)$ is the Dirac's delta function concentrated at zero. In practice, the optimization problem in Eq. \eqref{Eq:Population_Objective_Function_0} is solved via the Monte-Carlo sample average approximation of the objective function,
\small{
	\begin{subequations}
		\label{Eq:Empirical_Objective_Function}
		\begin{align}
		\nonumber
		\min_{\bm{W}\in \mathcal{W}} \min_{\widehat{\mu}^{N}\in \mathcal{P}_{N}}  \widehat{\mathrm{MMD}}^{\alpha}_{\widehat{\mu}^{N}}\Big[P_{\bm{V}},P_{\bm{W}}\Big]
		&={8\over n(n-1)\alpha}\sum_{1\leq i<j\leq n}\Big(\alpha y_{i}y_{j} -\expect_{\widehat{\mu}^{N}}[\varphi(\bm{x}_{i};\bm{\xi}^{k})\varphi(\bm{x}_{j};\bm{\xi}^{k})]\Big)^{2}\\
		&={8\over n(n-1)\alpha}\sum_{1\leq i<j\leq n}\Big(\alpha y_{i}y_{j} -{1\over N}\sum_{k=1}^{N}\varphi(\bm{x}_{i};\bm{\xi}^{k})\varphi(\bm{x}_{j};\bm{\xi}^{k})\Big)^{2},
		\end{align}
	\end{subequations}
}\normalsize
where $\mathcal{P}_{N}\df \ball^{N}_{R}(\widehat{\mu}_{0}^{N}) =\Big\{\widehat{\mu}^{N}\in \mathcal{M}(\real^{D}):W_{p}(\widehat{\mu}^{N},\widehat{\mu}_{0}^{N})\leq R\Big\}$, and $\widehat{\mu}_{0}^{N}$ is the empirical measure associated with the samples $\bm{\xi}_{0}^{1},\cdots,\bm{\xi}_{0}^{N}\sim_{\text{i.i.d.}}\mu_{0}$. The empirical objective function in Eq. \eqref{Eq:Empirical_Objective_Function} can be optimized with respect to the samples $\bm{\xi}^{1},\cdots,\bm{\xi}^{N}$ using the particle stochastic gradient descent. For the optimization problem in Eq. \eqref{Eq:Empirical_Objective_Function}, the (projected) stochastic gradient descent (SGD) takes the following recursive form,\footnote{To avoid clutter in our subsequent analysis, the normalization factor ${16\over n(n-1)}$ of the gradient is omitted by modifying the step-size $\eta$.}

\begin{subequations}
	\label{Eq:SGD}
	\small{\begin{align}
		\bm{\xi}^{k}_{m+1}&=\bm{\xi}^{k}_{m}-{\eta\over N}\left(y_{m}\widetilde{y}_{m}- {1\over \alpha N}\sum_{k=1}^{N}\varphi(\bm{x}_{m};\bm{\xi}_{m}^{k})\varphi(\widetilde{\bm{x}}_{m};\bm{\xi}_{m}^{k})\right)\nabla_{\bm{\xi}}\Big(\varphi(\bm{x}_{m};\bm{\xi}_{m}^{k})\varphi(\widetilde{\bm{x}}_{m};\bm{\xi}_{m}^{k})\Big),
		\end{align}}\normalsize
\end{subequations}
for $k=1,2,\cdots,N$, where $(y_{m},\bm{x}_{m}),(\widetilde{y}_{m},\widetilde{\bm{x}}_{m})\sim_{\text{i.i.d}} P_{\bm{x},y}$ and $\eta \in \real_{> 0}$ denotes the learning rate of the algorithm, and the initial particles are $\bm{\xi}_{0}^{1},\cdots,\bm{\xi}_{0}^{N}\sim_{\text{i.i.d.}}\mu_{0}$.  At each iteration of the SGD dynamic in Eq. \eqref{Eq:SGD}, a feasible solution for the inner optimization of the empirical risk function in Eq. \eqref{Eq:Empirical_Objective_Function} is generated via the empirical measure
\begin{align}
\label{Eq:Empirical_Measure}
\widehat{\mu}_{m}^{N}(\bm{\xi})=\dfrac{1}{N}\sum_{k=1}^{N}\delta(\bm{\xi}-\bm{\xi}^{k}_{m}).
\end{align}
Indeed, we prove in Section \ref{Section:Theoretical Results} that for an appropriate choice of the learning rate $\eta>0$, the empirical measure in Eq. \eqref{Eq:Empirical_Measure} remains inside the distribution ball $\widehat{\mu}_{m}^{N}\in \mathcal{P}_{N}$ for all $m\in [0,NT]\cap \integer$, and is thus a feasible solution for the empirical risk minimization \eqref{Eq:Empirical_Objective_Function} (see Corollary \ref{Corollary:1} in Section \ref{Section:Theoretical Results}).

\subsection{Proposed MMD GAN with kernel learning} 

In Algorithm \ref{Algorithm:1}, we describe the proposed method MMD GAN model with the kernel learning approach described earlier. Algorithm \ref{Algorithm:1} has an inner loop for the kernel training and an outer loop for training the generator, where we employ RMSprop \cite{tieleman2012lecture}. Our proposed MMD GAN model is distinguished from MMD GAN of \cite{li2017mmd} in that we learn a good kernel function in Eq. \eqref{Eq:distinguished} of the inner loop instead of optimizing the embedding function that is implemented by an auto-encoder. However, we mention that our kernel learning approach is compatible with the auto-encoder implementation of \cite{li2017mmd} for dimensionality reduction of features (and particles).  In the case of including an auto-encoder, the inner loop in Algorithm \ref{Algorithm:1} must be modified to add an additional step for training the auto-encoder. However, to convey the main ideas more clearly, the training step of the auto-encoder is omitted from Algorithm \ref{Algorithm:1}.

\begin{algorithm}[t!]\scriptsize{
		\caption{\small{MMD GAN with a supervised kernel learning (Monte-Carlo Approach)}}
		\label{Algorithm:1}
		\begin{algorithmic}
			\State {\bfseries Inputs:} {The learning rates $\tilde{\eta},\eta>0$ , the number of iterations of discriminator per generator update $T\in \mathbb{N}$, the batch-size $n$, the number of random features $N\in \mathbb{N}$. Regularization parameter $\alpha>0$.}
			\While{$\bm{\omega}$ has not converged}
			\For{$t=1,2,\cdots,T$}
			\State{Sample the labels $y,\widetilde{y}\sim_{\text{i.i.d}} \mathrm{Uniform}\{-1,1\}$. }
			
			\State{Sample the features $\bm{x}|y=+1\sim P_{\bm{V}}$, and $\bm{x}|y=-1\sim P_{\bm{W}}$. Similarly,  $\widetilde{\bm{x}}|\widetilde{y}=+1\sim P_{\bm{V}}$, and $\widetilde{\bm{x}}|\widetilde{y}=-1\sim P_{\bm{W}}$.} 
			
			\State{For all $k=1,2,\cdots,N$, update the particles,}
			\begin{align}
			\label{Eq:distinguished}
			\bm{\xi}^{k}\leftarrow\bm{\xi}^{k}-{\eta\over N}\left(\alpha y\widetilde{y}- {1\over N}\sum_{k=1}^{N}\varphi(\bm{x};\bm{\xi}^{k})\varphi(\widetilde{\bm{x}};\bm{\xi}^{k})\right)\nabla_{\bm{\xi}}\Big(\varphi(\bm{x};\bm{\xi}^{k})\varphi(\widetilde{\bm{x}};\bm{\xi}^{k})\Big),
			\end{align}
			\EndFor   
			
			\State{Sample a balanced minibatch of labels $\{y_{i}\}_{i=1}^{n}\sim_{\text{i.i.d.}} \mathrm{Uniform}\{-1,+1\}$.}
			
			\State{Sample the minibatch $\{\bm{x}\}_{i=1}^{n}$ such that $\bm{x}_{i}|y_{i}=+1\sim P_{\bm{V}}$, and $\bm{x}_{i}|y_{i}=-1\sim P_{\bm{W}}$ for all $i=1,2,\cdots,n$.}
			
			\State{Update the generator}
			\begin{subequations}
				\begin{align}
				\bm{g}_{\bm{\omega}}&\leftarrow \nabla_{\bm{\omega}}\widehat{D}^{\alpha}_{\widehat{\mu}^{N}}\Big[P_{\bm{V}},P_{\bm{W}}\Big], \quad 		\widehat{\mu}^{N}={1\over N}\sum_{k=1}^{N}\delta(\bm{\xi}-\bm{\xi}^{k}).\\
				\bm{w}&\leftarrow \bm{w}-\tilde{\eta}\text{RMSprop}(\bm{g}_{\bm{\omega}},\bm{\omega}).
				\end{align}
			\end{subequations}
			\EndWhile
		\end{algorithmic}
	}
\end{algorithm}\normalsize

\section{Kernel learning in discriminative models}
\label{Section:Kernel learning in discriminative models}

In this section, we now consider the kernel learning problem in discriminative models. We focus on the classical setting of supervised learning, whereby we are given are given $n$ feature vectors and their corresponding uni variate class labels  $(\bm{x}_{1},y_{1}),\cdots,(\bm{x}_{n},y_{n})\sim_{\text{i.i.d.}} P_{\bm{X},Y}$,  $(\bm{x}_{i},y_{i})\in \mathcal{X}\times \mathcal{Y}\subset \real^{d}\times \real$. In this paper, we are concerned with the binary classification and so the target spaces is given by $\mathcal{Y}=\{-1,1\}$.

Given a loss function $\ell:\mathcal{Y}\times \real \mapsto \real_{+}$, where $\ell(\cdot,z)$ is convex (\textit{e.g.} hinge function for SVMs, or quadratic for linear regression), and a reproducing Hilbert kernel space $(\mathcal{H}_{\mathcal{X}},K)$, a classifier $f\in \mathcal{H}_{\mathcal{X}}\oplus \{1\}$ is learned by minimizing the empirical risk minimization with a quadratic regularization,
\begin{align}
\label{Eq:Empirical_Loss_Minimization}
\inf_{f\in \mathcal{H}_{\mathcal{X}}}\widehat{R}[f]\df \dfrac{1}{n}\sum_{i=1}^{n} \ell(y_{i},f(\bm{x}_{i}))+\dfrac{\lambda}{2}\|f\|_{\mathcal{H}_{\mathcal{X}}}^{2},
\end{align}
where $\|\cdot\|_{\mathcal{H}_{\mathcal{X}}}$ is the Hilbert norm, $\lambda>0$ is the parameter of the regularization. Using the expansion $f(\bm{x})=\omega_{0}+\sum_{i=1}^{n-1}\omega_{i}K(\bm{x},\bm{x}_{i})$, and optimizing over the kernel class $\mathcal{K}$ yields
\begin{subequations}
\begin{align}
&\text{Primal}:\min_{\bm{\omega}\in \real^{n}}\min_{K\in \mathcal{K}} \dfrac{1}{n}\sum_{i=1}^{n} \ell\left(y_{i},\omega_{0}+\sum_{i=1}^{n-1}\omega_{i}K(\bm{x},\bm{x}_{i}) \right)+\dfrac{\lambda}{2}\|\bm{\omega}\|_{2}^{2}\\
&\text{Dual}:\max_{\bm{\alpha}\in \real^{n}}\min_{K\in \mathcal{K}}-\sum_{i=1}^{n}\ell^{\ast}(\beta_{i},y_{i})-\dfrac{1}{2\lambda}\bm{\beta}^{T}\bm{K}\bm{\beta},
\end{align}
\end{subequations}
where $\ell^{\ast}(\beta,y)=\sup_{z\in \real}\{\beta z-\ell(\beta,y)\}$ is the Fenchel's conjugate, and $\bm{K}\df [K(\bm{x}_{i},\bm{x}_{j})]_{(i,j)\in [n]\times [n]}$ is the kernel Gram matrix.

In the particular case of the soft margin SVMs classifier $\ell(y,z)=\max\{0,1-yz\}$, the primal and dual optimizations take the following forms
\begin{subequations}
\begin{align}
&\text{Primal}: \min_{\bm{\omega}\in \real^{n}}\min_{K\in \mathcal{K}}\dfrac{1}{n}\sum_{i=1}^{n}\max\left\{0,1-y_{i}\left(\omega_{0}+\sum_{i=1}^{n-1}\omega_{i}K(\bm{x},\bm{x}_{i}) \right) \right\}+\dfrac{\lambda}{2}\|\bm{\omega}\|_{2}^{2}     \\ \label{Eq:Dual_optimization}
&\text{Dual}: \max_{\bm{\beta}\in \real^{n}:\langle \bm{\beta},\bm{y} \rangle=0,\bm{0}\preceq \bm{\beta} \preceq C\bm{1}}\min_{K\in \mathcal{K}}  \langle \bm{\beta},\bm{1}\rangle-\dfrac{1}{2}\bm{\beta} (\bm{K}\odot \bm{y}^T\bm{y})  \bm{\beta}^T,
\end{align}
\end{subequations}
where $\odot$ is the Hadamard (element-wise) product. The last term in the dual objective function in Eq. \eqref{Eq:Dual_optimization}, suggest that for a fixed dual vector $\bm{\beta}\in \real^{n}$, the optimal choice of the kernel can be obtained by maximizing the kernel target alignment, \textit{i.e.},
\begin{align}
\label{Eq:Conditional_distributions}
\max_{\mu \in \mathcal{P}} \widehat{\mathrm{MMD}}_{\mu}[P_{+},P_{-}]= \dfrac{8}{n(n-1)} \sum_{1\leq i<j\leq n}y_{i}y_{j}K_{\mu}(\bm{x}_{i},\bm{x}_{j}),
\end{align}
where in the above optimization we used the random feature model of Eq. \eqref{Eq:Generative_Distribution} to transform the optimization of the kernel over the kernel class $\mathcal{K}$ to a distributional optimization over the distributional ball $\mathcal{P}$. In Equation \eqref{Eq:Conditional_distributions}s, $P_{+}\df P_{\bm{X}|Y=+1}$ and $P_{-}\df P_{\bm{X}|Y=-1}$ are the conditional distributions of features given the class labels. 

Using a Monte-Carlo sample average approximation similar to Eq. \eqref{Eq:Population_Objective_Function_0}, yields the following optimization problem
\begin{align}
\label{Eq:Discriminative_Model}
\max_{\widehat{\mu}_{N}\in \mathcal{P}_{N}} \dfrac{2}{n(n-1)\alpha}\left(\alpha y_{i}y_{j}-\dfrac{1}{N}\sum_{k=1}^{N}\varphi(\bm{x}_{i};\bm{\xi}^{k})\varphi(\bm{x}_{j};\bm{\xi}^{k})\right)^{2},
\end{align}
where we recall from Eq. \eqref{Eq:Population_Objective_Function_0} that $\alpha>0$ is a regularization parameter. Due to similarity of the kernel-target alignment in the generative model of Eq. \eqref{Eq:Population_Objective_Function_0} and the underlying optimization of \eqref{Eq:Discriminative_Model}, we apply the SGD method of Eq. \eqref{Eq:SGD} to solve Eq. \eqref{Eq:Discriminative_Model}.

\subsection{Proposed SVMs classifier with SGD kernel learning}

In Algorithm \ref{Algorithm:2}, we describe the proposed SVMs classifer with the kernel learning approach. Instead of a joint optimization over the dual vector $\bm{\beta}\in \real^{n}$ and the kernel $K\in \mathcal{K}$, we consider a two phase optimization problem. In the first phase (the \textbf{for} loop), we optimize the kernel using the SGD optimization of Eq. \eqref{Eq:SGD}. In the second phase, we construct the Gram matrix $\bm{\Phi}\bm{\Phi}^{T}$ from random features, where $\bm{\Phi}\df [\varphi(\bm{x}_{i};\bm{\xi}^{k})]_{i\in [n],k\in [N]}$. 

Let $\bm{K}=(K_{\mu}(\bm{x}_{i},\bm{x}_{j}))_{i,j\in [n]}$, and suppose $\bm{\xi}^{1},\cdots,\bm{\xi}^{N}\sim_{\text{i.i.d.}}\mu$. Then, the following concentration results are established by some of the authors of this paper (cf. \cite[Lemma 5.3]{khuzani2019multiple}),
\begin{subequations}
\begin{align}
\label{Eq:Justifies}
\prob(| \|\bm{\Phi}\|^{2}_{F}-N \mathrm{Tr}(\bm{K})|)\geq \delta)&\leq 2^{8}\left(\dfrac{\sigma_{p}nN \mathrm{Diam}(\mathcal{X})}{\delta}\right)\exp\left(-\dfrac{\delta^{2}}{4Nn^{2}(d+2)}\right),\\
\prob(| \|\bm{\Phi}\|^{2}_{2}-N\|\bm{K}\|_{2}|)\geq \delta)&\leq 2^{8}e^{3n\log 3}\left(\dfrac{\sigma_{p}nN \mathrm{Diam}(\mathcal{X})}{\delta}\right)\exp\left(-\dfrac{\delta^{2}}{4Nn^{2}(d+2)}\right),
\end{align}
\end{subequations}
where $\|\cdot\|_{F}$ and $\|\cdot\|_{2}$ are Frobenius and spectral norms, respectively,  $\mathrm{Diam}(\mathcal{X})$ is the diameter of the feature space, $d$ is the dimension of features $\bm{x}_{1},\cdots,\bm{x}_{n}\in \mathcal{X}\subset \real^{d}$, and $\sigma_{p}\df \expect_{\mu}[\|\bm{\xi}\|_{2}^{2}]$ is the second moment of the random vector $\bm{\xi}\sim \mu$. From Eq. \eqref{Eq:Justifies}, we observe that as the number of particles $N$ tends to infinity, the Gram matrix $\bm{\Phi}\bm{\Phi}^{T}$ concentrates in various norms around the (scaled) kernel matrix $\bm{K}$. Due to the concentration phenomenon of Eq. \eqref{Eq:Justifies}, the substitution of the kernel matrix $\bm{K}$ in the dual optimization of Eq. \eqref{Eq:Dual_optimization} with the scaled Gram matrix $\bm{\Phi}\bm{\Phi}^{T}/N$ yields a good approximation. Equation \eqref{Eq:Subsutition_we_made} of Algoirthm \ref{Algorithm:2} is the result of such an approximation.

\begin{algorithm}[t!]\scriptsize{
		\caption{\small{SVMs with a supervised kernel learning (Monte-Carlo Approach)}}
		\label{Algorithm:2}
		\begin{algorithmic}
			\State {\bfseries Inputs:} {The learning rates $\tilde{\eta},\eta>0$ , the number of iterations of discriminator per generator update $T\in \mathbb{N}$, the batch-size $n$, the number of random features $N\in \mathbb{N}$. Regularization parameter $\alpha>0$.}
			\For{$t=1,2,\cdots,T$}
			\State{Sample the labels $y,\widetilde{y}\sim_{\text{i.i.d}} \mathrm{Uniform}\{-1,1\}$. }
			
			\State{Sample the features $\bm{x}|y=+1\sim P_{+}$, and $\bm{x}|y=-1\sim P_{-}$ from the training data-set $\{(y_{1},\bm{x}_{1}),\cdots,(y_{n},\bm{x}_{n})\}$. Similarly,  $\widetilde{\bm{x}}|\widetilde{y}=+1\sim P_{+}$, and $\widetilde{\bm{x}}|\widetilde{y}=-1\sim P_{-}$.} 
			
			\State{For all $k=1,2,\cdots,N$, update the particles,}
			\begin{align}
			\label{Eq:distinguished}
			\bm{\xi}^{k}\leftarrow\bm{\xi}^{k}-{\eta\over N}\left(\alpha y\widetilde{y}- {1\over N}\sum_{k=1}^{N}\varphi(\bm{x};\bm{\xi}^{k})\varphi(\widetilde{\bm{x}};\bm{\xi}^{k})\right)\nabla_{\bm{\xi}}\Big(\varphi(\bm{x};\bm{\xi}^{k})\varphi(\widetilde{\bm{x}};\bm{\xi}^{k})\Big),
			\end{align}
			\EndFor  			
		    \State{Solve the following dual optimization problem}
		    \begin{align}
		    \label{Eq:Subsutition_we_made}
		     \max_{\bm{\beta}\in \real^{n}:\langle \bm{\beta},\bm{y} \rangle=0,\bm{0}\preceq \bm{\beta} \preceq C\bm{1}}\langle \bm{\beta},\bm{1}\rangle-\dfrac{1}{2}\bm{\beta} \left(\dfrac{1}{N}\bm{\Phi}\bm{\Phi}^{T}\odot \bm{y}^T\bm{y}\right)  \bm{\beta}^T,
		    \end{align}
		    where $\bm{\Phi}\df [\varphi(\bm{x}_{i};\bm{\xi}^{k})]_{i\in [n],k\in [N]}\in \real^{n\times N}$.	    		
		\end{algorithmic}
	}
\end{algorithm}\normalsize

\section{Main results: consistency and mean-field analysis}
\label{Section:Theoretical Results}
In this section, we provide theoretical guarantees for the consistency of various approximations we made to optimize the population MMD loss function in \eqref{Eq:Population_MMD_Loss}. We defer the proofs of the following theoretical results to Section \ref{Section:Proofs_of_Main_Theoretical_Results} of Appendix. Due to similarity of the kernel learning approach for generative and classification tasks, we state the underlying assumptions and theoretical results \textit{only} for the generative model. Nevertheless, it is easy to see that our theoretical results generalizes to the classification problem.

\subsection{Assumptions} 

Before we delve into technical results, we state the main assumptions underlying them:
\begin{itemize}
	\item[$\mathbf{(A.1)}$] The feature space $\mathcal{X}=\mathcal{V}\cup \mathcal{W} \subset \real^{d}$ is compact with a finite diameter
	$\mathrm{diam}(\mathcal{X})<\infty$, where $\mathcal{V}=\mathrm{support}(P_{\bm{V}})$ and $\mathcal{W}=\mathrm{support}(P_{\bm{W}})$ are the supports of the distributions $P_{\bm{V}}$ and $P_{\bm{W}}$ respectively.
	
	\item[$\mathbf{(A.2)}$] The feature maps are bounded and Lipchitz almost everywhere (a.e.) $\bm{\xi}\in \real^{D}$. In particular, we assume $\sup_{\bm{x}\in \mathcal{X}}|\varphi(\bm{x};\bm{\xi})|\leq L_{0}$,  $\sup_{\bm{x}\in \mathcal{X}}\|\nabla_{\bm{\xi}}\varphi(\bm{x};\bm{\xi})\|_{2}\leq L_{1}$, and $\sup_{\bm{\xi}\in \real^{D}}\|\nabla_{\bm{x}}\varphi(\bm{x};\bm{\xi})\|\leq L_{2}$. Let $L\df \max\{L_{0},L_{1},L_{2}\}<+\infty$.
	
	\item[$\mathbf{(A.3)}$] Let $\widehat{\mu}^{N}_{0}(\bm{\xi})\df {1\over N}\sum_{k=1}^{N}\delta(\bm{\xi}-\bm{\xi}_{0}^{k})$ denotes the empirical measure for the initial particles $\bm{\xi}_{0}^{1},\cdots,\bm{\xi}_{0}^{N}$. We assume that $\widehat{\mu}^{N}_{0}(\bm{\xi})$ converges (weakly) to a deterministic measure $\mu_{0}\in \mathcal{M}(\real^{D})$. Furthermore, we assume the limiting measure $\mu_{0}$  is absolutely continuous \textit{w.r.t.} Lebesgue measure and has a compact support $\mathrm{support}(\mu_{0})=\Xi\subset \real^{D}$.
\end{itemize}

\subsection{Consistency of finite-sample estimator} 

n this part, we prove that the solution to finite sample optimization problem in \eqref{Eq:Empirical_Objective_Function}  approaches its population optimum in  \eqref{Eq:Population_MMD_Loss} as the number of data points as well as the number of random feature samples tends to infinity. 

In this part, we prove that the solution to finite sample optimization problem in \eqref{Eq:Empirical_Objective_Function}  approaches its population optimum in  \eqref{Eq:Population_MMD_Loss} as the number of data points as well as the number of random feature samples tends to infinity. 

\begin{theorem}\textsc{{(Non-asymptotic Consistency of Finite-Sample Estimator)}}
	\label{Thm:Consistency of Monte Carlo Estimation}
	Suppose conditions $\mathbf{(A.1)}$-$\mathbf{(A.3)}$ of Appendix \ref{Section:Proofs_of_Main_Theoretical_Results} are satisfied. Consider the distribution balls $\mathcal{P}$ and $\mathcal{P}_{N}$ that are defined with respect to the $2$-Wasserstein distance $(p=2)$. Furthermore, consider the optimal MMD values of the population optimization and its finite sample estimate
	\begin{subequations}
		\begin{align}
		(\bm{W}_{\ast},\mu_{\ast})&\df \arg \min_{\bm{W}\in \mathcal{W}}\arg\sup_{\mu\in \mathcal{P}} \mathrm{MMD}_{\mu}[P_{\bm{V}},P_{\bm{W}}].\\
		(\widehat{\bm{W}}^{N}_{\ast},\widehat{\mu}^{N}_{\ast} )&\df \arg\min_{\bm{W}\in \mathcal{W}}\arg\inf_{\widehat{\mu}^{N}\in \mathcal{P}_{N} }\widehat{\mathrm{MMD}}_{\widehat{\mu}^{N}}^{\alpha}[P_{\bm{V}},P_{\bm{W}}],
		\end{align}
	\end{subequations}
	respectively. Then, with the probability of (at least) $1-3\varrho$ over the training data samples $\{(\bm{x}_{i},y_{i})\}_{i=1}^{n}$ and the random feature samples $\{\bm{\xi}_{0}^{k}\}_{k=1}^{N}$, the following non-asymptotic bound holds 
	\begin{align}
	\label{Eq:Upper_Bound_of}
	&\Big|\mathrm{MMD}_{\mu_{\ast}}[P_{\bm{V}},P_{\bm{W}_{\ast}}]-\mathrm{MMD}_{\widehat{\mu}^{N}_{\ast}}[P_{\bm{V}},P_{\widehat{\bm{W}}^{N}_{\ast}}] \Big|\\ \nonumber
	&\leq    \sqrt{\dfrac{L^{2}(d+2)}{N}}\ln^{1\over 2}\left(\dfrac{2^{8}N\mathrm{diam}^{2}(\mathcal{X})}{\varrho}\right)+2\max\left\{\dfrac{c_{1}L^{2}}{n} \ln^{1\over 2}\left(\dfrac{4}{\varrho}\right),\dfrac{c_{2}RL^{4} }{n^{2}}\ln\left(\dfrac{4e^{L^{4}\over 9}}{\varrho}\right) \right\}+\dfrac{8L^{2}}{\alpha},
	\end{align}
	where $c_{1}=3^{1\over 4}\times 2^{4}$, and $c_{2}=9\times 2^{11}$.
\end{theorem}
The proof of Theorem \ref{Thm:Consistency of Monte Carlo Estimation} is presented in Appendix \ref{App:Proof of Theorem_1}.

Notice that there are three key parameters involved in the upper bound of Theorem \ref{Thm:Consistency of Monte Carlo Estimation}. Namely, the number of training samples $n$, the number of random feature samples $N$, and the regularization parameter $\alpha$. The upper bound in  \eqref{Eq:Upper_Bound_of} thus shows that when $n,N, \alpha\rightarrow +\infty$, the solution obtained from solving the empirical risk minimization  in \eqref{Eq:Population_Objective_Function_0} yields a MMD population value tending to the optimal value of the distributional optimization in \eqref{Eq:Population_MMD_Loss}.

\subsection{Consistency of particle SGD algorithm} The consistency result of Theorem \ref{Thm:Consistency of Monte Carlo Estimation} is concerned with the MMD value of the optimal empirical measure $\widehat{\mu}^{N}_{\ast}(\bm{\xi})={1\over N}\sum_{k=1}^{N}\delta(\bm{\xi}-\bm{\xi}^{k}_{\ast})$ of the empirical risk minimization \eqref{Eq:Empirical_Objective_Function}. In practice, the particle SGD is executed for a few iterations and its values are returned as an estimate for $(\bm{\xi}_{\ast}^{1},\cdots,\bm{\xi}_{\ast}^{N})$. Consequently, it is desirable to establish a consistency type result for the particle SGD estimates $(\bm{\xi}_{m}^{1},\cdots,\bm{\xi}_{m}^{N})$ at the $m$-th iteration, where the notion of consistency will be made precise shortly. To prove such a consistency result, we define the scaled empirical measure as follows
\begin{align}
\label{Eq:Scaled_Empirical_Measures}
\mu_{t}^{N}=\widehat{\mu}^{N}_{\lfloor Nt\rfloor}=\dfrac{1}{N}\sum_{k=1}^{N}\delta(\bm{\xi}-\bm{\xi}_{\lfloor Nt\rfloor}), \quad 0\leq t\leq T.
\end{align}
At any time $t$, the scaled empirical measure $\mu_{t}^{N}$ is a random element, and thus $(\mu_{t}^{N})_{0\leq t\leq T}$ is a measured-valued stochastic process. Therefore, we characterize the evolution of its Lebesgue density $p_{t}^{N}(\bm{\xi})\df \mu_{t}^{N}(\mathrm{d}\bm{\xi})/\mathrm{d}\bm{\xi}$ in the following theorem:

\begin{theorem} \textsc{(McKean-Vlasov Mean-Field PDE)} 
	\label{Theorem:Density Evolution}	
	Suppose conditions $\mathbf{(A.1)}$-$\mathbf{(A.3)}$ in Section \ref{Section:Proofs_of_Main_Theoretical_Results} are satisfied. Further, suppose that the Radon-Nikodyme derivative  $q_{0}(\bm{\xi})=\mu_{0}(\mathrm{d}\bm{\xi})/\mathrm{d}\bm{\xi}$ exists.  Then, there exists a unique solution $(p^{\ast}_{t}(\bm{\xi}))_{0\leq t\leq T}$ to the following non-linear partial differential equation
	\small{\begin{align}
		\label{Eq:McKean-Vlasov}
		\begin{cases}
		\dfrac{\partial p_{t}(\bm{\xi})}{\partial t}&= -{\eta\over \alpha} \iint_{\mathcal{X}\times \mathcal{Y}} \left( \int_{\real^{p}} \varphi(\bm{x},\widetilde{\bm{\xi}})\varphi(\widetilde{\bm{x}},\widetilde{\bm{\xi}})p_{t}(\widetilde{\bm{\xi}})\mathrm{d}\widetilde{\bm{\xi}}-\alpha y \widetilde{y} \right)\nabla_{\bm{\xi}}(p_{t}(\bm{\xi})\nabla_{\bm{\xi}}(\varphi(\bm{x};\bm{\xi})\varphi(\widetilde{\bm{x}};\bm{\xi}))\mathrm{d}P_{\bm{x},y}^{\otimes 2},
		\\ 
		p_{0}(\bm{\xi})&=q_{0}(\bm{\xi}).
		\end{cases}
		\end{align}}\normalsize
	Moreover, the measure-valued process $\{(\mu_{t}^{N})_{0\leq t\leq T}\}_{N\in \integer}$ defined in Eq. \eqref{Eq:Scaled_Empirical_Measures} converges (weakly) to the unique solution $\mu_{t}^{\ast}(\bm{\xi})=p^{\ast}_{t}(\bm{\xi})\mathrm{d}\bm{\xi}$ as the number of particles tends to infinity $N\rightarrow \infty$. \footnote{The notion of the weak convergence of a sequence of empirical measures is formally defined in Appendix.}.
\end{theorem} 

The proof of Theorem \ref{Theorem:Density Evolution} is presented in Appendix \ref{App:Proof of Theorem_1}.

Due to the mean-field analysis of Theorem \ref{Theorem:Density Evolution}, we can prove that the empirical measure $\widehat{\mu}^{N}_{m}$ of the particles in SGD dynamic \eqref{Eq:SGD} remains inside the feasible distribution ball $\mathcal{P}_{N}$:
\begin{corollary}
	\label{Corollary:1}
	Consider the learning rate $\eta=\mathcal{O}\left(R^{p}\ \over T\sqrt{NT}\log(2/\delta)\right)$ for the SGD in \eqref{Eq:SGD}. Then, the empirical measure $\widehat{\mu}^{N}_{m}$ of the particles remains inside the distributional ball $\widehat{\mu}^{N}_{m}\in \mathcal{P}_{N}=\{\widehat{\mu}^{N}\in \mathcal{M}(\real^{D}):W_{p}(\widehat{\mu}^{N},\widehat{\mu}_{0}^{N})\leq R\}$ for all $m\in [0,NT]\cap \integer$, with the probability of (at least) $1-\delta$.
\end{corollary}

Let us make two remarks about the PDE in Eq. \eqref{Eq:McKean-Vlasov}. 

First, the seminal works of \cite{otto2001geometry}, and \cite{jordan1998variational} establishes an intriguing connection between the McKean-Vlasov  type PDEs specified in \eqref{Eq:McKean-Vlasov} and the gradient flow on the Wasserstein manifolds. More specifically, the PDE equation in Eq. \eqref{Eq:McKean-Vlasov} can be thought of as the minimization of the energy functional
\begin{subequations}
	\begin{align}
	\label{Eq:Gradient_Flow1}
	\inf_{\mu\in \mathcal{M}(\real^{D})} E_{\alpha}(p_{t}(\bm{\xi}))&\df {1\over \alpha}\int_{\real^{p}} R_{\alpha}(\bm{\xi},p_{t}(\bm{\xi}))p_{t}(\bm{\xi})\mathrm{d}\bm{\xi}\\
	\label{Eq:Gradient_Flow2}
	R_{\alpha}(\bm{\xi},p_{t}(\bm{\xi}))&\df -\alpha (\expect_{P_{\bm{x},y}}[y\varphi(\bm{x};\bm{\xi})])^{2}+\expect_{\widetilde{\bm{\xi}}\sim p_{t}}\Big[\Big(\expect_{P_{\bm{X}}}[\varphi(\bm{x};\bm{\xi})\varphi(\bm{x};\widetilde{\bm{\xi}})]\Big)^{2}\Big],
	\end{align}
\end{subequations}
using the following gradient flow dynamics
\begin{align}
\label{Eq:Gradient_Flow2}
\dfrac{\mathrm{d}p_{t}(\bm{\xi})}{\mathrm{d}t}=-\eta\cdot \mathrm{grad}_{p_{t}} E_{\alpha}(p_{t}(\bm{\xi})), \quad p_{0}(\bm{\xi})=q_{0}(\bm{\xi}),
\end{align}
where $\mathrm{grad}_{p_{t}}E(p_{t}(\bm{\xi}))=\nabla_{\bm{\xi}}\cdot(p_{t}(\bm{\xi})\nabla_{\bm{\xi}} R_{\alpha}(p_{t}(\bm{\xi})))$ is the Riemannian gradient of $R_{\alpha}(\mu_{t}(\bm{\xi}))$ with respect to the metric of the Wasserstein manifold . This shows that when the number of particles in particle SGD \eqref{Eq:SGD} tends to infinity $(N\rightarrow +\infty)$, their empirical distribution follows a gradient descent path for minimization of the population version (with respect to data samples) of the distributional risk optimization in Eq. \eqref{Eq:Population_Objective_Function_0}. In this sense, the particle SGD is a `consistent' approximation algorithm for solving the distributional optimization.

Second, notice that as the scaling parameter tends to infinity $\alpha\rightarrow \infty$, the energy functional tends to the limit $E_{\alpha}(p_{t}(\bm{\xi}))\rightarrow E_{\infty}(p_{t}(\bm{\xi}))\df (\expect_{P_{\bm{x},y}}[y\varphi(\bm{x};\bm{\xi})])^{2}$. Interestingly, this limiting energy functional is precisely the kernel polarization of Baram \cite{baram2005learning}, measuring the correlation between the class labels and random features

\subsection{Kernel Learning by Solving A PDE}

The kernel selection methods in the literature focuses on optimization methods based on kernel alignment optimization problem; see, \textit{e.g.}, \cite{cortes2012algorithms,lanckriet2004learning,cristianini2002kernel}.

The PDE in Eq. \eqref{Eq:McKean-Vlasov} puts forth an alternative method to finding good kernel functions. Namely, a good kernel function can be computed in two stages: first compute the stationary solution $p^{\ast}(\bm{\xi})$ corresponding to a solution of the PDE in Eq. \eqref{Eq:McKean-Vlasov} with $\partial p_{t}^{\ast}(\bm{\xi})/\partial t=0$, where the expectation with respect to the unknown data distribution is replaced by its finite-sample average. Then approximating the kernel via a Monte-Carlo sampling  $\bm{\xi}^{1},\cdots,\bm{\xi}^{N}\sim_{\text{i.i.d.}} p^{\ast}(\bm{\xi})$,
\begin{align}
K_{\mu^{\ast}}(\bm{x},\bm{y})=\int_{\real^{p}}\varphi(\bm{x};\bm{\xi})\varphi(\bm{y};\bm{\xi})p^{\ast}(\bm{\xi})\mathrm{d}\bm{\xi}\approx {1\over N}\sum_{i=1}^{N}\varphi(\bm{x};\bm{\xi}^{k})\varphi(\bm{y};\bm{\xi}^{k}).
\end{align}
While such a method can be successful when the random features $\bm{\xi}\in \real^{D}$ are low dimensional, in the high-dimensional settings $(D\gg 1)$, solving the PDE in Eq. \eqref{Eq:McKean-Vlasov} numerically and sampling from a high dimensional density function $p^{\ast}(\bm{\xi})$ appears to be computationally more expensive than solving the kernel-target alignment via SGD iterative optimization.

\subsection{Propagation of Chaos}
We now establish the so called `propagation of chaos' property of particle SGD. At a high level, the propagation of chaos means that when the number of samples $\{\bm{\xi}^{k}\}_{k=1}^{N}$ tends to infinity $(N\rightarrow +\infty)$, their dynamics are decoupled.

\begin{definition}\textsc{(Exchangablity)}
	Let $\nu$ be a probability measure on a Polish space $\mathcal{S}$ and. For $N\in \integer$, we say that $\nu^{\otimes N}$ is an exchangeable probability measure on the product space $\mathcal{S}^{n}$ if it is invariant under the permutation $\bm{\pi}\df (\pi(1),\cdots,\pi(N))$ of indices. In particular,
	\begin{align}
	\label{Eq:Exchangable}
	\nu^{\otimes N}(\bm{\pi}\cdot B)=\nu^{\otimes N}(B),
	\end{align}
	for all Borel subsets $B\in \mathcal{B}(\mathcal{S}^{n})$.	
\end{definition}

An interpretation of the exchangablity condition \eqref{Eq:Exchangable} can be provided via De Finetti's representation theorem which states that the joint distribution of an infinitely exchangeable sequence of random variables is as if a random parameter were drawn from some distribution and then the random variables in question were independent and identically distributed, conditioned on that parameter.

Next, we review the mathematical definition of chaoticity, as well as the propagation of chaos in the product measure spaces:
\begin{definition}\textsc{(Chaoticity)} 
	\label{Def:Chaoticity}
	Suppose $\nu^{\otimes N}$ is exchangeable. Then, the sequence $\{\nu^{\otimes N}\}_{N\in \integer}$ is $\nu$-chaotic if, for any natural number $\ell \in \integer$ and any test function $f_{1},f_{2},\cdots,f_{k}\in C_{b}^{2}(\mathcal{S})$, we have
	\begin{align}
	\label{Eq:Chaoticity}
	\lim_{N\rightarrow \infty}\left\langle\prod_{k=1}^{\ell} f_{k}(s^{k}),\nu^{\otimes N}(\mathrm{d}s^{1},\cdots,\mathrm{d}s^{N})\right\rangle= \prod_{k=1}^{\ell}\langle f_{k},\nu\rangle
	\end{align}
\end{definition}

According to Eq. \eqref{Eq:Chaoticity} of Definition \ref{Def:Chaoticity}, a sequence of probability measures on the product spaces $\mathcal{S}$ is $\nu$-chaotic if, for fixed $k$ the joint probability measures for the first $k$ coordinates tend to the product measure $\nu(\mathrm{d}s_{1})\nu(\mathrm{d}s_{2})\cdots \nu(\mathrm{d}s_{k})=\nu^{\otimes k}$ on $\mathcal{S}^{k}$. If the measures $\nu^{\otimes N}$ are thought of as giving the joint distribution of $N$ particles residing in the space $\mathcal{S}$, then $\{\nu^{\otimes N}\}$ is $\nu$-chaotic if $k$ particles out of $N$ become more and more independent as $N$ tends to infinity, and each particle’s distribution tends to $\nu$. A sequence of symmetric probability measures on $\mathcal{S}^{N}$ is chaotic if it is $\nu$-chaotic for some probability measure $\nu$ on $\mathcal{S}$.

If a Markov process on $\mathcal{S}^{N}$ begins in a random state with the distribution $\nu^{\otimes N}$, the distribution of the state after $t$ seconds of Markovian random motion can be expressed in terms of the transition function $\mathcal{K}^{N}$ for the Markov process. The distribution at time $t>0$ is the probability measure $U_{t}^{N}\nu^{\otimes N}$ is defined by the kernel
\begin{align}
\label{Eq:Definition_of_Ut}
U_{t}^{N}\nu^{\otimes N}(B)\df \int_{\mathcal{S}^{N}}\mathcal{K}^{N}(s,B,t)\nu^{\otimes N}(\mathrm{d}s).
\end{align}

\begin{definition}\textsc{(Propogation of Chaos)} A sequence functions 
	\begin{align}
	\Big\{\mathcal{K}^{N}(s,B,t)\Big\}_{N\in \integer}
	\end{align}	
	whose $N$-th term is a Markov transition function on $\mathcal{S}^{N}$ that satisfies the  permutation condition
	\begin{align}
	\mathcal{K}^{N}(s,B,t)=\mathcal{K}^{N}(\bm{\pi}\cdot s,\bm{\pi}\cdot B,t),
	\end{align}
	propagates chaos if whenever $\{\nu^{\otimes N} \}_{N\in \integer}$ is chaotic, so is $\{U_{t}^{N} \}$ for any $t\geq 0$, where $U_{t}^{N}$ is defined in Eq. \eqref{Eq:Definition_of_Ut}.
\end{definition}

We note that for finite systems size $N$,the states of the particles  are not independent of each other. However, as we prove in the following result, in the limiting system $N\rightarrow +\infty$, the particles are mutually independent. This phenomena is known as the propagation of chaos (\textit{a.k.a.} asymptotic independence):

\begin{theorem}\textsc{(Chaoticity in Particle SGD)}
	\label{Thm:Chaoticity in Particle SGD}
	Consider Assumptions $\mathbf{(A.1)}-\mathbf{(A.3)}$. Furthermore, suppose that $\{\bm{\xi}^{k}_{0}\}_{1\leq k\leq N}\sim_{\text{i.i.d.}} \mu_{0}$ is exchangable in the sense that the joint law is invariant under the permutation of indices. Then, at each time instant $t\in (0,T]$, the scaled empirical measure $\mu_{t}^{N}\in \mathcal{M}(\real^{D})$ defined via scaling
	\begin{align}
	\mu_{t}^{N}(\mathrm{d}\bm{\xi}^{1},\cdots,\mathrm{d}\bm{\xi}^{N})\df \widehat{\mu}^{N}_{\lfloor Nt \rfloor}(\mathrm{d}\bm{\xi}^{1},\cdots,\mathrm{d}\bm{\xi}^{N})= \prob\{ \bm{\xi}^{1}_{\lfloor Nt\rfloor}\in \mathrm{d}\bm{\xi}^{1},\cdots,\bm{\xi}^{N}_{\lfloor Nt\rfloor}\in \mathrm{d}\bm{\xi}^{N}\},
	\end{align}
	is $\mu^{\ast}_{t}$-chaotic, where $\mu^{\ast}_{t}$ is mean-field solution of Eq.  \eqref{Eq:Mean_Field_Equation}.
\end{theorem}

\section{Empirical Results for Generative Tasks}
\label{Section:Experiments_on_Synthetic_Data_Set}
We now turn to empirical evaluations. We test the performance of Algorithm \ref{Algorithm:1} on synthetic data-set, as well as on benchmark data-sets.

\subsection{Experiments on Synthetic Data-Set}
The synthetic data-set we consider is as follows:

\begin{itemize}
	\item The distribution of training data is $P_{\bm{V}}=\mathsf{N}(\bm{0},(1+\lambda)\bm{I}_{d\times d})$,
	
	\item The  distribution of generated data is $P_{\bm{W}}=\mathsf{N}(\bm{0},(1-\lambda)\bm{I}_{d\times d})$.
\end{itemize}
To reduce the dimensionality of data, we consider the embedding $\iota:\real^{d}\mapsto \real^{d_{0}}, \bm{x}\mapsto \iota(\bm{x})=\bm{\Sigma}\bm{x}$, where $\bm{\Sigma}\in \real^{d_{0}\times d}$ and $d_{0}<d$. In this case, the distribution of the embedded features are $P_{\bm{X}|Y=+1}=\mathsf{N}(\bm{0},(1+\lambda)\bm{\Sigma}\bm{\Sigma}^{T})$, and $P_{\bm{X}|Y=-1}=\mathsf{N}(\bm{0},(1-\lambda)\bm{\Sigma}\bm{\Sigma}^{T})$. 

Note that $\lambda\in [0,1]$ is a parameter that determines the separation of distributions. In particular, the Kullback-Leibler divergece of the two multi-variate Gaussian distributions is controlled by $\lambda\in [0,1]$,
\begin{align}
D_{\mathrm{KL}}(P_{\bm{X}|Y=-1},P_{\bm{X}|Y=+1})=\dfrac{1}{2}\left[ \log \left(\dfrac{1-\lambda}{1+\lambda}\right)-d_{0}+d_{0}(1-\lambda^{2})\right].
\end{align}

In Figure \ref{Fig:4}, we show the distributions of \textit{i.i.d.} samples from the distributions $P_{\bm{V}}$ and $P_{\bm{W}}$ for different choices of variance parameter of $\lambda=0.1$, $\lambda=0.5$, and $\lambda=0.9$.  Notice that for larger $\lambda$ the divergence is reduced and thus performing the two-sample test is more difficult. From Figure \ref{Fig:4}, we clearly observe that for large values of $\lambda$, the data-points from the two distributions $P_{\bm{V}}$ and $P_{\bm{W}}$ have a large overlap and conducting a statistical test to distinguish between these two distributions is more challenging.

\begin{figure}[!t]
	\begin{center}
		\hspace*{-5mm}		\subfigure{
			\includegraphics[trim={.2cm .2cm .2cm  .6cm},width=.35\linewidth]{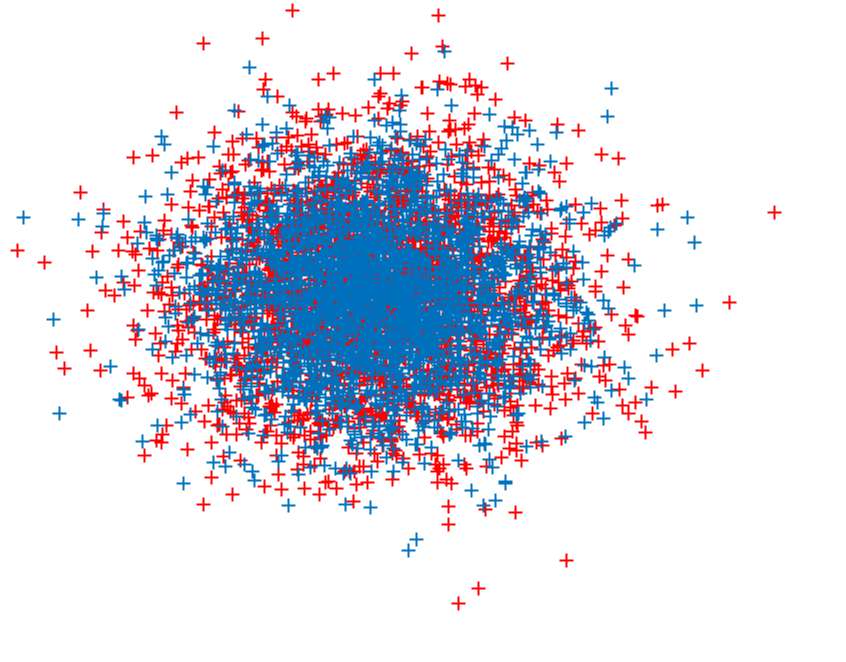} 
			\includegraphics[trim={.2cm .2cm .2cm  .6cm},width=.35\linewidth]{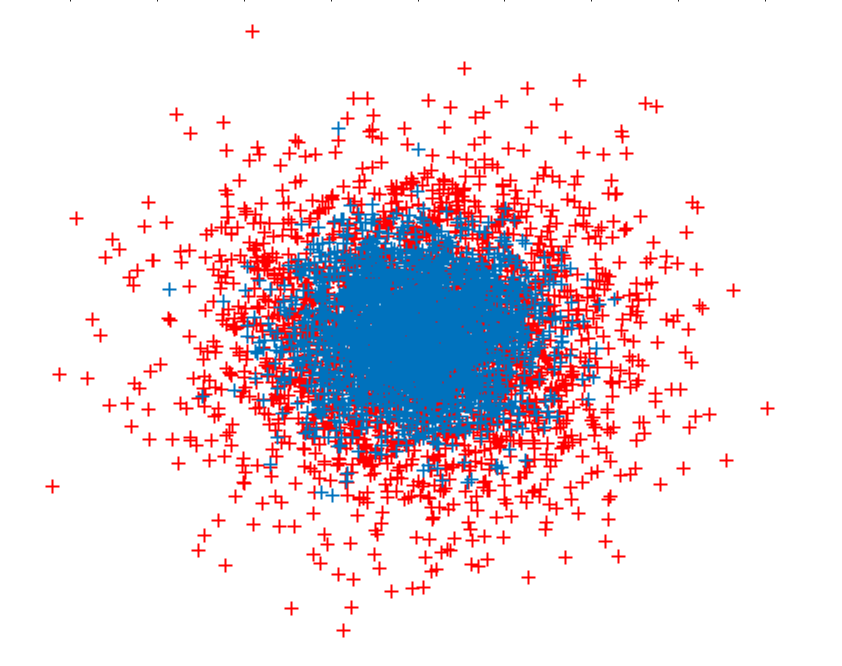} 
			\includegraphics[trim={.2cm .2cm .2cm  .2cm},width=.35\linewidth]{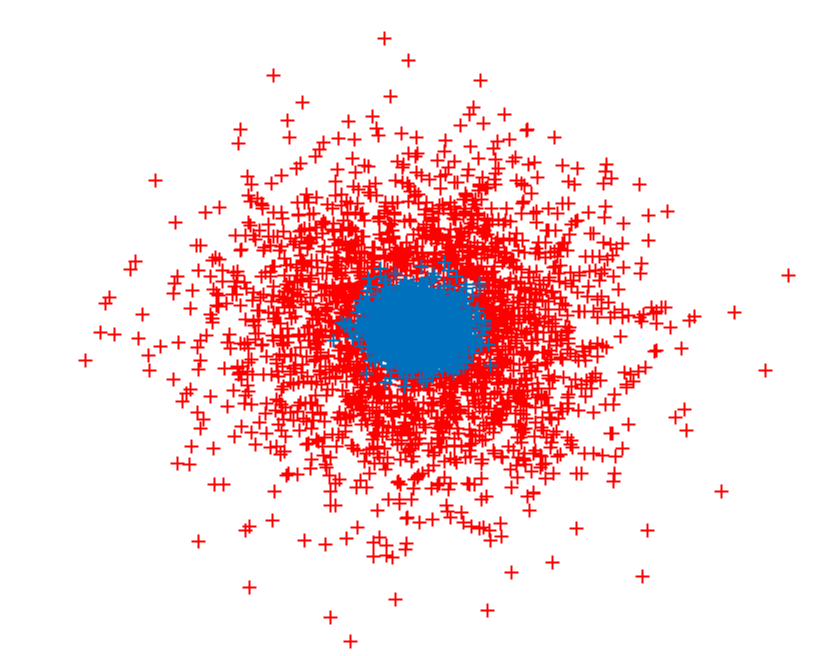}}
		\\ \vspace{-2mm}
		\subfigure{\footnotesize{\hspace{5mm} (a) \hspace{45mm} (b) \hspace{45mm} (c)} }
		\vspace{-2mm}
		\caption{\footnotesize{Visualization of data-points from the synthetic data-set $P_{\bm{V}}=\mathsf{N}(\bm{0},(1+\lambda)\bm{I}_{d\times d})$ and $P_{\bm{W}}=\mathsf{N}(\bm{0},(1-\lambda)\bm{I}_{d\times d})$ for $d=2$. Panel (a): $\lambda=0.1$, Panel (b): $\lambda=0.5$, and Panel (c): $\lambda=0.9$.}}
		\label{Fig:4} 
	\end{center}
\end{figure}

\begin{figure}[!ht]
	\begin{center}
		\subfigure{
			\includegraphics[trim={.2cm .2cm .2cm  .6cm},width=.4\linewidth]{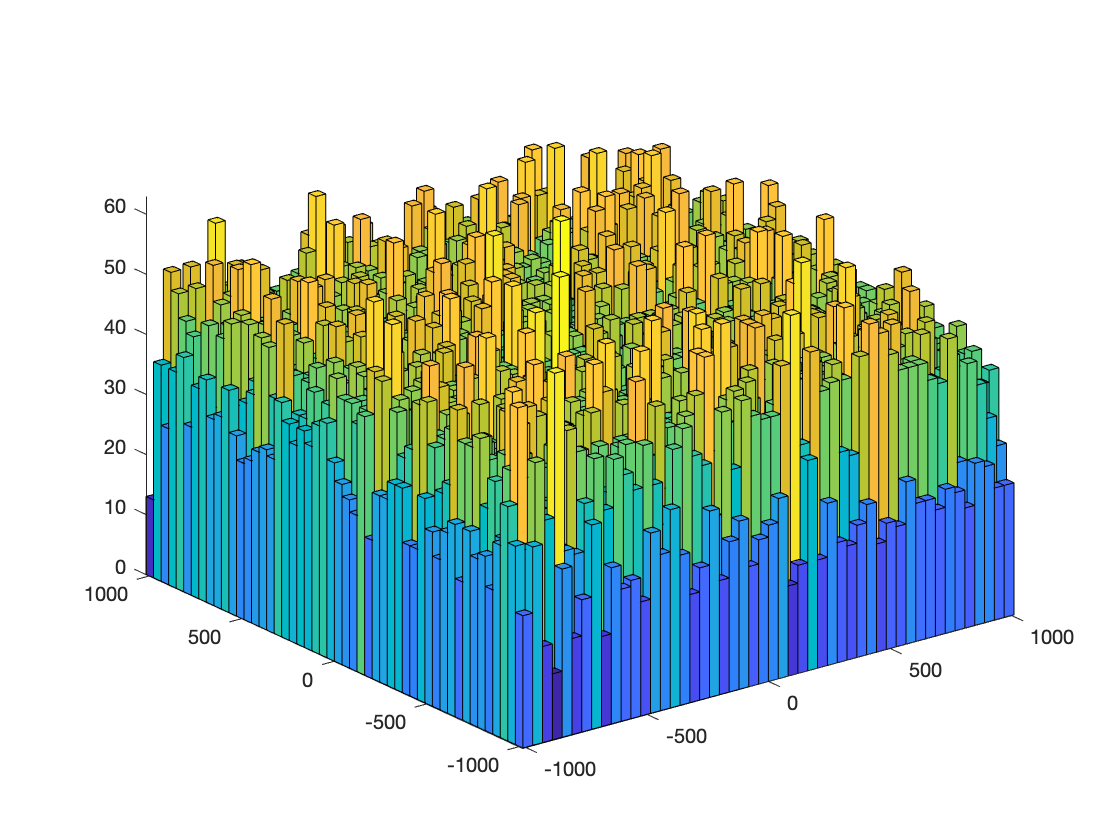} 
			\includegraphics[trim={.2cm .2cm .2cm  .6cm},width=.4\linewidth]{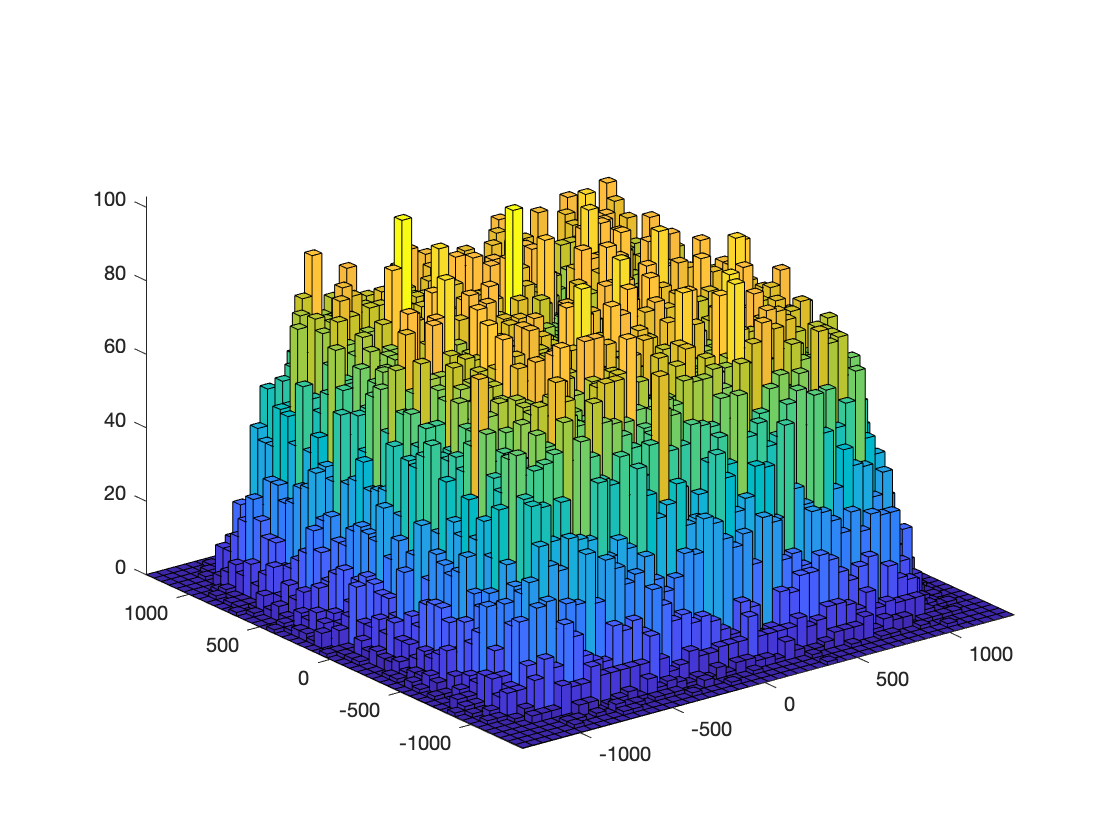}} \\
		\subfigure{\footnotesize{\hspace{5mm} (a) \hspace{45mm} (b) } }
		\subfigure{
			\includegraphics[trim={.2cm .2cm .2cm  .6cm},width=.4\linewidth]{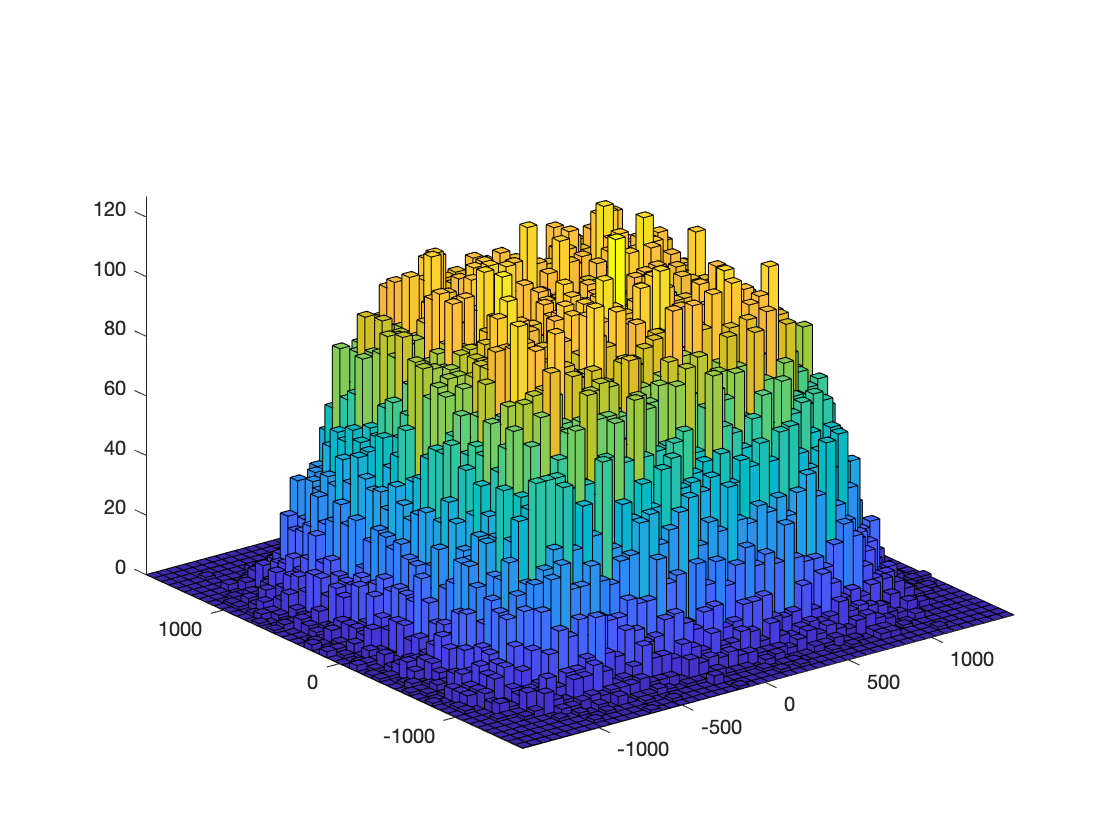} 
			\includegraphics[trim={.2cm .2cm .2cm  .2cm},width=.4\linewidth]{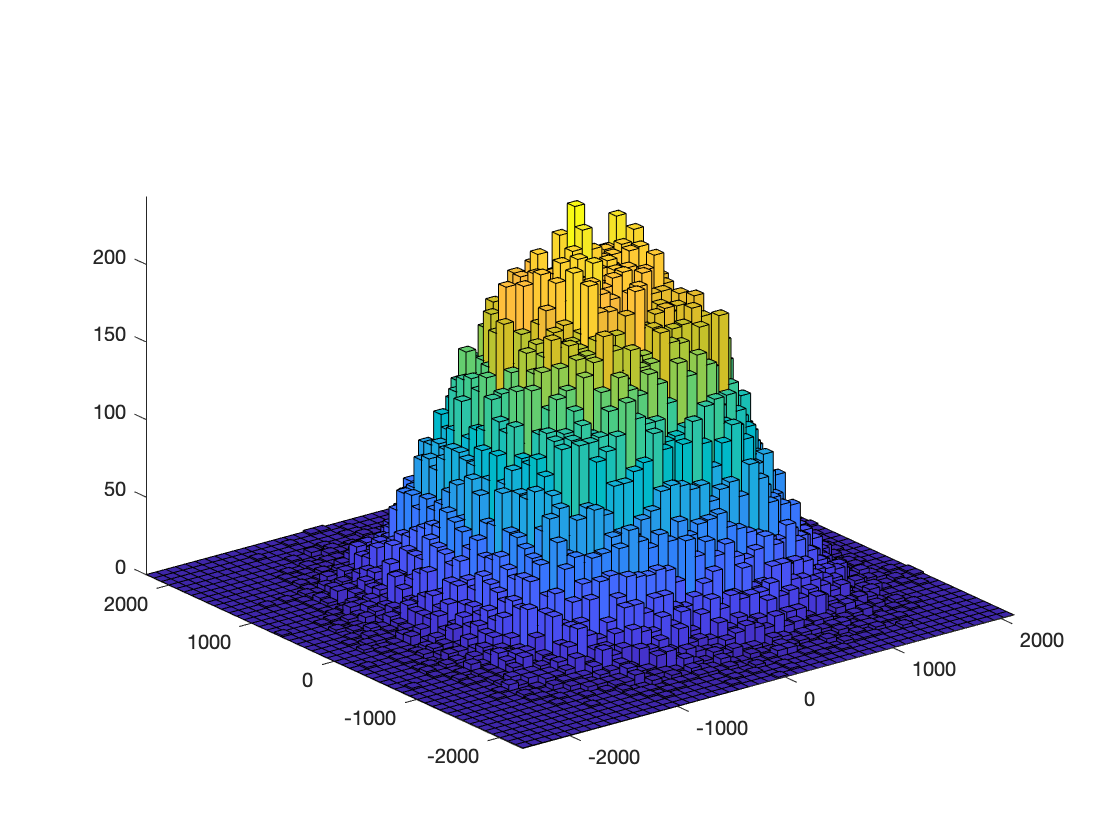}}\\
		\subfigure{\footnotesize{\hspace{5mm} (c) \hspace{45mm} (d) } }

		\caption{\footnotesize{The evolution of the empirical measure $\mu^{N}_{m}(\bm{\xi})={1\over N}\sum_{k=1}^{N}\delta(\bm{\xi}-\bm{\xi}_{m}^{k})$ of the SGD particles $\bm{\xi}_{m}^{1},\cdots,\bm{\xi}_{m}^{N}\in \real^{2}$ at different iterations $m$. The empirical measure of random feature maps seemingly converges to a Gaussian stationary measure corresponding to a Gaussian RBF kernel. Panel (a): $m=0$, Panel (b): $m=300$, Panel (c):  $m=1000$, and Panel (d): $m=2500$. }}
		\label{Fig:7} 
	\end{center}	
\end{figure}

\subsubsection{Kernel Learning Approach} 
\label{subsubsection:Kernel Learning Approach}
Figure \ref{Fig:5} depicts our two-phase kernel learning procedure which we also employed in our implementations of Algorithm \ref{Algorithm:1} on benchmark data-sets of Section \ref{Subsection:Performance on benchmark datasets} in the main text. The kernel learning approach consists of training the auto-encoder and the kernel optimization sequentially, \textit{i.e.}, 
\begin{align}
\label{Eq:solution_of_the_optimization}
\sup_{\widehat{\mu}^{N}\in\mathcal{P}_{N}}\sup_{\iota \in \mathcal{Q}}\widehat{\mathrm{MMD}}_{K_{\widehat{\mu}^{N}}\circ \iota}^{\alpha}[P_{\bm{V}},P_{\bm{W}}].
\end{align}
where the function class is defined $\mathcal{Q}\df \{\iota(\bm{z})=\sigma(\bm{\Sigma}\bm{z}+\bm{b}),\bm{\Sigma}\in \real^{d_{0}\times d},\bm{b}\in \real^{d_{0}}\}$, and $(K_{\widehat{\mu}^{N}}\circ \iota)(\bm{x}_{1},\bm{x}_{2})=K_{\widehat{\mu}^{N}}(\iota(\bm{x}_{1}),\iota(\bm{x}_{2}))$. Here, $\bm{\sigma}(\cdot)$ is the sigmoid non-linearity. Now, we consider a two-phase optimization procedure: 
\begin{itemize}
	\item \textbf{Phase (I)}: we fix the kernel function, and optimize the auto-encoder to compute a co-variance matrix $\bm{\Sigma}$ and the bias term $\bm{b}$ for the dimensionality reduction.
	
	\item \textbf{Phase (II)}: we optimize the kernel based from the learned embedded features $\iota(\bm{x})$. 
\end{itemize}
This two-phase procedure significantly improves the computational complexity of SGD as it reduces the dimensionality of random feature samples $\bm{\xi}\in \real^{D}$, $D=d_{0}\ll d$. When the kernel function $K$ is fixed, optimizing the auto-encoder is equivalent to the kernel learning step of \cite{li2017mmd}.

\begin{figure}[!t]	
	
	\begin{center}
		\hspace*{-12mm}		\subfigure{
			\includegraphics[trim={.2cm .2cm .2cm  .6cm},width=.40\linewidth]{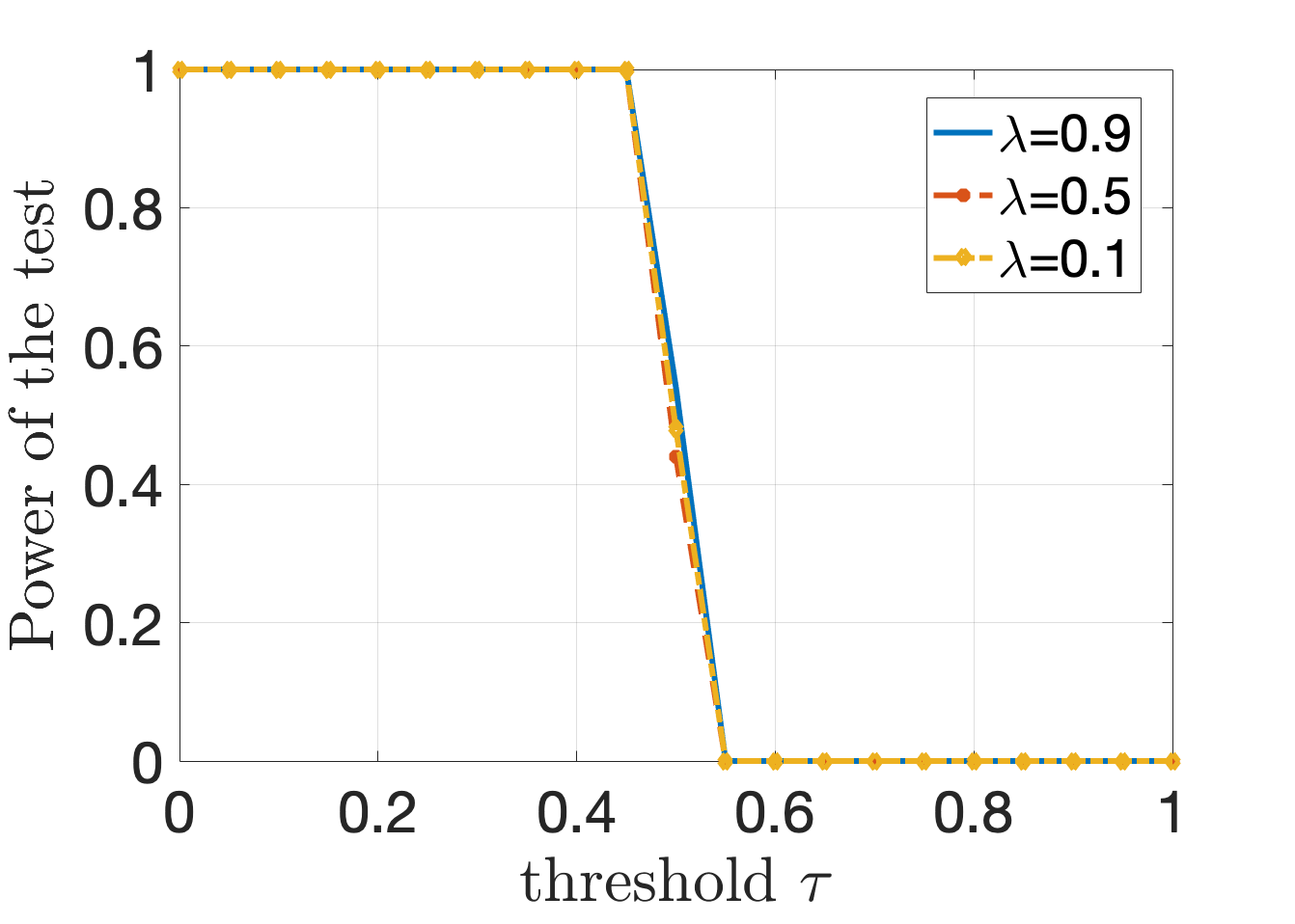} \hspace{-2mm}
			\includegraphics[trim={.2cm .2cm .2cm  .6cm},width=.39\linewidth]{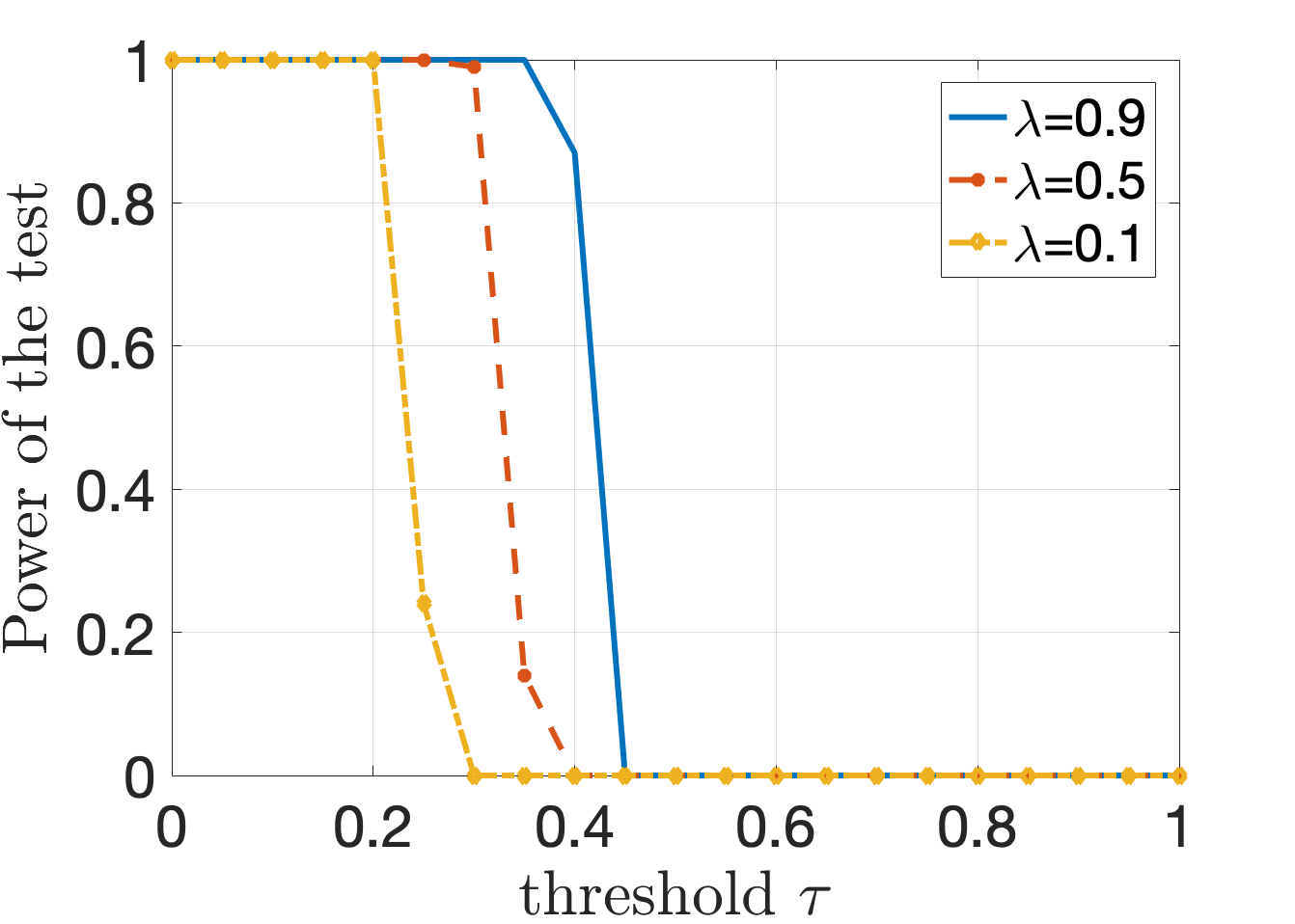}\hspace{-2mm}
			\includegraphics[trim={.2cm .2cm .2cm  .6cm},width=.40\linewidth]{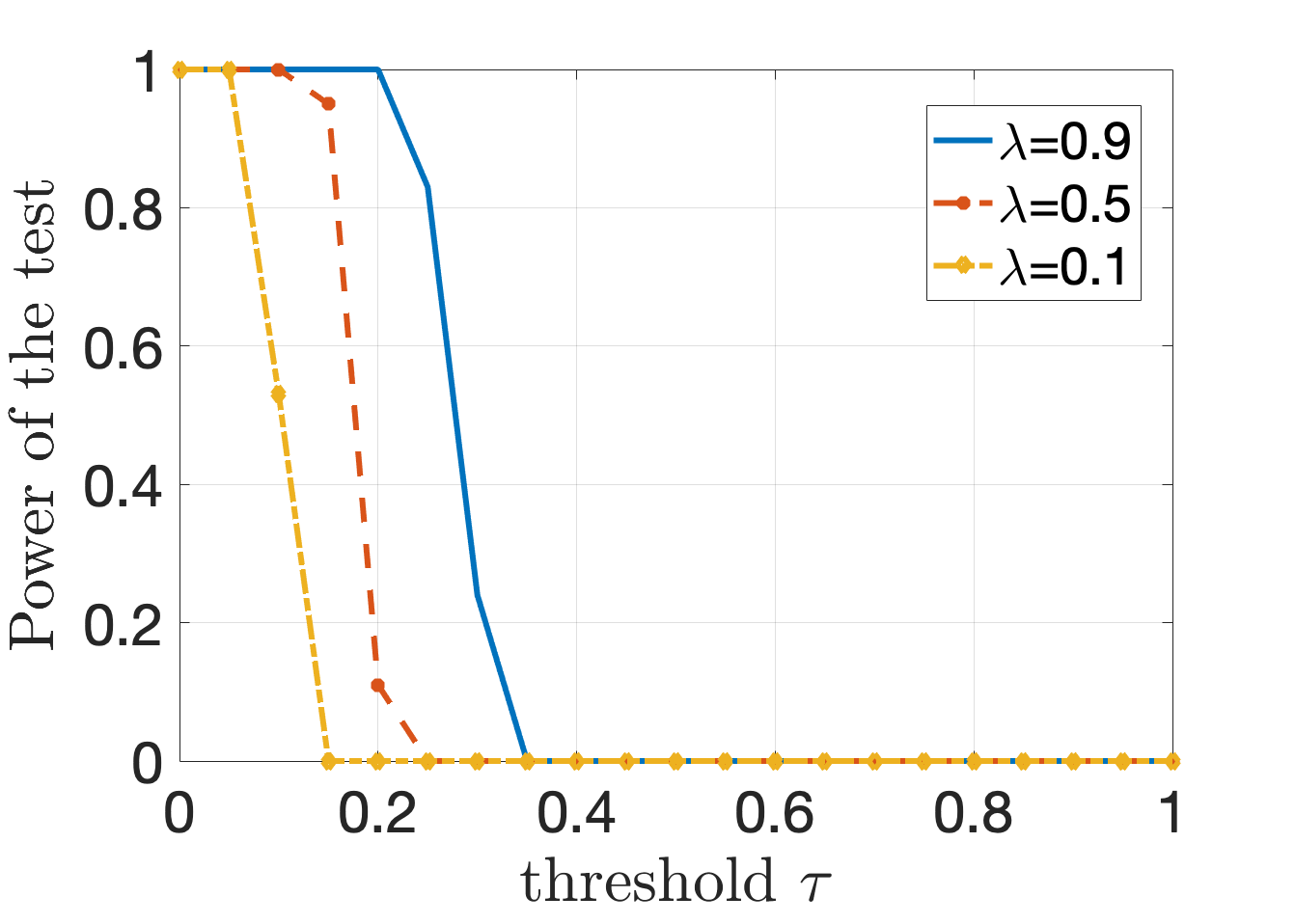}	
		}
		\\ \vspace{-2mm}
		\subfigure{\footnotesize{\hspace{5mm} (a) \hspace{48mm} (b) \hspace{48mm} (c)  } }
		\vspace{-2mm}
		\caption{\footnotesize{The statistical power versus the threshold $\tau$ for the binary hypothesis testing via the unbiased estimator of the kernel MMD. The parameters for this simulations are $\lambda\in \{0.1,0.5,0.9\}$, $d=100$, $n+m=100$, $d_{0}=50$.  Panel (a): Trained kernel using the two-phase procedure with the particle SGD in \eqref{Eq:SGD} and an auto-encoder, Panel (b): Trained kernel with an auto-encoder and a fixed Gaussian kernel with the bandwidth $\sigma=1$, Panel (c): Untrained kernel without an auto-encoder.}}
		\label{Fig:5} 
	\end{center}
	
	\begin{center}
		\hspace*{-12mm}		\subfigure{
			\includegraphics[trim={.2cm .2cm .2cm  .6cm},width=1.1\linewidth]{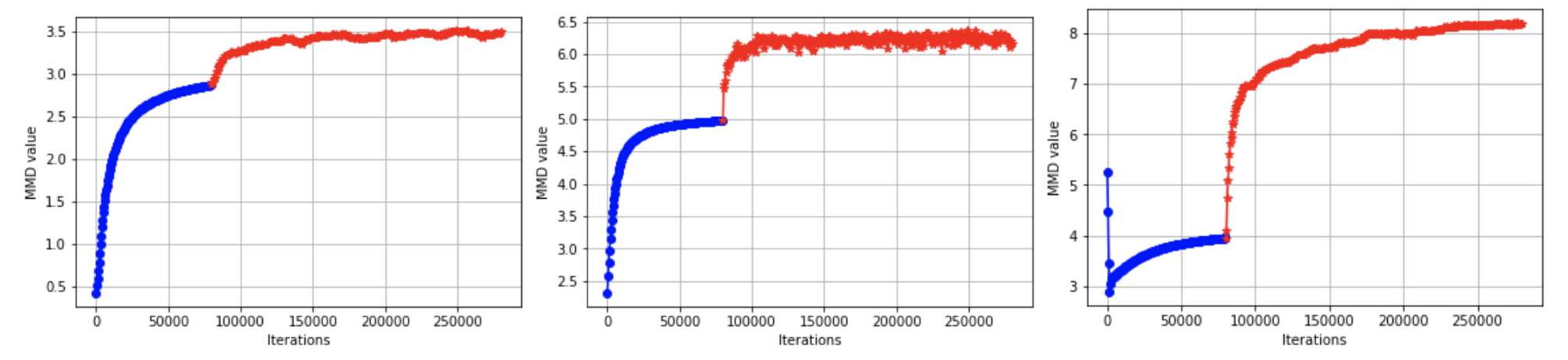}
		}
		\\ \vspace{-2mm}
		\subfigure{\footnotesize{\hspace{5mm} (a) \hspace{48mm} (b) \hspace{48mm} (c)  } }
		\vspace{-2mm}
		\caption{\footnotesize{The MMD value during the two phase procedure for the kernel training. In the first phase, an auto-encoder is trained (blue curve). In the second phase, the kernel is trained using the embedded features (red curve). Panel (a): $\lambda = 0.9$, Panel (b): $\lambda = 0.5$, Panel (c):$\lambda$ = 0.1 }}
		\label{Fig:6} 
	\end{center}
\end{figure}

\subsubsection{Statistical Hypothesis Testing with the Kernel MMD}

Let $\bm{V}_{1},\cdots,\bm{V}_{m}\sim_{\text{i.i.d.}} P_{\bm{V}}=\mathsf{N}(\bm{0},(1+\lambda)\bm{I}_{d\times d})$, and $\bm{W}_{1},\cdots,\bm{W}_{n}\sim_{\text{i.i.d.}} P_{\bm{W}}=\mathsf{N}(\bm{0},(1-\lambda)\bm{I}_{d\times d})$. Given these i.i.d. samples, the statistical test $\mathcal{T}(\{\bm{V}_{i}\}_{i=1}^{m},\{\bm{W}_{i}\}_{j=1}^{n}):\mathcal{V}^{m}\times \mathcal{W}^{n}\rightarrow \{0,1\}$ is used to distinguish between these hypotheses:
\begin{itemize}
	\item \textbf{Null hypothesis} $\mathsf{H}_{0}:P_{\bm{V}}=P_{\bm{W}}$ (thus $\lambda=0$),
	\item \textbf{Alternative hypothesis} $\mathsf{H}_{1}:P_{\bm{V}}\not= P_{\bm{W}}$ (thus $\lambda>0$).
\end{itemize}
To perform hypothesis testing via the kernel MMD, we require that $\mathcal{H}_{\mathcal{X}}$ is a universal RKHS, defined on a compact metric space $\mathcal{X}$. Universality requires that the kernel $K(\cdot,\cdot)$ be continuous and, $\mathcal{H}_{\mathcal{X}}$ be dense in $C(\mathcal{X})$.
Under these conditions, the following theorem establishes that the kernel MMD is indeed a metric:

\begin{theorem}\textsc{(Metrizablity of the RKHS)}
	Let $\mathcal{F}$ denotes a unit ball in a universal RKHS $\mathcal{H}_{\mathcal{X}}$ defined on a compact metric space $\mathcal{X}$ with the associated continuous kernel $K(\cdot,\cdot)$. Then,  the kernel MMD is a metric in the sense that $\mathrm{MMD}_{K}[P_{\bm{V}},P_{\bm{W}}]=0$ if and only if $P_{\bm{V}}=P_{\bm{W}}$.
\end{theorem}

To design a test, let $\widehat{\mu}^{N}_{m}(\bm{\xi})={1\over N}\sum_{k=1}^{N}\delta(\bm{\xi}-\bm{\xi}_{m}^{k})$ denotes the solution of SGD in \eqref{Eq:SGD} for solving the optimization problem. Consider the following MMD estimator consisting of two $U$-statistics and an empirical function
\begin{align}
\nonumber
\widehat{\mathrm{MMD}}_{K_{\widehat{\mu}_{m}^{N}}\circ \iota}\big[\{\bm{V}_{i}\}_{i=1}^{m},\{\bm{W}_{i}\}_{i=1}^{n}\big]&=\dfrac{1}{m(m-1)}\sum_{k=1}^{N}\sum_{i\not =j}\varphi(\iota(\bm{V}_{i}),\bm{\xi}_{m}^{k})\varphi(\iota(\bm{V}_{j}),\bm{\xi}_{m}^{k})\\ \nonumber
&\hspace{4mm}+\dfrac{1}{n(n-1)}\sum_{k=1}^{N}\sum_{i\not =j}\varphi(\iota(\bm{W}_{i}),\bm{\xi}_{m}^{k})\varphi(\iota(\bm{W}_{j}),\bm{\xi}_{m}^{k})\\ \label{Eq:Negative_Estimator}
&\hspace{4mm}-\dfrac{2}{nm}\sum_{k=1}^{N}\sum_{i=1}^{m}\sum_{j=1}^{n}\varphi(\iota(\bm{W}_{i}),\bm{\xi}_{m}^{k})\varphi(\iota(\bm{V}_{j}),\bm{\xi}^{k}_{m}).
\end{align}
Given the samples $\{\bm{V}_{i}\}_{i=1}^{m}$ and $\{\bm{W}_{i}\}_{i=1}^{n}$, we design a test statistic as below
\begin{align}
\label{Eq:Test_statistics}
\mathcal{T}(\{\bm{V}_{i}\}_{i=1}^{m},\{\bm{W}_{i}\}_{i=1}^{n})\df \begin{cases}
\mathsf{H}_{0}  & \text{if } \widehat{\mathrm{MMD}}_{K_{\widehat{\mu}_{m}^{N}}\circ \iota}\big[\{\bm{V}_{i}\}_{i=1}^{m},\{\bm{W}_{i}\}_{i=1}^{n}\big]\leq \tau \\
\mathsf{H}_{1}  & \text{if } \widehat{\mathrm{MMD}}_{K_{\widehat{\mu}_{m}^{N}}\circ \iota}\big[\{\bm{V}_{i}\}_{i=1}^{m},\{\bm{W}_{i}\}_{i=1}^{n}\big]> \tau,
\end{cases}.
\end{align}
where $\tau\in \real$ is a threshold. Notice that the unbiased MMD estimator of  \eqref{Eq:Negative_Estimator} can be negative despite the fact that the population MMD is non-negative. Consequently, negative values for the statistical threshold $\tau$  \eqref{Eq:Test_statistics} are admissible. In the following simulations, we only consider non-negative values for the threshold $\tau$. 

A Type I error is made when $\mathsf{H}_{0}$ is rejected based on the observed samples, despite the null hypothesis having generated the data. Conversely, a Type II error occurs when $\mathsf{H}_{0}$ is accepted despite the alternative hypothesis $\mathsf{H}_{1}$ being true. The \textit{significance level} $\alpha$ of a test is an upper bound on the probability of a Type I error: this is a design parameter of the test which must be set in advance, and is used to determine the threshold to which we compare the test statistic. The \textit{power of a test} is the probability of rejecting the null hypothesis $\mathsf{H}_{0}$ when it is indeed incorrect. In particular,
\begin{align}
\mathrm{Power}\df \prob(\text{reject}\ \mathsf{H}_{0}| \mathsf{H}_{1} \ \text{is true}).
\end{align}

In this sense, the statistical power controls the probability of making Type II errors.

\subsection{Empirical results on benchmark data-sets}


In Figure \ref{Fig:7}, we show the evolution of the empirical measure $\mu_{m}^{N}(\bm{\xi})$ of SGD particles by plotting the 2D histogram of the particles $\bm{\xi}_{m}^{1},\cdots,\bm{\xi}_{m}^{N}\in \real^{d}$ at different iterations of SGD. Clearly, starting with a uniform distribution in \ref{Fig:7}(a), the empirical measure seemingly evolves into a Gaussian measure in Figure \ref{Fig:7}(d). The evolution to a Gaussian distribution demonstrates that the RBF Gaussian kernel corresponding to a Gaussian distribution for the random features indeed provides a good kernel function for the underlying hypothesis test with Gaussian distributions.

In Figure \ref{Fig:5}, we evaluate the power of the test for $100$ trials of hypothesis test using the test statistics of \eqref{Eq:Test_statistics}. To obtain the result, we used an autoencoder to reduce the dimension from $d=100$ to $p=50$. Clearly, for the trained kernel in Panel (a) of Figure \ref{Fig:5}, the threshold $\tau$ for which $\mathrm{Power}=1$ increases after learning the kernel via the two phase procedure described earlier. In comparison, in Panel (b), we observe that training an auto-encoder only with a fixed standard Gaussian kernel $K(\bm{x},\bm{y})=\exp(-\|\bm{x}-\bm{y}\|_{2}^{2})$ attains lower thresholds compared to our two-phase procedure. In Panel (c), we demonstrate the case of a fixed Gaussian kernel without an auto-encoder. In this case, the threshold is significantly lower due to the large dimensionality of the data. 

From Figure \ref{Fig:5}, we also observe that interestingly, the phase transition in the statistical threshold $\tau$ is less sensitive to the parameter $\lambda$. This phenomenon is also observed in Figure \ref{Fig:6}, where the improvements due to SGD for more difficult two sample tests (\textit{i.e.}, smaller values of $\lambda$) is larger than the case of distinguishable distributions. 

\subsection{Performance on benchmark datasets}
\label{Subsection:Performance on benchmark datasets}

\begin{figure}[!t]
	\begin{center}
		\hspace*{-5mm}		\subfigure{
			\includegraphics[trim={.2cm .2cm .2cm  .6cm},width=.25\linewidth]{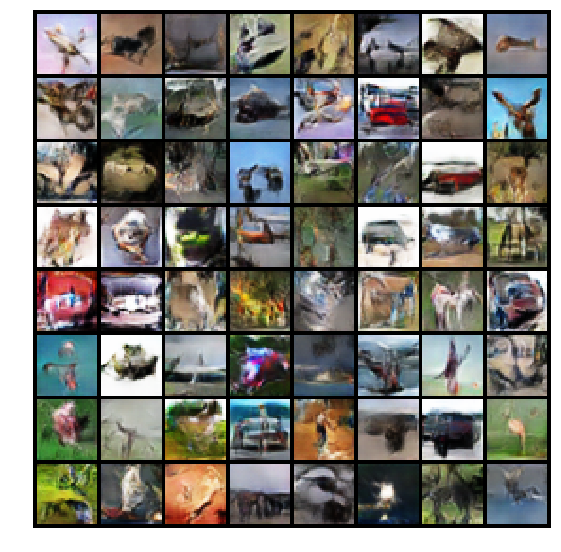} 
			\includegraphics[trim={.2cm .2cm .2cm  .6cm},width=.25\linewidth]{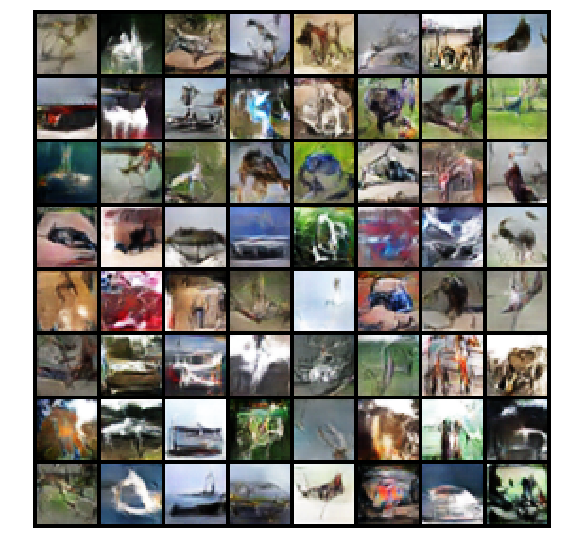} 
			\includegraphics[trim={.2cm .2cm .2cm  .6cm},width=.25\linewidth]{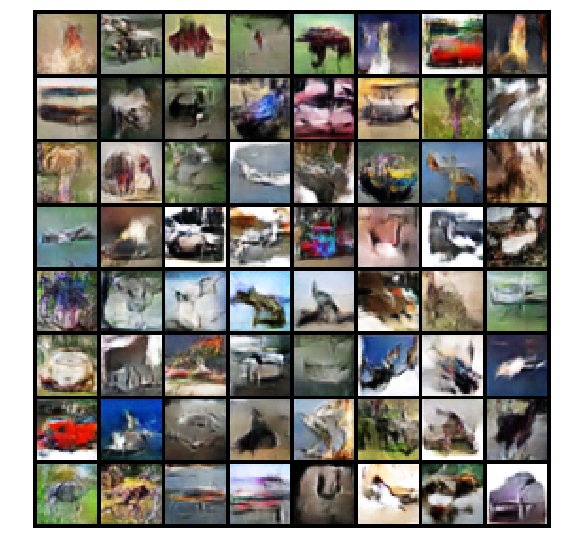}  
			\hspace{1mm}	\vspace*{-2mm}				\includegraphics[width=.24\linewidth]{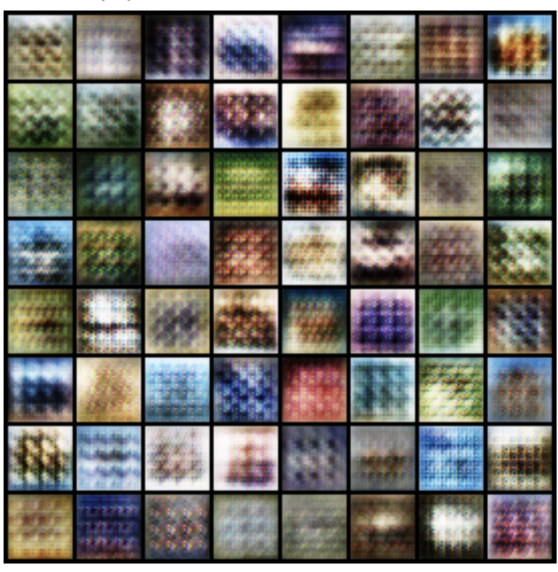}  }	
		\\ 
		\vspace{-1mm}
		\subfigure{\footnotesize{\hspace{5mm} (a) \hspace{30mm} (b) \hspace{30mm} (c) \hspace{30mm} (d)} }
		
		\hspace*{-7mm}		\subfigure{
			\includegraphics[trim={.2cm .2cm .2cm  .6cm},width=.25\linewidth]{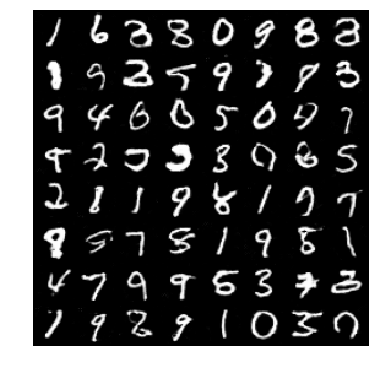} 
			\includegraphics[trim={.2cm .2cm .2cm  .6cm},width=.25\linewidth]{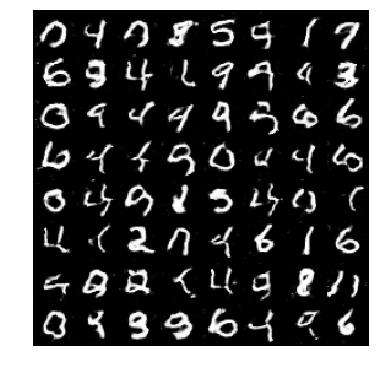} 
			\includegraphics[trim={.2cm .2cm .2cm  .2cm},width=.25\linewidth]{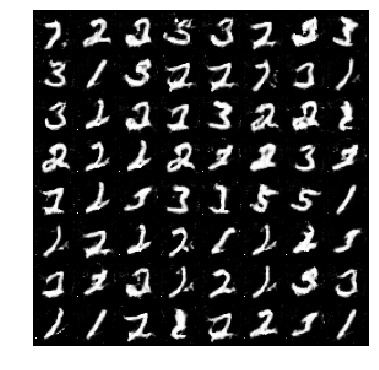} }  
		\includegraphics[trim={.2cm .15cm .15cm  .15cm},width=.24\linewidth]{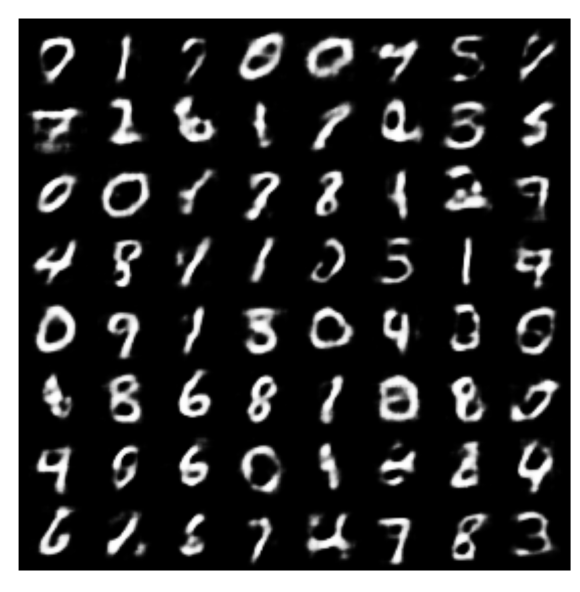}  
		\\ 
		\vspace{-1mm}
		\subfigure{\footnotesize{\hspace{5mm} (e) \hspace{30mm} (f) \hspace{30mm} (g) \hspace{30mm} (h)} }
		\vspace{-3mm}
		\caption{\footnotesize{Sample generated images using CIFAR-10 (top row), and MNIST (bottom row) data-sets. Panels (a)-(e): Proposed MMD GAN with an \textit{automatic} kernel selection via the particle SGD (Algorithm \ref{Algorithm:1}), Panels (b)-(f): MMD GAN \cite{li2017mmd} with an auto-encoder for dimensionality reduction in conjunction with a mixed RBF Gaussian kernel whose bandwidths are optimized via cross-validation, Panels (c)-(g): MMD GAN in \cite{li2017mmd} with a single RBF Gaussian kernel with an auto-encoder for dimensionality reduction in conjunction with a single RBF Gaussian kernel whose bandwidth is optimized via cross-validation, and Panel (d)-(g): GMMN without an auto-encoder \cite{li2015generative}.}}
		\label{Fig:1} 
	\end{center}
\end{figure}

We evaluate our kernel learning approach on large-scale benchmark data-sets. We train our MMD GAN model on two distinct types of data-sets, namely on MNIST \cite{lecun1998gradient} and CIFAR-10 \cite{lecun1998gradient}, where the size of training instances are $60\times 10^{3}$ and $50\times 10^{3}$, respectively. All the generated samples are from a fixed noise random vectors and are not singled out.

\subsubsection{Implementation and hyper-parameters.} We implement Algorithm \ref{Algorithm:1} as well as MMD GAN \cite{li2017mmd} in ${\tt{Pytorch}}$ using NVIDIA Titan V100 32GB graphics processing units (GPUs). The source code of Algorithm \ref{Algorithm:1} is built upon the code of \cite{li2017mmd}, and retains the auto-encoder implementation. In particular, we use a sequential training of the auto-encoder and kernel as explained in the synthetic data-set. For a fair comparison, our hyper-parameters are adjusted as in \cite{li2017mmd}, \textit{i.e.}, the learning rate of 0.00005 is considered for RMSProp \cite{tieleman2012lecture}.  Moreover, the batch-size for training the generator and auto-encoder is $n=64$. The learning rate of particle SGD is tuned to $\eta=10$.

\subsubsection{Random feature maps.} To approximate the kernel, we use the the random feature model of Rahimi and Recht \cite{rahimi2008random,rahimi2009weighted}, where  $\varphi(\bm{x};\bm{\xi})=\sqrt{2/N}\cos(\bm{x}^{T}\bm{\xi}+b)$. Here $b\sim \mathrm{Uniform}[-\pi,\pi]$ is a random bias term. 

\subsubsection{Practical considerations.} When data-samples $\{\bm{V}_{i}\}\in \real^{d}$ are high dimensional (\textit{e.g.} CIFAR-10), the particles $\bm{\xi}^{1},\cdots,\bm{\xi}^{N}\in \real^{p},p=d$ in SGD \eqref{Eq:SGD} are also high-dimensional. To reduce the dimensionality of the particles, we apply an auto-encoder architecture similar to \cite{li2017mmd}, and train our kernel on top of learned embedded features. More specifically, in our simulations, we train an auto-encoder where the dimensionality of the latent space is $h=10$ for MNIST, and $h=128$ (thus $p=d=128$) for CIFAR-10. Therefore, the particles $\bm{\xi}^{1},\cdots,\bm{\xi}^{N}$ in subsequent kernel training phase have the dimension of $D=10$, and $D=128$, respectively.

\subsubsection{Choice of the scaling parameter $\alpha$.} There is a trade-off in the choice of $\alpha$. While for large values of $\alpha$, the kernel is better able to separate data-samples from generated samples, in practice, a large value of $\alpha$ slows down the convergence of particle SGD. This is due to the fact that the coupling strength between the particles in Eq. \eqref{Eq:SGD} decrease as $\alpha$ increase. The scaling factor is set to be $\alpha=1$ in all the following experiments.

\subsubsection{Qualitative comparison.} We now show that \textit{without} the bandwidth tuning for Gaussian kernels and using the particle SGD to learn the kernel, we can attain better visual results on benchmark data-sets.  In Figure \ref{Fig:1}, we show the generated samples on  CIFAR-10 and MNIST data-sets, using our Algorithm \ref{Algorithm:1}, MMD GAN  \cite{li2017mmd}  with a mixed and homogeneous Gaussian RBF kernels, and GMMN \cite{li2015generative}. 

Figure \ref{Fig:1}(a) shows the samples from Algorithm \ref{Algorithm:1}, Figure \ref{Fig:1}(b) shows the samples from MMD GAN  \cite{li2017mmd} with the RBF Gaussian mixture kernel of Eq. , where $\sigma_{k}\in \{1,2,4,8,16 \}$ are the bandwidths of the Gaussian kernels that are fine tuned and optimized. We observe that our MMD GAN with \textit{automatic} kernel learning visually attains similar results to MMD GAN \cite{li2017mmd} which requires \textit{manual} tuning of the hyper-parameters. In Figure \ref{Fig:1}(c), we show the MMD GAN result with a single kernel RBF Gaussian kernel whose bandwidth are optimized via cross-validation and is $\sigma=16$. Lastly, in Figure \ref{Fig:1}(d), we show the samples from GMMN \cite{li2015generative} which does not exploit an auto-encoder or kernel training. Clearly, GMMN yield a poor results compared to other methods due to high dimensionality of features, as well as the lack of an efficient method to train the kernel.

On MNIST data-set in Figure \ref{Fig:1}(e)-(h), the difference between our method and MMD GAN \cite{li2017mmd} is visually more pronounced. We observe that without a manual tuning of the kernel bandwidth and by using the particle SGD \eqref{Eq:SGD} to optimize the kernel, we attain better generated images in Figure \ref{Fig:1}(e), compared to MMD GAN with mixed RBF Gaussian kernel and \textit{manual} bandwidth tuning in Figure \ref{Fig:1}(f). Moreover, using a single RBF Gaussian kernel yields a poor result regardless of the choice of its bandwidth. The generated images from GMMN is also shown in Figure \ref{Fig:1}(h).

\subsubsection{Quantitivative comparison.} To quantitatively measure the quality and diversity of generated samples, we compute the inception score (IS) \cite{salimans2016improved} as well as Fr\`{e}chet Inception Distance (FID) \cite{heusel2017gans} on CIFAR-10 images. 

Intuitively, the inception score is used for GANs to measure samples quality and diversity.  This score is based on the Inception-v3 Network \cite{szegedy2016rethinking} which is a deep convolutional architecture designed for classification tasks on ImageNet \cite{deng2009imagenet}, a dataset consisting of 1.2 million RGB images from 1000 classes. Given an image $\bm{x}$, the task of the network is to output a class label $y$ in the form of a vector of probabilities. The inception score uses an Inception-v3 Network pre-trained on ImageNet and calculates a statistic of the network’s outputs when applied to generated images. More precisely, the inception score of a generative model is 
\begin{align}
\mathrm{IS}\df \exp\left( \expect_{\bm{X}\sim P_{\bm{W}}} \mathrm{D}_{\mathrm{KL}}(P_{Y|\bm{X}}||P_{Y})\right),
\end{align}
where $\mathrm{D}_{\mathrm{KL}}(\cdot||\cdot)$ is the Kullback-Leibler divergence. The definition of the inception score  is motivated by the following two observations:
\begin{itemize}
	\item[(\textit{i})] The images generated should contain meaningful objects, \textit{i.e.},  $P_{Y|\bm{X}}$ should be low entropy. In other words, images with meaningful objects are supposed to have low label (output) entropy, that is, they belong to few object classes
	
	\item[(\textit{ii})] The generative algorithm should generate diverse images from all the different classes in ImageNet, \textit{i.e.}, the distribution of labels $P_{Y}$ should have a high entropy.
\end{itemize}

The FID improves on IS by actually comparing the statistics of generated samples to real samples, instead of evaluating generated samples independently.  In particular, Heusel, \textit{et al}. \cite{heusel2017gans}  propose to use  the Fr\'{e}chet distance between two multivariate Gaussians $\mathsf{N}(\bm{\mu}_{1},\bm{\Sigma}_{1})$ and $\mathsf{N}(\bm{\mu}_{2},\bm{\Sigma}_{2})$ as follows
\begin{align}
\mathrm{FID}\df \|\bm{\mu}_{1}-\bm{\mu}_{2}\|_{2}^{2}+\mathrm{Tr}\left(\bm{\Sigma}_{1}+\bm{\Sigma}_{2}-2(\bm{\Sigma}_{1}\bm{\Sigma}_{2})^{1\over 2}\right).
\end{align}

In Table \ref{Table:1}, we report the quantitative measures for different MMD GAN model using different scoring metric. Note that in Table \ref{Table:1} lower FID scores and higher IS scores indicate a better performance. We observe from Table \ref{Table:1} that our approach attain lower FID score, and higher IS score compared to MMD GAN with single Gaussian kernel (bandwidth $\sigma=16$), and a mixture Gaussian kernel (bandwidths $\{1,2,4,8,16\}$).

\vspace{-4mm} 
\begin{center}
	\centering
	\begin{table}[t]
	\begin{tabular}{||c| c c||}
			\hline
				\rowcolor{Gray}
			Method & FID $(\downarrow)$ & IS $(\uparrow)$  \\ [0.5ex] 
			\hline\hline
			MMD GAN (Gaussian) \cite{li2017mmd} & $67.244\pm  0.134$ & 5.608$\pm$0.051   \\ 
			\hline
			MMD GAN (Mixture Gaussian) \cite{li2017mmd} & $67.129\pm  0.148$ & 5.850$\pm$0.055    \\
			\hline
			SGD Alg. \ref{Algorithm:1} & $\bm{65.059\pm  0.153}$ &  $\bm{5.97\pm 0.046}$  \\
			\hline
			Real-data &0 & 11.237$\pm$0.116 \\
			\hline
		\end{tabular}
		\caption{Comparison of the quantitative performance measures of MMD GANs with different kernel learning approaches.}
		\label{Table:1}
	\end{table}
	\vspace{-2mm}
\end{center}
\vspace{-3mm}

\section{Empirical Results for Classification Tasks}
\label{Section:Empirical Results for Classification Tasks}

We now turn to empirical evaluations for classification on synthetic and benchmark data-sets. We compare our method with two alternative kernel learning techniques, namely, the importance sampling of Sinha and Duchi \cite{sinha2016learning} , and the Gaussian bandwidth optimization via $k$ nearest neighbor ($k$NN) \cite{chen2017simple}.  In the sequel, we provide a brief description of each method:
\begin{itemize}
\item \textbf{Importance Sampling}: the importance sampling of Sinha and Duchi \cite{sinha2016learning} which proposes to assign a weight $w^{m}$ to each sample $\bm{\xi}^{m}$ for $m=1,2,\cdots,N$. The weights are then optimized via the following standard optimization procedure
\begin{align}
\label{Eq:Importance}
\max_{w^{1},\cdots,w^{N}\in \mathcal{Q}_{N}}\dfrac{1}{2n(n-1)} \sum_{1\leq i<j\leq n}y_{i}y_{j}\sum_{k=1}^{N}w^{m}\varphi(\bm{x}_{i};\bm{\xi}^{k})\varphi(\bm{x}_{j};\bm{\xi}^{k}),
\end{align}    
where $\mathcal{Q}_{N}\df \{\bm{w}\in \real_{+}^{N}: \langle \bm{w},\bm{1}\rangle=1, \chi(\bm{w}|| \bm{1}/N)\leq R\}$ is the distributional ball, and $\chi(\cdot||\cdot)$ is the $\chi^{2}-$divergence. Notice that the complexity of solving the optimization problem in Eq. \eqref{Eq:Importance} is insensitive to the dimension of random features $\bm{x}_{i}\in \real^{d}$. 

\item \textbf{Gaussian Kernel with $k$NN Bandwidth Selection Rule}: In this method, we fix the Gaussian kernel $K(\bm{x},\bm{y})=\exp(\|\bm{x}-\bm{y}]\|_{2}^{2}/\sigma^{2})$. To select a good bandwidth for the kernel, we use the deterministic $k$ nearest neighbor ($k$NN) approach of \cite{chen2017simple}. In particular, we choose the bandwidth according to the following rule
\begin{align}
\sigma^{2}= \dfrac{1}{n}\sum_{i=1}^{n}\|\bm{x}_{i}-\bm{x}_{i}^{k\mathrm{NN}}\|_{2}^{2},
\end{align}
where $\bm{x}_{i}^{k\mathrm{NN}}$ is defined as $k$ nearest neighbor of $\bm{x}_{i}$. In our experiments, we let $k=3$. We then generate random Fourier features $\varphi(\bm{x};\bm{\xi}_{i})=\sqrt{{2\over N}}\cos(\langle \bm{w},\bm{\xi}^{k} \rangle+b^{k})$, with $(\bm{\xi}^{k})_{1\leq k\leq N}\sim_{\mathrm{i.i.d.}} {1\over \sqrt{2\pi\sigma^{2}}}\exp(-\sigma^{2}\|\bm{\xi}\|_{2}^{2})$ and $(b^k)_{1\leq k\leq N}\sim_{\mathrm{i.i.d.}} \mathrm{Uniform}[-\pi,\pi]$.

\end{itemize}

\subsection{Empirical Results on the Synthetic Data-Set}

For experiments with the synthetic data, we follow the model of Sinha and Duchi \cite{sinha2016learning} and compare our results with the importance sampling method proposed therein and the publicly available codes on Github.\footnote{\url{https://github.com/amansinha/learning-kernels}} In particular, we generate the features from the standard Gaussian distribution $\bm{x}\sim \mathsf{N}(0,\bm{I}_{d\times d})$ with the class labels $y=\mathrm{sign}(\|\bm{x}\|_{2}-d)$, where $\bm{x}\in \real^{d}$. The underpinning idea behind this data generation model is that the Gaussian kernel is ill-suited for this classification task, as the Euclidean distance used in this kernel does not capture the underlying structure of the classes. Nevertheless, we compare our results with the Gaussian kernel with $k$NN bandwidth selection procedure.

\begin{figure}[!t]
	\begin{center}
			\subfigure{
			\includegraphics[trim={.2cm .2cm .2cm  .6cm},width=.5\linewidth]{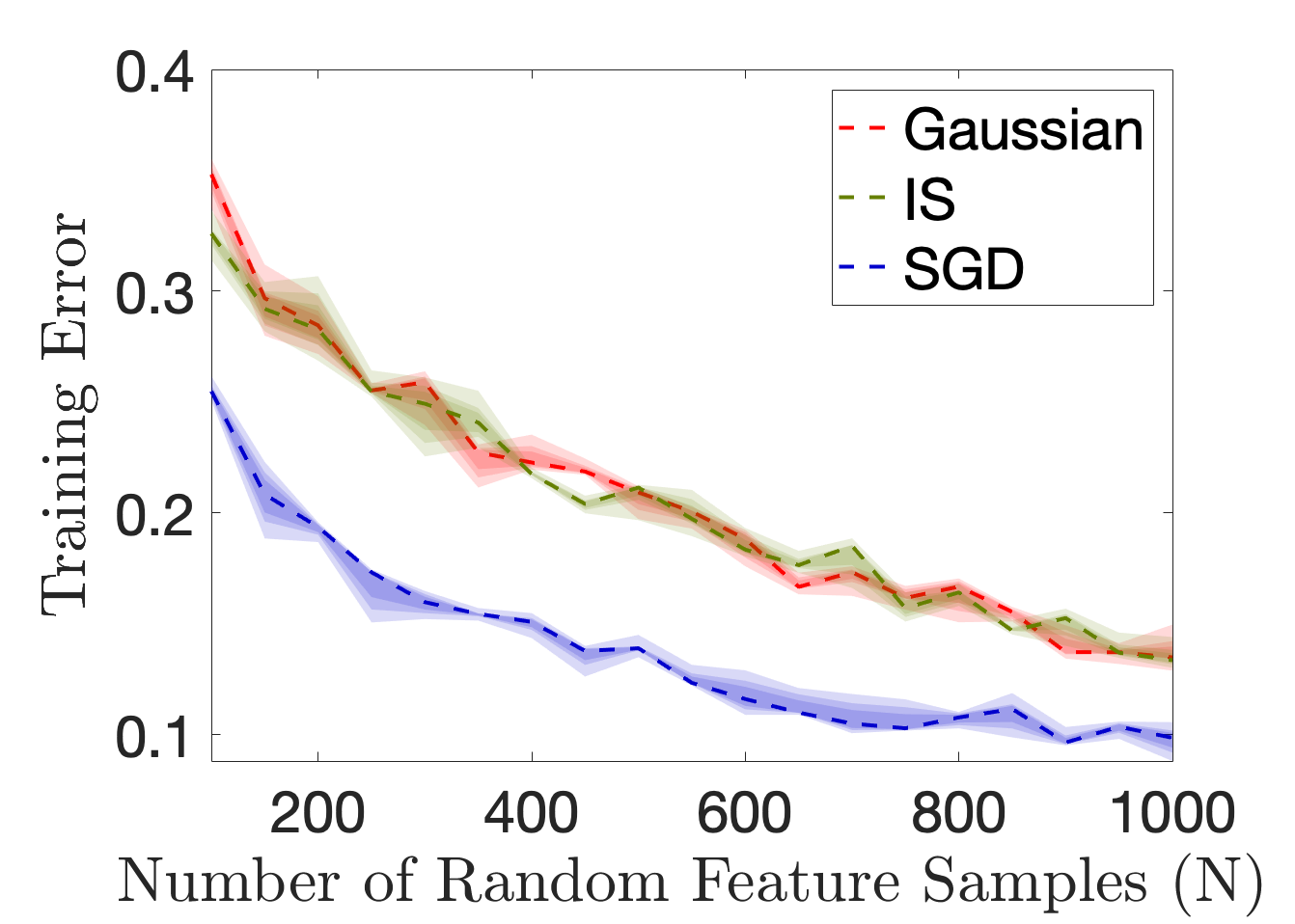} \hspace{-2mm}
			\includegraphics[trim={.2cm .2cm .2cm  .6cm},width=.5\linewidth]{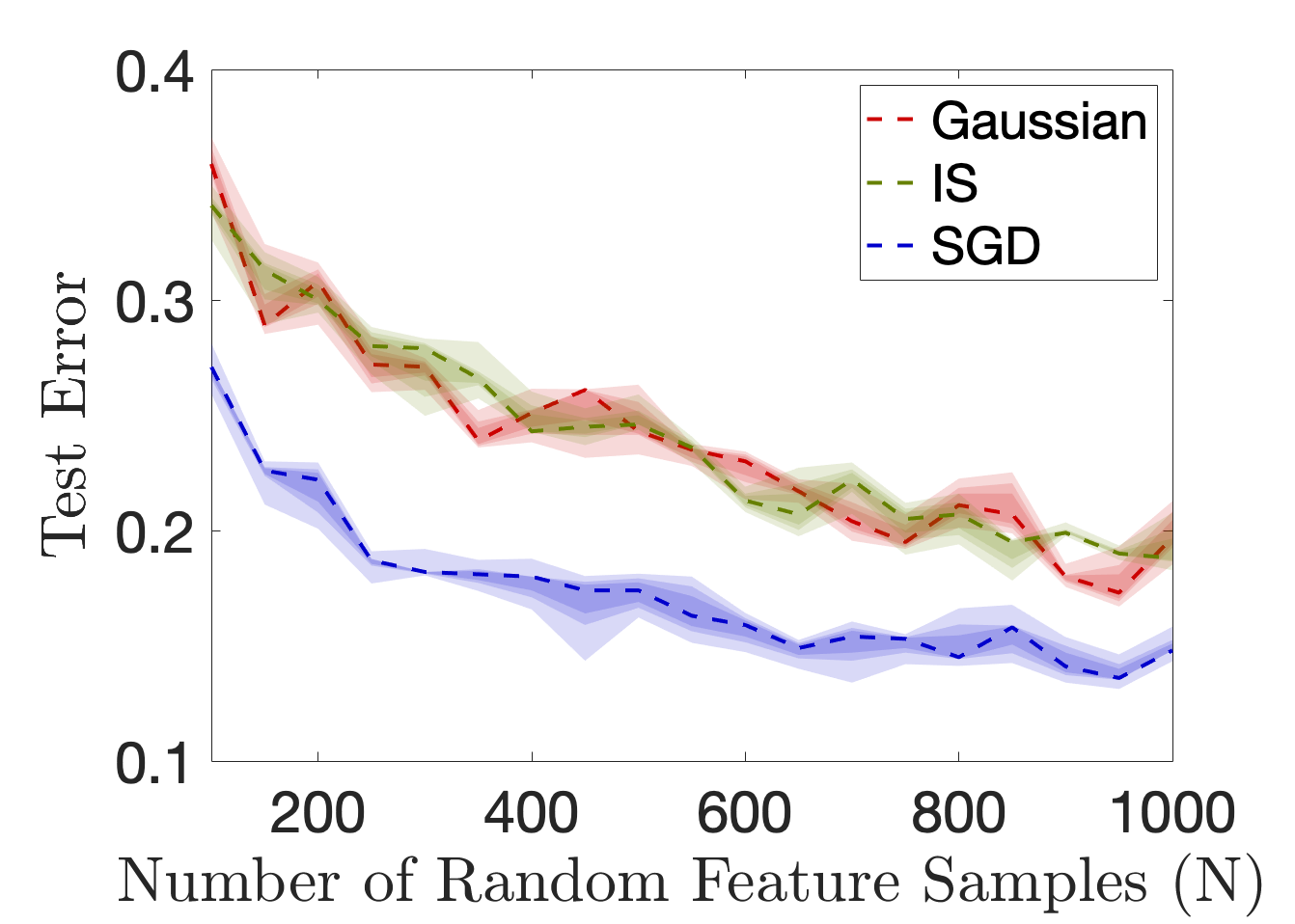} \hspace{-2mm}}\\
				\subfigure{\footnotesize{\hspace{5mm} (a) \hspace{60mm} (b)}}
		\subfigure{
			\includegraphics[trim={.2cm .2cm .2cm  .6cm},width=.5\linewidth]{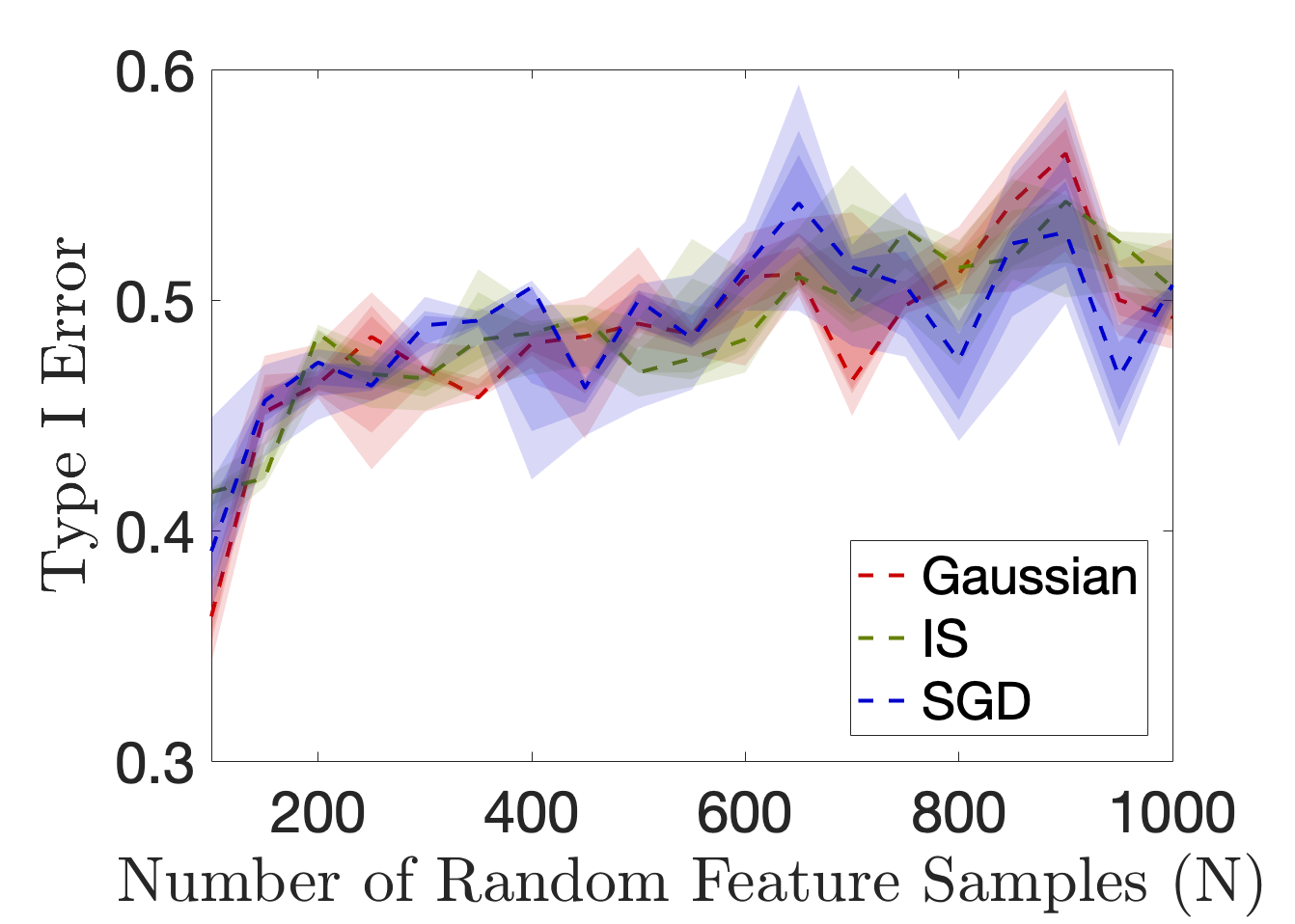}  \hspace{-2mm}
			\includegraphics[width=.5\linewidth]{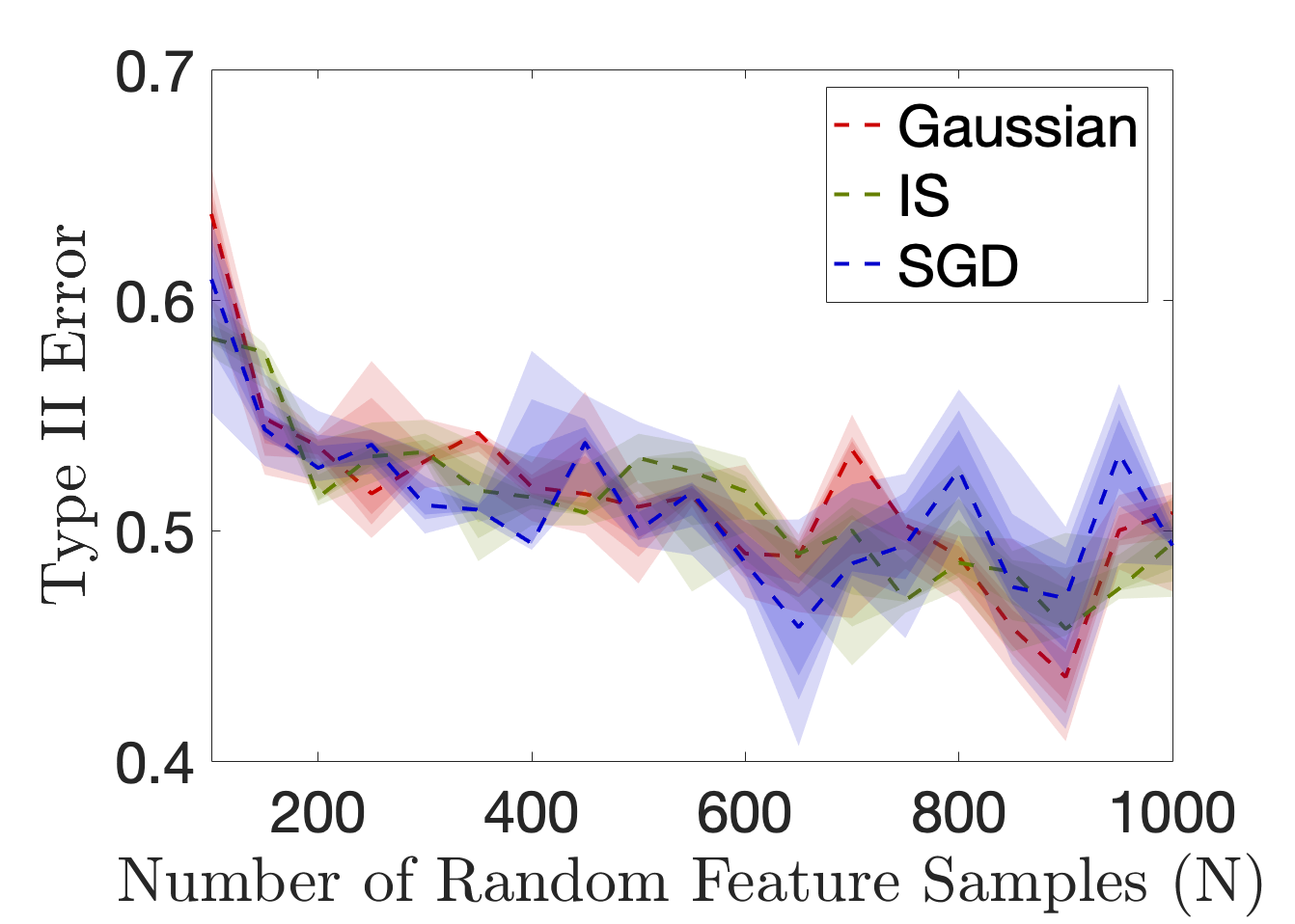}  }
		\\ 
				\subfigure{\footnotesize{\hspace{5mm} (c) \hspace{60mm} (d)}} 
		\vspace{-3mm}
		\caption{\footnotesize{ Performance comparison of SGD optimization algorithm, importance sampling, and a Gaussian kernel with the \textit{optimized} bandwidth for the data dimension $d=10$, and different number of random feature samples $N$. The shaded regions denote $\%20$, $\%30$ and $\%50$ percentile errors. Panel (a): Training error, Panel (b): Test error, Panel (c): Type I error, Panel (d): Type 2 error .}}
		\label{Fig:1} 
	\end{center}
\end{figure}

\begin{figure}[!t]
	\begin{center}
		\subfigure{
			\includegraphics[trim={.2cm .2cm .2cm  .6cm},width=.5\linewidth]{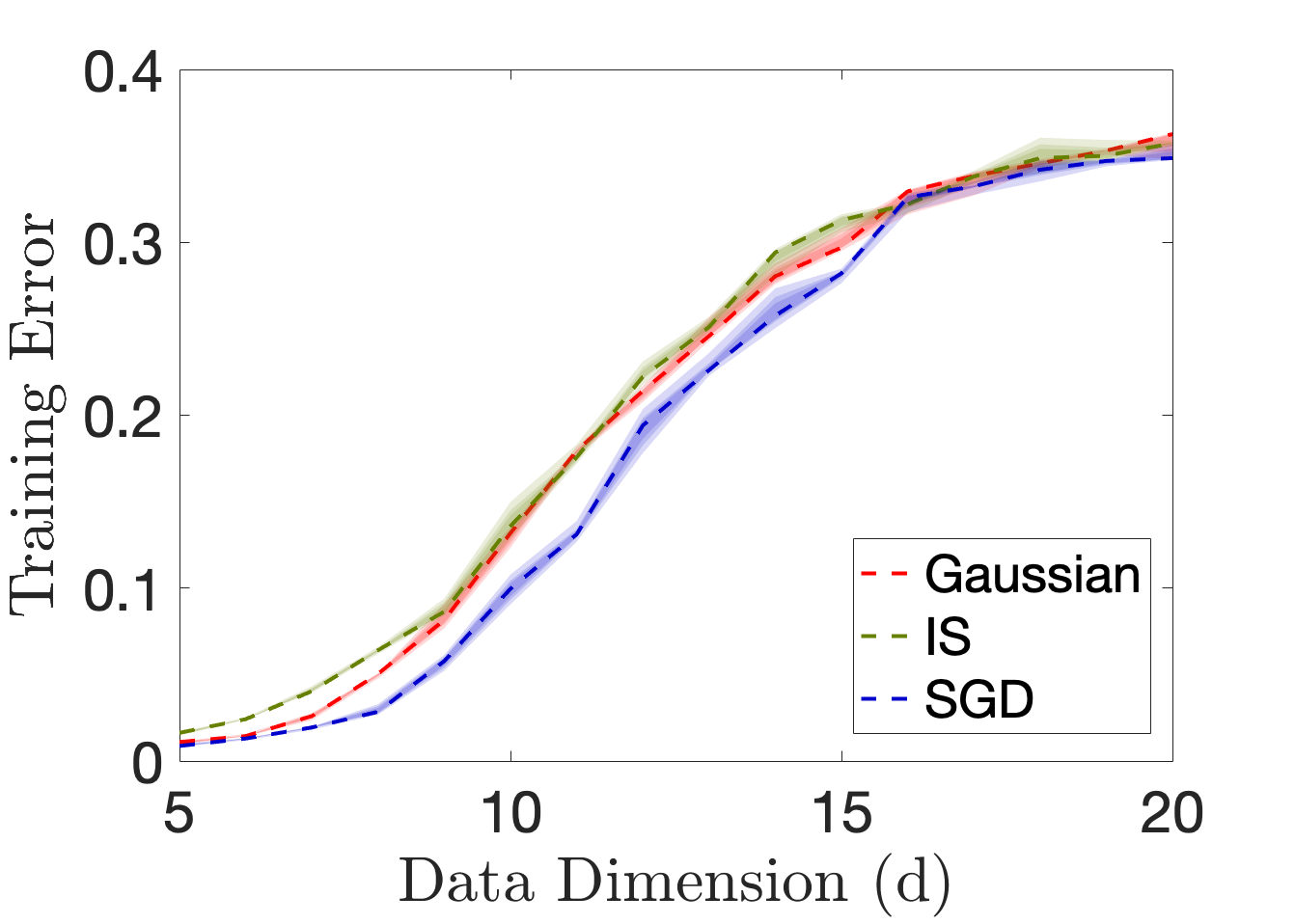} \hspace{-2mm}
			\includegraphics[trim={.2cm .2cm .2cm  .6cm},width=.5\linewidth]{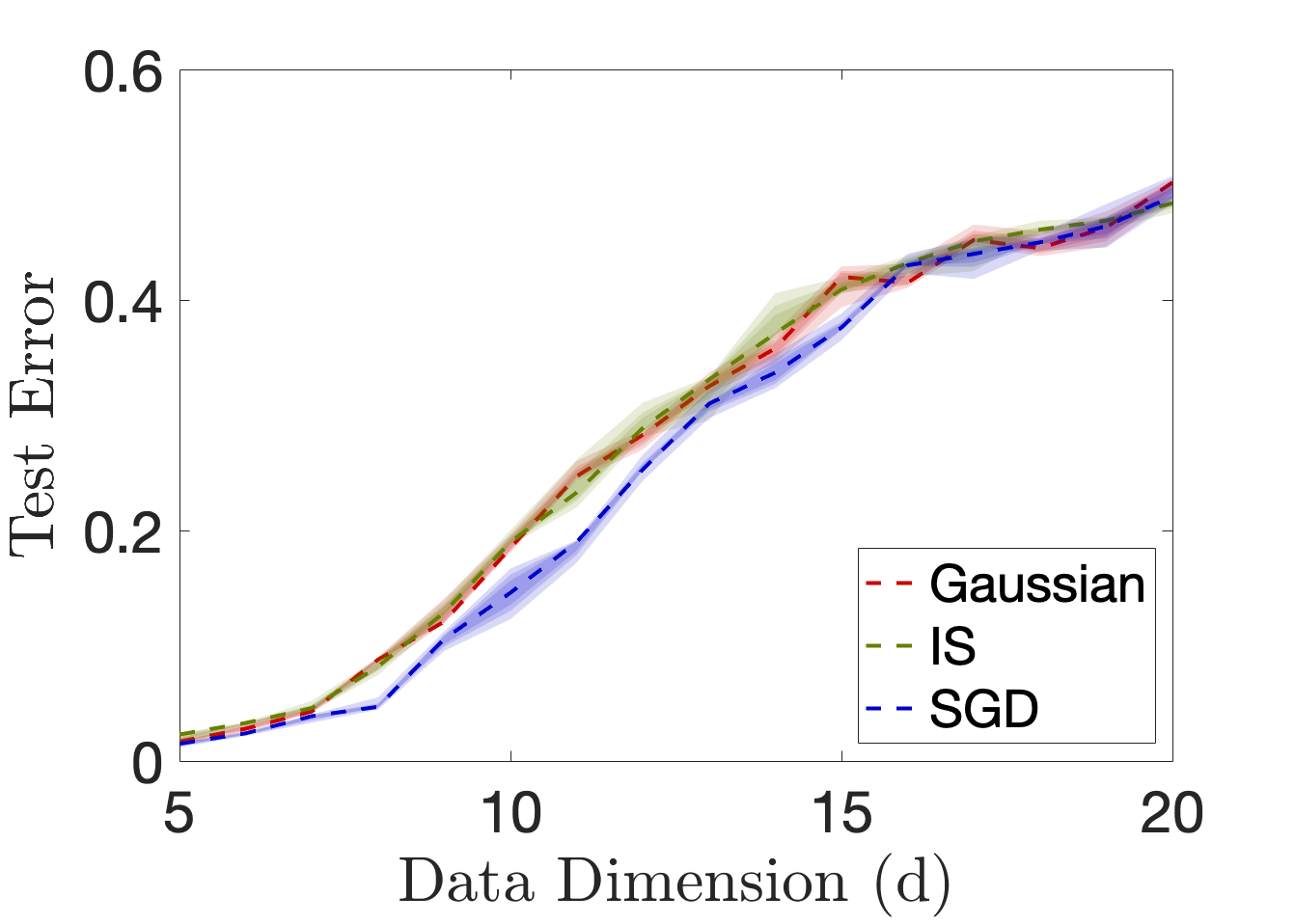} \hspace{-2mm}}\\
		\subfigure{\footnotesize{\hspace{5mm} (a) \hspace{65mm} (b)}}
		\subfigure{
			\includegraphics[trim={.2cm .2cm .2cm  .6cm},width=.5\linewidth]{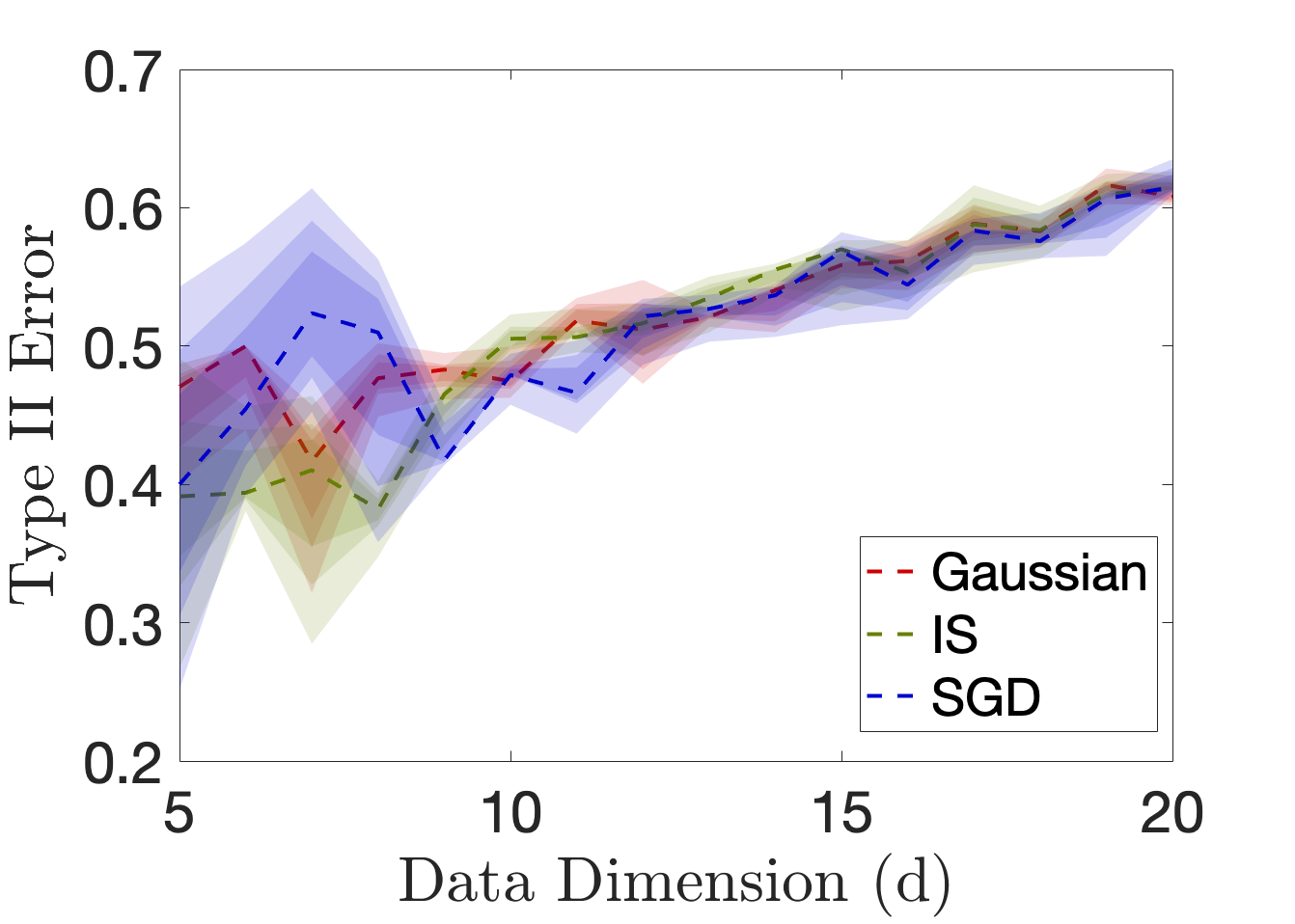}  \hspace{-2mm}
			\includegraphics[width=.5\linewidth]{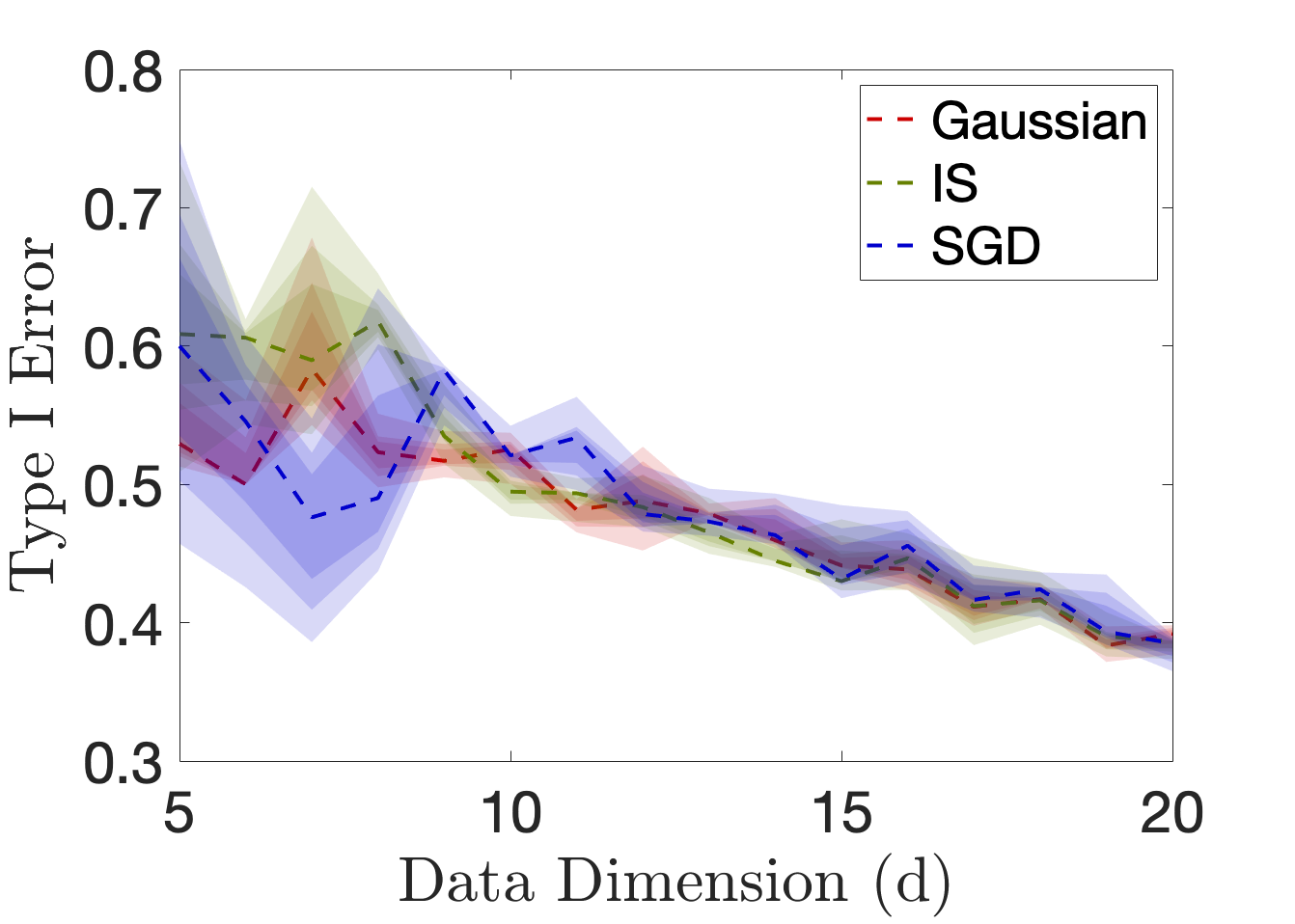}  }
		\\ 
		\subfigure{\footnotesize{\hspace{5mm} (c) \hspace{60mm} (d)}} 
		\vspace{-3mm}
		\caption{\footnotesize{Performance comparison of SGD optimization algorithm, importance sampling, and a Gaussian kernel with the \textit{optimized} bandwidth for $N=1000$ random feature samples, and different dimensions $d$.  The shaded regions denote $\%20$, $\%30$ and $\%50$ percentile errors. Panel (a): Training error, Panel (b): Test error, Panel (c): Type I error, Panel (d): Type 2 error  .}}
		\label{Fig:1} 
	\end{center}
\end{figure}

\begin{figure}[!t]
	\begin{center}
		\subfigure{
\hspace{-18mm}			\includegraphics[trim={.2cm .2cm .2cm  .6cm},width=.4\linewidth]{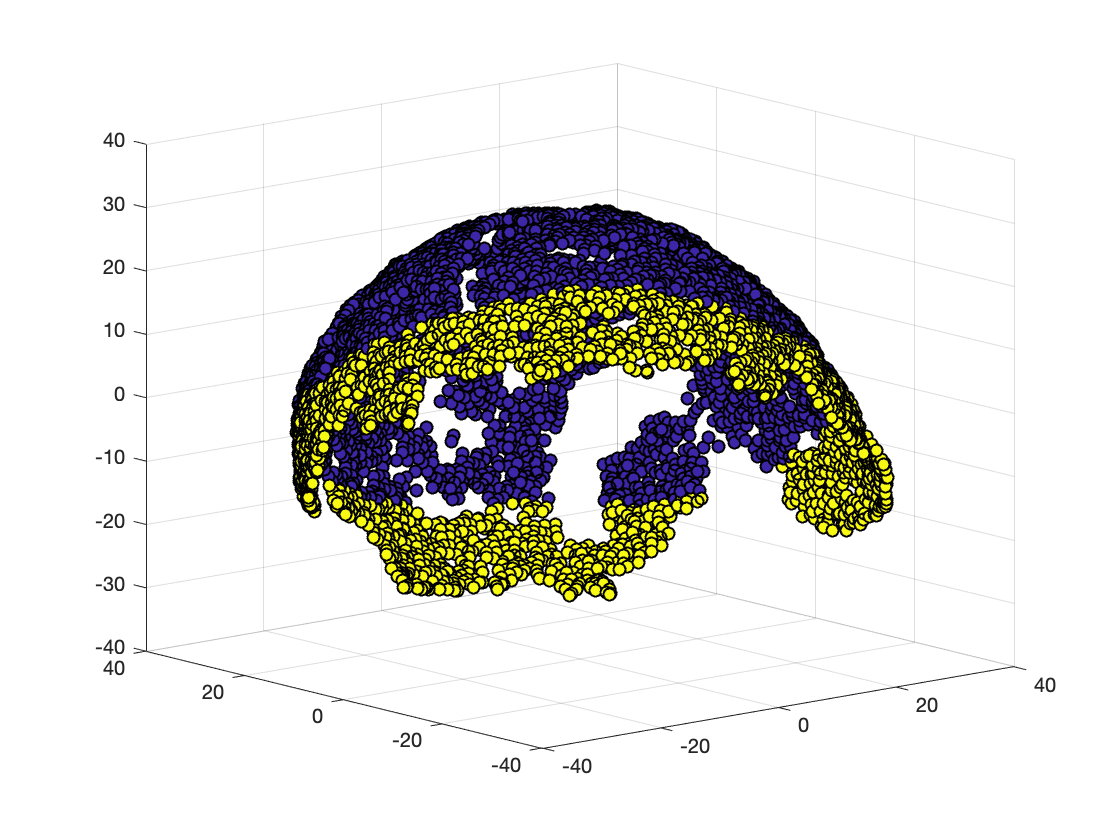} \hspace{-2mm}
			\includegraphics[trim={.2cm .2cm .2cm  .6cm},width=.4\linewidth]{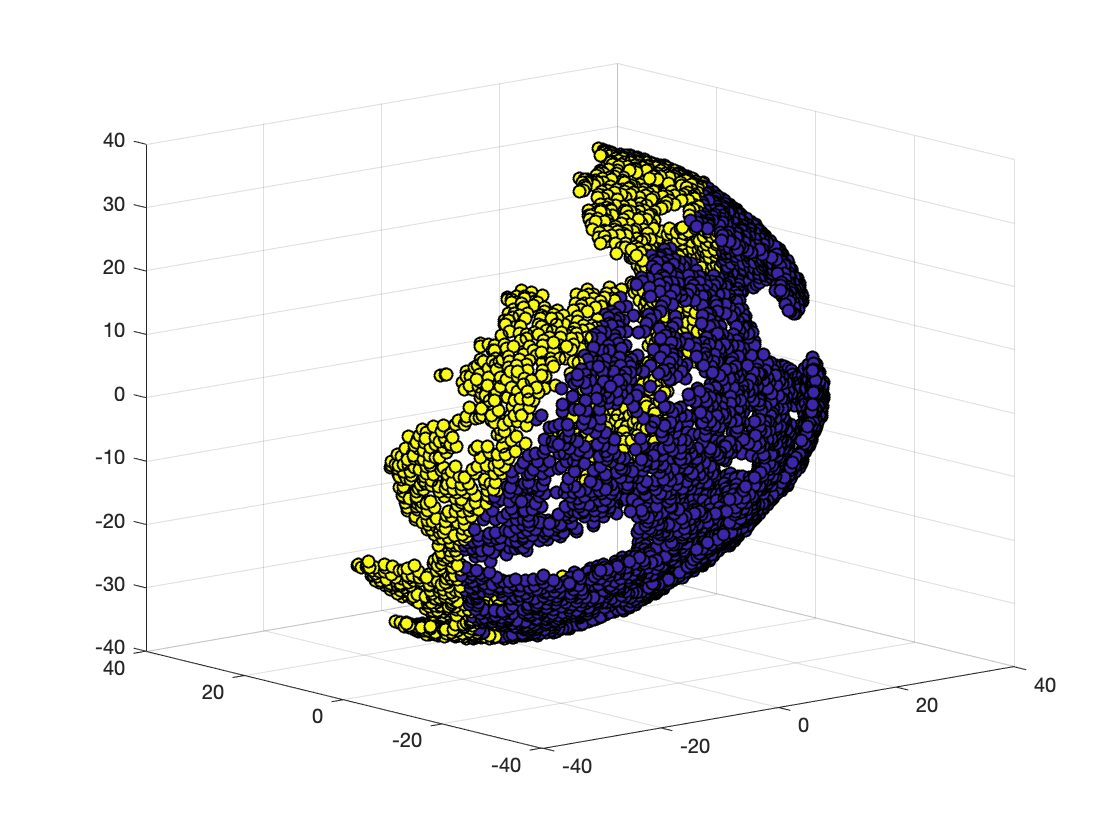} \hspace{-2mm}
			\includegraphics[trim={.2cm .2cm .2cm  .6cm},width=.4\linewidth]{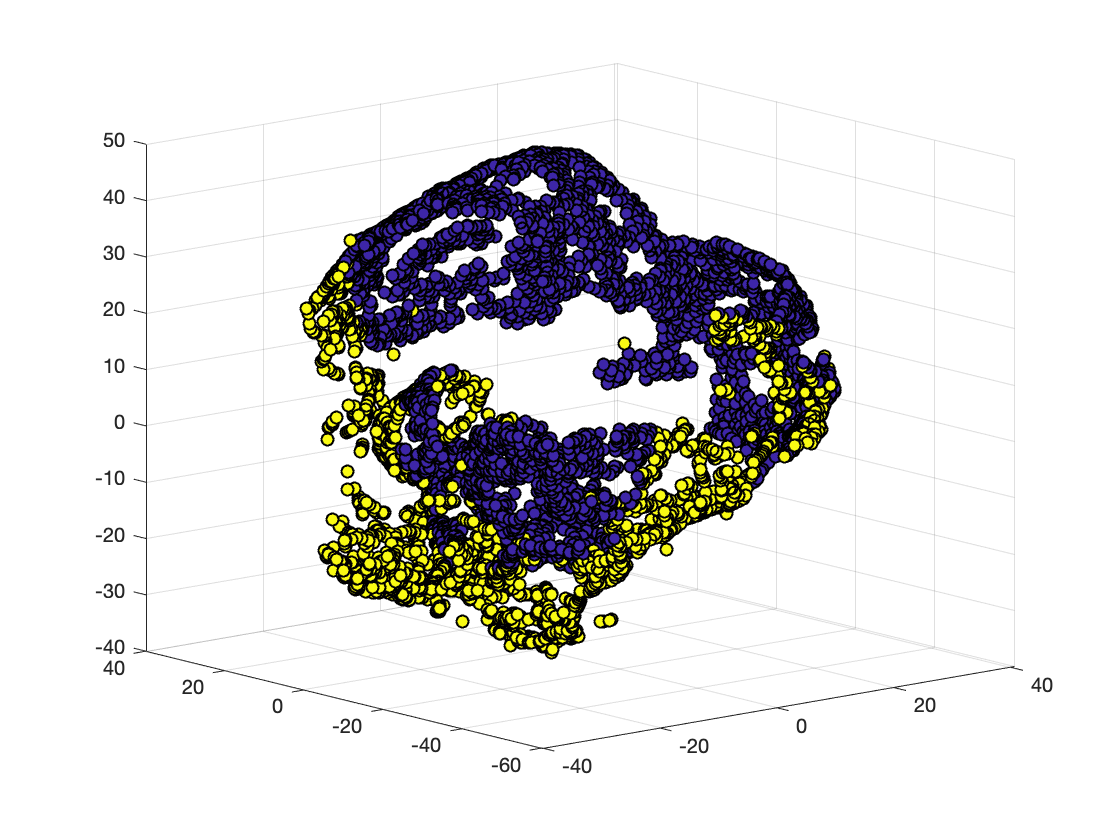} } \\ 
				\subfigure{
\hspace{-18mm}			\includegraphics[trim={.2cm .2cm .2cm  .6cm},width=.4\linewidth]{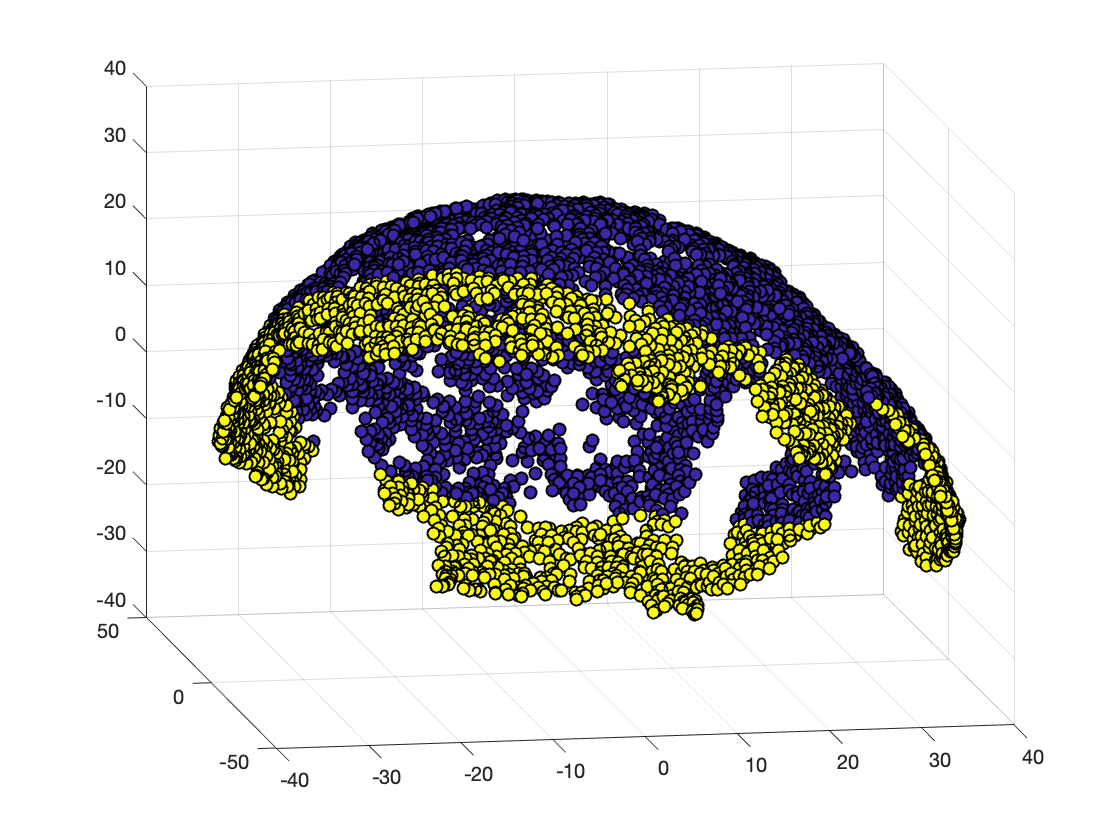} \hspace{-2mm}
			\includegraphics[trim={.2cm .2cm .2cm  .6cm},width=.4\linewidth]{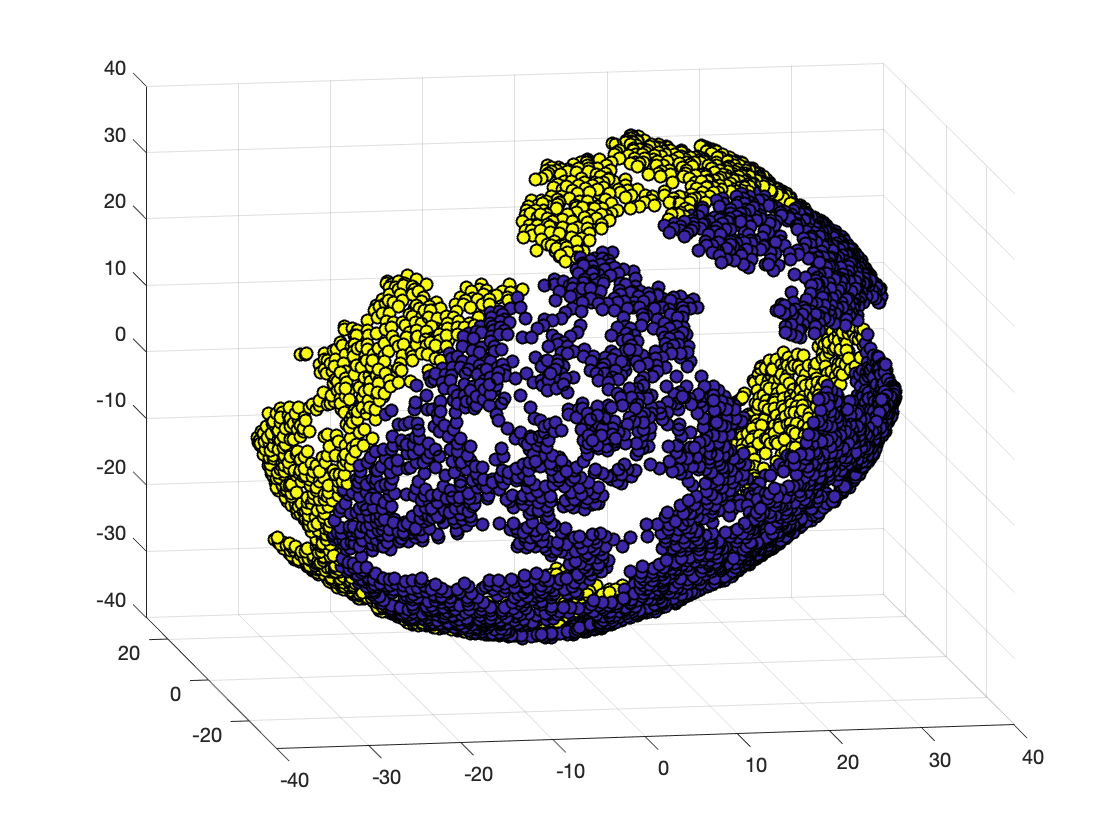} \hspace{-2mm}
			\includegraphics[trim={.2cm .2cm .2cm  .6cm},width=.4\linewidth]{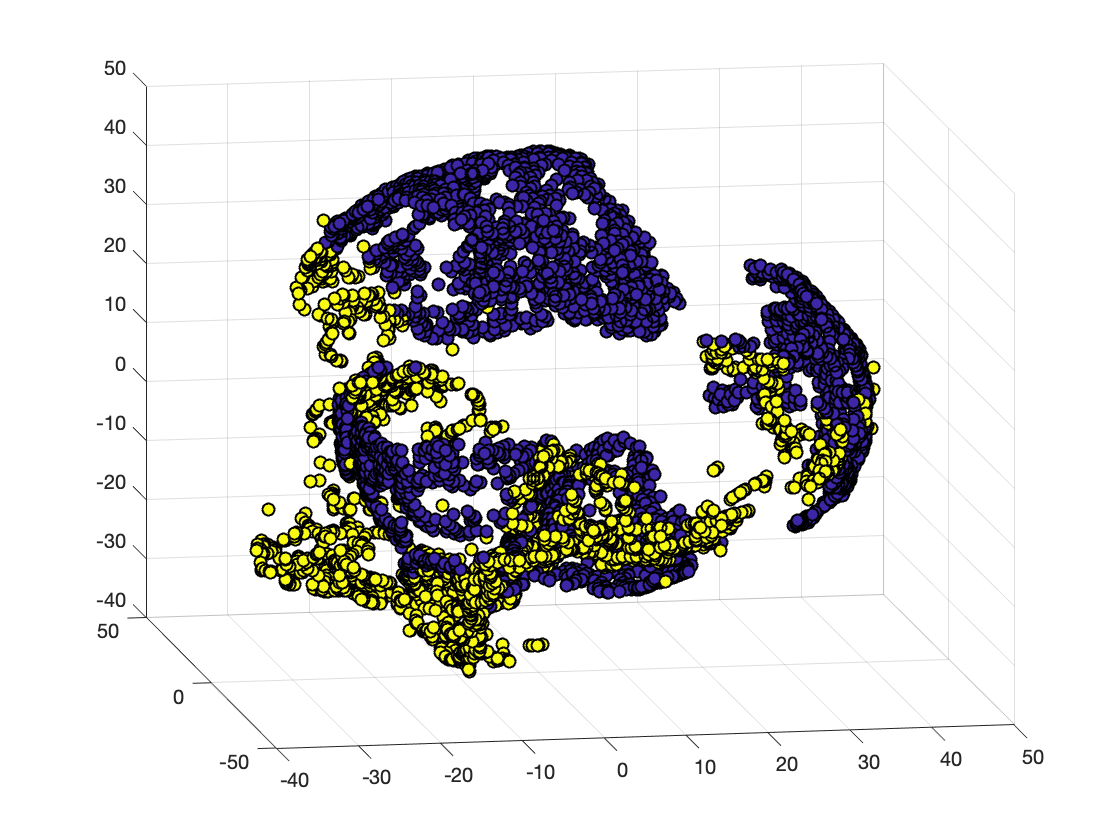} }
		\subfigure{\footnotesize{\hspace{5mm} (a) \footnotesize{\hspace{55mm} (b) \hspace{55mm} (c)}}} 
		\vspace{-3mm}
		\caption{\footnotesize{$t$-SNE plots of random features for synthetic data from different perspectives (each row corresponds to a different angle view). Panel (a): SGD algorithm, Panel (b): Importance Sampling, Panel (c): Gaussian kernel.}}
		\label{Fig:1} 
	\end{center}
\end{figure}

\begin{center}
	\centering
\begin{table}[t]
		\small{\hspace{-10mm}	\begin{tabular}{||c|| c| c| c| c||}
			\hline
			\rowcolor{Gray}
			SGD  &$d$=5 &$d$=10 & $d$=15  & $d$=20  \\ [0.5ex] 
			\hline\hline
			Training Error &0.00935$\pm$1.60e$-3$   &0.093$\pm$ 7.09e$-$3 &0.259$\pm$15.823e$-3$  & 0.346$\pm$7.561e$-3$     \\ 
			\hline
			Test Error &0.0154$\pm$2.59e$-3$  & 0.1411$\pm$17.7e$-3$ &0.347$\pm$23.187e$-3$   &0.463$\pm$10.264e$-3$ \\
			\hline
			Run Time (sec) &0.193$\pm$0.357e$-3$  &  $0.617\pm$150.33e$-3$  &3.228$\pm$0.547   &4.523$\pm$0.911   \\
			\hline
			\hline
			\rowcolor{Gray}
			Importance Sampling &$d=5$ & $d$=10 & $d$=15 &$d$=20  \\ [0.5ex] 
			\hline\hline
			Training Error &0.0163$\pm$1.55e$-3$ &0.135$\pm$9.54e$-3$ & 0.308$\pm$5.130e$-3$ & 0.363$\pm$9.667e$-3$  \\ 
			\hline
			Test Error &0.0186$\pm$5.18e$-3$ & 0.180$\pm$15.4e$-3$ & 0.407$\pm$10.031e$-3$ & 0.483$\pm$18.161e$-3$ \\
			\hline
			Run Time (sec)    &0.236$\pm$68.072e$-3$ &  0.154$\pm$28.7e$-3$  & 0.206$\pm$28.414e$-3$ & 0.302$\pm$111.43e$-3$ \\
			\hline
			\hline
			\rowcolor{Gray}
			 Gaussian Kernel &$d$=5 & $d$=10 & $d$=15 &$d$=20   \\ [0.5ex] 
			\hline\hline
			Training Error &0.0108$\pm$1.19e$-3$  &0.136$\pm$8.88e$-3$ &0.307$\pm$7.613e$-3$  &0.356$\pm$4.266e$-3$     \\ 
			\hline
			Test Error   &0.0175$\pm$5.01e$-3$ & 0.193$\pm$13.1e$-3$ & 0.403$\pm$14.032e$-3$ &0.485$\pm$17.47e$-3$  \\
			\hline
		\end{tabular}}\normalsize
		\caption{\footnotesize{Comparison between the SGD kernel learning method (Algorithm \ref{Algorithm:2}), the Importance Sampling \cite{sinha2016learning}, and a Gaussian kernel with the optimized bandwidth for data classification via the kernel SVMs for different data dimension $d$ and a fixed number of samples $N=1000$. The reported average numbers and standard deviations are evaluated over $10$ trials. The parameters of SGD for this simulations are (Iterations: $T=10000$ Regularization:$\alpha=1$, Step-size: $\eta=1$).}}
		\label{Table:3}
	\end{table}
\end{center}

\begin{center}
	\centering
	\begin{table}[t]
		\hspace{-11mm}	\small{\begin{tabular}{||c||  c| c| c|c||}
			\rowcolor{Gray}
				\hline
			SGD & $N$=100 & $N$=1000 & $N$=2000 &$N$=3000 \\ [0.5ex] 
			\hline\hline
			Training Error   &0.249$\pm$13.64e$-3$ &0.093$\pm$ 7.09e$-$3 &  0.0568$\pm$5.196e$-3$ &0.0373$\pm$4.196e$-3$  \\ 
			\hline
			Test Error   & 0.261$\pm$21.24e$-3$ & 0.1411$\pm$17.7e$-3$ & 0.111$\pm$12.59e$-3$ & 0.0997$\pm$10.719e$-3$ \\
			\hline
			Run Time (sec)   &  0.267$\pm$81.2e$-3$&  $0.617\pm$150.33e$-3$&3.943$\pm$0.839 &3.050$\pm$0.205\\
			\hline
			\hline
			\rowcolor{Gray}
			Importance Sampling & $N$=100 & $N$=1000 & $N$=2000 &$N$=3000\\ [0.5ex] 
			\hline\hline
			Training Error  & 0.332$\pm$20.31e$-3$ & 0.135$\pm$9.54e$-3$ & 0.1043$\pm$5.99e$-3$ &0.089$\pm$6.79e$-3$ \\ 
			\hline
			Test Error &  0.342$\pm$22.74e$-3$ &0.180$\pm$15.4e$-3$&  0.124$\pm$10.54e$-3$ & 0.1006$\pm$14.478e$-3$ \\
			\hline
			Run Time (sec)     &  0.0243$\pm$8.80e$-3$  &  0.154$\pm$28.7e$-3$  & 0.435$\pm$0.0422 &0.465$\pm$0.038\\
			\hline
			\hline
			\rowcolor{Gray}
			Gaussian Kernel & $N$=100 & $N$=1000 & $N$=2000 &$N$=3000\\ [0.5ex] 
			\hline\hline
			Training Error   &0.345$\pm$13.8e$-3$ &0.136$\pm$8.88e$-3$ & 0.074$\pm$7.30e$-3$ &0.0466$\pm$4.215e$-3$ \\ 
			\hline
			Test Error  & 0.355$\pm$34.30e$-3$ & 0.193$\pm$13.1e$-3$ &   0.145$\pm$20.41e$-3$ &0.111$\pm$7.33e$-3$ \\
			\hline
		\end{tabular}}\normalsize
		\caption{\footnotesize{Comparison between the SGD kernel learning method (Algorithm \ref{Algorithm:2}), the importance Sampling of \cite{sinha2016learning}, and a Gaussian kernel with the optimized bandwidth in conjunction with the kernel SVMs for different random feature samples $N$ and a fixed data dimension $d=10$. The reported average numbers and standard deviations are computed over $10$ trials}.}
		\label{Table:4}
	\end{table}
\end{center}

\subsection{Quantitiative Comparison}

In Table \ref{Table:3} and Table \ref{Table:4}, we present the training and test errors using kernel SVMs in conjunction with the trained kernel of Algorithm \ref{Algorithm:2}, the trained kernel via auto-encoder embedding only, the trained kernel via the importance sampling optimization of Eq. \eqref{Eq:Importance}, and the standard Gaussian kernel.\footnote{The simulation studies of Sinha and Duchi \cite{sinha2016learning} leverages a regression model to classify the synthetic data, whereas we use kernel SVMs. .}\footnote{We measured the elapsed time of our algorithm using {\tt{tic-toc}} command in MATLAB.}\footnote{For simulations of this section, we used a MacBook Pro 2017 with 2.3 GHz Intel Core i5 Processor and 8 GB 2133 MHz LPDDR3 Memory for the simulations.} 

In Table \ref{Table:3}, the number of random feature samples is fixed $N=1000$, while the data dimension $d$ changes from $d=5$ to $d=20$. We observe that already at dimension $d=10$, the SGD optimization method outperforms the importance sampling method, albeit with a larger computational cost. 

In Table \ref{Table:4}, we the test and training errors of the kernel learning methods for different number of random feature samples $N$. Clearly, already at $N=1000$ random feature samples, the training and test errors for SGD is comparable to that of importance sampling using $N=10000$ samples. We observe the improvement in the accuracy of both SGD and importance sampling as the number of random features increase.

\section{Implementation on Benchmark Data-Sets}

We now apply our kernel learning method for classification and regression tasks on real-world data-sets. 

\subsection{Data-Sets Description}

We apply our kernel learning approach to classification and regression tasks of real-world data-sets. In Table \ref{Table:Benchmark_Data_Descrpition}, we provide the characteristics of each data-set. All of these datasets are publicly available at UCI repository.\footnote{\url{https://archive.ics.uci.edu/ml/index.php}}
\subsubsection{Online news popularity}

This data-set summarizes a heterogeneous set of features about articles published by Mashable in a period of two years. The goal is to predict the number of shares in social networks (popularity).

\subsubsection{Buzz in social media dataset}

This data-set contains examples of buzz events from two different social networks: Twitter, and Tom's Hardware, a forum network focusing on new technology with more conservative dynamics

\subsubsection{Adult}

Adult data-set contains the census information of individuals including education, gender, and capital gain. The assigned classification task is to predict whether a person earns over 50K annually. The train and test sets are two separated files consisting of roughly 32000 and 16000 samples respectively.

\subsubsection{Epileptic Seizure Detection}

The epileptic seizure detection dataset consists of a recording of brain activity for 23.6 seconds. The corresponding time-series is sampled into 4097 data points. Each data point is the value of the EEG recording at a different point in time. So we have total 500 individuals with each has 4097 data points for 23.5 seconds. The 4097 data points are then divided and shuffled every into 23 segments, each segment contains 178 data points for 1 second, and each data point is the value of the EEG recording at a different point in time. 

\subsection{Quantitative Comparison}

In Figure \ref{Fig:Seizure}, we present the training and test results for regression and classification tasks on benchmark data-sets, using top $d=35$ features from each data-set and for different number of random feature samples $N$. In all the experiments, the SGD method provides a better accuracy in both the training and test phases. Nevertheless, in the case of seizure detection, we observe that for a small number of random feature samples, the importance sampling and Gaussian kernel with $k$NN for bandwidth outperforms SGD. 

In Figure \ref{Fig:Seizure_1}, we illustrate the time consumed for training the kernel using SGD and importance sampling methods versus the number of random features $N$. For the linear regression on \textsc{Buzz}, and \textsc{Online news popularity} the difference between the run-times are negligible. However, for classification tasks using \textsc{Adult} and \textsc{Seizure}, the difference in run-times are more pronounced.

Those visual observations are also repeated in a tabular form in Table \ref{Table:long_table}, where the training and test errors as well algorithmic efficiencies (run-times) are presented for $N=100$, $N=1000$, $N=2000$, and $N=3000$ number of features.

\begin{center}
	\centering
	\begin{table}[t]
		\begin{tabular}{||c|| c| c| c|c||}
			\hline
			\rowcolor{Gray}
			Data-set      & \text{Task}   & $d$  &  $n_{\text{training}}$   &$n_{\text{test}}$  \\
			\hline
			\hline
			Buzz& Regression &77  &93800 &46200 \\
			\hline
			Online news popularity& Regression & 58  &26561  &13083 \\ 
			\hline
			Adult & Classification  & 122 &32561  &16281  \\
			\hline
			Seizure &Classification  & 178  & 8625  &2875 \\
			\hline 
		\end{tabular}
		\caption{\footnotesize{Description of the benchmark data-sets used in this paper}.}
		\label{Table:Benchmark_Data_Descrpition}
	\end{table}
\end{center}

\begin{figure}[!h]
	\begin{center}
		\subfigure{	\hspace{-20mm}
			\includegraphics[trim={.2cm .2cm .2cm  .6cm},width=.3\linewidth]{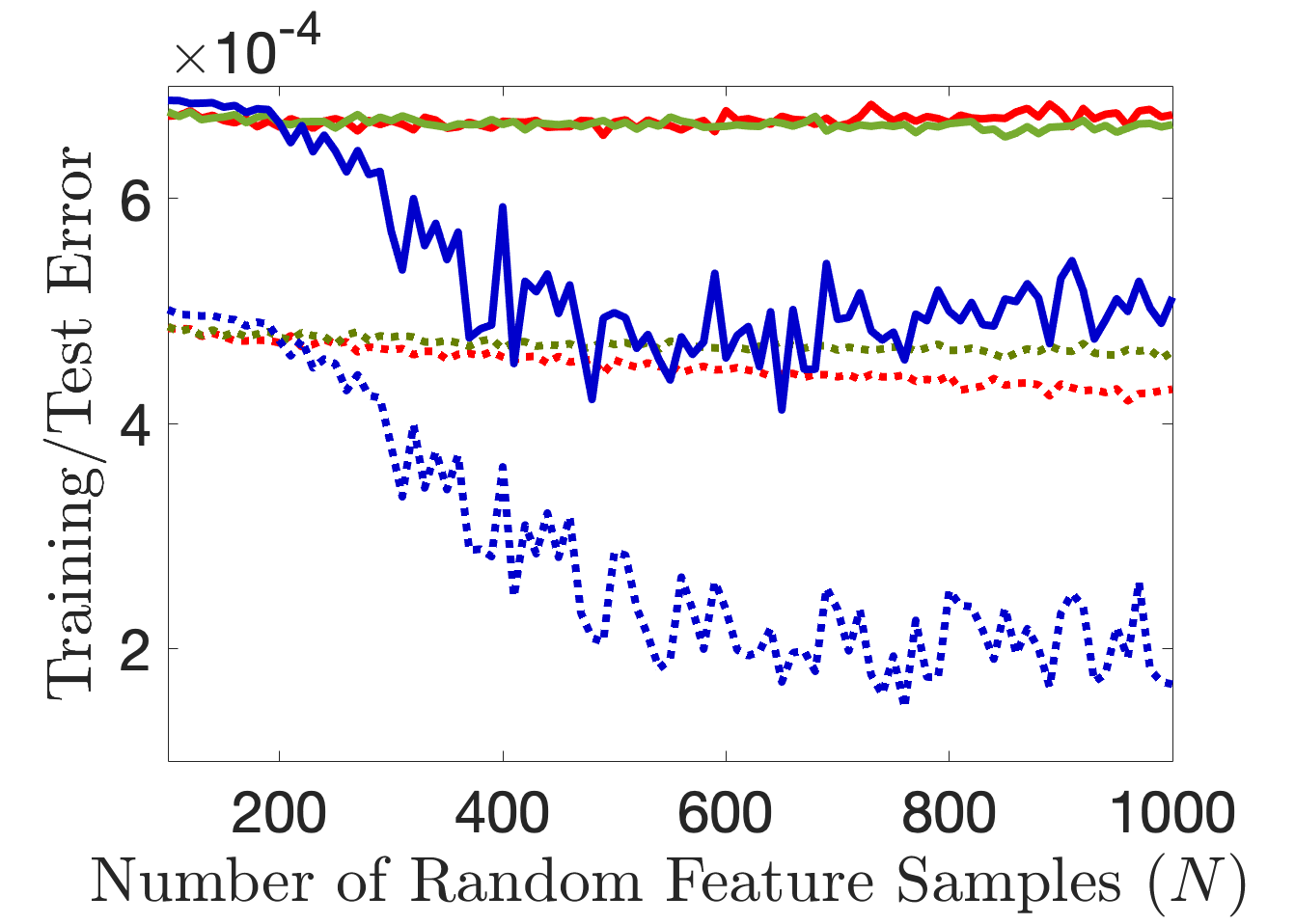} \hspace{-2mm}
			\includegraphics[trim={.2cm .2cm .2cm  .6cm},width=.3\linewidth]{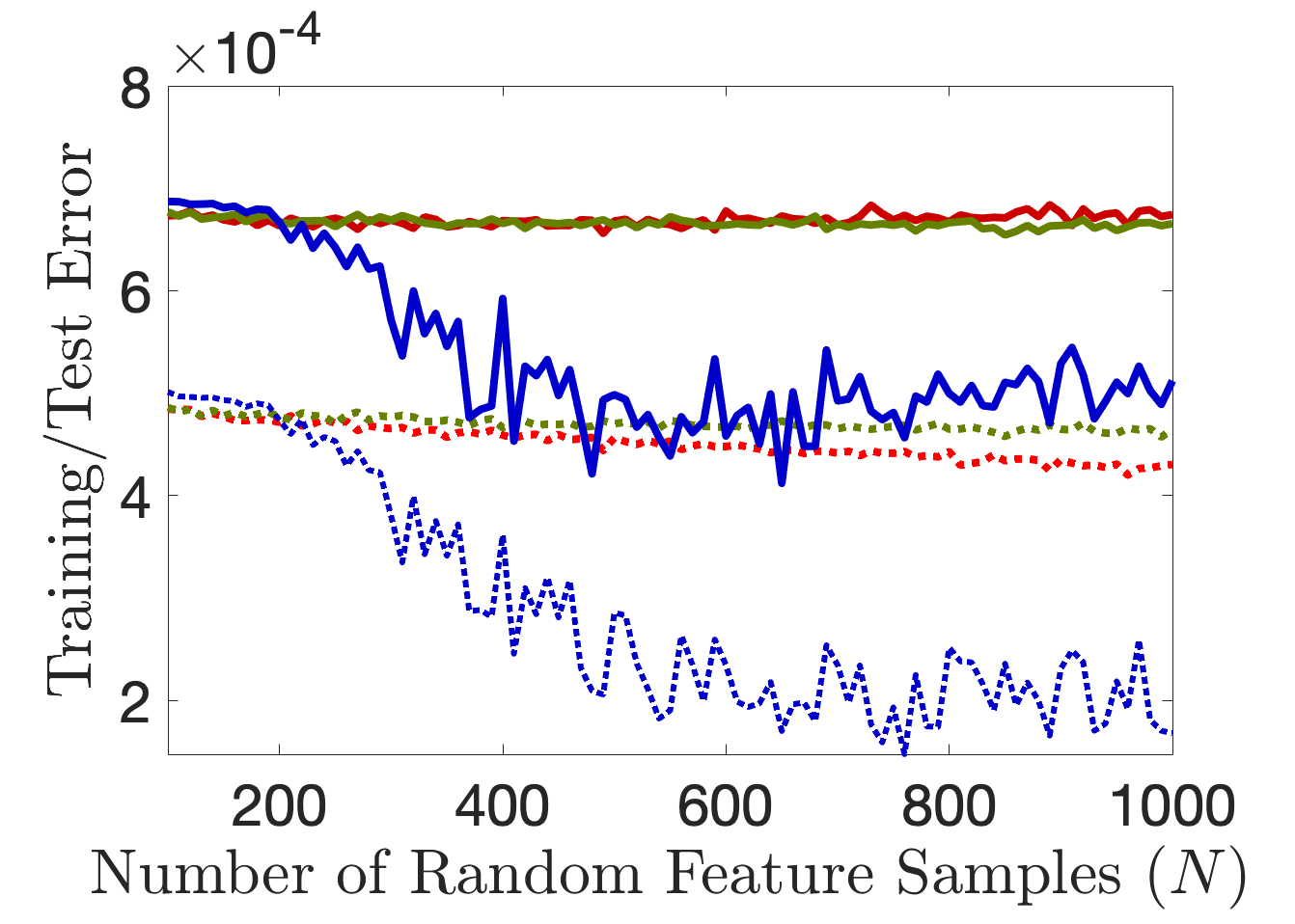} \hspace{-2mm}
			\includegraphics[trim={.2cm .2cm .2cm  .6cm},width=.3\linewidth]{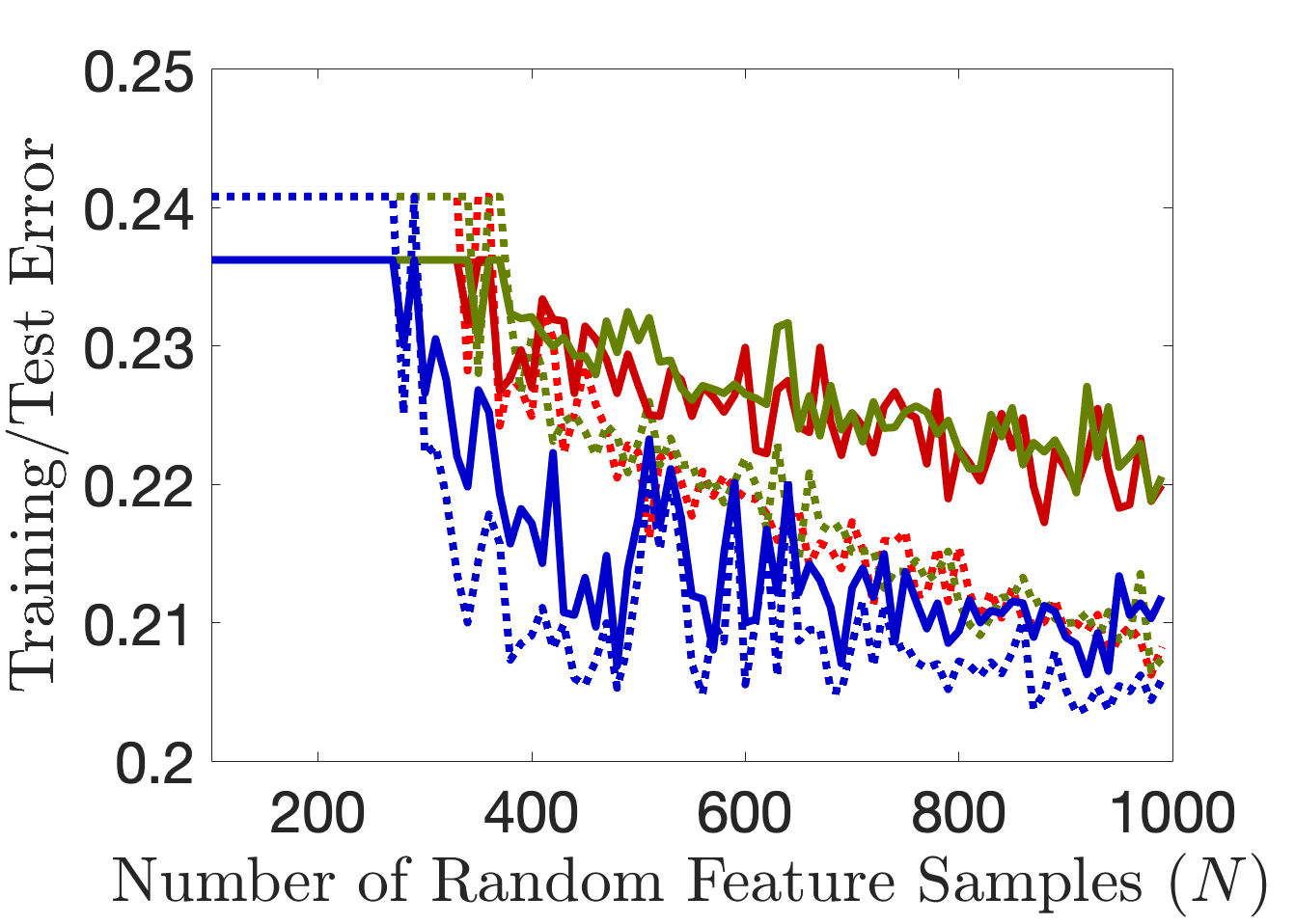}  \hspace{-2mm}
			\includegraphics[width=.3\linewidth]{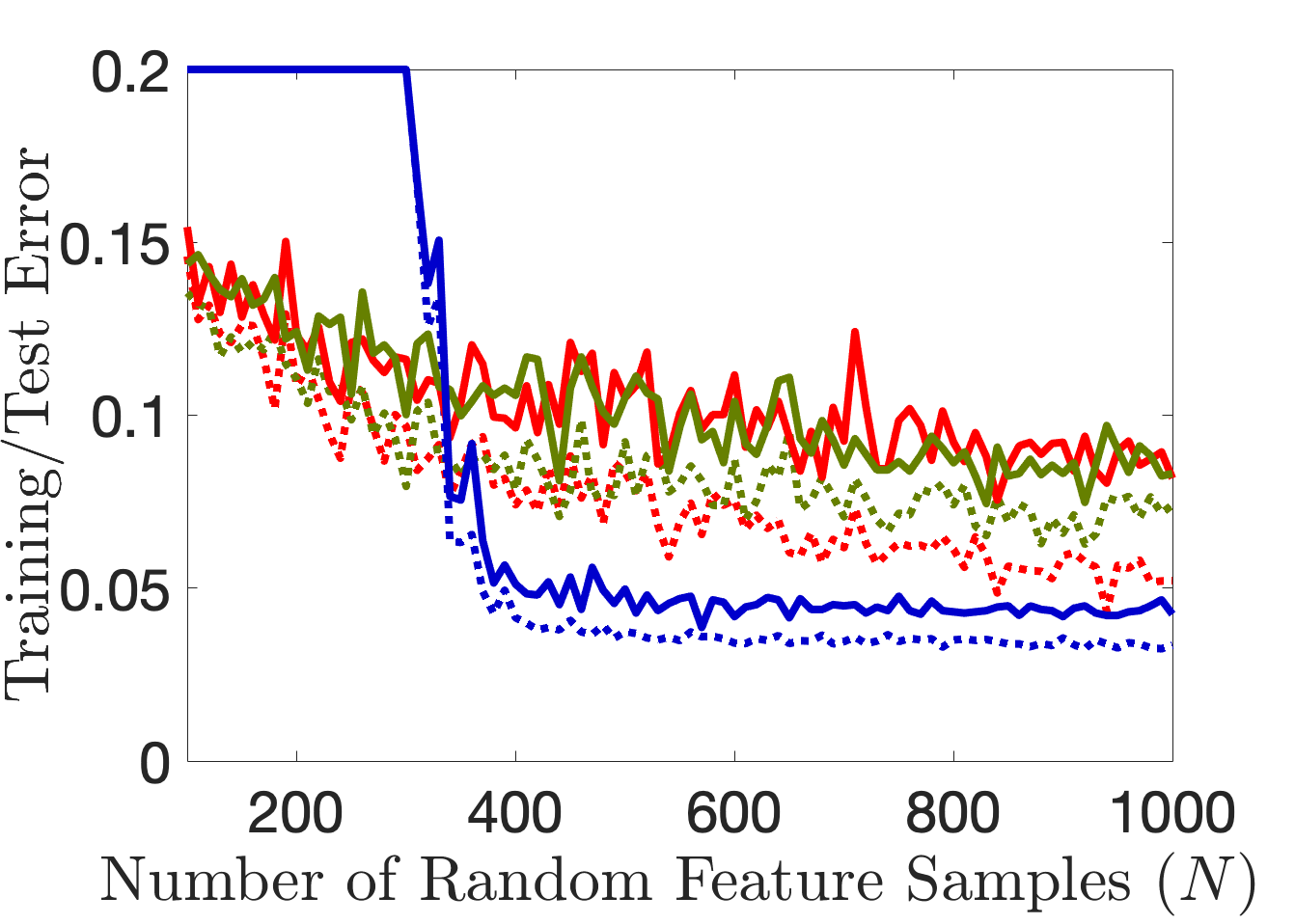}}
		\\ 
		\subfigure{\footnotesize{(a)\hspace{40mm} (b)  \hspace{40mm} (c) \hspace{40mm} (d)}} 
		\vspace{-3mm}
		\caption{\footnotesize{Performance comparison of SGD optimization algorithm (blue lines), importance sampling (green lines), and a Gaussian kernel with the \textit{optimized} bandwidth (red lines) for classification and linear regression tasks. The training and test errors are depicted with dashed and solid lines, respectively.  Panel (a): \textsc{Buzz}, Panel (b): \textsc{Online news popularity}, Panel (c): \textsc{Adult}, Panel (d): \textsc{Seizure}.}}
		\label{Fig:Seizure} 
	\end{center}
	\begin{center}
		\subfigure{
			\hspace{-20mm}	\includegraphics[trim={.2cm .2cm .2cm  .6cm},width=.3\linewidth]{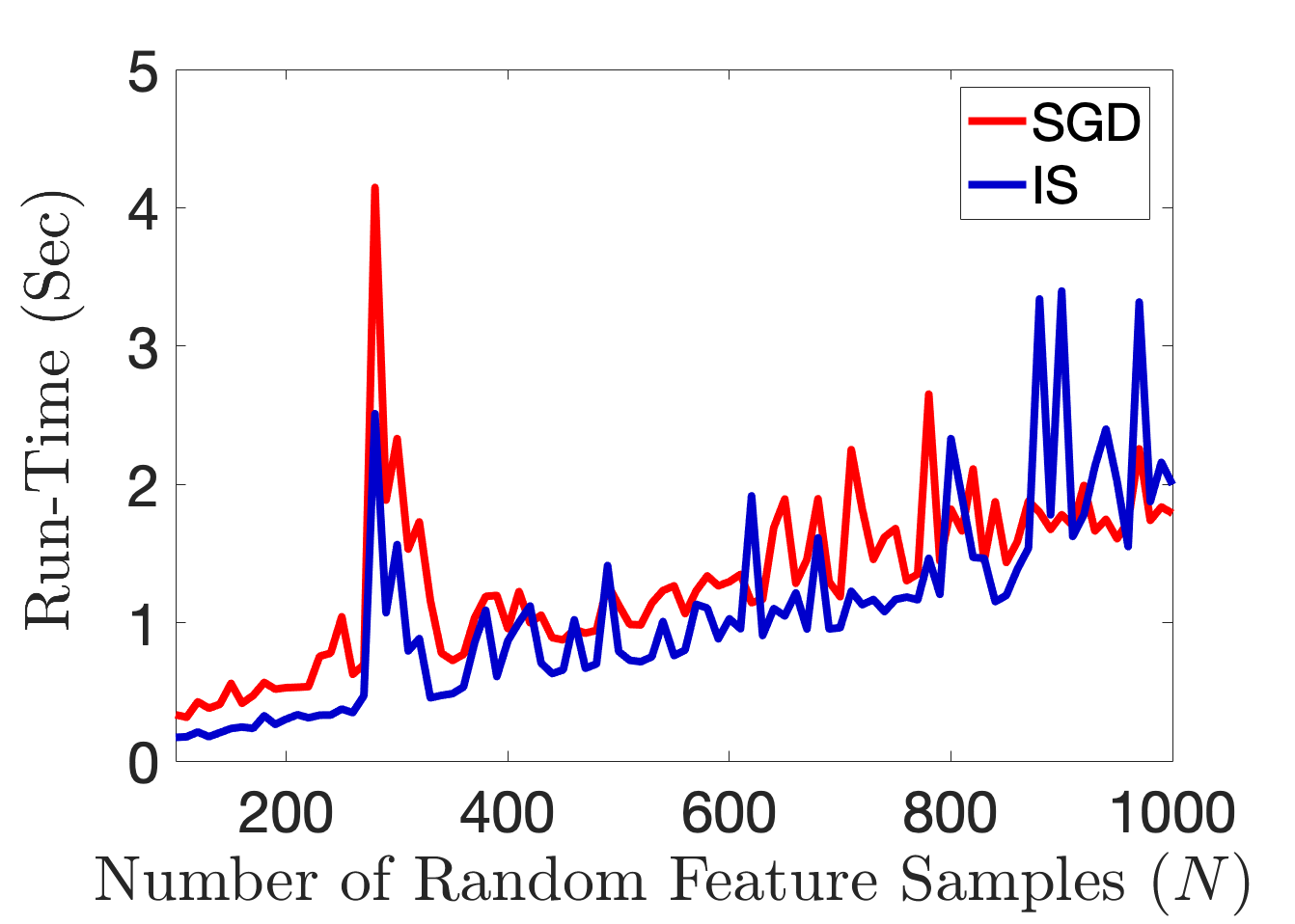} \hspace{-2mm}
			\includegraphics[trim={.2cm .2cm .2cm  .6cm},width=.3\linewidth]{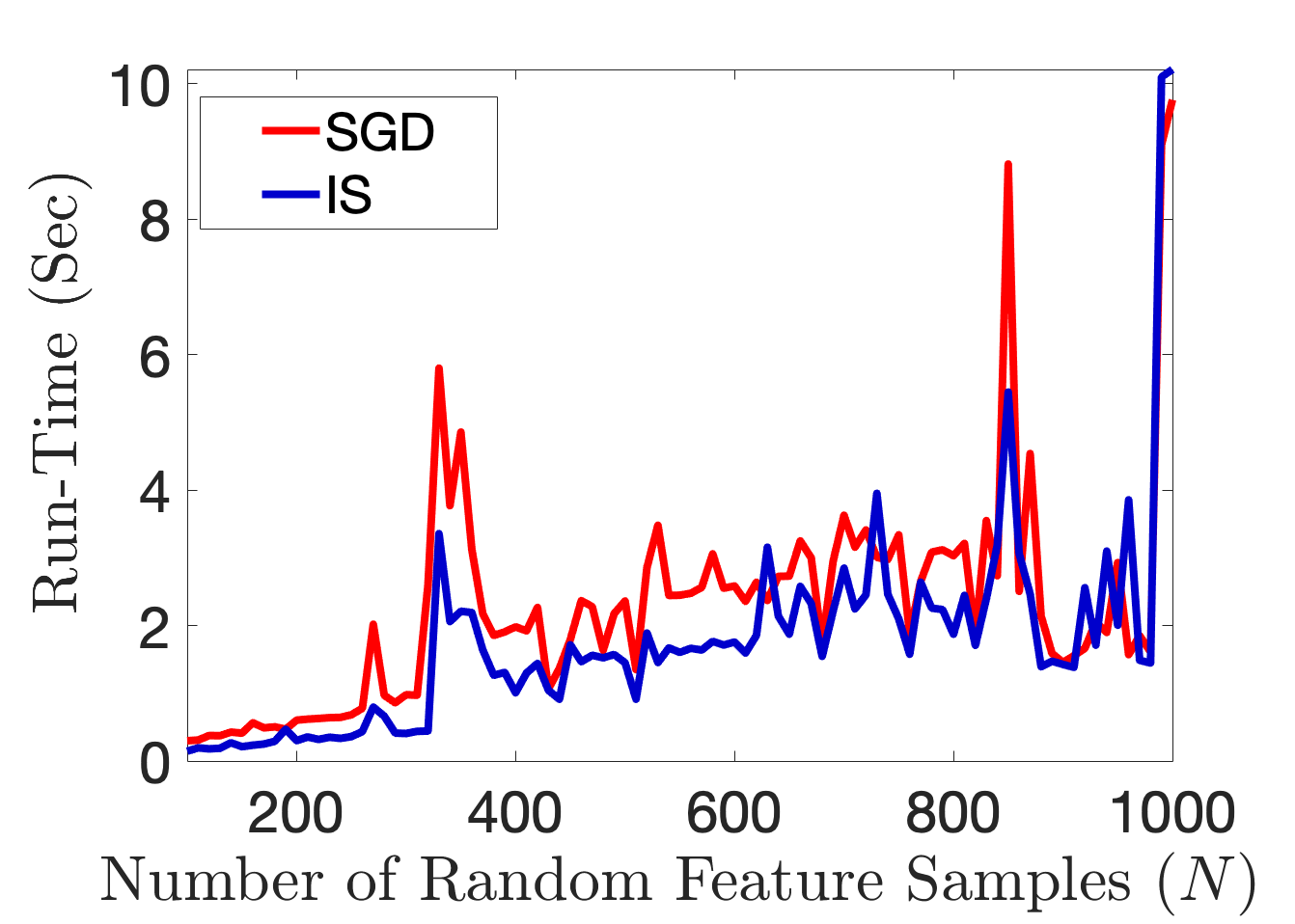} 
			\includegraphics[trim={.2cm .2cm .2cm  .6cm},width=.3\linewidth]{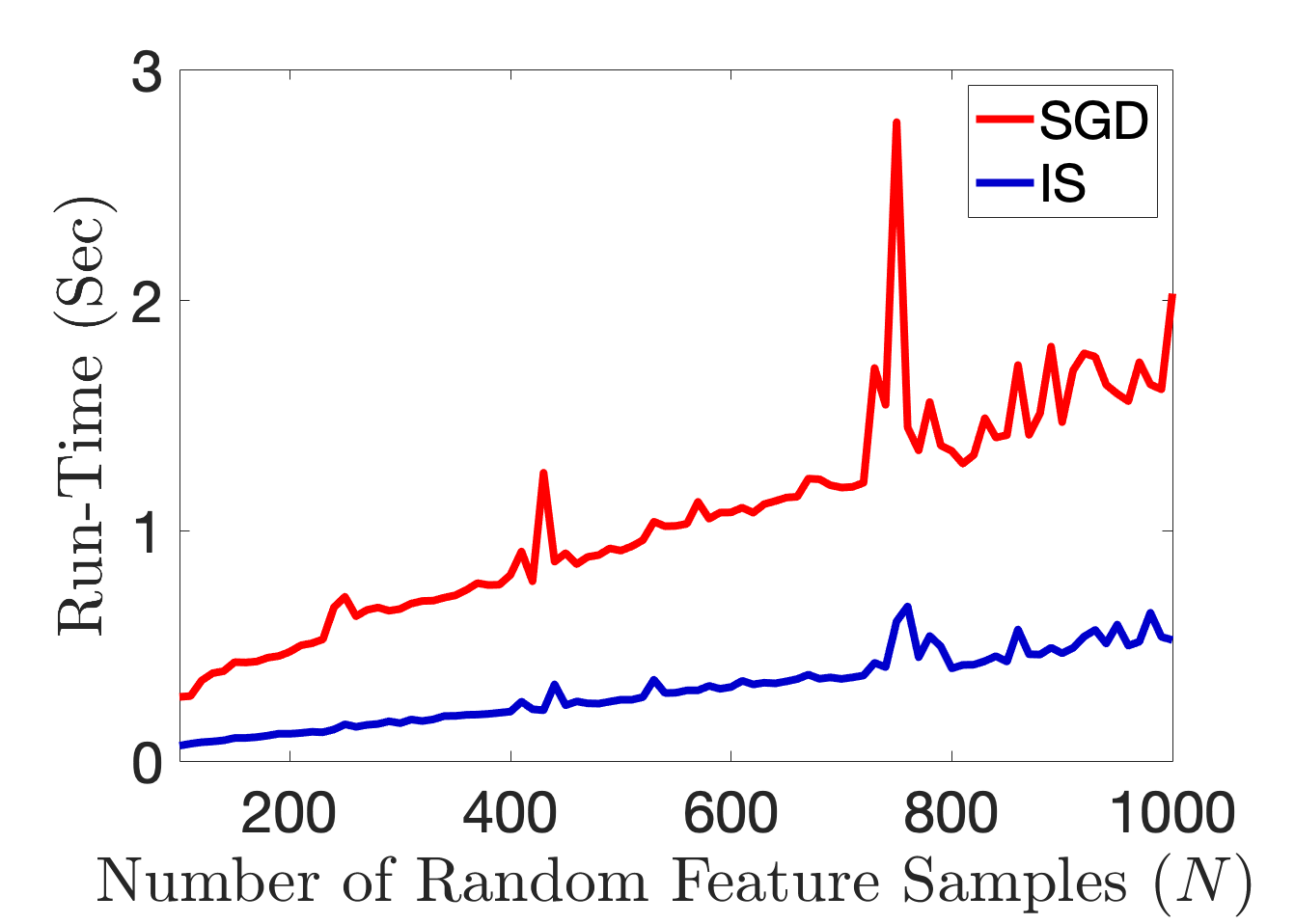}  \hspace{-2mm}
			\includegraphics[trim={.2cm .2cm .2cm  .6cm},width=.3\linewidth]{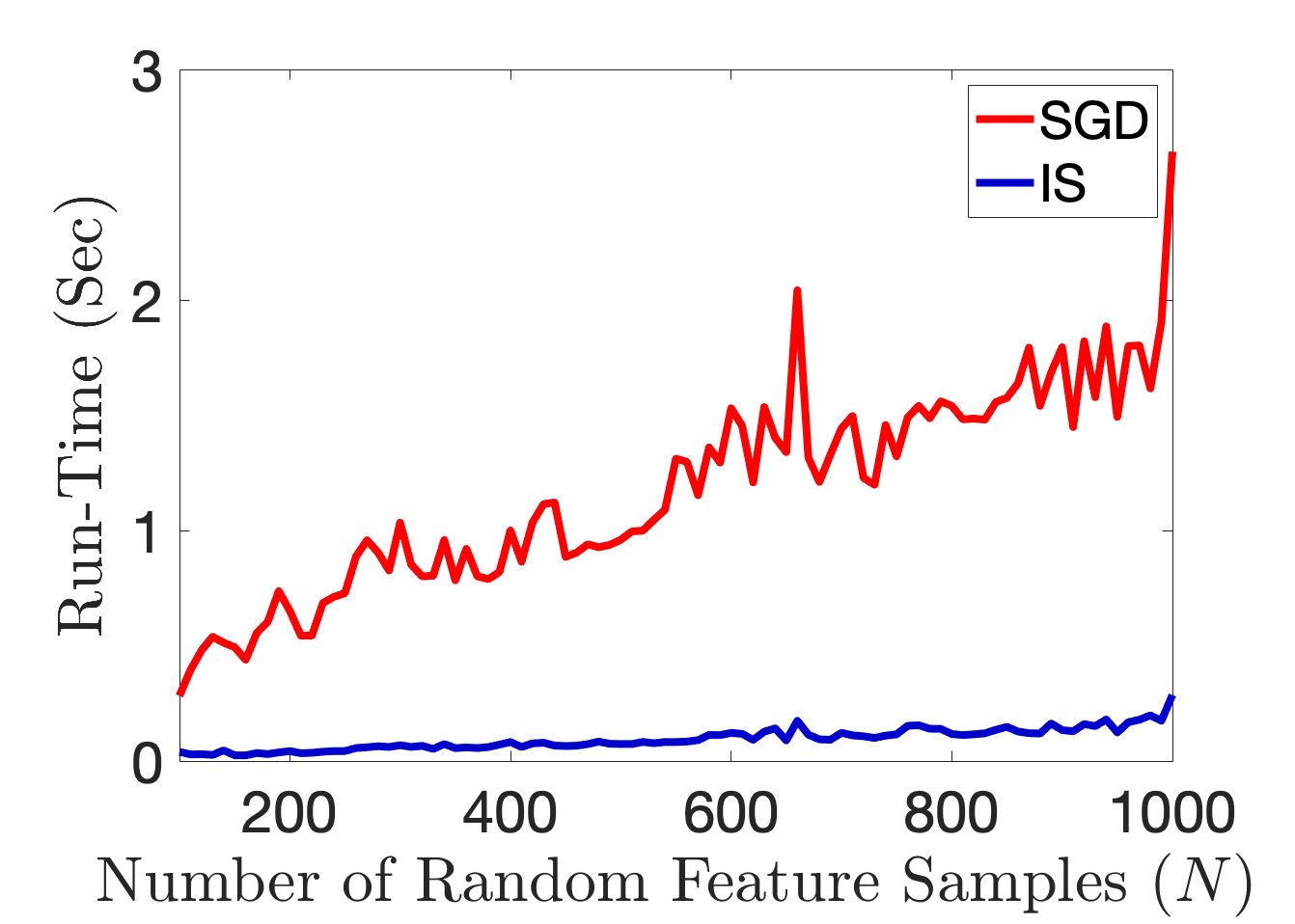}  }
		\\ 
		\subfigure{\footnotesize{(a)\hspace{40mm} (b)  \hspace{40mm} (c) \hspace{40mm} (d)}} 
		\vspace{-3mm}
		\caption{\footnotesize{The run-times of kernel learning algorithms using SGD and Importance Sampling (IS).  Panel (a): \textsc{Buzz}, Panel (b): \textsc{Online news popularity}, Panel (c): \textsc{Adult}, Panel (d): \textsc{Seizure}.}}
		\label{Fig:Seizure_1} 
	\end{center}
\end{figure}

\vspace{-14mm}
\section{Conclusions}
\label{Section:Conclusion}
We have presented a novel supervised method to optimize the kernel in generative and discriminative machine learning models. Our kernel learning approach is based on a distributionally robust optimization for the distribution of random features over the Wasserstein distributional ball. To obtain a tractable finite dimensional optimization problem, we used a sample average approximation (SAA) for random features, and subsequently applied a particle stochastic gradient descent (SGD). We provided theoretical guarantees for the consistency the finite sample (Monte-Carlo) approximations. In addition, using methods of mean-field theory, we proved that SGD is a consistent algorithm in the sense that the empirical distribution of the particles is an admissible solution for the distributional optimization problem. Our empirical evaluations on synthetic and benchmark data-sets shows the efficacy of our kernel learning method for generative and discriminative machine learning tasks.

\begin{center}
	\centering
	\begin{table}[!htbp]
	\footnotesize{ \begin{tabular}{||c|| c| c| c|c||}
				\hline	
				\multicolumn{5}{c}{{\textbf{Buzz}}}\\ \hline
				\rowcolor{Gray}
				\hline
				SGD & $N$=100 & $N$=1000 & $N$=2000 &$N$=3000 \\ [0.5ex] 
				\hline\hline
				Training Error   &0.501e$-3$   &0.190e$-3$   & 0.180e$-3$    &0.151e$-3$  \\ 
				\hline
				Test Error   & 0.501e$-3$    &0.190e$-3$    &0.180e$-3$    &0.151e$-3$ \\
				\hline
				Run Time (sec)   &  0.316&  1.971 &4.617 &14.621\\
				\hline
				\hline
				\rowcolor{Gray}
				Importance Sampling & $N$=100 & $N$=1000 & $N$=2000 &$N$=3000\\ [0.5ex] 
				\hline\hline
				Training Error  & 0.486e$-3$    &0.466e$-3$    &0.460e$-3$    &0.455e$-3$ \\ 
				\hline
				Test Error &  0.677$-3$   & 0.661$-3$    &0.662$-3$    &0.661e$-3$ \\
				\hline
				Run Time (sec)     &  0.188  &  1.527 & 3.592 &4.955\\
				\hline
				\hline
				\rowcolor{Gray}
				Gaussian Kernel & $N$=100 & $N$=1000 & $N$=2000 &$N$=3000\\ [0.5ex] 
				\hline\hline
				Training Error   &0.484e$-3$    &0.429e$-3$    &0.379e$-3$   & 0.327e$-3$ \\ 
				\hline
				Test Error  &  0.673e$-3$    &0.673$-3$    &0.709e$-3$    &0.744e$-3$
				\\				\hline
				\multicolumn{5}{c}{{\textbf{Online news popularity}}}\\ 
				\hline
				\rowcolor{Gray}	
				SGD & $N$=100 & $N$=1000 & $N$=2000 &$N$=3000 \\ [0.5ex] 
				\hline\hline
				Training Error   &0.501e$-3$    &0.190e$-3$   & 0.180e$-3$  &  0.151e$-3$  \\ 
				\hline
				Test Error   & 0.687e$-3$    &0.524e$-3$    &0.567e$-3$   & 0.588e$-3$ \\
				\hline
				Run Time (sec)   &  0.309    &2.071    &3.152    &5.054\\
				\hline
				\hline
				\rowcolor{Gray}
				Importance Sampling & $N$=100 & $N$=1000 & $N$=2000 &$N$=3000\\ [0.5ex] 
				\hline\hline
				Training Error  & 0.486e$-3$    &0.466e$-3$    &0.460e$-3$    &0.455e$-3$ \\ 
				\hline
				Test Error &  0.677e$-3$     &0.662e$-3$     &0.662e$-3$     &0.661e$-3$  \\
				\hline
				Run Time (sec)     & 0.222    &1.640    &4.254   &10.782\\
				\hline
				\hline
				\rowcolor{Gray}
				Gaussian Kernel & $N$=100 & $N$=1000 & $N$=2000 &$N$=3000\\ [0.5ex] 
				\hline\hline
				Training Error   &0.484e$-3$    &0.429e$-3$    &0.379e$-3$    &0.327e$-3$ \\ 
				\hline
				Test Error  & 0.673e$-3$    &0.673e$-3$   &0.709e$-3$    &0.744e$-3$	
				\\				\hline
				\multicolumn{5}{c}{{\textbf{Adult}}}\\ \hline
				\rowcolor{Gray}	
				SGD & $N$=100 & $N$=1000 & $N$=2000 &$N$=3000 \\ [0.5ex] 
				\hline\hline
				Training Error   &
				0.240    &0.204    &0.199    &0.197 \\ 
				\hline
				Test Error   & 0.236    &0.209   &0.208   &0.210
				\\
				\hline
				Run Time (sec)   &  0.274    &1.784    &9.272   &11.295\\
				\hline
				\hline
				\rowcolor{Gray}
				Importance Sampling & $N$=100 & $N$=1000 & $N$=2000 &$N$=3000\\ [0.5ex] 
				\hline\hline
				Training Error  &0.240   & 0.207    &0.214    &0.220
				\\ 
				\hline
				Test Error &  0.236    &0.219    &0.222   &0.223 \\
				\hline
				Run Time (sec)     &  0.058    &0.511    &2.661    &8.068\\
				\hline
				\hline
				\rowcolor{Gray}
				Gaussian Kernel & $N$=100 & $N$=1000 & $N$=2000 &$N$=3000\\ [0.5ex] 
				\hline\hline
				Training Error   & 0.240    &0.206    &0.201    &0.197 \\ 
				\hline
				Test Error  &  0.236    &0.222    &0.216    &0.214
				\\				\hline
				\multicolumn{5}{c}{{\textbf{Seizure}}}\\ \hline
				\rowcolor{Gray}	
				SGD & $N$=100 & $N$=1000 & $N$=2000 &$N$=3000 \\ [0.5ex] 
				\hline\hline
				Training Error   &0.200    &0.034   & 0.031    &0.032  \\ 
				\hline
				Test Error   & 0.200    &0.043    &0.043    &0.043 \\
				\hline
				Run Time (sec)   &  0.350    &1.747    &3.833    &5.222\\
				\hline
				\hline
				\rowcolor{Gray}
				Importance Sampling & $N$=100 & $N$=1000 & $N$=2000 &$N$=3000\\ [0.5ex] 
				\hline\hline
				Training Error  & 0.135    &0.078    &0.055    &0.051
				\\ 
				\hline
				Test Error &  0.144    &0.099    &0.063    &0.056 \\
				\hline
				Run Time (sec)     &  0.038    &0.138    &0.352    &0.404\\
				\hline
				\hline
				\rowcolor{Gray}
				Gaussian Kernel & $N$=100 & $N$=1000 & $N$=2000 &$N$=3000\\ [0.5ex] 
				\hline\hline
				Training Error   & 0.146    &0.056    &0.033    &0.030 \\ 
				\hline
				Test Error  & 0.154    &0.093   & 0.076    &0.069 \\
				\hline
		\end{tabular}}\normalsize
		\caption{\footnotesize{Comparison between the SGD kernel learning method (Algorithm \ref{Algorithm:2}), the importance Sampling of \cite{sinha2016learning}, and a regular Gaussian kernel in conjunction with the kernel SVMs for different random feature samples $N$ on benchmark data-sets.}}
		\label{Table:long_table}
	\end{table}
	\vspace{-2mm}
\end{center}

\section{Acknowledgement}

\textsc{M.B.K.} would like to thank his colleague \textsc{Dr. Chuang Wang} from Harvard University for his intriguing theoretical work on the scaling limits of the online algorithms  \cite{wang2017scaling} which motivated this work. \textsc{M.B.K.} would also like to thank \textsc{Prof. Ravi Mazumdar} from The University of Waterloo for his theoretical paper \cite{vasantam2018occupancy} which inspired the proof of tightness of Skhorhod spaces in this paper. We also thank \textsc{Dr. Maxime Bassenne} for useful discussion regarding the MacKean-Vlasov PDE.

The research of \textsc{MBK}, \textsc{LS}, and \textsc{LX} is supported by NIH under the grant \textsc{1R01 CA176553} and \textsc{R01E0116777}. The content of this article are solely the responsibility of the authors and do not necessarily reflect the official NIH views.

\appendix

\section*{Appendix}

\section{Proofs of  Main Theoretical Results}
\label{Section:Proofs_of_Main_Theoretical_Results}

\textbf{Notation}: We denote vectors by lower case bold letters, \textit{e.g.} $\bm{x}=(x_{1},\cdots,x_{n})\in \real^{n}$, and matrices by the upper case bold letters, \textit{e.g.}, $\bm{M}=[M_{ij}]\in \real^{n\times m}$. The Frobenius norm of a matrix is denoted by $\|\bm{M}\|_{F}=\sum_{i=1}^{n}\sum_{j=1}^{m}|M_{ij}|^{2}$. Let $\ball_{r}(\bm{x})\df \{\bm{y}\in \real^{d}: \|\bm{y}-\bm{x}\|_{2}\leq r \}$ denote the Euclidean ball of radius $r$ centered at $\bm{x}$. For a given metric space $\mathcal{X}$, Let $C_{b}(\real^{d})$ denote the space of bounded and continuous functions on $\mathcal{X}$ equipped with the usual supremum norm 
\begin{align}
\|f\|_{\infty}\df \sup_{\bm{x}\in \mathcal{X}}|f(\bm{x})|.
\end{align}
Further, $C_{b}^{k}(\mathcal{X})$ the space of all functions in $C_{b}(\mathcal{X})$ whose partial derivatives up to order $k$ are bounded and continuous, and $C_{c}^{k}(\mathcal{X})$ the space of functions whose partial derivatives up to order $k$ are continuous with compact support. 

We denote the class of the integrable functions $f$ with $f(t) \geq 0$ \textit{a.e.},  on $0\leq t\leq T$ by $L_{+}^{1}[0, T]$. Similarly, $L_{+}^{\infty}[0, T]$ will denote the essentially bounded functions with $f(t)\geq 0$ almost everywhere. For a given metric space $\mathcal{X}$, we denote the Borel $\sigma$-algebra by $\mathcal{B}(\mathcal{X})$. For a Borel set $B\in \mathcal{B}(\mathcal{X})$, the measure value of the set $B$ with respect to the measure is given by $\mu(B)$.
The space of finite non-negative measures defined on $\mathcal{X}$ is denoted by $\mathcal{M}(\mathcal{X})$. The Dirac measure with the unit mass at $x\in \mathcal{X}$ is denoted by $\delta(x)$. For any measure $\mu\in \mathcal{M}(\mathcal{X})$ and any bounded function $f\in C_{b}(\mathcal{X})$, we define 
\begin{align}
\langle \mu,f \rangle\df \int_{\mathcal{X}}f(x)\mu(\mathrm{d}x).
\end{align}
The space $\mathcal{M}(\mathcal{X})$ is equipped with the weak topology, \textit{i.e.}, a (random) sequence $\{\mu^{N}\}_{N\in \integer}$ converges weakly to a deterministic measure $\mu\in \mathcal{M}(\mathcal{X})$ if and only if $\langle \mu^{N},f \rangle\rightarrow \langle \mu,f \rangle$ for all $f\in C^{b}(\mathcal{X})$. We denote the weak convergence by $\mu^{N}_{t}\stackrel{\text{weakly}}{\rightarrow} \mu$.
Notice that when $\mathcal{X}$ is Polish, then $\mathcal{M}(\mathcal{X})$ equipped with the weak topology is also Polish.\footnote{A topological space is Polish if it is homeomorphic to a complete, separable metric space.} For a Polish space $\mathcal{X}$, let $\mathcal{D}_{\mathcal{X}}([0,T])$ denotes the Skorokhod space of the c\'{a}dl\'{a}g functions that take values in $\mathcal{X}$ defined on $[0,T]$. We assume that $\mathcal{D}_{\mathcal{X}}([0,T])$ is equipped with the Skorokhod's $J_{1}$-topology \cite{billingsley2013convergence}, which in that case $\mathcal{D}_{\mathcal{X}}([0,T])$ is also a Polish space.

To establish the concentration results in this paper, we require the following definitions:
\begin{definition}\textsc{(Orlicz Norm)}
	\label{Def:Orlicz}
	The Young-Orlicz modulus is a convex non-decreasing function $\psi:\real_{+}\rightarrow \real_{+}$ such that $\psi(0)=0$ and $\psi(x)\rightarrow \infty$ when $x\rightarrow \infty$. Accordingly, the Orlicz norm of an integrable random variable $X$ with respect to the modulus $\psi$ is defined as
	\begin{align}
	\|X\|_{\psi}\df \inf \{\beta>0:\expect[\psi(||X|-\expect[|X|]|/\beta)]\leq 1\}.
	\end{align}
\end{definition}

In the sequel, we consider the Orlicz modulus $\psi_{\nu}(x)\df \exp(x^{\nu})-1$ . Accordingly, the cases of $\|\cdot\|_{\psi_{2}}$ and $\|\cdot\|_{\psi_{1}}$ norms are called the sub-Gaussian and the sub-exponential norms and have the following alternative definitions:
\begin{definition}\textsc{(Sub-Gaussian Norm)}
	\label{Definition:Sub-Gaussian Norm}
	The sub-Gaussian norm of a random variable $Z$, denoted by $\|Z\|_{\psi_{2}}$, is defined as
	\begin{align}
	\|Z\|_{\psi_{2}}= \sup_{q\geq 1} q^{-1/2}(\expect|Z|^{q})^{1/q}.
	\end{align}
	For a random vector $\bm{Z}\in \real^{n}$, its sub-Gaussian norm is defined as  follows
	\begin{align}
	\label{Eq:Sub_Gaussian_random_vector}
	\|\bm{Z}\|_{\psi_{2}}=\sup_{\bm{x}\in \mathrm{S}^{n-1}}\|\langle  \bm{x},\bm{Z}\rangle \|_{\psi_{2}}.
	\end{align}
\end{definition}

\begin{definition}\ \textsc{(Sub-exponential Norm)}
	The sub-exponential norm of a random variable $Z$, denoted by $\|Z\|_{\psi_{1}}$, is defined as follows
	\begin{align}
	\|Z\|_{\psi_{1}}=\sup_{q\geq 1} q^{-1}(\expect[|Z|^{q}])^{1/q}.
	\end{align}
	For a random vector $\bm{Z}\in \real^{n}$, its sub-exponential norm is defined below
	\begin{align}
	\|\bm{Z}\|_{\psi_{1}}= \sup_{\bm{x}\in \mathrm{S}^{n-1}} \|\langle \bm{Z},\bm{x}\rangle\|_{\psi_{1}}.
	\end{align}
\end{definition}

%

We use asymptotic notations throughout the paper. We use the standard asymptotic notation for sequences. If $a_{n}$ and $b_{n}$ are positive
sequences, then $a_{n}=\mathcal{O}(b_{n})$ means that $\lim \sup_{n\rightarrow \infty} a_{n}/b_{n}< \infty$, whereas  $a_{n} = \Omega(b_{n})$ means that
$\lim \inf_{n\rightarrow \infty} a_{n}/b_{n} > 0$. Furthermore, $a_{n}=\widetilde{\mathcal{O}}(b_{n})$ implies $a_{n}=\mathcal{O}(b_{n}\text{poly}\log(b_{n}))$. Moreover $a_{n}=o(b_{n})$ means that $\lim_{n\rightarrow \infty}a_{n}/b_{n}=0$ and $a_{n}=\omega(b_{n})$ means that $\lim_{n\rightarrow \infty} a_{n}/b_{n}=\infty$. Lastly, we have $a_{n}=\Theta(b_{n})$ if $a_{n}=\mathcal{O}(b_{n})$ and $a_{n}=\Omega(b_{n})$.

\subsection{Proof of Theorem \ref{Thm:Consistency of Monte Carlo Estimation}}
\label{App:Proof of Theorem_1}

By the triangle inequality, we have that
\begin{align}
\label{Eq:A1-A4}
&\Big|\mathrm{MMD}_{\mu_{\ast}}[P_{\bm{V}},P_{\bm{W}_{\ast}}]-\mathrm{MMD}_{\widehat{\mu}_{\ast}^{N}}\Big[P_{\bm{V}},P_{\widehat{\bm{W}}^{N}_{\ast}}\Big]\Big|\leq \mathsf{A}_{1}+\mathsf{A}_{2}+\mathsf{A}_{3}+\mathsf{A}_{4},
\end{align}
where the terms $\mathsf{A}_{i},i=1,2,3,4$ are defined as follows
\begin{align*}
\mathsf{A}_{1}&\df \Big|\mathrm{MMD}_{\mu_{\ast}}[P_{\bm{V}},P_{{\bm{W}}_{\ast}}]-\min_{\bm{W}\in \mathcal{W}}\sup_{\mu\in \mathcal{P}} \widehat{\mathrm{MMD}}_{\mu}[P_{\bm{V}},P_{\bm{W}}] \Big|  \\
\mathsf{A}_{2}&\df \Big|\min_{\bm{W}\in \mathcal{W}}\sup_{\mu\in \mathcal{P}}\widehat{\mathrm{MMD}}_{\mu}[P_{\bm{V}},P_{\bm{W}}]-\min_{\bm{W}\in \mathcal{W}}\sup_{\widehat{\mu}^{N}\in \mathcal{P}_{N}}\widehat{\mathrm{MMD}}_{\widehat{\mu}^{N}}\Big[P_{\bm{V}},P_{\bm{W}}\Big] \Big|  \\
\mathsf{A}_{3}&\df \Big|\min_{\bm{W}\in \mathcal{W}}\sup_{\widehat{\mu}^{N}\in \mathcal{P}_{N}}\widehat{\mathrm{MMD}}_{\widehat{\mu}^{N}}\Big[P_{\bm{V}},P_{\bm{W}}\Big]-\widehat{\mathrm{MMD}}_{\widehat{\mu}_{\ast}^{N}}\Big[P_{\bm{V}},P_{{\bm{W}}}\Big]\Big|  \\
\mathsf{A}_{4}&\df \Big|\widehat{\mathrm{MMD}}_{\widehat{\mu}_{\ast}^{N}}\Big[P_{\bm{V}},P_{\widehat{\bm{W}}_{\ast}^{N}}\Big]-\mathrm{MMD}_{\widehat{\mu}^{N}_{\ast}}\Big[P_{\bm{V}},P_{\widehat{\bm{W}}^{N}_{\ast}}\Big]\Big| .
\end{align*}

In the sequel, we compute an upper bound for each term on the right hand side of Eq. \eqref{Eq:A1-A4}:
\vspace*{5mm}

\noindent\textbf{Upper bound on $\mathsf{A}_{1}$:}

First, notice that the squared kernel MMD loss in Eq. \eqref{Eq:kernel_MMD} can be characterized in terms of class labels and features defined in Section \ref{subection:Kernel-Target Alignment as an Unbiased Estimator of MMD Loss} as follows
\begin{align}
\label{Eq:To_See_Equivalence}
\mathrm{MMD}_{\mu}[P_{\bm{V}},P_{\bm{W}}]&=4\expect_{P_{\bm{x},y}^{\otimes 2}}\left[y\widehat{y}K_{\mu}(\bm{x},\widehat{\bm{x}})\right].
\end{align}
To see this equivalence, we first rewrite the right hand side of Eq. \eqref{Eq:To_See_Equivalence} as follows
\begin{align}
\nonumber
\expect_{P_{y,\bm{x}}^{\otimes 2}}\left[y\widehat{y}K_{\mu}(\bm{x},\widehat{\bm{x}})\right]&=\prob\{y=+1\}\prob\{\widehat{y}=+1\}\expect_{\bm{x},\widehat{\bm{x}}\sim P^{\otimes 2}_{\bm{x}|y=+1}}[K_{\mu}(\bm{x},\widehat{\bm{x}})]
\\ \nonumber
&\hspace{4mm}+\prob\{y=-1\}\prob\{\widehat{y}=-1\}\expect_{\bm{x},\widehat{\bm{x}}\sim P^{\otimes 2}_{\bm{x}|y=-1}}[K_{\mu}(\bm{x},\widehat{\bm{x}})]\\ \nonumber
&\hspace{4mm}-\prob\{y=-1\}\prob\{\widehat{y}=+1\}\expect_{\bm{x}\sim P_{\bm{x}|y=-1},\widehat{\bm{x}}\sim P_{\bm{x}|y=+1}}[K_{\mu}(\bm{x},\widehat{\bm{x}})]  \\ \label{Eq:Vaght_Tange}
&\hspace{4mm}-\prob\{y=+1\}\prob\{\widehat{y}=-1\}\expect_{\bm{x}\sim P_{\bm{x}|y=+1},\widehat{\bm{x}}\sim P_{\bm{x}|y=-1}}[K_{\mu}(\bm{x},\widehat{\bm{x}})].
\end{align}
Now, recall from Section \ref{subection:Kernel-Target Alignment as an Unbiased Estimator of MMD Loss} that $P_{\bm{x}|y=+1}=P_{\bm{V}}$, and $P_{\bm{x}|y=-1}=P_{\bm{W}}$ by construction of the labels and random features. Moreover, $y,\widehat{y}\sim_{\mathrm{i.i.d.}}\mathrm{Uniform}\{-1,+1\}$, and thus $\prob\{y=-1\}=\prob\{y=+1\}={1\over 2}$. Therefore, from Eq. \eqref{Eq:Vaght_Tange}, we derive
\begin{align}
\nonumber
\expect_{P_{y,\bm{x}}^{\otimes 2}}[y\widehat{y}K_{\mu}(\bm{x},\widehat{\bm{x}})]&={1\over 4}\expect_{P^{\otimes 2}_{\bm{V}}}[K_{\mu}(\bm{x};\widehat{\bm{x}})]+{1\over 4}\expect_{ P^{\otimes 2}_{\bm{W}}}[K_{\mu}(\bm{x};\widehat{\bm{x}})]
-{1\over 2}\expect_{P_{\bm{V}},P_{\bm{W}}}[K_{\mu}(\bm{x};\widehat{\bm{x}})]\\ \nonumber
&={1\over 4}\mathrm{MMD}_{\mu}[P_{\bm{V}},P_{\bm{W}}].
\end{align}
For any given $\bm{W}\in \mathcal{W}$, we have that
\begin{align}
\nonumber
\Big|\sup_{\mu\in \mathcal{P}}\widehat{\mathrm{MMD}}_{\mu}[P_{\bm{V}},P_{\bm{W}}]&-\sup_{\mu\in \mathcal{P}}\mathrm{MMD}_{\mu}[P_{\bm{V}},P_{\bm{W}}]\Big|\\ \nonumber
&\leq \sup_{\mu\in \mathcal{P}}\big|\widehat{\mathrm{MMD}}_{\mu}[P_{\bm{V}},P_{\bm{W}}]-\mathrm{MMD}_{\mu}[P_{\bm{V}},P_{\bm{W}}]\big|\\ \nonumber
&=4\sup_{\mu\in \mathcal{P}}\Bigg|{1\over n(n-1)}\sum_{i\not= j}y_{i}y_{j}K_{\mu}(\bm{x}_{i},\bm{x}_{j})-\expect_{P_{y,\bm{x}}^{\otimes 2}}\left[y\widehat{y}K_{\mu}(\bm{x},\widehat{\bm{x}})\right]\Bigg|\\ \nonumber
&=4\sup_{\mu\in \mathcal{P}} \big|\expect_{\mu}[E_{n}(\bm{\xi})]\big|\\ \nonumber
&\leq 4\left|\sup_{\mu\in \mathcal{P}} \expect_{\mu}[E_{n}(\bm{\xi})]\right|,
\end{align}
where the error term is defined using the random features
\begin{align}
\label{Eq:Error_Term}
E_{n}(\bm{\xi})\df  \dfrac{1}{n(n-1)}\sum_{i\not= j}y_{i}y_{j}\varphi(\bm{x}_{i};\bm{\xi})\varphi(\bm{x}_{j};\bm{\xi})-\expect_{P^{\otimes 2}_{\bm{x},y}}[y\widehat{y}\varphi(\bm{x};\bm{\xi})\varphi(\widehat{\bm{x}},\bm{\xi})].
\end{align}

Now, we invoke the following strong duality theorem \cite{gao2016distributionally}:

\begin{theorem}\textsc{(Strong Duality for Robust Optimization, \cite[Theorem 1]{gao2016distributionally})}
	\label{Thm:Strong_Duality}
	Consider the general metric space $(\Xi,d)$, and any normal distribution $\nu\in \mathcal{M}(\Xi)$, where $\mathcal{M}(\Xi)$ is the set of Borel probability measures on $\Xi$. Then, 
	\begin{align}
	\sup_{\mu\in \mathcal{M}(\Xi)} \Big\{\expect_{\mu}[\Psi(\bm{\xi})]: W_{p}(\mu,\nu)\leq R\Big\}=\min_{\lambda\geq 0}\left\{\lambda R^{p}-\int_{\Xi}\inf_{\bm{\xi}\in \Xi}[\lambda d^{p}(\bm{\xi},\bm{\zeta})-\Psi(\bm{\xi})] \nu(\mathrm{d}\bm{\zeta}) \right\},
	\end{align}
	provided that $\Psi$ is upper semi-continuous in $\bm{\xi}$.
\end{theorem}

Under the strong duality of Theorem \ref{Thm:Strong_Duality}, we obtain that
\begin{align}
\nonumber
\Big|\sup_{\mu\in \mathcal{P}}\widehat{\mathrm{MMD}}_{\mu}[P_{\bm{V}},P_{\bm{W}}]&-\sup_{\mu\in \mathcal{P}}\mathrm{MMD}_{\mu}[P_{\bm{V}},P_{\bm{W}}]\Big|\\ \label{Eq:Returning_To_Strong_Duality}
&\leq 4\left|\min_{\lambda\geq 0}\left\{\lambda R^{p}-\int_{\real^{D}}\inf_{\bm{\zeta}\in \real^{D}}\big[\lambda \|\bm{\xi}-\bm{\zeta}\|_{2}^{p}-E_{n}(\bm{\zeta})\big] \mu_{0}(\mathrm{d}\bm{\xi}) \right\}\right|.
\end{align}
In the sequel, let $p=2$. The \textit{Moreau's envelope}  \cite{parikh2014proximal} of a function $f:\mathcal{X}\rightarrow \real$ is defined as follows
\begin{align}
\label{Eq:Moreau's envelope}
M_{f}^{\beta}(\bm{y})\df\inf_{\bm{x}\in \mathcal{X}} \left\{\dfrac{1}{2\beta}\|\bm{x}-\bm{y}\|_{2}^{2}+f(\bm{x})\right\},\quad \forall \bm{y}\in \mathcal{X},
\end{align}
where $\beta>0$ is the regularization parameter. When the function $f$ is differentiable, the following lemma can be established:

\begin{lemma}\textsc{(Moreau's envelope of Differentiable Functions)}
	\label{Lemma:M_envelope}
	Suppose the function $f:\mathcal{X}\rightarrow \real$ is differentiable. Then, 	the Moreau's envelope defined in Eq. \eqref{Eq:Moreau's envelope} has the following upper bound and lower bounds
	\begin{align}
	\label{Eq:THe_Moreau_lower}
	f(\bm{y})-\dfrac{\beta}{2} \int_{0}^{1}\sup_{\bm{x}\in \mathcal{X}} \|\nabla f(\bm{y}+s(\bm{x}-\bm{y}))\|_{2}^{2}\mathrm{d}s \leq M_{f}^{\beta}(\bm{y})\leq f(\bm{y}).
	\end{align}
	In particular, when $f$ is $L_{f}$-Lipschitz, we have
	\begin{align}
	f(\bm{y})-\dfrac{\beta L^{2}_{f}}{2}\leq M_{f}^{\beta}(\bm{y})\leq f(\bm{y}).
	\end{align}
\end{lemma}
The proof is presented in Appendix \ref{Proof_of_Lemma:M_envelop}. 

Now, we return to Equation \eqref{Eq:Returning_To_Strong_Duality}. We leverage the lower bound on Moreau's envelope in Eq. \eqref{Eq:THe_Moreau_lower} of Lemma \ref{Lemma:M_envelope} as follows
\begin{align}
\nonumber
&\Big|\sup_{\mu\in \mathcal{P}}\widehat{\mathrm{MMD}}_{\mu}[P_{\bm{V}},P_{\bm{W}}]-\sup_{\mu\in \mathcal{P}}\mathrm{MMD}_{\mu}[P_{\bm{V}},P_{\bm{W}}]\Big|\\ \nonumber
&\leq 4\left|\min_{\lambda\geq 0}\left\{\lambda R^{2}-\int_{\real^{D}}M_{-E_{n}}^{{1\over 2\lambda}}(\bm{\xi}) \mu_{0}(\mathrm{d}\bm{\xi}) \right\}\right|\\ \nonumber
&\leq  4\left|\min_{\lambda\geq 0}\left\{\lambda R^{2}+\expect_{\mu_{0}}[E_{n}(\bm{\xi})]+{1\over 4\lambda}  \expect_{\mu_{0}}\left[\int_{0}^{1}\sup_{\bm{\zeta}\in \real^{D}}\|\nabla E_{n}((1-s)\bm{\xi}+s\bm{\zeta})\|_{2}^{2}\mathrm{d}s\right]  \right\}\right|\\  \label{Eq:To_proceed_From_1}
&\leq 4|\expect_{\mu_{0}}[E_{n}(\bm{\xi})]|+4R\expect_{\mu_{0}}\left[\int_{0}^{1}\sup_{\bm{\zeta}\in \real^{D}}\|\nabla E_{n}((1-s)\bm{\xi}+s\bm{\zeta})\|_{2}^{2}\mathrm{d}s \right].
\end{align}
Let $\bm{\zeta}_{\ast}=\bm{\zeta}_{\ast}(\bm{\xi},s)=\arg\sup_{\bm{\zeta}\in \real^{D}}\|\nabla E_{n}(1-s)\bm{\xi}+s\bm{\zeta}\|_{2}$. Then, applying the union bound in conjunction with Inequality \eqref{Eq:To_proceed_From_1} yields
\begin{align}
\nonumber
&\prob\Bigg(\Big|\sup_{\mu\in \mathcal{P}}\widehat{\mathrm{MMD}}_{\mu}[P_{\bm{V}},P_{\bm{W}}]-\sup_{\mu\in \mathcal{P}}\mathrm{MMD}_{\mu}[P_{\bm{V}},P_{\bm{W}}]\Big|\geq \delta \Bigg) \\ \nonumber
&\leq \prob\Bigg(\Bigg|\int_{\real^{D}}E_{n}(\bm{\xi})\mu_{0}(\mathrm{d}\bm{\xi}) \Bigg|\geq {\delta\over 8}  \Bigg)+\prob\Bigg(\int_{\real^{D}}\hspace{-1mm}\int_{0}^{1} \|\nabla E_{n}((1-s)\bm{\xi}+s\bm{\zeta}_{\ast})\|^{2}_{2}\mathrm{d}s\mu_{0}(\mathrm{d}\bm{\xi})\geq {\delta\over 8R} \Bigg).\\ \label{Eq:union_bound}
\end{align}

Now, we state the following lemma:
\begin{lemma}\textsc{(Tail Bounds for the Finite Sample Estimation Error)}
	\label{Lemma:Tail Bounds for the Finite Sample Estimation Error}
	Consider the estimation error $E_{n}$ defined in Eq. \eqref{Eq:Error_Term}. Then, the following statements hold:
	\begin{itemize}
		\item $Z=\|\nabla E_{n}(\bm{\xi})\|_{2}^{2}$ is a sub-exponential random variable with the Orlicz norm of $\|Z\|_{\psi_{1}}\leq {9\times 2^{9}\times L^{4}\over n^{2}}$ for every $\bm{\xi}\in \real^{D}$. Moreover, 
		\begin{align}
		\label{Eq:probability_bound_1}
		\hspace{10mm}\prob\Bigg( \int_{\real^{D}}\int_{0}^{1}\| \nabla E_{n}((1-s)\bm{\xi}+s\bm{\zeta}_{\ast}) \|_{2}^{2}\mathrm{d}s\mu_{0}(\mathrm{d}\bm{\xi})\geq \delta \Bigg) \leq 2 e^{-{n^{2}\delta\over 9\times 2^{9}\times L^{4}}+{L^{4}\over 9}},
		\end{align}
		\item $E_{n}(\bm{\xi})$ is zero-mean sub-Gaussian random variable with the Orlicz norm of $\|E_{n}(\bm{\xi})\|_{\psi_{2}}\leq {16\sqrt{3}L^{4}\over n}$ for every $\bm{\xi}\in \real^{D}$. Moreover,
		\begin{align}
		\label{Eq:probability_bound_2}
		\prob\left(\left|\int_{\real^{D}} E_{n}(\bm{\xi})\mu_{0}(\mathrm{d}\bm{\xi})\right|\geq \delta \right)\geq 2e^{-{n^{2}\delta^{2} \over 16\sqrt{3}L^{4}} }.
		\end{align}
	\end{itemize}
\end{lemma}
The proof of Lemma \ref{Lemma:Tail Bounds for the Finite Sample Estimation Error} is presented in Appendix \ref{Appendix:Proof_of_Thm_McKean}. 

Now, we leverage the probability bounds \eqref{Eq:probability_bound_1} and \eqref{Eq:probability_bound_2} of Lemma \ref{Lemma:Tail Bounds for the Finite Sample Estimation Error} to upper bound the terms on the right hand side of Eq. \eqref{Eq:union_bound} as below
\small{\begin{align}
	\nonumber
	\prob\Bigg(\Big|\sup_{\mu\in \mathcal{P}}\widehat{\mathrm{MMD}}_{\mu}[P_{\bm{V}},P_{\bm{W}}]-\sup_{\mu\in \mathcal{P}}\mathrm{MMD}_{\mu}[P_{\bm{V}},P_{\bm{W}}]\Big|\geq \delta \Bigg)&\leq 2e^{-{n^{2}\delta^{2} \over \sqrt{3}\times 2^{11}\times L^{4}} }+2 e^{-{n^{2}\delta\over 9\times 2^{12}\times RL^{4}}+{L^{4}\over 9}}\\ 
	&\leq 4\max\left\{ e^{-{n^{2}\delta^{2} \over \sqrt{3}\times 2^{11}\times L^{4}} },e^{-{n^{2}\delta\over 9\times 2^{12}\times RL^{4}}+{L^{4}\over 9}}\right\},
	\end{align}}\normalsize 
where the last inequality comes from the basic inequality $a+b\leq 2\max\{a,b\}$.
Therefore, with the probability of at least $1-\varrho$, we have that
\small{\begin{align}
	\nonumber
	\Big|\sup_{\mu\in \mathcal{P}}\widehat{\mathrm{MMD}}_{\mu}[P_{\bm{V}},P_{\bm{W}}]&-\sup_{\mu\in \mathcal{P}}\mathrm{MMD}_{\mu}[P_{\bm{V}},P_{\bm{W}}]\Big|\\ \label{Eq:varrho}
	&\leq  \min\left\{\dfrac{3^{1\over 4}\times 2^{11\over 2}\times L^{2}}{n} \ln^{1\over 2}\left(\dfrac{4}{\varrho}\right),\dfrac{9\times 2^{12}\times RL^{4} }{n^{2}}\ln\left(\dfrac{4e^{L^{4}\over 9}}{\varrho}\right) \right\},
	\end{align}}\normalsize
for all $\bm{W}\in \mathcal{W}$.

\begin{lemma}\textsc{(Distance between minima of Adjacent Functions)}
	\label{Lemma:A_Basic_Inequality}
	Let $\Psi(\bm{W}):\mathcal{W}\rightarrow \real$ and $\Phi(\bm{W}):\mathcal{W}\rightarrow \real$. Further, suppose $\|\Psi(\bm{W})-\Phi(\bm{W})\|_{\infty}\leq \delta$ for some $\delta>0$. Then,
	\begin{align}
	\label{Eq:Basic_Inequality}
	\left|\min_{\bm{W}\in \mathcal{W}}\Psi(\bm{W})-\min_{\bm{W}\in \mathcal{W}}\Phi(\bm{W})\right|\leq \delta.
	\end{align}
\end{lemma}
See Appendix \ref{Appendix:Proof_of_a_Basic_Inequality} for the proof.

Let $\Psi(\bm{W})\df \sup_{\mu\in \mathcal{P}} \widehat{\mathrm{MMD}}_{\mu}[P_{\bm{V}},P_{\bm{W}}]$, and $\Phi(\bm{W})\df \sup_{\mu\in \mathcal{P}}\mathrm{MMD}_{\mu}[P_{\bm{V}},P_{\bm{W}}]$.  Then, from Inequality \eqref{Eq:Basic_Inequality}, we have the following upper bound on $\mathsf{A}_{1}$
\begin{align}
\nonumber
\mathsf{A}_{1}&=\Big|\mathrm{MMD}_{\mu_{\ast}}[P_{\bm{V}},P_{{\bm{W}}_{\ast}}]-\min_{\bm{W}\in \mathcal{W}}\sup_{\mu\in \mathcal{P}} \widehat{\mathrm{MMD}}_{\mu}[P_{\bm{V}},P_{\bm{W}}] \Big|\\  \nonumber
&=\Big|\min_{\bm{W}\in \mathcal{W}}\sup_{\mu\in \mathcal{P}} \mathrm{MMD}_{\mu}[P_{\bm{V}},P_{\bm{W}}]-\min_{\bm{W}\in \mathcal{W}}\sup_{\mu\in \mathcal{P}} \widehat{\mathrm{MMD}}_{\mu}[P_{\bm{V}},P_{\bm{W}}]  \Big|\\
&\leq  \min\left\{\dfrac{3^{1\over 4}\times 2^{4}\times L^{2}}{n} \ln^{1\over 2}\left(\dfrac{4}{\varrho}\right),\dfrac{9\times 2^{11}\times RL^{4} }{n^{2}}\ln\left(\dfrac{4e^{L^{4}\over 9}}{\varrho}\right) \right\}. 
\end{align}
with the probability of (at least) $1-\varrho$.

\vspace*{5mm}
\noindent\textbf{Upper bound on $\mathsf{A}_{2}$}:

To establish the upper bound on $\mathsf{A}_{2}$, we recall that
\begin{subequations}
	\begin{align}
	\widehat{\mathrm{MMD}}_{\widehat{\mu}^{N}}[P_{\bm{V}},P_{\bm{W}}]&= \dfrac{1}{n(n-1)}\dfrac{1}{N}\sum_{i\not=j}\sum_{k=1}^{N}y_{i}y_{j}\varphi(\bm{x}_{i};\bm{\xi}^{k})\varphi(\bm{x}_{j};\bm{\xi}^{k}) \\
	\widehat{\mathrm{MMD}}_{\mu}[P_{\bm{V}},P_{\bm{W}}]&={1\over n(n-1)}\sum_{i\not=j}y_{i}y_{j}\expect_{\mu}[\varphi(\bm{x}_{i};\bm{\xi})\varphi(\bm{x}_{j};\bm{\xi})] .
	\end{align}
\end{subequations}
Therefore,
\begin{align}
\nonumber
\Big|\sup_{\mu\in \mathcal{P}}\widehat{\mathrm{MMD}}_{\mu}[P_{\bm{V}},P_{\bm{W}}]&-\sup_{\widehat{\mu}^{N}\in \mathcal{P}_{N}} \widehat{\mathrm{MMD}}_{\widehat{\mu}^{N}}[P_{\bm{V}},P_{\bm{W}}]\Big|\\ \label{Eq:Trick}
&\leq\Big | \sup_{\mu\in \mathcal{P}}\expect_{\mu}[\varphi(\bm{x}_{i};\bm{\xi})\varphi(\bm{x}_{j};\bm{\xi})]- \sup_{\widehat{\mu}^{N}\in \mathcal{P}_{N}}\expect_{\widehat{\mu}^{N}}[\varphi(\bm{x}_{i};\bm{\xi})\varphi(\bm{x}_{j};\bm{\xi})]\Big|.
\end{align}
Here, the last inequality is due to Theorem \ref{Thm:Strong_Duality} and the following duality results hold
\small{\begin{align}
	\nonumber
	\sup_{\mu\in \mathcal{P}}\expect_{\mu}[\varphi(\bm{x};\bm{\xi})\varphi(\widehat{\bm{x}};\bm{\xi})]&= \inf_{\lambda\geq 0}\left\{ \lambda R^{2}-\int_{\real^{D}}\inf_{\bm{\zeta}\in \real^{D}}\{ \lambda\|\bm{\xi}-\bm{\zeta}\|^{2}_{2}-\varphi(\bm{x}_{i};\bm{\zeta})\varphi(\bm{x}_{j};\bm{\zeta})\}\mu_{0}(\mathrm{d}\bm{\xi}) \right\}\\ \nonumber
	\sup_{\widehat{\mu}^{N}\in \mathcal{P}_{N}}\expect_{\widehat{\mu}^{N}}[\varphi(\bm{x}_{i};\bm{\xi})\varphi(\bm{x}_{j};\bm{\xi})]&= \inf_{\lambda\geq 0}\left\{ \lambda R^{2}-\dfrac{1}{N}\sum_{k=1}^{N}\inf_{\bm{\zeta}\in \real^{D}}\{ \lambda\|\bm{\xi}_{0}^{k}-\bm{\zeta}\|^{2}_{2}-\varphi(\bm{x}_{i};\bm{\zeta})\varphi(\bm{x}_{j};\bm{\zeta})\} \right\}.
	\end{align}}\normalsize
Now, in the sequel, we establish a uniform concentration result for the following function
\small{\begin{align}
	\nonumber
	T_{(\bm{x},\widehat{\bm{x}})}^{\lambda}:\mathcal{\real}^{N\times D}&\mapsto \real\\ \nonumber
	\hspace{-4mm}(\bm{\xi}_{0}^{1},\cdots,\bm{\xi}_{0}^{N})&\mapsto T_{(\bm{x},\widehat{\bm{x}})}^{\lambda}(\bm{\xi}_{0}^{1},\cdots,\bm{\xi}_{0}^{N})=\dfrac{1}{N}\sum_{k=1}^{N}M_{-\varphi(\bm{x},\cdot)\varphi(\widehat{\bm{x}},\cdot)}^{{1\over 2\lambda}}(\bm{\xi}^{k}_{0})-\int_{\real^{D}}M_{-\varphi(\bm{x},\cdot)\varphi(\widehat{\bm{x}},\cdot)}^{1\over 2\lambda}(\bm{\xi})\mu_{0}(\mathrm{d}\bm{\xi}).
	\end{align}}\normalsize
Then, from Eq. \eqref{Eq:Trick} we have
\begin{align}
\label{Eq:Trick_2}
\Big|\sup_{\mu\in \mathcal{P}}\widehat{\mathrm{MMD}}_{\mu}[P_{\bm{V}},P_{\bm{W}}]&-\sup_{\widehat{\mu}^{N}\in \mathcal{P}_{N}} \widehat{\mathrm{MMD}}_{\widehat{\mu}^{N}}[P_{\bm{V}},P_{\bm{W}}]\Big|\leq \sup_{\lambda\geq 0}\sup_{\bm{x},\widehat{\bm{x}}\in \mathcal{X}}|T_{(\bm{x},\widehat{\bm{x}})}^{\lambda}(\bm{\xi}_{0}^{1},\cdots,\bm{\xi}_{0}^{N})|.
\end{align}

We now closely follow the argument of \cite{rahimi2008random} to establish a  uniform concentration result with respect to the data points $\bm{x},\widehat{\bm{x}}\in \mathcal{X}$. In particular, consider an $\epsilon$-net cover of $\mathcal{X}\subset \real^{d}$. Then, we require $N_{\epsilon}=\left({4\mathrm{diam}(\mathcal{X})\over \epsilon} \right)^{d}$ balls of the radius $\epsilon>0$,  \textit{e.g.}, see \cite[Lemma 4.1, Section 4]{pollard1990empirical}. Let $\mathcal{Z}=\{\bm{z}_{1},\cdots,\bm{z}_{N_{\epsilon}}\}\subset \mathcal{X}$ denotes the center of the covering net. Now, let $(\bm{\xi}_{0}^{1},\cdots,\bm{\xi}_{0}^{k},\cdots,\bm{\xi}_{0}^{N})\in \real^{N\times D}$ and $(\bm{\xi}_{0}^{1},\cdots,\widetilde{\bm{\xi}}_{0}^{k},\cdots,\bm{\xi}_{0}^{N})\in \real^{N\times D}$ be two sequences that differs in the $k$-th coordinate for $1\leq k\leq N$. Then,
\begin{align}
\nonumber
\big|T_{(\bm{z}_{i},\bm{z}_{j})}(\bm{\xi}_{0}^{1},\cdots,\bm{\xi}_{0}^{k},\cdots,\bm{\xi}_{0}^{N})&-T_{(\bm{z}_{i},{\bm{z}_{j}})}(\bm{\xi}_{0}^{1},\cdots,\widetilde{\bm{\xi}}_{0}^{k},\cdots,\bm{\xi}_{0}^{N})\big|\\ \label{Eq:customer}
&=\dfrac{1}{N}\Big|M_{-\varphi(\bm{z}_{i};\cdot)\varphi(\bm{z}_{j};\cdot)}^{{1\over 2\lambda}}(\bm{\xi}^{k}_{0})-M_{-\varphi(\bm{z}_{i},\cdot)\varphi({\bm{z}_{j}},\cdot)}^{{1\over 2\lambda}}(\widetilde{\bm{\xi}}^{k}_{0})\Big|.
\end{align}
Without loss of generality suppose $M_{-\varphi(\bm{z}_{i};\cdot)\varphi(\bm{z}_{j};\cdot)}^{{1\over 2\lambda}}(\bm{\xi}^{k}_{0})\geq M_{-\varphi(\bm{z}_{i};\cdot)\varphi(\bm{z}_{j};\cdot)}^{{1\over 2\lambda}}(\widetilde{\bm{\xi}}^{k}_{0})$. Then, 
\begin{align}
\nonumber
&M_{-\varphi(\bm{z}_{i};\cdot)\varphi(\bm{z}_{j};\cdot)}^{{1\over 2\lambda}}(\bm{\xi}^{k}_{0})-M_{-\varphi(\bm{z}_{i},\cdot)\varphi({\bm{z}_{j}},\cdot)}^{{1\over 2\lambda}}(\widetilde{\bm{\xi}}^{k}_{0})\\ \nonumber
&=\inf_{\bm{\zeta}\in \real^{D}}\left\{\lambda\|\bm{\zeta}-\bm{\xi}_{0}^{k}\|_{2}^{2}-\varphi(\bm{z}_{i};\bm{\zeta})\varphi(\bm{z}_{j};\bm{\zeta})  \right\} -\inf_{\bm{\zeta}\in \real^{D}}\left\{\lambda\|\bm{\zeta}-\widetilde{\bm{\xi}}_{0}^{k}\|_{2}^{2}-\varphi(\bm{z}_{i};\bm{\zeta})\varphi(\bm{z}_{j};\bm{\zeta})  \right\} \\ \nonumber
&\stackrel{\rm{(a)}}{\leq} -\varphi(\bm{z}_{i};\bm{\xi}_{0}^{k})\varphi(\bm{z}_{j};\bm{\xi}_{0}^{k})-\inf_{\bm{\zeta}\in \real^{D}}\left\{\lambda\|\bm{\zeta}-\widetilde{\bm{\xi}}_{0}^{k}\|_{2}^{2}-\varphi(\bm{z}_{i};\bm{\zeta})\varphi(\bm{z}_{j};\bm{\zeta})  \right\}\\  \nonumber
&\stackrel{\rm{(b)}}{\leq} -\varphi(\bm{z}_{i};\bm{\xi}_{0}^{k})\varphi(\bm{z}_{j};\bm{\xi}_{0}^{k})+\sup_{\bm{\zeta}\in \real^{D}}\left\{\varphi(\bm{z}_{i};\bm{\zeta})\varphi(\bm{z}_{j};\bm{\zeta})  \right\}\\ \label{Eq:Mystery_1}
&\stackrel{\rm{(c)}}{\leq} 2L^{2},
\end{align}
where $\mathrm{(a)}$ follows by letting $\bm{\zeta}=\bm{\xi}_{0}^{k}$ in the first optimization problem, $\mathrm{(b)}$ follows by using the fact that $-\lambda\|\bm{\zeta}-\widetilde{\bm{\xi}}_{0}^{k}\|_{2}$ is non-positive for any $\bm{\zeta}\in \real^{D}$ and can be dropped, and $\mathrm{(c)}$ follows from Assumption $\mathbf{(A.2)}$.

Now, plugging the upper bound in Eq. \eqref{Eq:Mystery_1}  into Eq.  \eqref{Eq:customer} yields
\begin{align}
\nonumber
\big|T^{\lambda}_{(\bm{z}_{i},\bm{z}_{j})}(\bm{\xi}_{0}^{1},\cdots,\bm{\xi}_{0}^{k},\cdots,\bm{\xi}_{0}^{N})-T^{\lambda}_{(\bm{z}_{i},\bm{z}_{j})}(\bm{\xi}_{0}^{1},\cdots,\widetilde{\bm{\xi}}_{0}^{k},\cdots,\bm{\xi}_{0}^{N})\big|\leq \dfrac{2L^{2}}{N}.
\end{align}
From McDiarmid's Martingale inequality \cite{mcdiarmid1989method} and the union bound, we obtain that
\begin{align}
\prob\left(\cup_{\bm{z}_{i},\bm{z}_{j}\in \mathcal{Z}} |T^{\lambda}_{(\bm{z}_{i},\bm{z}_{j})}(\bm{\xi}_{0}^{1},\cdots,\bm{\xi}_{0}^{N})|\geq \delta \right)\leq \left(\dfrac{4\mathrm{diam}(\mathcal{X})}{\epsilon} \right)^{d}\cdot\exp\left(-{N\delta^{2}\over L^{2}}\right),
\end{align}
for all $\lambda\geq 0$. Now, consider arbitrary points $(\bm{x},\widehat{\bm{x}})\in \mathcal{X}\times \mathcal{X}$. Let the center of the balls containing those points be $\bm{z}_{i},\bm{z}_{j}\in \mathcal{Z}$, \textit{i.e.}, $\bm{x}\in \ball_{\varepsilon}(\bm{z}_{i})$ and $\widehat{\bm{x}}\in \ball_{\varepsilon}(\bm{z}_{j})$ for some $\bm{z}_{i},\bm{z}_{j}\in \mathcal{Z}$. Then, by the triangle inequality, we have that
\small{\begin{align}
	\nonumber
	|T_{(\bm{x},\widehat{\bm{x}})}^{\lambda}(\bm{\xi}_{0}^{1},\cdots,\bm{\xi}_{0}^{N})&- T_{(\bm{z}_{i},\bm{z}_{j})}(\bm{\xi}_{0}^{1},\cdots,\bm{\xi}_{0}^{N})|\\ \nonumber
	&\leq 
	|T_{(\bm{x},\widehat{\bm{x}})}^{\lambda}(\bm{\xi}_{0}^{1},\cdots,\bm{\xi}_{0}^{N})-T^{\lambda}_{(\bm{z}_{i},\widehat{\bm{x}})}(\bm{\xi}_{0}^{1},\cdots,\bm{\xi}_{0}^{N})|\\ \nonumber
	&+|T^{\lambda}_{(\bm{z}_{i},\widehat{\bm{x}})}(\bm{\xi}_{0}^{1},\cdots,\bm{\xi}_{0}^{N})-T^{\lambda}_{(\bm{z}_{i},{\bm{z}_{j}})}(\bm{\xi}_{0}^{1},\cdots,\bm{\xi}_{0}^{N})|\\ \nonumber
	&\leq \|\nabla_{\bm{x}} T_{(\bm{x},\widehat{\bm{x}})}^{\lambda}(\bm{\xi}_{0}^{1},\cdots,\bm{\xi}_{0}^{N}) \|_{2}\|\bm{x}-\bm{z}_{i}\|_{2}+ \|\nabla_{\widehat{\bm{x}}} T_{(\bm{x},\widehat{\bm{x}})}^{\lambda}(\bm{\xi}_{0}^{1},\cdots,\bm{\xi}_{0}^{N}) \|_{2}\|\widehat{\bm{x}}-\bm{z}_{j}\|_{2}\\
	&\leq 2L_{T}\epsilon,
	\end{align}}\normalsize
where $L_{T}=L_{T}(\bm{\xi}_{0}^{1},\cdots,\bm{\xi}_{0}^{N})\df \sup_{\bm{x},\widehat{\bm{x}}\in \mathcal{X}}\|\nabla_{\bm{x}}T_{(\bm{x},\widehat{\bm{x}})}^{\lambda}(\bm{\xi}_{0}^{1},\cdots,\bm{\xi}_{0}^{N})\|_{2}$ is the Lipschitz constant of the mapping $T$. Note that the Lipschitz constant $L_{T}$ is a random variable with respect to the random feature samples $\bm{\xi}_{0},\cdots,\bm{\xi}_{N}$. Let $(\bm{x}_{\ast},\widehat{\bm{x}}_{\ast})\df \arg\sup_{\bm{x},\widehat{\bm{x}}\in \mathcal{X}}\|\nabla_{\bm{x}}T_{(\bm{x},\widehat{\bm{x}})}^{\lambda}(\bm{\xi}_{0}^{1},\cdots,\bm{\xi}_{0}^{N})\|_{2}$. We compute an upper bound on the second moment of the random variable $L_{T}$ as follows
\small{\begin{align}
	\nonumber
	\expect_{\mu_{0}}\big[L_{T}^{2}\big]&=\expect_{\mu_{0}}\left[\|\nabla_{\bm{x}}T^{\lambda}_{(\bm{x}_{\ast},\widehat{\bm{x}}_{\ast})} (\bm{\xi}_{0}^{1},\cdots,\bm{\xi}_{0}^{N})\|_{2}^{2}\right]\\ \nonumber
	&=\expect_{\mu_{0}}\left[\left\|\dfrac{1}{N}\sum_{k=1}^{N}\nabla_{\bm{x}}M_{-\varphi(\bm{x}_{\ast};\cdot)\varphi(\widehat{\bm{x}}_{\ast};\cdot)}^{1\over 2\lambda}(\bm{\xi}_{0}^{k})-\int_{\real^{D}}\nabla_{\bm{x}}M_{-\varphi(\bm{x}_{\ast};\cdot)\varphi(\widehat{\bm{x}}_{\ast};\cdot)}^{1\over 2\lambda}(\bm{\xi})\mu_{0}(\mathrm{d}\bm{\xi})\right\|_{2}^{2} \right]\\ \nonumber
	&= \expect_{\mu_{0}}\left[\left\|\dfrac{1}{N}\sum_{k=1}^{N}\nabla_{\bm{x}} M^{1\over 2\lambda}_{-\varphi(\bm{x}_{\ast};\cdot)\varphi(\widehat{\bm{x}}_{\ast};\cdot)}(\bm{\xi}_{0}^{k})  \right\|_{2}^{2}\right]-\expect_{\mu_{0}}\left[\left\|\int_{\real^{D}}\nabla_{\bm{x}}M_{-\varphi(\bm{x}_{\ast};\cdot)\varphi(\widehat{\bm{x}}_{\ast};\cdot)}^{1\over 2\lambda}(\bm{\xi})\mu_{0}(\mathrm{d}\bm{\xi})\right\|_{2}^{2}\right]\\ \nonumber
	&\leq \dfrac{1}{N^{2}} \expect_{\mu_{0}}\left[\left\|\sum_{k=1}^{N}\nabla_{\bm{x}} M^{1\over 2\lambda}_{-\varphi(\bm{x}_{\ast};\cdot)\varphi(\widehat{\bm{x}}_{\ast};\cdot)}(\bm{\xi}_{0}^{k}) \right\|_{2}^{2}\right].
	\end{align}}\normalsize
We further proceed using the triangle inequality as well as the basic inequality $(a_{1}+a_{2}+\cdots+a_{N})^{2}\leq N(a_{1}^{2}+a_{2}^{2}+\cdots+a_{N}^{2})$,
\begin{align}
\nonumber
\expect_{\mu_{0}}\big[L_{T}^{2}\big]&=\dfrac{1}{N^{2}} \expect_{\mu_{0}}\left[\left\|\sum_{k=1}^{N}\nabla_{\bm{x}} M^{1\over 2\lambda}_{-\varphi(\bm{x}_{\ast};\cdot)\varphi(\widehat{\bm{x}}_{\ast};\cdot)}(\bm{\xi}_{0}^{k}) \right\|_{2}^{2}\right]\\ \nonumber
&\leq \dfrac{1}{N^{2}}\expect_{\mu_{0}}\left[\left(\sum_{k=1}^{N}\left\|\nabla_{\bm{x}} M^{1\over 2\lambda}_{-\varphi(\bm{x}_{\ast};\cdot)\varphi(\widehat{\bm{x}}_{\ast};\cdot)}(\bm{\xi}_{0}^{k}) \right\|_{2}\right)^{2}\right]\\ \label{Eq:to_proceed}
&\leq \dfrac{1}{N}\sum_{k=1}^{N}\expect_{\mu_{0}}\left[\left\|\nabla_{\bm{x}} M^{1\over 2\lambda}_{-\varphi(\bm{x}_{\ast};\cdot)\varphi(\widehat{\bm{x}}_{\ast};\cdot)}(\bm{\xi}_{0}^{k}) \right\|_{2}^{2}\right].
\end{align}
To proceed from \eqref{Eq:to_proceed}, we leverage the following lemma:
\begin{lemma}\textsc{(Moreau's Envelop of Parametric Functions)}
	\label{Lemma:Gradient_of_M}
	Consider the parametric function $f:\mathcal{X}\times\Theta\rightarrow \real$  and the associated Moreau's envelope for a given $\bm{\theta}\in \Theta\subset \real^{d}$:
	\begin{align}
	M_{f(\cdot;\bm{\theta})}^{\beta}(\bm{x})=\inf_{\bm{y}\in \mathcal{X}}\left\{\dfrac{1}{2\beta}\|\bm{x}-\bm{y}\|_{2}^{2}+f(\bm{y};\bm{\theta})\right\}.
	\end{align}	
	Furthermore, define the proximal operator as follows
	\begin{align}
	\mathrm{Prox}_{f(\cdot;\bm{\theta})}^{\beta}(\bm{x})=\arg\inf_{\bm{y}\in \mathcal{X}}\left\{\dfrac{1}{2\beta}\|\bm{x}-\bm{y}\|_{2}^{2}+f(\bm{y};\bm{\theta})\right\}.
	\end{align}
	Then, Moreau's envelope has the following upper bound
	\begin{align}
	\label{Eq:cruelty}
	\left\|\nabla_{\bm{\theta}} M_{f(\cdot;\bm{\theta})}^{\beta}(\bm{x})\right\|_{2}\leq \left\|\nabla_{\bm{\theta}}f\left(\mathrm{Prox}^{\beta}_{f(\cdot;\bm{\theta})}(\bm{x});\bm{\theta}\right)\right\|_{2}.
	\end{align}
\end{lemma}
The proof is presented in Appendix \ref{Proof_of_Lemma:}.

Equipped with Inequality \eqref{Eq:cruelty} of Lemma \ref{Lemma:Gradient_of_M}, we now compute an upper bound on the right hand side of Eq. \eqref{Eq:to_proceed} as follows
\begin{align}
\expect_{\mu_{0}}\big[L_{T}^{2}\big]\leq \dfrac{1}{N}\sum_{k=1}^{N}\expect_{\mu_{0}}[|\varphi(\widehat{\bm{x}}_{\ast};\bm{\xi}) |^{2}\cdot\|\nabla_{\bm{x}}\varphi(\bm{x}_{\ast};\bm{\xi}_{0}^{k}) \|_{2}^{2}]\leq L^{4},
\end{align}
where the last inequality is due to $(\mathbf{A.2})$.

Invoking Markov's inequality now yields
\begin{align}
\nonumber
\prob\left(|T_{(\bm{x},\widehat{\bm{x}})}^{\lambda}(\bm{\xi}_{0}^{1},\cdots,\bm{\xi}_{0}^{N})- T^{\lambda}_{(\bm{z}_{i},\bm{z}_{j})}(\bm{\xi}_{0}^{1},\cdots,\bm{\xi}_{0}^{N})|\geq \delta \right)&=\prob\left( L_{T}\geq {\delta\over 2\epsilon} \right)\\ \nonumber
&\leq \left(\dfrac{2\epsilon}{\delta}\right)^{2}\expect_{\mu_{0}}[L_{T}^{2}]\\ 
&\leq \left(\dfrac{2\epsilon}{\delta}\right)^{2}L^{4}.
\end{align}
Now, using the union bound, for every arbitrary pair of data points $(\bm{x},\widehat{\bm{x}})\in \mathcal{X}\times \mathcal{X}$ the following inequality holds
\begin{align}
\nonumber
\prob\left(|T_{(\bm{x},\widehat{\bm{x}})}^{\lambda}(\bm{\xi}_{0}^{1},\cdots,\bm{\xi}_{0}^{N})|\geq \delta \right) \nonumber&\leq \prob\left(|T^{\lambda}_{(\bm{z}_{i},\bm{z}_{j})}(\bm{\xi}_{0}^{1},\cdots,\bm{\xi}_{0}^{N})|\geq \delta/2 \right)\\
&\hspace{4mm}+\prob\left(|T_{(\bm{x},\widehat{\bm{x}})}^{\lambda}(\bm{\xi}_{0}^{1},\cdots,\bm{\xi}_{0}^{N})- T^{\lambda}_{(\bm{z}_{i},\bm{z}_{j})}(\bm{\xi}_{0}^{1},\cdots,\bm{\xi}_{0}^{N})|\geq \delta/2 \right)\\ \nonumber
&\leq \left(\dfrac{2\epsilon}{\delta}\right)^{2}L^{4}+\left(\dfrac{4\mathrm{diam}(\mathcal{X})}{\epsilon} \right)^{d}\cdot\exp\left(-{N\delta^{2}\over L^{2}}\right).
\end{align}
Following the proposal of \cite{rahimi2008random}, we choose $\epsilon=(\kappa_{1}/\kappa_{2})^{1\over {d+2}}$, where $\kappa_{1}\df (4\mathrm{diam}(\mathcal{X}))^{d}\cdot e^{-{2N\lambda\delta^{2}\over 2L^{2}\lambda+L^{4}}}$ and $\kappa_{2}\df (2/\delta)^{2}L^{4}$. Then,
\begin{align}
\nonumber
\prob\left(\sup_{\bm{x},\widehat{\bm{x}}\in \mathcal{X}}|T_{(\bm{x},\widehat{\bm{x}})}^{\lambda}(\bm{\xi}_{0}^{1},\cdots,\bm{\xi}_{0}^{N})|\geq \delta \right) \nonumber&\leq  2^{8}\left(\dfrac{L^{2}\mathrm{diam}(\mathcal{X})}{\delta} \right)^{2}\cdot\exp\left(-{N\delta^{2}\over L^{2}(d+2)}\right).
\end{align}
Thus, with the probability of at least $1-\varrho$, the following inequality holds
\begin{align}
\sup_{\lambda\geq 0}\sup_{\bm{x},\widehat{\bm{x}}\in \mathcal{X}}|T_{(\bm{x},\widehat{\bm{x}})}^{\lambda}(\bm{\xi}_{0}^{1},\cdots,\bm{\xi}_{0}^{N})|\leq \left(\dfrac{L^{2}(d+2)}{N}\mathrm{W}\left(\dfrac{2^{8}N\mathrm{diam}^{2}(\mathcal{X})}{\varrho}\right)\right)^{1\over 2},
\end{align}
where $\mathrm{W}(\cdot)$ is the Lambert $W$-function.\footnote{Recall that the lambert $W$-function is the inverse of the function $f(W)=We^{W}$.} Since $\mathrm{W}(x)\leq \ln(x)$ for $x>e$, we can rewrite the upper bound in terms of elementary functions
\begin{align}
\label{Eq:ddoops}
\sup_{\lambda\geq 0}\sup_{\bm{x},\widehat{\bm{x}}\in \mathcal{X}}|T_{(\bm{x},\widehat{\bm{x}})}^{\lambda}(\bm{\xi}_{0}^{1},\cdots,\bm{\xi}_{0}^{N})|\leq \sqrt{\dfrac{L^{2}(d+2)}{N}}\ln^{1\over 2}\left(\dfrac{2^{8}N\mathrm{diam}^{2}(\mathcal{X})}{\varrho}\right), 
\end{align}
provided that $N$ is sufficiently large and/or $\varrho$ is sufficiently small so that ${2^{8}N\mathrm{diam}^{2}(\mathcal{X})\over \varrho}\geq e$. Plugging Inequality  \eqref{Eq:ddoops} in Eq. \eqref{Eq:Trick_2} now results in the following inequality
\small{\begin{align}
	\Big|\sup_{\mu\in \mathcal{P}}\widehat{\mathrm{MMD}}_{\mu}[P_{\bm{V}},P_{\bm{W}}]&-\sup_{\widehat{\mu}^{N}\in \mathcal{P}_{N}} \widehat{\mathrm{MMD}}_{\widehat{\mu}^{N}}[P_{\bm{V}},P_{\bm{W}}]\Big|\leq \sqrt{\dfrac{L^{2}(d+2)}{N}}\ln^{1\over 2}\left(\dfrac{2^{8}N\mathrm{diam}^{2}(\mathcal{X})}{\varrho}\right),
	\end{align}}\normalsize
for all $\bm{W}\in \mathcal{W}$. Employing\eqref{Eq:Basic_Inequality} from Lemma \ref{Lemma:A_Basic_Inequality} now yields the following upper bound
\begin{align}
\nonumber
\mathsf{A}_{2}&= \Big|\min_{\bm{W}\in \mathcal{W}}\sup_{\mu\in \mathcal{P}}\widehat{\mathrm{MMD}}_{\mu}[P_{\bm{V}},P_{\bm{W}}]-\min_{\bm{W}\in \mathcal{W}}\sup_{\widehat{\mu}^{N}\in \mathcal{P}_{N}}\widehat{\mathrm{MMD}}_{\widehat{\mu}^{N}}\Big[P_{\bm{V}},P_{\bm{W}}\Big] \Big|\\
&\leq \sqrt{\dfrac{L^{2}(d+2)}{N}}\ln^{1\over 2}\left(\dfrac{2^{8}N\mathrm{diam}^{2}(\mathcal{X})}{\varrho}\right).
\end{align}

\vspace*{5mm}
\noindent\textbf{Upper bound on $\mathsf{A}_{3}$}:

Recall that the solution of the empirical risk function of Eq.  \eqref{Eq:Empirical_Objective_Function} is denoted by
\begin{align}
\nonumber
(\widehat{\bm{W}}_{\ast}^{N},\widehat{\mu}^{N}_{\ast})&\df \arg\min_{\bm{W}\in \mathcal{W}} \arg\inf_{\widehat{\mu}^{N}\in \mathcal{P}_{N}} \widehat{\mathrm{MMD}}_{\widehat{\mu}^{N}}^{\alpha}[P_{\bm{V}},P_{\bm{W}}]\\ \nonumber
&=\arg\min_{\bm{W}\in \mathcal{W}} \arg\sup_{\widehat{\mu}^{N}\in \mathcal{P}_{N}}\dfrac{8}{n(n-1)} \sum_{1\leq i<j\leq n}y_{i}y_{j}\expect_{\widehat{\mu}^{N}\in \mathcal{P}_{N}}[\varphi(\bm{x}_{i};\bm{\xi})\varphi(\bm{x}_{j};\bm{\xi})]\\ \label{Eq:optimization111}
&\hspace{30mm}-\dfrac{8}{n(n-1)\alpha} \sum_{1\leq i<j\leq n} (\expect_{\widehat{\mu}^{N}}[\varphi(\bm{x}_{i};\bm{\xi})\varphi(\bm{x}_{j};\bm{\xi})])^{2}.
\end{align}
We also define the solution of the empirical kernel alignment as follows
\begin{align}
\nonumber
(\widehat{\bm{W}}^{N}_{\diamond},\widehat{\mu}^{N}_{\diamond})&\df  \arg\min_{\bm{W}\in \mathcal{W}} \arg\sup_{\widehat{\mu}^{N}\in \mathcal{P}_{N}} \widehat{\mathrm{MMD}}_{\widehat{\mu}^{N}}[P_{\bm{V}},P_{\bm{W}}]\\ \label{Eq:Optimization222}
&=\dfrac{8}{n(n-1)}\sum_{1\leq i<j\leq n}y_{i}y_{j}\expect_{\widehat{\mu}^{N}}[\varphi(\bm{x}_{i};\bm{\xi})\varphi(\bm{x}_{j};\bm{\xi})].
\end{align}
Due to the optimality of the empirical measure $\widehat{\mu}_{\ast}^{N}$ for the inner optimization in Eq. \eqref{Eq:optimization111}, the following inequality holds
\begin{align}
\nonumber
\widehat{\mathrm{MMD}}_{\widehat{\mu}_{\diamond}^{N}}^{\alpha}\big[P_{\bm{V}},P_{\widehat{\bm{W}}^{N}_{\ast}}\big] &\leq \widehat{\mathrm{MMD}}_{\widehat{\mu}_{\ast}^{N}}^{\alpha}[P_{\bm{V}},P_{\widehat{\bm{W}}^{N}_{\ast}}]
\\  \label{Eq:rearrange}
&\leq  \dfrac{8}{n(n-1)} \sum_{1\leq i<j\leq n}y_{i}y_{j}\expect_{\widehat{\mu}_{\ast}^{N}}[\varphi(\bm{x}_{i};\bm{\xi})\varphi(\bm{x}_{j};\bm{\xi})].
\end{align}
Upon expansion of $\widehat{\mathrm{MMD}}_{\widehat{\mu}_{\diamond}^{N}}^{\alpha}\big[P_{\bm{V}},P_{\widehat{\bm{W}}^{N}_{\ast}}\big]$, and after rearranging the terms in Eq. \eqref{Eq:rearrange}, we arrive at
\begin{align}
\nonumber
\widehat{\mathrm{MMD}}_{\widehat{\mu}_{\diamond}^{N}}&\big[P_{\bm{V}},P_{\widehat{\bm{W}}^{N}_{\ast}}\big]-\widehat{\mathrm{MMD}}_{\widehat{\mu}_{\ast}^{N}}\big[P_{\bm{V}},P_{\widehat{\bm{W}}^{N}_{\ast}}\big]\\  \nonumber
&=\dfrac{8}{n(n-1)} \sum_{1\leq i<j\leq n}y_{i}y_{j}(\expect_{\widehat{\mu}_{\diamond}^{N}}[\varphi(\bm{x}_{i};\bm{\xi})\varphi(\bm{x}_{j};\bm{\xi})]-\expect_{\widehat{\mu}_{\ast}^{N}}[\varphi(\bm{x}_{i};\bm{\xi})\varphi(\bm{x}_{j};\bm{\xi})])\\ \nonumber
&\leq  \dfrac{8}{n(n-1)\alpha} \sum_{1\leq i<j\leq n} \left(\expect_{\widehat{\mu}_{\diamond}^{N}}[\varphi(\bm{x}_{i};\bm{\xi})\varphi(\bm{x}_{j};\bm{\xi})]\right)^{2}\\  \label{Eq:Combine_1}
&\leq \dfrac{8L^{4}}{\alpha}, 
\end{align}
where the last inequality is due to the fact that $\|\varphi\|_{\infty}<L$ by $\textbf{(A.1)}$. Now, due to optimality of $\widehat{\bm{W}}^{N}_{\diamond}$ for the outer optimization problem in Eq. \eqref{Eq:Optimization222}, we have 
\begin{align}
\label{Eq:Combine_11}
\widehat{\mathrm{MMD}}_{\widehat{\mu}_{\diamond}^{N}}&\big[P_{\bm{V}},P_{\widehat{\bm{W}}^{N}_{\diamond}}\big]\leq \widehat{\mathrm{MMD}}_{\widehat{\mu}_{\diamond}^{N}}\big[P_{\bm{V}},P_{\widehat{\bm{W}}^{N}_{\ast}}\big].
\end{align}
Putting together Inequalities \eqref{Eq:Combine_1} and \eqref{Eq:Combine_11} yields
\begin{align}
\label{Eq:combine_MMD_1}
\widehat{\mathrm{MMD}}_{\widehat{\mu}_{\diamond}^{N}}&\big[P_{\bm{V}},P_{\widehat{\bm{W}}^{N}_{\diamond}}\big]-\widehat{\mathrm{MMD}}_{\widehat{\mu}_{\ast}^{N}}\big[P_{\bm{V}},P_{\widehat{\bm{W}}^{N}_{\ast}}\big]\leq \dfrac{8L^{4}}{\alpha}.
\end{align}
Similarly, due to the optimality of the empirical measure $\widehat{\mu}^{N}_{\diamond}$ for the optimization in Eq. \eqref{Eq:Optimization222} we have that
\begin{align}
\nonumber
\widehat{\mathrm{MMD}}_{\widehat{\mu}_{\ast}^{N}}\Big[P_{\bm{V}},P_{\widehat{\bm{W}}^{N}_{\ast}}\Big]&\leq \widehat{\mathrm{MMD}}_{\widehat{\mu}_{\ast}^{N}}\Big[P_{\bm{V}},P_{\widehat{\bm{W}}^{N}_{\diamond}}\Big]\\
\label{Eq:combine_MMD_2}
&\leq \widehat{\mathrm{MMD}}_{\widehat{\mu}_{\diamond}^{N}}\Big[P_{\bm{V}},P_{\widehat{\bm{W}}^{N}_{\diamond}}\Big].
\end{align}
Combining Eqs. \eqref{Eq:combine_MMD_1} and \eqref{Eq:combine_MMD_2} then yields
\begin{align}
\label{Eq:S1}
\mathsf{A}_{3}=\Big|\widehat{\mathrm{MMD}}_{\widehat{\mu}_{\ast}^{N}}\big[P_{\bm{V}},P_{\widehat{\bm{W}}^{N}_{\ast}}\big]-\widehat{\mathrm{MMD}}_{\widehat{\mu}_{\diamond}^{N}}\big[P_{\bm{V}},P_{\widehat{\bm{W}}^{N}_{\diamond}}\big]\Big|\leq \dfrac{8L^{4}}{\alpha}.
\end{align}

\vspace*{5mm}
\noindent\textbf{Upper bound on $\mathsf{A}_{4}$}: 

The upper bound on $\mathsf{A}_{4}$ can be obtained exactly the same way as $\mathsf{A}_{1}$. Ideed, from Eq. \eqref{Eq:varrho} it follows directly that
\begin{align}
\nonumber
\mathsf{A}_{4}&=\Big|\widehat{\mathrm{MMD}}_{\widehat{\mu}_{\ast}^{N}}\Big[P_{\bm{V}},P_{\widehat{\bm{W}}_{\ast}^{N}}\Big]-\mathrm{MMD}_{\widehat{\mu}^{N}_{\ast}}\Big[P_{\bm{V}},P_{\widehat{\bm{W}}^{N}_{\ast}}\Big]\Big|\\ \label{Eq:Bound_on_A4}
&\leq \sup_{\widehat{\mu}^{N}\in \mathcal{P}_{N}} \Big|\widehat{\mathrm{MMD}}_{\widehat{\mu}^{N}}\Big[P_{\bm{V}},P_{\widehat{\bm{W}}_{\ast}^{N}}\Big]-\mathrm{MMD}_{\widehat{\mu}^{N}}\Big[P_{\bm{V}},P_{\widehat{\bm{W}}^{N}_{\ast}}\Big]\Big|\\
&\leq   \min\left\{\dfrac{3^{1\over 4}\times 2^{11\over 2}\times L^{2}}{n} \ln^{1\over 2}\left(\dfrac{4}{\varrho}\right),\dfrac{9\times 2^{12}\times RL^{4} }{n^{2}}\ln\left(\dfrac{4e^{L^{4}\over 9}}{\varrho}\right) \right\}.
\end{align}

Now, plugging the derived upper bounds in $\mathsf{A}_{1}$-$\mathsf{A}_{4}$ in Eq. \eqref{Eq:A1-A4} and employing the union bound completes the proof.

$\hfill \square$
\subsection{Proof of Theorem \ref{Theorem:Density Evolution}}
\label{Appendix:Proof_of_Thm_McKean}

The proof has three main ingredients and follows the standard procedure in the literature, see, \textit{e.g.}, \cite{wang2017scaling,luo2017scaling}. In the first step, we identify the mean-field limit of the particle SGD in Eq. \eqref{Eq:SGD}.  In the second step, we prove the convergence of the measured-valued process $\{(\mu^{N}_{t})_{0\leq t\leq T}\}$ to the mean-field solution by establishing the pre-compactness of Sokhorhod space. Lastly, we prove the uniqueness of the mean-field solution of the particle SGD.

\textbf{Step 1-Identification of the scaling limit}: First, we identify the weak limit of converging sub-sequences via the action of the empirical measure $\widehat{\mu}_{m}^{N}(\bm{\xi})={1\over N}\sum_{k=1}^{N}\delta(\bm{\xi}-\bm{\xi}^{k}_{m})$ on a test function $f\in C_{b}^{3}(\real^{d}_{0})$. In particular, we use the standard techniques of computing the scaling limits from \cite{luo2017scaling}.

Recall that the action of an empirical measure on a bounded function is defined as follows
\begin{align}
\langle f,\widehat{\mu}_{m}^{N} \rangle\df \dfrac{1}{N}\sum_{k=1}^{N}f(\bm{\xi}^{k}_{m}).
\end{align}
We analyze the evolution of the empirical measure $\widehat{\mu}_{m}^{N}$ via its action on a test function $f \in C_{b}^{3}(\real^{p})$. Using Taylor's expansion, we obtain
\begin{align}
\nonumber
\langle f,\widehat{\mu}_{m+1}^{N} \rangle- \langle f,\widehat{\mu}_{m}^{N} \rangle
&=\langle f,\widehat{\mu}_{m+1}^{N} \rangle- \langle f,\widehat{\nu}_{m+1}^{N} \rangle \\    \nonumber
&=\dfrac{1}{N}\sum_{k=1}^{N}f(\bm{\xi}^{k}_{m+1})-f(\bm{\xi}^{k}_{m})\\ \nonumber
&=\dfrac{1}{N}\sum_{k=1}^{N}\nabla f(\bm{\xi}^{k}_{m})(\bm{\xi}^{k}_{m+1}-\bm{\xi}^{k}_{m})^{T}+R^{N}_{m}.
\end{align}
where $R_{m}^{N}$ is a remainder term defined as follows
\begin{align}
\label{Eq:Remainder_1}
R^{N}_{m}\df\dfrac{1}{N}\sum_{k=1}^{N}(\bm{\xi}_{m+1}^{k}-\bm{\xi}_{m}^{k})^{T}\nabla^{2}f(\widetilde{\bm{\xi}}^{k})(\bm{\xi}_{m+1}^{k}-\bm{\xi}_{m}^{k}),
\end{align}
where $\widetilde{\bm{\xi}}^{k}\df (\widetilde{\xi}^{k}(1),\cdots,\widetilde{\xi}^{k}(p))$, and $\widetilde{\xi}^{k}(i)\in [\xi_{m}^{k}(i),\xi_{m+1}^{k}(i)]$, for $i=1,2,\cdots,p$.

Plugging the difference term $(\bm{\xi}_{m+1}^{k}-\bm{\xi}_{m}^{k})$ from the SGD equation in Eq. \eqref{Eq:SGD} results in
\small{\begin{align}
	\label{Eq:Recast_Equation}
	&\langle f,\widehat{\mu}_{m+1}^{N} \rangle- \langle f,\widehat{\mu}_{m}^{N} \rangle\\ \nonumber
	&=\dfrac{\eta}{N^{2}\alpha}\sum_{k=1}^{N}\nabla f(\bm{\xi}_{m}^{k})\cdot \left( \left({1\over N}\sum_{\ell=1}^{N}\varphi(\bm{x}_{m};\bm{\xi}_{m}^{\ell})\varphi(\widetilde{\bm{x}}_{m};\bm{\xi}_{m}^{\ell})-\alpha y_{m}\widetilde{y}_{m}\right)\nabla_{\bm{\xi}}\Big(\varphi(\bm{x}_{m};\bm{\xi}_{m}^{k})\varphi(\widetilde{\bm{x}}_{m};\bm{\xi}_{m}^{k})\Big)\right)+R^{N}_{m}.
	\end{align}}\normalsize
Now, we define the drift and Martingale terms as follows

\begin{subequations}
	{\small\begin{align}
		\label{Eq:D}
		D_{m}^{N}&\df \dfrac{\eta}{N\alpha}\iint_{\mathcal{X}\times \mathcal{Y}} \left( \langle\varphi(\bm{x},\bm{\xi})\varphi(\widetilde{\bm{x}},\bm{\xi}), \widehat{\mu}^{N}_{m} \rangle-\alpha y \widetilde{y} \right)\\ \nonumber
		&\hspace{10mm}\times \langle \nabla f(\bm{\xi})(\varphi(\widetilde{\bm{x}};\bm{\xi})\nabla_{\bm{\xi}}\varphi(\bm{x};\bm{\xi})+\varphi(\bm{x};\bm{\xi})\nabla_{\bm{\xi}}\varphi(\widetilde{\bm{x}};\bm{\xi})),\widehat{\mu}^{N}_{m}\rangle \mathrm{d} P_{\bm{x},y}^{\otimes 2}((\bm{x},y),(\widetilde{\bm{x}},\widetilde{y}))\\
		\label{Eq:M}
		M_{m}^{N}&\df {\eta \over N\alpha}\left( \langle\varphi(\bm{x}_{m},\bm{\xi})\varphi(\widetilde{\bm{x}}_{m},\bm{\xi}), \widehat{\mu}^{N}_{m} \rangle-\alpha y_{m} \widetilde{y}_{m} \right)\\  \nonumber
		&\hspace{10mm}\times \langle \nabla f(\bm{\xi})(\varphi(\widetilde{\bm{x}}_{m};\bm{\xi})\nabla_{\bm{\xi}}\varphi(\bm{x}_{m};\bm{\xi})+\varphi(\widetilde{\bm{x}}_{m};\bm{\xi})\nabla_{\bm{\xi}}\varphi(\bm{x}_{m};\bm{\xi})),\widehat{\mu}^{N}_{m}\rangle-\mathcal{D}_{m}^{N}.
		\end{align}}\normalsize
\end{subequations}
respectively. Using the definitions of $D_{m}^{N}$ and $M_{m}^{N}$ in Eqs. \eqref{Eq:D}-\eqref{Eq:M}, we recast Equation \eqref{Eq:Recast_Equation} as follows
\begin{align}
\label{Eq:difference_term}
\langle f,\widehat{\mu}^{N}_{m+1} \rangle-\langle f,\widehat{\mu}^{N}_{m} \rangle=D_{m}^{N}+M_{m}^{N}+R_{m}^{N}.
\end{align}
Summation over $\ell=0,1,2\cdots,m-1$ and using the telescopic sum yields
\begin{align}
\label{Eq:telescopic_sum}
\langle f,\widehat{\mu}_{m}^{N}\rangle-\langle f,\widehat{\mu}_{0}^{N}\rangle=\sum_{\ell=0}^{m-1}D_{\ell}^{N}+\sum_{\ell=0}^{m-1} M_{\ell}^{N}+\sum_{\ell=0}^{m-1}R_{\ell}^{N}.
\end{align}
We also define the following continuous embedding of the drift, martingale, and the remainder terms as follows
\begin{subequations}
	\begin{align}
	\hspace{15mm} \mathcal{D}^{N}_{t}&\df \sum_{\ell=0}^{\lfloor Nt \rfloor}D_{\ell}^{N} \\ \label{Eq:M1}
	\mathcal{M}^{N}_{t}&\df \sum_{\ell=0}^{\lfloor Nt \rfloor}M_{\ell}^{N}\\ \label{Eq:Remainder_2}
	\mathcal{R}^{N}_{t}&\df \sum_{\ell=0}^{\lfloor Nt\rfloor}R_{\ell}^{N}, \quad t\in (0,T].
	\end{align}
\end{subequations}
The scaled empirical measure $\mu_{t}^{N}\df \widehat{\mu}^{N}_{\lfloor Nt \rfloor}$ then can be written as follows
\begin{align}
\label{Eq:Expression}
\langle f,\mu^{N}_{t} \rangle-\langle f,\mu^{N}_{0} \rangle&=\mathcal{D}^{N}_{t}+\mathcal{M}^{N}_{t}+\mathcal{R}^{N}_{t}.
\end{align}
Since the drift process $(D^{N}_{t})_{0\leq t\leq T}$ is a piecewise c\'{a}dl\'{a}g process, we have
\begin{align}
D_{\ell}^{N}=\int_{{\ell \over N}}^{{\ell+1\over N}}R[\mu_{s}]\mathrm{d}s,
\end{align}
where the functional $R[\mu_{s}]$ is defined as follows 
\begin{align}
\label{Eq:Functional_Rs}
R[\mu_{s}]&\df {\eta \over \alpha}\iint_{\mathcal{X}\times \mathcal{Y}} \left( \langle\varphi(\bm{x},\bm{\xi})\varphi(\widetilde{\bm{x}},\bm{\xi}),\mu_{s} \rangle-\alpha y \widetilde{y} \right)\\ \nonumber
&\hspace{10mm}\times \langle \nabla f(\bm{\xi})(\varphi(\widetilde{\bm{x}};\bm{\xi})\nabla_{\bm{\xi}}\varphi(\bm{x};\bm{\xi})+\varphi(\bm{x};\bm{\xi})\nabla_{\bm{\xi}}\varphi(\widetilde{\bm{x}};\bm{\xi}))^{T},\mu_{s}\rangle P_{\bm{x},y}^{\otimes 2}((\mathrm{d}\bm{x},\mathrm{d}\widetilde{\bm{x}}),(\mathrm{d}y,\mathrm{d}\widetilde{y})).
\end{align}
Therefore, the expression in Eq. \eqref{Eq:Expression} can be rewritten as follows
\begin{align}
\label{Eq:Yield}
\langle f,\mu^{N}_{t} \rangle-\langle f,\mu^{N}_{0} \rangle&=\int_{0}^{t}R[\mu_{s}]\mathrm{d}s +\mathcal{M}^{N}_{t}+\mathcal{R}^{N}_{t}.
\end{align}

In the following lemma, we prove that the remainder term $\sup_{0\leq t\leq T}|\mathcal{R}^{N}_{t}|$ vanishes in probabilistic sense as the number of particles tends to infinity $N\rightarrow \infty$:
\begin{lemma}\textsc{(Large $N$-Limit of the Remainder Process)}
	\label{Lemma:Remainder}
	Consider the remainder process $(\mathcal{R}^{N}_{t})_{0\leq t\leq T}$ defined via scaling in Eqs. \eqref{Eq:Remainder_1}-\eqref{Eq:Remainder_2}. Then, there exists a constant $C_{0}>0$ such that
	\begin{align}
	\label{Eq:Cheater_1}
	\sup_{0\leq t\leq T}|\mathcal{R}_{t}^{N}|\leq \dfrac{C_{0}T}{N}\left(\eta L^{2}+{2\eta L^{4}\over \alpha}\right).
	\end{align}
	and thus $\lim \sup_{N\rightarrow \infty}\sup_{0\leq t\leq T}|\mathcal{R}_{t}^{N}|=0$ almost surely.
\end{lemma}
\begin{proof}
	The proof is relegated to Appendix \ref{Proof_of_Lemma:Remainder}.
\end{proof}  
We can also prove a similar result for the process defined by the remainder term: 
\begin{lemma}\textsc{(Large $N$-Limit of the Martingale Process)}
	\label{Lemma:Martingale}
	Consider the Martingale process $(\mathcal{M}^{N}_{t})_{0\leq t\leq T}$ defined via scaling in Eqs. \eqref{Eq:M}-\eqref{Eq:M1}. Then, for some constant $C_{1}>0$, the following inequality holds
	\begin{align}
	\label{Eq:Cheater_2}
	\prob\left(\sup_{0\leq t\leq T}|\mathcal{M}_{t}^{N}|\geq \varepsilon \right)\leq  \dfrac{1}{N\alpha \varepsilon}4\sqrt{2}L^{2}\sqrt{\lfloor NT \rfloor}\eta C_{1}(L^{2}+\alpha)^{2}.
	\end{align}
	In particular, with the probability of at least $1-\rho$, we have
	\begin{align}
	\sup_{0\leq t\leq T}|\mathcal{M}_{t}^{N}|\leq  \dfrac{1}{N\alpha \rho}4\sqrt{2}L^{2}\sqrt{\lfloor NT \rfloor}\eta C_{1}(L^{2}+\alpha)^{2}.
	\end{align}
	and thus $\lim \sup_{N\rightarrow \infty}\sup_{0\leq t\leq T}|\mathcal{M}_{t}^{N}|=0$ almost surely.
\end{lemma}
\begin{proof}
	The proof is deferred to Appendix \ref{Proof_of_Lemma:Martingale}.
\end{proof}

Now, using the results of Lemmata \ref{Lemma:Martingale} and \ref{Lemma:Remainder} in conjunction with Eq. \eqref{Eq:Yield} yields the following mean-field equation as $N\rightarrow \infty$,
\begin{align}
\label{Eq:Mean_Field_Equation}
\langle \mu_{t},f\rangle&=\langle \mu_{0},f \rangle+{\eta\over \alpha} \int_{0}^{t}\Bigg(\iint_{\mathcal{X}\times \mathcal{Y}} \left( \langle\varphi(\bm{x},\bm{\xi})\varphi(\widetilde{\bm{x}},\bm{\xi}),\mu_{s} \rangle-\alpha y \widetilde{y} \right)\\ \nonumber
&\hspace{4mm}\times \langle \nabla f(\bm{\xi})(\varphi(\widetilde{\bm{x}};\bm{\xi})\nabla_{\bm{\xi}}\varphi(\bm{x};\bm{\xi})+\varphi(\bm{x};\bm{\xi})\nabla_{\bm{\xi}}\varphi(\widetilde{\bm{x}};\bm{\xi})),\mu_{s}\rangle P_{\bm{x},y}^{\otimes 2}((\mathrm{d}\bm{x},\mathrm{d}\widetilde{\bm{x}}),(\mathrm{d}y,\mathrm{d}\widetilde{y}))\Bigg)\mathrm{d}s.
\end{align}
Notice that he mean-field equation in Eq. \eqref{Eq:Mean_Field_Equation} is in the weak form. When the Lebesgue density $p_{t}(\bm{\xi})=\mathrm{d}\mu_{t}/\mathrm{d}\bm{\xi}$ exists, the McKean-Vlasov PDE in Eq. \eqref{Eq:McKean-Vlasov} can be readily obtained from Eq. \eqref{Eq:Mean_Field_Equation}.

\textbf{Step 2: Pre-compactness of the Skorkhod space}:  To establish our results in this part of the proof, we need a definition and a theorem:

\begin{definition}\textsc{(Tightness)}
	A set $\mathcal{A}$ of probability measures on a metric space $\mathcal{S}$ is tight if there exists a compact subset $\mathcal{S}_{0}\subset \mathcal{S}$ such that
	\begin{align}
	\nu(\mathcal{S}_{0})\geq 1-\varepsilon, \quad \text{for all}\ \nu\in \mathcal{A},
	\end{align}
	for all $\varepsilon >0$. A sequence $\{X^{N}\}_{N\in \integer}$ of random elements of the metric space $\mathcal{S}$ is tight if there exists a compact subset $\mathcal{S}_{0}\subset \mathcal{S}$ such that
	\begin{align}
	\nu(X^{N}\in \mathcal{S}_{0})>1-\varepsilon,
	\end{align}
	for all $\varepsilon>0$, and all $N\in \integer$.
\end{definition}

Now, to show the tightness of the measured valued process $(\mu^{N}_{t})_{0\leq t\leq T}$, we must verify Jakubowski's criterion \cite[Thm. 1]{jakubowski1986skorokhod}:

\begin{theorem} \textsc{(Jakubowski's criterion \cite[Thm. 1]{jakubowski1986skorokhod})}
	\label{Thm:Jack}
	A sequence of measured-valued process $\{(\zeta_{t}^{N})_{0\leq t\leq T}\}_{N\in \integer}$ is tight in $\mathcal{D}_{\mathcal{M}(\real^{D})}([0,T])$ if and only if the following two conditions are satisfied:	
	\begin{itemize}
		\item[$\mathbf{(J.1)}$] For each $T>0$ and $\gamma>0$, there exists a compact set $\mathcal{U}_{T,\gamma}$ such that
		\begin{align}
		\lim_{N\rightarrow \infty} \inf \prob\left(\zeta_{t}^{N}\in \mathcal{U}_{T,\gamma},\forall t\in (0,T] \right)>1-\gamma.
		\end{align}
		This condition is referred to as the compact-containment condition. 
		
		\item[$\mathbf{(J.2)}$] There exists a family $\mathcal{H}$ of real-valued functions $H:\mathcal{M}(\real^{D})\mapsto \real$ that separates points in $\mathcal{M}(\real^{D})$ and is closed under addition such that for every $H\in \mathcal{H}$, the sequence $\{(H(\xi_{t}^{N}))_{0\leq t\leq T}\}_{N\in \integer}$ is tight in $\mathcal{D}_{\real}([0,T])$.
	\end{itemize}
\end{theorem}

To establish $\mathbf{(J1)}$, we closely follow the proof of \cite[Lemma 6.1.]{giesecke2013default}. In particular, for each $L>0$, we define $\mathcal{S}_{L}=[0,B]^{p}$. Then, $\mathcal{S}_{B}\subset \real^{p}$ is compact, and for each $t\geq 0$, and $N\in \integer$, we have
\begin{align}
\expect[\mu_{t}^{N}(\real^{p}\slash \mathcal{S}_{B})]&={1\over N}\sum_{k=1}^{N}\prob\left(\|\bm{\xi}^{k}_{\lfloor Nt\rfloor}\|_{2}\geq B \right)\\
&\stackrel{\rm{(a)}}{\leq} {1\over N}\sum_{k=1}^{N}\dfrac{\expect[\|\bm{\xi}^{k}_{\lfloor Nt\rfloor} \|_{2}]}{B}\\
&\stackrel{\rm{(b)}}{\leq} \dfrac{c_{0}+\eta\alpha L^{2}T+2\eta L^{4}T}{B},
\end{align}
where $\rm{(a)}$ follows from Markov's inequality, and $\rm{(b)}$ follows from the upper bound on the norm of the particles in Eq. \eqref{Eq:From_Bound_On_x} of Appendix \ref{Section:Proofs of Auxiliary Results}. We now define the following set
\begin{align}
\mathcal{U}_{B}=\left\{\mu\in \mathcal{M}(\real^p):\mu(\real^{p}\slash \mathcal{S}_{(B+j)^{2}})<\dfrac{1}{\sqrt{B+j}} \ \text{for all}\ j\in \integer \right\}.
\end{align}
We let $\mathcal{U}_{T,\gamma}=\overline{\mathcal{U}}_{B}$, where $\overline{\mathcal{U}}_{B}$ is the completion of the set $\mathcal{U}_{B}$. By definition, $\mathcal{U}_{T,\gamma}$ is a compact subset of $\mathcal{M}(\real^{D})$. Now, we have
\begin{align}
\nonumber
\prob\left(\mu_{t}^{N}\not\in \mathcal{U}_{T,\gamma}\right)&\leq \sum_{j=1}^{\infty}\prob\Bigg(\mu_{t}^{N}(\real^{p}\slash\mathcal{S}_{(B+j)^{2}}) >\dfrac{1}{\sqrt{B+j}} \Bigg)\\ 
\nonumber
&\leq  \sum_{j=1}^{\infty}\dfrac{\expect[\mu_{t}^{N}(\real^{p}\slash\mathcal{S}_{(B+j)^{2}})]}{1/\sqrt{B+j}}\\
\nonumber
&\leq \sum_{j=1}^{\infty}\dfrac{c_{0}+\eta L^{2}T+2(\eta/\alpha) L^{4}T}{(B+j)^{2}/\sqrt{B+j}}\\
&=\sum_{j=1}^{\infty}\dfrac{c_{0}+\eta L^{2}T+2(\eta/\alpha) L^{4}T}{(B+j)^{3/2}}.
\end{align}
Now, since
\begin{align}
\lim_{B\rightarrow \infty}\sum_{j=1}^{\infty}\dfrac{c_{0}+\eta L^{2}T+2(\eta/\alpha) L^{4}T}{(B+j)^{3/2}}=0,
\end{align}
this implies that for any $\gamma>0$, there exists a $B>0$, such that
\begin{align}
\lim_{N\rightarrow \infty} \inf \prob\left(\mu_{t}^{N}\in \overline{\mathcal{U}}_{B},\forall t\in (0,T] \right)>1-\gamma.
\end{align}
This completes the proof of $(\textbf{J.1})$. To verify $\textbf{(J.2)}$, we consider the following class of functions
\begin{align}
\label{Eq:define_the_class}
\mathcal{H}\df \{H: \exists f\in C_{b}^{3}(\real^{p})\  \text{such that}\ H(\mu)=\langle \mu,f\rangle, \forall \mu\in \mathcal{M}(\real^{D})\}.
\end{align}
By definition, every function $H\in \mathcal{H}$ is continuous with respect to the weak topology of $\mathcal{M}(\real^{D})$ and further the class of functions $\mathcal{H}$ separate points in $\mathcal{M}(\real^{D})$ and is closed under addition. Now, we state the following sufficient conditions to establish $\textbf{(J.2)}$. The statement of the theorem is due to \cite[Thm. C.9]{robert2013stochastic}:

\begin{theorem}\textsc{(Tightness in $\mathcal{D}_{\real}([0,T])$, \cite[Thm. C.9]{robert2013stochastic})} 
	A sequence $\{(Z^{N}_{t})_{0\leq t\leq T}\}_{N\in \integer}$ is tight in $\mathcal{D}_{\real}([0,T])$ iff for any $\delta>0$, we have
	\begin{itemize}
		\item[$\mathbf{(T.1)}$] There exists $\epsilon>0$, such that 
		\begin{align}
		\prob(|Z_{0}^{N}|>\epsilon)\leq \delta, 
		\end{align}	
		for all $N\in \integer$.
		
		\item[$\mathbf{(T.2)}$] For any $\rho>0$, there exists $\sigma>0$ such that
		\begin{align}
		\prob\left(\sup_{t_{1},t_{2}\leq T,|t_{1}-t_{2}|\leq \rho} |Z_{t_{1}}^{N}-Z_{t_{2}}^{N}|>\sigma\right)\leq \delta, 
		\end{align}
	\end{itemize}
\end{theorem}

This completes the tightness proof of the of the laws of the measured-valued process $\{(\mu^{N}_{t})_{0\leq t\leq T}\}_{N\in \integer}$. Now, we verify the condition $\textbf{(J.2)}$ by showing that the sufficient conditions $\textbf{(T.1)}$ and $\textbf{(T.2)}$ hold for function values $\{(H(\mu_{t}^{N}))_{0\leq t\leq T}\}_{N\in \integer}$, where $H\in \mathcal{H}$ and $\mathcal{H}$ is defined in Eq. \eqref{Eq:define_the_class}.  Now, condition $\mathbf{(T.1)}$ is readily verified since
\begin{align}
H(\mu^{N}_{0})&=\langle \mu_{0}^{N},f \rangle=\int_{\real^{p}}f(\bm{\xi})\mu_{0}^{N}(\mathrm{d}\bm{\xi})\\
&\leq \|f\|_{\infty} \int_{\real^{p}}\mu_{0}^{N}(\mathrm{d}\bm{\xi})\\
&\leq b,
\end{align}
where in the last step, we used the fact that $f\in C_{b}^{3}(\real^{p})$, and hence, $\|f\|_{\infty}\leq b$. Thus, $\prob(H(\mu^{N}_{0})\geq b)=0$ for all $N\in \integer$, and the condition $\mathbf{(T.1)}$ is satisfied. Now, consider the condition $\textbf{(T.2)}$. From Equation \eqref{Eq:Yield}, and with $0\leq t_{1}<t_{2}\leq T$ we have
\begin{align}
\nonumber
|H(\mu^{N}_{t_{1}})-H(\mu^{N}_{t_{1}})|&=|\langle f,\mu^{N}_{t_{1}} \rangle-\langle f,\mu^{N}_{t_{2}} \rangle |\\ \label{Eq:last_term_of}
&\leq \int_{t_{1}}^{t_{2}}\left|R[\mu_{s}]\right|\mathrm{d}s  +|\mathcal{M}_{t_{1}}^{N}-\mathcal{M}_{t_{2}}^{N}|+|\mathcal{R}_{t_{1}}^{N}-\mathcal{R}_{t_{2}}^{N}|.
\end{align}

To bound the first term, recall the definition of $R[\mu_{s}]$ from Eq. \eqref{Eq:Functional_Rs}. The following chain of inequalities holds,
\small{\begin{align}
	\nonumber
	|R[\mu_{s}]|
	&\leq \dfrac{\eta}{\alpha}\expect_{P_{\bm{x},y}^{\otimes 2}}[ |\langle \varphi(\bm{x},\bm{\xi}) \varphi(\widetilde{\bm{x}},\bm{\xi}),\mu_{s}\rangle-\alpha y \widetilde{y}| |\langle \nabla f(\bm{\xi})(\varphi(\widetilde{\bm{x}};\bm{\xi})\nabla_{\bm{\xi}}\varphi(\bm{x};\bm{\xi})+\varphi(\bm{x};\bm{\xi})\nabla_{\bm{\xi}}\varphi(\widetilde{\bm{x}};\bm{\xi}))^{T},\mu_{s}\rangle|  ]\\ \label{Eq:From_Proceed}
	&\leq {\eta\over \alpha} \expect_{P_{\bm{x}}^{\otimes 2}}[ (|\langle \varphi(\bm{x},\bm{\xi}) \varphi(\widetilde{\bm{x}},\bm{\xi}),\mu_{s}\rangle|+\alpha ) |\langle \nabla f(\bm{\xi})(\varphi(\widetilde{\bm{x}};\bm{\xi})\nabla_{\bm{\xi}}\varphi(\bm{x};\bm{\xi})+\varphi(\bm{x};\bm{\xi})\nabla_{\bm{\xi}}\varphi(\widetilde{\bm{x}};\bm{\xi}))^{T},\mu_{s}\rangle| ].
	\end{align}}\normalsize
Let $I:\real^{p}\rightarrow \real, I(\bm{\xi})=1$ denotes the identity function. Notice that $\langle I,\mu_{s} \rangle=\int_{\real^{p}}\mu_{s}(\mathrm{d}s)=1$. From \eqref{Eq:From_Proceed}, we proceed as follows
\small{\begin{align}
	\nonumber
	|R[\mu_{s}]|
	&\leq {\eta\over \alpha} \expect_{P_{\bm{X}}^{\otimes 2}}[ (\|\varphi\|^{2}_{\infty}\cdot |\langle I ,\mu_{s}\rangle|+\alpha ) \cdot \|\nabla f(\bm{\xi})(\varphi(\widetilde{\bm{x}};\bm{\xi})\nabla_{\bm{\xi}}\varphi(\bm{x};\bm{\xi})+\varphi(\bm{x};\bm{\xi})\nabla_{\bm{\xi}}\varphi(\widetilde{\bm{x}};\bm{\xi}))^{T}\|_{\infty}\cdot|\langle I,\mu_{s}\rangle| ]\\ \nonumber
	&\leq {\eta\over \alpha} \expect_{P_{\bm{X}}^{\otimes 2}}[ (\|\varphi\|^{2}_{\infty}+\alpha ) \cdot \|\nabla f(\bm{\xi})(\varphi(\widetilde{\bm{x}};\bm{\xi})\nabla_{\bm{\xi}}\varphi(\bm{x};\bm{\xi})+\varphi(\bm{x};\bm{\xi})\nabla_{\bm{\xi}}\varphi(\widetilde{\bm{x}};\bm{\xi}))^{T}\|_{\infty}]\\  \label{Eq:leading}
	&\leq \dfrac{2\eta}{\alpha}(L^{2}+\alpha)L^{2}C_{1},
	\end{align}}\normalsize
where the last inequality is due to $\textbf{(A.1)}$. Therefore,
\begin{align}
\label{Eq:Put_Togeter_1}
\int_{t_{1}}^{t_{2}}|R[\mu_{s}]|\mathrm{d}s\leq \mathfrak{s}_{0} |t_{2}-t_{1}|,
\end{align}
where $\mathfrak{s}_{0}\df \dfrac{2\eta}{\alpha}(L^{2}+\alpha)L^{2}C_{1}$.

Consider the middle term of \eqref{Eq:last_term_of}. Using the definition of the martingale term in Eq. \eqref{Eq:M1}, we obtain that
\begin{align}
\nonumber
|\mathcal{M}_{t_{1}}^{N}-\mathcal{M}_{t_{2}}^{N}|&=\left| \sum_{\ell=0}^{\lfloor Nt_{1} \rfloor} M_{\ell}^{N} -\sum_{\ell=0}^{\lfloor Nt_{2} \rfloor}M_{\ell}^{N} \right|\\ \label{Eq:Inequality_123}
&\leq \left| \sum_{\ell=\lfloor Nt_{1} \rfloor}^{\lfloor Nt_{2} \rfloor}M_{\ell}^{N}\right|.
\end{align}
In Equation of Section \ref{Section:Proofs of Auxiliary Results}, we have proved the following concentration bound
\begin{align}
\label{Eq:Clear_1}
\prob(|M^{N}_{m}|\geq \varepsilon)\leq 2\exp\left(-\dfrac{N^{2}\alpha^{2}\varepsilon^{2}}{8mL^{4}\eta^{2}C_{1}^{2}(L^{2}+\alpha)^{2} } \right), \quad \forall m\in [0,NT]\cap \integer.
\end{align}
Now, recall the alternative definition of the sub-Gaussian random variables:

\begin{definition}\textsc{(Sub-Gaussian Random Variables \cite{boucheron2013concentration})}
	A random variable X is $\sigma^{2}$-sub-Gaussian if
	\begin{align}
	\expect[\exp(\lambda (X-\expect[X]))]\leq \exp\Big({\lambda^{2}\sigma^{2}\over 2}\Big). 
	\end{align}
\end{definition}

We enumerate a few standard consequences of sub-Gaussianity \cite{boucheron2013concentration}. If $X_{i}$ are independent and $\sigma_{i}^{2}$-sub-Gaussian, then $\sum_{i=1}^{n}X_{i}$ is $\sum_{i=1}^{n}\sigma_{i}^{2}$ -sub-Gaussian. Moreover, $X$ is $\sigma^{2}$-sub-Gaussian if and only if 
\begin{align}
\label{Eq:Sub_Gaussian}
\prob(|X-\expect[X]|\geq \varepsilon)\leq \exp\left(-\dfrac{\varepsilon^{2}}{2\sigma^{2}}\right).
\end{align}
Now, it is clear from \eqref{Eq:Clear_1} andthat $M^{N}_{m}$ is sub-Gaussian random variable with a zero mean, and with the parameter $\sigma_{m}^{2}=\dfrac{4mL^{4}\eta^{2}C_{1}^{2}(L^{2}+\alpha)^{2}}{N^{2}\alpha^{2}}$. Therefore, $\sum_{\ell=\lfloor Nt_{1}\rfloor}^{\lfloor Nt_{2} \rfloor}M_{\ell}^{N}$ is sub-Gaussian with the parameter $\sigma^{2}(t_{1},t_{2})\df \dfrac{2L^{4}\eta^{2}C_{1}^{2}(L^{2}+\alpha)^{2}}{N^{2}\alpha^{2}} (\lfloor Nt_{1}\rfloor-\lfloor Nt_{2}\rfloor+1)(\lfloor Nt_{1}\rfloor+\lfloor Nt_{2} \rfloor)$. Consequently, from Inequality \eqref{Eq:Inequality_123} and the concentration inequality in Eq. \eqref{Eq:Sub_Gaussian}, we have
\begin{align}
\label{Eq:Put_Together_2}
\prob\left(\sup_{t_{1},t_{2}\leq T,|t_{1}-t_{2}|\leq \rho} |\mathcal{M}_{t_{1}}^{N}-\mathcal{M}_{t_{2}}^{N}|\geq \varepsilon \right)&\leq \prob\left(\sup_{t_{1},t_{2}\leq T,|t_{1}-t_{2}|\leq \rho}\left| \sum_{\ell=\lfloor Nt_{1} \rfloor}^{\lfloor Nt_{2} \rfloor}M_{\ell}^{N}\right|\geq \varepsilon \right)\\
&=\prob\left(\left| \sum_{\ell=\lfloor Nt^{\ast}_{1} \rfloor}^{\lfloor Nt^{\ast}_{2} \rfloor}M_{\ell}^{N}\right|\geq \varepsilon \right)\\
&\leq 2\exp\left(-\dfrac{\varepsilon^{2}}{\sigma^{2}(t^{\ast}_{1},t^{\ast}_{2})} \right)\\
&\leq 2\exp\left(-\dfrac{\alpha^{2}\varepsilon^{2}}{4L^{4}\eta^{2}C_{1}^{2}(L^{2}+\alpha)^{2}(\rho+1)T}\right),
\end{align}
where $(t_{1}^{\ast},t_{2}^{\ast})\df \arg\sup_{t_{1},t_{2}\leq T,|t_{1}-t_{2}|\leq \rho}\left| \sum_{\ell=\lfloor Nt_{1} \rfloor}^{\lfloor Nt_{2} \rfloor}M_{\ell}^{N}\right|$.

We first compute a bound for the last term of \eqref{Eq:last_term_of} using the definition of the scaled term $\mathcal{R}_{t}^{N}$ from \eqref{Eq:Remainder_2}. We have
\begin{align}
\nonumber
|\mathcal{R}_{t_{1}}^{N}-\mathcal{R}_{t_{2}}^{N}|&=\left|\sum_{\ell=0}^{\lfloor Nt_{1} \rfloor}R_{\ell}^{N} -\sum_{\ell=0}^{\lfloor Nt_{2} \rfloor}R_{\ell}^{N} \right|\\ \nonumber
&=\left| \sum_{\ell=\lfloor Nt_{1} \rfloor}^{\lfloor Nt_{2} \rfloor}R_{\ell} \right|\\ \nonumber
&\leq \sum_{\ell=\lfloor Nt_{1} \rfloor}^{\lfloor Nt_{2} \rfloor}|R_{\ell}|\\ \nonumber
&\stackrel{\rm{(a)}}{\leq} |\lfloor Nt_{2} \rfloor- \lfloor Nt_{1} \rfloor| \dfrac{C_{0}}{N^{2}}(\eta L^{2}+(L^{4}/\alpha)) \\  \label{Eq:Put_Together_3}
&\stackrel{\rm{(b)}}{\leq} \mathfrak{s}_{1}|t_{2}-t_{1}|,
\end{align}
where $\rm{(a)}$ follows from the upper bound in Eq. \eqref{Eq:Inequality_1} of Section \ref{Section:Proofs of Auxiliary Results}, and in $\rm{(b)}$ we define $\mathfrak{s}_{1}\df \dfrac{C_{0}}{N}(\eta L^{2}+(L^{4}/\alpha))$.

Putting together \eqref{Eq:Put_Togeter_1}, \eqref{Eq:Put_Together_2}, and \eqref{Eq:Put_Together_3}, we conclude from Inequality \eqref{Eq:last_term_of} that
\small{\begin{align}
	\nonumber
	\prob\Big(\hspace{-.4mm}\sup_{t_{1},t_{2}\leq T,|t_{1}-t_{2}|\leq \rho}|H(\mu^{N}_{t_{1}})-H(\mu^{N}_{t_{1}})|\geq \sigma\Big)&\leq  \prob\Big(\hspace{-.4mm}\sup_{t_{1},t_{2}\leq T,|t_{1}-t_{2}|\leq \rho}|\mathcal{M}_{t_{1}}^{N}-\mathcal{M}_{t_{2}}^{N}|+(\mathfrak{s}_{0}+\mathfrak{s}_{1})\rho \geq \sigma\Big)\\ \nonumber
	&\leq 2\exp\left(-\dfrac{\alpha^{2}(\sigma-(\mathfrak{s}_{0}+\mathfrak{s}_{1})\rho)^{2}}{4L^{4}\eta^{2}C_{1}^{2}(L^{2}+\alpha)^{2}(\rho+1)T}\right).
	\end{align}}\normalsize
Therefore, condition $\textbf{(T.2)}$ is also satisfied. Since the sufficient conditions $\textbf{(T.1)}$ and $\textbf{(T.2)}$ are satisfied, the condition $\textbf{(J.2)}$ is satisfied. This completes the  tightness proof of the measured-valued sequence $\{\mu_{t}^{N}\}_{N\in \integer}$.

Now, we prove its convergence to a mean-field solution $(\mu^{\ast}_{t})_{0\leq t\leq T}$. 

\begin{theorem}\textsc{(Prokhorov's theorem \cite{prokhorov1956convergence})}
	\label{Thm:Prokhorov}
	A subset of probability measures on a complete separable metric space is tight if and only if it is pre-compact.
\end{theorem}
According to Theorem \ref{Thm:Prokhorov}, the tightness of the Skorkhod Space $\mathcal{D}_{\mathcal{M}(\real^{D})}([0,T])$ implies its pre-compactness which in turn implies the existence of a converging sub-sequence $\{(\mu^{N}_{t})_{0\leq t\leq T}\}_{N_{k}}$ of $\{\mu_{t}^{N}\}_{N\in \integer}$ . Notice that $\{(\mu^{N}_{t})_{0\leq t\leq T}\}_{N_{k}}$ is a stochastic process defined on the Skorkhod space. Therefore, let $\pi^{N_{k}}$ denotes the law of the converging sub-sequence $\{(\mu^{N}_{t})_{0\leq t\leq T}\}_{N_{k}}$. By definition, $\pi^{N_{k}}$ is an element of the measure space $\mathcal{M}(\mathcal{D}_{[0,T]}(\mathcal{M}(\real^{D})))$. In the sequel, we closely follow the argument of \cite[Proposition 4]{wang2017scaling} to show that the limiting measure $\pi^{\infty}$ is a Dirac's delta function concentrated at \textit{a} mean-field solution $\mu_{t}^{\ast}\in \mathcal{D}_{[0,T]}(\mathcal{M}(\real^{D}))$. We define the following functional
\begin{align}
\nonumber
F_{t}:\mathcal{D}_{[0,T]}(\mathcal{M}(\real^{D}))&\rightarrow \real,\\
\mu_{t}&\mapsto F_{t}[\mu_{t}]=\left|\langle \mu_{t},f \rangle-\langle \mu_{0},f\rangle-\int_{0}^{t}R[\mu_{s}]\mathrm{d}s\right|.
\end{align}
We compute the expectation of the functional $F_{t}$ with respect to $\pi^{N_{k}}$. We then have
\begin{align}
\nonumber
\expect_{\pi^{N_{k}}}[F_{t}(\mu)]&=\expect[F_{t}[\mu_{t}^{N}]]\\ \label{Eq:Dam1}
&=\expect\left[\left|\langle \mu^{N_{k}}_{t},f \rangle-\langle \mu^{N}_{0},f\rangle-\int_{0}^{t}R[\mu^{N_{k}}_{s}]\mathrm{d}s\right|. \right].
\end{align}
Now, from Equation \eqref{Eq:Yield}, we have that
\begin{align}
\label{Eq:Dam}
\langle \mu^{N_{k}}_{t},f \rangle-\langle \mu^{N_{k}}_{0},f\rangle-\int_{0}^{t}R[\mu^{N_{k}}_{s}]\mathrm{d}s=\mathcal{M}_{t}^{N_{k}}+\mathcal{R}_{t}^{N_{k}}.
\end{align}
Plugging \eqref{Eq:Dam} in Eq. \eqref{Eq:Dam1} gives 
\begin{align}
\nonumber
\expect_{\pi^{N_{k}}}[F_{t}(\mu)]&=\expect[F_{t}[\mu_{t}^{N_{k}}]]\\ \nonumber
&=\expect\left[\left|\mathcal{M}_{t}^{N_{k}}+\mathcal{R}_{t}^{N_{k}}\right| \right]\\ \nonumber
&\leq \expect\Big[\sup_{0\leq t\leq T} |\mathcal{M}_{t}^{N_{k}}|\Big]+\expect\Big[\sup_{0\leq t\leq T}|\mathcal{R}_{t}^{N_{k}}|\Big]\\  \label{Eq:Dam3}
&=\dfrac{1}{N\alpha \rho}4\sqrt{2}L^{2}\sqrt{\lfloor NT \rfloor}\eta C_{1}(L^{2}+\alpha)^{2}+ \dfrac{C_{0}T}{N}(\eta\alpha L^{2}T+2\eta L^{4}T),
\end{align}
where the last equality is due to the bounds in Eqs. \eqref{Eq:Cheater_1} and \eqref{Eq:Cheater_2} of Lemmata \ref{Lemma:Remainder} and \ref{Lemma:Martingale}, respectively. Taking the limit of $N\rightarrow \infty$ from Eq. \eqref{Eq:Dam3} yields
\begin{align}
\label{Eq:From_Imply}
\lim_{N_{k}\rightarrow \infty}\expect_{\pi^{N_{k}}}[|F_{t}[\mu]|]=0.
\end{align}
It can be shown that the functional $F_{t}[\cdot]$ is continuous and bounded. Therefore, due the weak convergence of the sequence $\{\pi^{N_{k}}\}_{N_{k}\in \integer}$ to $\pi^{\infty}$, Eq. \eqref{Eq:From_Imply} implies that
\begin{align}
\label{Eq:Label_11}
\expect_{\pi^{\infty}}[|F_{t}(\mu)|]=0.
\end{align}
Since the identity \eqref{Eq:Label_11} holds for all bounded test functions $f\in C^{3}_{b}(\real^{D})$ and for all $t\in (0,T]$, it follows that $\pi^{\infty}$ is a Dirac's delta function concentrated at \textit{a} solution $(\mu_{t}^{\ast})_{0\leq t\leq T}$ of the mean-field equation.

\textbf{Step 3: Uniqueness of a mean-field solution}: Before we establish the uniqueness result we make two remarks:

First, we make it clear that from the compact-containment condition $\textbf{(J.1)}$ of Jakubowski's criterion in Theorem \ref{Thm:Jack}, the support of the measured-valued process $(\mu_{t}^{N})_{0\leq t\leq T}=(\widehat{\mu}^{N}_{\lfloor Nt\rfloor})_{0\leq t\leq T}$ is compact for all $0\leq t\leq T$. Moreover, in Step 2 of the proof, we established that the measure valued process $(\mu_{t}^{N})_{0\leq t\leq T}$ converges weakly to \textit{a} mean-field solution as the number of particles tends to infinity (\textit{i.e.}, $N\rightarrow \infty$). Thus, all the possible solutions of the mean-field equation also have compact supports. Let $\widehat{\Xi}\subset \real^{D}$ denotes a compact set containing the supports of all such solutions at $0\leq t\leq T$. In the sequel, it suffices to establish the uniqueness of the mean-field solution for the test functions with a compact domain, \textit{i.e.}, let $f\in C^{3}_{b}(\widehat{\Xi})$.

Second, for all bounded continuous test functions $f\in C^{3}_{b}(\widehat{\Xi})$, the operator $f \rightarrow \langle \mu_{t},f\rangle$ is a linear
operator with $\mu_{t}(\real^{D}) = 1$. Hence, from Riesz-Markov-Kakutani representation theorem \cite{rudin1987real,varadarajan1958theorem} by assuming $\mu_{t}\in \mathcal{M}(\real^{D})$
, existence of unique operator implies $f\mapsto \langle f,\mu_{t}\rangle$ implies the existence of the unique probability measure $\mu_{t}$. Now, we equip the measure space $\mathcal{M}(\real^{D})$ with the following norm
\begin{align}
\label{Eq:Measure_Space_norm}
\|\mu\|\df \sup_{\substack{f\in C^{3}_{b}(\widehat{\Xi})\\ \|f\|_{\infty}\not= 0}} \dfrac{|\langle f,\mu \rangle|}{\|f\|_{\infty}}.
\end{align}

Given an initial measure $\mu_{0}$, we next prove that there exists at most one mean-field model solution by showing that there exists at most one real valued process $\langle \mu_{t},f\rangle$ corresponding to the mean-field model. Suppose $(\mu_{t}^{\ast,1})_{0\leq t\leq T}$, $(\mu_{t}^{\ast,2})_{0\leq t\leq T}$ are two solutions satisfying the mean-field equations \eqref{Eq:Mean_Field_Equation} with the initial distributions $\mu_{0}^{1}$, $\mu_{0}^{2}\in \mathcal{M}(\real^{D})$, respectively. For any test function $f\in C^{3}_{b}(\widehat{\Xi})$ we have that
\begin{align}
\label{Eq:Equation_B}
\langle \mu_{t}^{\ast,1}-\mu_{t}^{\ast.2},f\rangle&=\langle \mu_{0}^{1}-\mu_{0}^{2},f \rangle+{\eta \over \alpha}\int_{0}^{t}\Bigg(\iint_{\mathcal{X}\times \mathcal{Y}} \left( \langle\varphi(\bm{x},\bm{\xi})\varphi(\widetilde{\bm{x}},\bm{\xi}),\mu_{s}^{\ast,1}-\mu_{s}^{\ast,2} \rangle-\alpha y \widetilde{y} \right)\\  \nonumber
&\hspace{4mm}\times \langle \nabla f(\bm{\xi})(\nabla_{\bm{\xi}}(\varphi(\widetilde{\bm{x}};\bm{\xi})\varphi(\bm{x};\bm{\xi})))^{T},\mu_{s}^{\ast,1}-\mu_{s}^{\ast,2}\rangle P_{\bm{x},y}^{\otimes 2}((\mathrm{d}\bm{z},\mathrm{d}\widetilde{\bm{z}})\Bigg)\mathrm{d}s
.
\end{align}

We bound the first term on the right side of Equation \eqref{Eq:Equation_B} as follows
\begin{align}
\langle \mu_{0}^{1}-\mu_{0}^{2},f \rangle&\leq \|\mu_{0}^{1}-\mu_{0}^{2}\|\cdot\|f\|_{\infty}\\ \label{Eq:Bound_1}
&\leq b\|\mu_{0}^{1}-\mu_{0}^{2}\|,
\end{align}
where used the definition of the norm $\|\cdot\|$ on the measure space $\mathcal{M}(\real^{D})$ from Eq. \eqref{Eq:Measure_Space_norm}.

Furthermore, let 
\begin{align}
\nonumber
&\iint_{\mathcal{X}\times \mathcal{Y}}\alpha y\widetilde{y} \langle \nabla f(\bm{\xi})(\nabla_{\bm{\xi}}(\varphi(\widetilde{\bm{x}};\bm{\xi})\varphi(\bm{x};\bm{\xi})))^{T},\mu_{s}^{\ast,1}-\mu_{s}^{\ast,2}\rangle P_{\bm{x},y}^{\otimes 2}(\mathrm{d}(\bm{x},y),\mathrm{d}(\widetilde{\bm{x}},\widetilde{y}))\\ \nonumber
&\leq \iint_{\mathcal{X}\times \mathcal{Y}} \alpha |y\widetilde{y}|\cdot  |\langle \nabla f(\bm{\xi})(\nabla_{\bm{\xi}}(\varphi(\widetilde{\bm{x}};\bm{\xi})\varphi(\bm{x};\bm{\xi})))^{T},\mu_{s}^{\ast,1}-\mu_{s}^{\ast,2}\rangle| P_{\bm{x}}^{\otimes 2}(\mathrm{d}(\bm{x},y),\mathrm{d}(\widetilde{\bm{x}},\widetilde{y}))\\ \nonumber
&\leq \alpha \|\mu_{s}^{\ast,1}-\mu_{s}^{\ast,2}\| \int_{\mathcal{X}} \|\nabla f(\bm{\xi})(\nabla_{\bm{\xi}}(\varphi(\widetilde{\bm{x}};\bm{\xi})\varphi(\bm{x};\bm{\xi})))^{T}\| P_{\bm{x}}^{\otimes 2}(\mathrm{d}\bm{x},\mathrm{d}\widetilde{\bm{x}})\\ \nonumber
&\leq \alpha \|\mu_{s}^{\ast,1}-\mu_{s}^{\ast,2}\| \int_{\mathcal{X}}\|\nabla f(\bm{\xi})\|_{\infty} \cdot \|\nabla_{\bm{\xi}}\varphi(\widetilde{\bm{x}};\bm{\xi})\varphi(\bm{x};\bm{\xi})\|_{\infty}P_{\bm{x}}^{\otimes 2}(\mathrm{d}\bm{x},\mathrm{d}\widetilde{\bm{x}})\\ \label{Eq:Bound_2}
&\leq \alpha L^{2}C_{1} \|\mu_{s}^{\ast,1}-\mu_{s}^{\ast,2}\|,
\end{align}
where in the last inequality, we used the fact that $\|\nabla f(\bm{\xi})\|\leq C_{1}$ since the test function is three-times continuously differentiable $f\in C^{3}_{b}(\widehat{\Xi})$ on a compact support.

Similarly, we have
\begin{align}
\nonumber
&\int_{\mathcal{X}} \langle \varphi(\bm{x},\bm{\xi})\varphi(\widetilde{\bm{x}},\bm{\xi}),\mu_{s}^{\ast,1}-\mu_{s}^{\ast,2} \rangle \langle \nabla f(\bm{\xi})(\nabla_{\bm{\xi}}\varphi(\bm{x},\bm{\xi})\varphi(\bm{x},\bm{\xi}) ),\mu_{s}^{\ast,1}-\mu_{s}^{\ast,2}\rangle P^{\otimes 2}_{\bm{x}}(\mathrm{d}\bm{x},\mathrm{d}\widetilde{\bm{x}})\\ \nonumber
&\leq \|\mu_{s}^{\ast,1}-\mu_{s}^{\ast,2}\|^{2} \int_{\mathcal{X}} \| \varphi(\bm{x},\bm{\xi})\varphi(\widetilde{\bm{x}},\bm{\xi})\|_{\infty}\|\nabla f(\bm{\xi})(\nabla_{\bm{\xi}}\varphi(\bm{x},\bm{\xi})\varphi(\bm{x},\bm{\xi}))^{T}\|_{\infty} P_{\bm{x}}^{\otimes 2}(\mathrm{d}\bm{x},\mathrm{d}\widetilde{\bm{x}})\\ \label{Eq:Bound_3}
&\leq L^{4}C_{1}\|\mu_{s}^{\ast,1}-\mu_{s}^{\ast,2}\|^{2}.
\end{align}
Putting together the inequalities in Eqs. \eqref{Eq:Bound_1},\eqref{Eq:Bound_2}, and \eqref{Eq:Bound_3}  yield
\small{\begin{align}
	\langle \mu_{t}^{\ast,1}-\mu_{t}^{\ast,2},f\rangle \leq b\|\mu_{0}^{1}-\mu_{0}^{2}\|+ {L^{2}C_{1}\eta}\int_{0}^{t}\|\mu_{s}^{\ast,1}-\mu_{s}^{\ast,2}\|\mathrm{d}s+{\eta L^{4}C_{1}\over \alpha}\int_{0}^{t}\|\mu_{s}^{\ast,1}-\mu_{s}^{\ast,2}\|^{2}\mathrm{d}s.
	\end{align}}\normalsize
The above inequality holds for all bounded functions $f\in C^{3}_{b}(\widehat{\Xi})$. Thus, by taking the supremum with respect to $f$ we obtain
\begin{align}
\label{Eq:From_1}
\|\mu_{t}^{\ast,1}-\mu_{t}^{\ast,2}\|&=\sup_{f\in C^{3}_{b}(\widehat{\Xi})} \langle \mu_{t}^{\ast,1}-\mu_{t}^{\ast,2},f \rangle \\
&\leq b\|\mu_{0}^{1}-\mu_{0}^{2}\|+{ L^{2}C_{1}\eta}\int_{0}^{t}\|\mu_{s}^{\ast,1}-\mu_{s}^{\ast,2}\|\mathrm{d}s+{L^{4}C_{1}\eta\over \alpha}\int_{0}^{t}\|\mu_{s}^{\ast,1}-\mu_{s}^{\ast,2}\|^{2}\mathrm{d}s.
\end{align}
Now, we employ the following result which generalizes Gronewall's inequality when higher order terms are involved:
\begin{lemma}\textsc{(Extended Gronewall's inequality, \cite[Thm 2.1.]{webb2018extensions})}
	Let $p\in \integer$ and suppose that for a.e. $t\in [0, T]$, $u\in L_{+}^{\infty}[0,T]$ satisfies
	\begin{align}
	\label{Eq:B7}
	u_{t}\leq c_{0}(t)+\int_{0}^{t}(c_{1}(s)u_{s}+c_{2}(s)u^{2}_{s}+\cdots+c_{{p+1}}(s)u^{p+1}_{s})\mathrm{d}s,
	\end{align}
	where $c_{0}\in L_{\infty}[0,T]$ is non-decreasing, and $c_{j}\in L_{+}^{1}[0,T]$ for $j\in \{1,\cdots,p+1\}$. Then, if
	\begin{align}
	\int_{0}^{T}c_{j+1}(s)u^{j}_{s}\mathrm{d}s\leq M_{j}, \quad j\in \{1,2,\cdots,p\}.
	\end{align}
	It follows that for a.e. $t\in [0, T]$
	\begin{align}
	\label{Eq:from_2}
	u_{t}\leq c_{0}(t)\exp\left (\int_{0}^{t}c_{1}(s)\mathrm{d}s\right)\exp(M_{1}+\cdots+M_{p}).
	\end{align}
\end{lemma}

We now apply the extended Gronewall's  Inequality \eqref{Eq:B7} with $p=1$, $c_{0}(t)=b\|\mu_{0}^{1}-\mu_{0}^{2}\|, 0\leq t\leq T$, $c_{1}(t)=\eta L^{2}C_{1},0\leq t\leq T$,  $c_{2}(t)={\eta L^{4}C_{1}\over \alpha},0\leq t\leq T$, and $u_{s}=\|\mu_{s}^{\ast,1}-\mu_{s}^{\ast,2}\|$. In this case, it is easy to see that $M_{1}={2b\eta T L^{4}C_{1}\over \alpha}$. Hence, from Eqs. \eqref{Eq:From_1} and \eqref{Eq:from_2}, we obtain that
\begin{align}
\|\mu_{t}^{\ast,1}-\mu_{t}^{\ast,2}\|\leq b\|\mu_{0}^{1}-\mu_{0}^{2}\|\cdot \exp\Big(\eta L^{2}C_{1}t+{2bT\eta L^{4}C_{1}\over \alpha}\Big), \quad 0\leq t\leq T.
\end{align}
Thus, starting from an initial measure $\mu_{0}^{1}=\mu_{0}^{2}=\mu_{0}$, there exists at most one solution for the mean-field model equations \eqref{Eq:Mean_Field_Equation}. 

$\hfill$ $\square$

\subsection{Proof of Corollary \ref{Corollary:1}}

To establish the proof, we recall from Eq. \eqref{Eq:telescopic_sum} that
\begin{align}
\label{Eq:Take_Supremum}
\langle f,\widehat{\mu}_{m}^{N}\rangle-\langle f,\widehat{\mu}_{0}^{N}\rangle=\sum_{\ell=0}^{m-1}D_{\ell}^{N}+\sum_{\ell=0}^{m-1} M_{\ell}^{N}+\sum_{\ell=0}^{m-1}R_{\ell}^{N},
\end{align}
for all $f\in C_{b}(\real^{p})$.The total variation (TV) distance between two measures $\nu_{1},\nu_{2}\in \mathcal{M}(\real^{D})$ is defined
\begin{align}
\|\nu_{1}-\nu_{2}\|_{\mathrm{TV}}\df \sup_{\|f\|_{\infty}\leq 1}{1\over 2}|\langle f,\nu_{1} \rangle-\langle f,\nu_{2} \rangle| .
\end{align}
Therefore, by taking the supremum of Eq. \eqref{Eq:Take_Supremum} over all the bounded functions $f\in C_{b}(\real^{p}),b=1$, we obtain that
\begin{align}
\label{Eq:Upper_TV}
\|\widehat{\mu}_{m}^{N}-\widehat{\mu}_{0}^{N} \|_{\mathrm{TV}}&\leq \dfrac{1}{2}\sum_{\ell=0}^{m-1}|D_{\ell}^{N}|+{1\over 2}\sum_{\ell=0}^{m-1}|M_{\ell}^{N}|+{1\over 2}\sum_{\ell=0}^{m-1}|R_{\ell}^{N}|.
\end{align}
Based on the upper bound \eqref{Eq:Inequality_1} on the remainder term, we have
\begin{align}
\label{Eq:Ara1}
|R_{\ell}^{N}|&\leq \dfrac{C_{0}}{N^{2}}\left(\eta L^{2}+2{\eta\over \alpha} L^{4}\right), \quad \ell\in [0,m-1],
\end{align}
for some constant $C_{0}>0$. Moreover, from the concentration inequality \eqref{Eq:martingale_concentration_inequality}, we also have that with the probability of at least $1-\delta$, the following inequality holds
\begin{align}
\label{Eq:Ara2}
|M_{\ell}^{N}|\leq {8\sqrt{\ell} \eta L^{2}C_{1}(L^{2}+\alpha)\over N\alpha }\log\left({2\over \delta}\right).
\end{align}
Lastly, recall the definition of the drift term in Eq. \eqref{Eq:D}. By carrying out a similar bounding method leading to Eq. \eqref{Eq:leading}, it can be shown that 
\begin{align}
\label{Eq:Ara3}
|\mathcal{D}_{\ell}^{N}|\leq \dfrac{2\eta}{N\alpha}(L^{2}+\alpha)L^{2}C_{1}.
\end{align}
By plugging Eq. \eqref{Eq:Ara1}, \eqref{Eq:Ara2}, and \eqref{Eq:Ara3} into Eq. \eqref{Eq:Upper_TV}, we derive that
\small{\begin{align}
	\label{Eq:Inequality_P}
	\|\widehat{\mu}_{m}^{N}-\widehat{\mu}_{0}^{N} \|_{\mathrm{TV}}\leq \dfrac{m\eta C_{0}}{2N^{2}}\left(L^{2}+2 {L^{4}\over \alpha}\right)+{{8m\sqrt{m} \eta L^{2}C_{1}(L^{2}+\alpha)\over N\alpha }\log\left({2\over \delta}\right)}\log\left({2\over \delta}\right)+\dfrac{m\eta}{N\alpha}(L^{2}+\alpha)L^{2}C_{1},
	\end{align}}\normalsize
with the probability of $1-\delta$. We now leverage the following lemma:

\begin{lemma}\textsc{(Bounded Equivalence of the Wasserstein and Total Variation Distances, \cite{singh2018minimax})}
	\label{Lemma:Asymptotic_Bound}
	Suppose $(\mathcal{X},d)$ is a metric space, and suppose $\mu$ and $\nu$ are Borel probability measures on $\mathcal{X}$ with countable support; \textit{i.e.}, there exists a countable set $\mathcal{X}'\subseteq \mathcal{X}$ such that $\mu(\mathcal{X}')=\nu(\mathcal{X}')=1$. Then, for any $p\geq 1$, we have 
	\begin{align}
	\label{Eq:Asymptotic_Bound}
	\mathrm{Sep}(\mathcal{X}')(2\mathrm{TV}(\mu,\nu))^{1\over p} \leq W_{p}(\mu,\nu)\leq \mathrm{Diam}(\mathcal{X}') (2 \mathrm{TV}(\mu,\nu))^{1\over p},
	\end{align}	
	where $\mathrm{Diam}(\mathcal{X}')\df \sup_{x,y\in \mathcal{X}'}d(x,y)$, and $\mathrm{Sep}(\mathcal{X}')\df \inf_{x\not=y\in \mathcal{X}'}d(x,y)$.
\end{lemma}

Consider the metric space $(\real^{D},\|\cdot\|_{2})$. Note that the empirical measures $\widehat{\mu}_{m}^{N},\widehat{\mu}_{0}^{N}$ have a countable support $\mathcal{X}'=\{\bm{\xi}_{m}^{k} \}_{k=1}^{N}\cup \{\bm{\xi}_{0}^{k}\}_{k=1}^{N}\subset \real^{D}$. Therefore, using the upper bounds in  \eqref{Eq:Asymptotic_Bound} of Lemma \ref{Lemma:Asymptotic_Bound} and \ref{Eq:Inequality_P}, we conclude that when the step-size is of the order
\begin{align}
\eta=\mathcal{O}\left(R^{p}\ \over T\sqrt{NT}\log(2/\delta)\right),
\end{align} 
then $W_{p}(\widehat{\mu}_{m}^{N},\widehat{\mu}^{N}_{0})\leq R$ for all $m\in [0,NT]\cap \integer$. $\hfill$ $\square$

\subsection{Proof of Theorem \ref{Thm:Chaoticity in Particle SGD}}
\label{Eq}

The proof is built upon \cite{vasantam2018occupancy}. It suffices to show that for every integer $\ell\in \integer$, and for all the test functions $f_{1},\cdots,f_{k}\in C_{b}^{3}(\real^{D})$, we have
\begin{align}
\lim\sup_{N\rightarrow \infty}\left|\expect\left[\prod_{k=1}^{\ell}f_{k}(\bm{\xi}^{k}_{\lfloor Nt\rfloor})\right] -\prod_{k=1}^{\ell}\langle \mu_{t}^{\ast},f_{k}\rangle\right|=0.
\end{align}
Using the triangle inequality, we now have that
\begin{align}
\nonumber
&\left|\expect\left[\prod_{k=1}^{\ell}f_{k}(\bm{\xi}^{k}_{\lfloor Nt\rfloor})\right] -\prod_{k=1}^{\ell}\langle \mu_{t}^{\ast},f_{k}\rangle\right|\\ \label{Eq:First_Second_Term}
&\leq \left|\expect\left[\prod_{k=1}^{\ell}\langle \widehat{\mu}^{N}_{t},f_{k}\rangle\right] -\prod_{k=1}^{\ell}\langle \mu_{t}^{\ast},f_{k}\rangle\right|+\left|\expect\left[\prod_{k=1}^{\ell}\langle \widehat{\mu}^{N}_{t},f_{k}\rangle\right]-\expect\left[\prod_{k=1}^{\ell}f_{k}(\bm{\xi}^{k}_{\lfloor Nt\rfloor})\right] \right|.
\end{align}
For the first term on the right side of Eq. \eqref{Eq:First_Second_Term} we have
\begin{align}
\nonumber
\lim\sup_{N\rightarrow \infty}\left|\expect\left[\prod_{k=1}^{\ell}\langle \widehat{\mu}^{N}_{t},f_{k}\rangle\right] -\prod_{k=1}^{\ell}\langle \mu_{t}^{\ast},f_{k}\rangle\right|&\stackrel{\rm{(a)}}{\leq} \lim\sup_{N\rightarrow \infty} \expect\left[\left|\prod_{k=1}^{\ell}\langle \widehat{\mu}_{t}^{N},f_{k} \rangle-\prod_{k=1}^{\ell}\langle \mu_{t}^{\ast},f_{k} \rangle   \right|\right]\\ \nonumber
&\stackrel{\rm{(b)}}{\leq} \expect\left[\lim\sup_{N\rightarrow \infty}\left|\prod_{k=1}^{\ell}\langle \widehat{\mu}_{t}^{N},f_{k} \rangle-\prod_{k=1}^{\ell}\langle \mu_{t}^{\ast},f_{k} \rangle   \right|\right]\\ \nonumber
&\stackrel{\rm{(c)}}{\leq}b^{\ell-1} \expect\left[\sum_{k=1}^{\ell}\lim\sup_{N\rightarrow \infty}\left|\langle \widehat{\mu}_{t}^{N},f_{k} \rangle-\langle \mu_{t}^{\ast},f_{k} \rangle   \right|\right]\\ \label{Eq:yield_1}
&\stackrel{\rm{(d)}}{=}0,
\end{align}
where $\rm{(a)}$ is by Jensen's inequality, $\rm{(b)}$ is by Fatou's lemma,  $\rm{(c)}$ follows from the basic inequality $\Big|\prod_{i=1}^{N}a_{i}-\prod_{i=1}^{N}b_{i}\Big|\leq \sum_{i=1}^{N}|a_{i}-b_{i}|$ for $|a_{i}|,|b_{i}|\leq 1, i=1,2,\cdots,N$, as well as the fact that $ \langle\mu_{t}^{\ast},f_{k}\rangle\leq b$ and $\langle \widehat{\mu}_{t}^{N},f_{k}\rangle\leq b$ for all $k=1,2,\cdots,N$ due to the boundedness of the test functions $f_{1},\cdots,f_{\ell}\in C_{b}^{3}(\real^{D})$, and $\rm{(d)}$ follows from the weak convergence $\widehat{\mu}_{t}^{N}\stackrel{\text{weakly}}{\rightarrow}  \mu_{t}^{\ast}$ to the mean-field solution \eqref{Eq:Mean_Field_Equation}.

Now, consider the second term on the right hand side of Eq. \eqref{Eq:First_Second_Term}. Due to the exchangability of the initial states $(\bm{\xi}^{k}_{0})_{1\leq k\leq N}$, the law of the random variables $(\bm{\xi}^{k}_{0})_{1\leq k\leq N}$ is also exchangable. Therefore, we obtain that
\begin{align}
\label{Eq:Subtract_1}
\expect\left[\prod_{k=1}^{\ell}f_{k}(\bm{\xi}^{k}_{\lfloor Nt\rfloor})\right] =\dfrac{\ell !}{N!}\expect\left[\sum_{\pi\in \Pi(\ell,N)}\prod_{k=1}^{\ell}f_{k}(\bm{\xi}^{\pi(k)}_{\lfloor Nt\rfloor}) \right],
\end{align}
where $\Pi(\ell,N)$ is the set of all permutations of $\ell$ numbers selected from $\{1,2,\cdots,N\}$. Notice that the right hand side of Eq. \eqref{Eq:Subtract_1} is the symmetrized version of the left hand side \eqref{Eq:Subtract_2}. 

Further, by the definition of the empirical measure $\widehat{\mu}_{t}^{N}$ we obtain that
\begin{align}
\label{Eq:Subtract_2}
\expect\left[\prod_{k=1}^{\ell}\langle \widehat{\mu}^{N}_{t},f_{k}\rangle\right]&=\dfrac{1}{N^{\ell}}\expect\left[\prod_{k=1}^{\ell}\left(\sum_{m=1}^{N} f_{k}(\bm{\xi}^m_{\lfloor Nt\rfloor})\right)\right]\\
&=\dfrac{1}{N^{\ell}}\expect\left[\sum_{\pi\in \widetilde{\Pi}(\ell,N)}\left(\prod_{k=1}^{\ell}f_{k}(\bm{\xi}^{\pi(k)}_{\lfloor Nt\rfloor})\right)\right].
\end{align}
Therefore, subtracting \eqref{Eq:Subtract_1} and \eqref{Eq:Subtract_2} yields
\begin{align}
\left|\expect\left[\prod_{k=1}^{\ell}\langle \widehat{\mu}^{N}_{t},f_{k}\rangle\right]-\expect\left[\prod_{k=1}^{\ell}f_{k}(\bm{\xi}^{k}_{\lfloor Nt\rfloor})\right] \right|\leq b^{\ell}\left(1-{N! \over \ell ! N^{\ell}}\right).
\end{align}
Hence,
\begin{align}
\label{Eq:yield_2}
\lim\sup_{N\rightarrow \infty}\left|\expect\left[\prod_{k=1}^{\ell}\langle \widehat{\mu}^{N}_{t},f_{k}\rangle\right]-\expect\left[\prod_{k=1}^{\ell}f_{k}(\bm{\xi}^{k}_{\lfloor Nt\rfloor})\right] \right|=0.
\end{align}
Combining Eqs. \eqref{Eq:yield_1}-\eqref{Eq:yield_2} yields the desired result. $\hfill$ $\square$

\section{Proofs of Auxiliary Results}
\label{Section:Proofs of Auxiliary Results}

\subsection{Proof of Lemma \ref{Lemma:M_envelope}}
\label{Proof_of_Lemma:M_envelop}

The upper bound follows trivially by letting $\bm{x}=\bm{y}$ in the optimization problem \eqref{Eq:Moreau's envelope}.

Now, consider the lower bound. Define the function $g:[0,1]\rightarrow \real, t\mapsto g(t)=f(\bm{y}+t(\bm{x}-\bm{y}))$. Then, when $f$ is differentiable, we have $g'(t)=\langle \bm{x}-\bm{y},\nabla f(\bm{y}+t(\bm{x}-\bm{y}))  \rangle$. In addition, $g(0)=f(\bm{y})$, and $g(1)=f(\bm{x})$. Based on the basic identity $g(1)=g(0)+\int_{0}^{1}g'(s)\mathrm{d}s$,  we derive
\begin{align}
\nonumber
f(\bm{x})&=f(\bm{y})+\int_{0}^{1}\langle \bm{x}-\bm{y},\nabla f(\bm{y}+s(\bm{x}-\bm{y}))  \rangle\mathrm{d}s\\
\label{Eq:Using_Inequality}
&\geq f(\bm{y})-\|\bm{x}-\bm{y}\|_{2}\int_{0}^{1}\|\nabla f(\bm{y}+s(\bm{x}-\bm{y}))\|_{2}\mathrm{d}s,
\end{align}
where the last step is due to the Cauchy-Schwarz inequality. Using Inequality \eqref{Eq:Using_Inequality} yields the following lower bound on Moreau's envelope
\begin{align}
\nonumber
M_{f}^{\beta}(\bm{y})&\geq f(\bm{y})+ \inf_{\bm{x}\in \mathcal{X}} \left\{\dfrac{1}{2\beta}\|\bm{x}-\bm{y}\|_{2}^{2}-\|\bm{x}-\bm{y}\|_{2} \int_{0}^{1}\|\nabla f(\bm{y}+s(\bm{x}-\bm{y}))\|_{2}\mathrm{d}s \right\}\\ \nonumber
&= f(\bm{y})+\inf_{\bm{x}\in \mathcal{X}}\Bigg\{ \left(\dfrac{1}{\sqrt{2\beta}}\|\bm{x}-\bm{y}\|_{2}-{\sqrt{\beta\over 2}}\int_{0}^{1}\|\nabla f(\bm{y}+s(\bm{x}-\bm{y}))\|_{2}\mathrm{d}s \right)^{2}\\ \nonumber
&\hspace{50mm}-{\beta\over 2}\left(\int_{0}^{1}\|\nabla f(\bm{y}+s(\bm{x}-\bm{y}))\|_{2}\mathrm{d}s\right)^{2}\Bigg\}\\ \nonumber
&\geq  f(\bm{y})-\dfrac{\beta}{2}\sup_{\bm{x}\in \mathcal{X}}\Bigg(\int_{0}^{1}\|\nabla f(\bm{y}+s(\bm{x}-\bm{y}))\|_{2}\mathrm{d}s \Bigg)^{2}\\ \nonumber
&\stackrel{\rm{(a)}}{\geq} f(\bm{y})-\dfrac{\beta}{2}\sup_{\bm{x}\in \mathcal{X}} \int_{0}^{1}\|\nabla f(\bm{y}+s(\bm{x}-\bm{y}))\|^{2}_{2}\mathrm{d}s\\ \label{Eq:5G}
&\geq f(\bm{y})-\dfrac{\beta}{2} \int_{0}^{1}\sup_{\bm{x}\in \mathcal{X}} \|\nabla f(\bm{y}+s(\bm{x}-\bm{y}))\|_{2}^{2}\mathrm{d}s,
\end{align}
where $\rm{(a)}$ is due to Jensen's inequality.
\hfill $\square$

\subsection{Proof of Lemma \ref{Lemma:Gradient_of_M}}
\label{Proof_of_Lemma:}

Let $\bm{u}\in \real^{d}$ denotes an arbitrary unit vector $\|\bm{u}\|_{2}=1$. From the definition of the gradient of a function, we have that
\begin{align}
\label{Eq:By_Cauchy}
\left\langle \nabla_{\bm{\theta}}M_{f(\cdot;\bm{\theta})}^{\beta}(\bm{x}),\bm{u}\right\rangle= D_{\bm{u}}[M_{f(\cdot;\bm{\theta})}(\bm{x})],
\end{align} 
where $D_{\bm{u}}[M_{f(\cdot;\bm{\theta})}(\bm{x})]$ is the directional derivative 
\begin{align}
D_{\bm{u}}[M_{f(\cdot;\bm{\theta})}(\bm{x})]\df \lim_{\delta\rightarrow  0}\dfrac{M_{f(\cdot;\bm{\theta}+\delta\bm{u})}(\bm{x})-M_{f(\cdot;\bm{\theta})}(\bm{x})}{\delta}.
\end{align}
We now have
\begin{align}
\nonumber
M_{f(\cdot;\bm{\theta}+\delta\bm{u})}(\bm{x})&=\inf_{\bm{x}\in \mathcal{X}}\left\{\dfrac{1}{2\beta}\|\bm{x}-\bm{y}\|_{2}^{2}+f(\bm{y};\bm{\theta}+\delta\bm{u}) \right\}\\ \nonumber
&=\inf_{\bm{x}\in \mathcal{X}}\left\{\dfrac{1}{2\beta}\|\bm{x}-\bm{y}\|_{2}^{2}+f(\bm{y};\bm{\theta})+\delta \langle \nabla_{\bm{\theta}} f(\bm{y};\bm{\theta}),\bm{u} \rangle\right\}+\mathcal{O}(\delta^{2})\\ \nonumber
&\stackrel{\rm{(a)}}{\leq} \dfrac{1}{2\beta}\|\bm{x}-\mathrm{Prox}_{f(\cdot;\bm{\theta})}(\bm{x})\|_{2}^{2}+f(\mathrm{Prox}_{f(\cdot;\bm{\theta})}(\bm{x});\bm{\theta})\\ \label{Eq:Vali_Delam_1}
&\hspace{4mm}+\delta\langle \nabla_{\bm{\theta}}f(\mathrm{Prox}_{f(\cdot;\bm{\theta})}(\bm{x});\bm{\theta}),\bm{u} \rangle+\mathcal{O}(\delta^{2}),
\end{align}
where the inequality in $\rm{(a)}$ follows by letting $\bm{y}=\mathrm{Prox}_{f(\cdot;\bm{\theta})}(\bm{x})$ in the optimization problem. Now, recall that
\begin{align*}
M_{f(\cdot;\bm{\theta})}(\bm{x})&=\inf_{\bm{y}\in \mathcal{X}}\left\{\dfrac{1}{2\beta}\|\bm{x}-\bm{y} \|_{2}^{2}+f(\bm{y};\bm{\theta})\right\},\\
\mathrm{Prox}_{f(\cdot;\bm{\theta})}(\bm{x})&=\arg\min_{\bm{y}\in \mathcal{X}}\left\{\dfrac{1}{2\beta}\|\bm{x}-\bm{y} \|_{2}^{2}+f(\bm{y};\bm{\theta})\right\}.
\end{align*}
Therefore, 
\begin{align}
\label{Eq:Vali_Delam_2}
M_{f(\cdot;\bm{\theta})}(\bm{x})=\dfrac{1}{2\beta}\|\bm{x}-\mathrm{Prox}_{f(\cdot;\bm{\theta})}(\bm{x}) \|_{2}^{2}+f(\mathrm{Prox}_{f(\cdot;\bm{\theta})}(\bm{x});\bm{\theta}).
\end{align}
Substitution of Eq. \eqref{Eq:Vali_Delam_2} in \eqref{Eq:Vali_Delam_1} yields
\begin{align}
M_{f(\cdot;\bm{\theta}+\delta\bm{u})}(\bm{x})&\leq M_{f(\cdot;\bm{\theta})}(\bm{x})+\delta\langle \nabla_{\bm{\theta}}f(\mathrm{Prox}_{f(\cdot;\bm{\theta})}(\bm{x});\bm{\theta}),\bm{u} \rangle+\mathcal{O}(\delta^{2}).
\end{align} 
Hence, $D_{\bm{u}}[M_{f(\cdot;\bm{\theta})}(\bm{x})]\leq \langle \nabla_{\bm{\theta}}f(\mathrm{Prox}_{f(\cdot;\bm{\theta})}(\bm{x});\bm{\theta}),\bm{u} \rangle$. From Eq. \eqref{Eq:By_Cauchy} and by using Cauchy-Schwarz inequality, we compute the following bound on the inner product of the gradient with the unit vectors $\bm{u}\in \real^{d},\|\bm{u}\|_{2}=1$,
\begin{align}
\left\langle \nabla_{\bm{\theta}}M_{f(\cdot;\bm{\theta})}^{\beta}(\bm{x}),\bm{u}\right\rangle\leq 
\|\nabla_{\bm{\theta}}f(\mathrm{Prox}_{f(\cdot;\bm{\theta})}(\bm{x});\bm{\theta})\|_{2}\cdot\|\bm{u}\|_{2}
&=\|\nabla_{\bm{\theta}}f(\mathrm{Prox}_{f(\cdot;\bm{\theta})}(\bm{x});\bm{\theta})\|_{2}.
\end{align}
Since the preceding upper bound holds for all the unit vectors $\bm{u}\in \real^{d}$, we let $\bm{u}={\nabla_{\bm{\theta}}M_{f(\cdot;\bm{\theta})}^{\beta}(\bm{x})\over \|\nabla_{\bm{\theta}}M_{f(\cdot;\bm{\theta})}^{\beta}(\bm{x})\|_{2}}$ to get Inequality \eqref{Eq:cruelty}. $\hfill$ $\square$

\subsection{Proof of Lemma \ref{Lemma:Tail Bounds for the Finite Sample Estimation Error}}
\label{Proof_of_Lemma:Tail Bounds for the Finite Sample Estimation Error}

Let $\bm{z}\in \mathrm{S}^{d-1}$ denotes an arbitrary vector on the sphere. Define the random variable
\small{\begin{align}
	\nonumber
	Q_{\bm{z}}\Big((y_{1},\bm{x}_{1}),\cdots,(y_{n},\bm{x}_{n})\Big)&\df \langle \bm{z},\nabla E_{n}(\bm{\xi})\rangle\\ \nonumber
	&=\dfrac{1}{n(n-1)}\sum_{i\not=j} y_{i}y_{j}\Big(\varphi(\bm{x}_{i};\bm{\xi}) \langle \bm{z},\nabla \varphi(\bm{x}_{j};\bm{\xi})\rangle+\varphi(\bm{x}_{j};\bm{\xi})\langle\bm{z},\nabla \varphi(\bm{x}_{i};\bm{\xi})\rangle\Big)\\
	&\hspace{4mm}-\expect_{P_{\bm{x},y}^{\otimes 2}}\Big[y\widehat{y}\Big(\varphi(\bm{x}_{i};\bm{\xi}) \langle \bm{z},\nabla \varphi(\bm{x}_{j};\bm{\xi})\rangle+\varphi(\bm{x}_{j};\bm{\xi})\langle\bm{z},\nabla \varphi(\bm{x}_{i};\bm{\xi})\rangle\Big)\Big].
	\end{align}}\normalsize
Clearly, $\expect_{P_{\bm{x},y}}[Q_{\bm{z}}]=0$. Now, let $(\widehat{y}_{m},\widehat{\bm{x}}_{m})\in \mathcal{Y}\times \mathcal{X},1\leq m\leq n$. By repeated application of the triangle inequality, we obtain that 
\begin{align}
\nonumber
\Big|Q_{\bm{z}}((y_{1},\bm{x}_{1}),&\cdots, (y_{m},\bm{x}_{m}),\cdots,(y_{n},\bm{x}_{n}))-Q_{\bm{z}}((y_{1},\bm{x}_{1}),\cdots, (\widehat{y}_{m},\widehat{\bm{x}}_{m}),\cdots,(y_{n},\bm{x}_{n}))\Big|\\ \nonumber
&\leq {1\over n(n-1)}\Big|\sum_{i\not=m}y_{i}\varphi(\bm{x}_{i};\bm{\xi})\langle \bm{z},y_{m}\nabla\varphi(\bm{x}_{m};\bm{\xi})-\widehat{y}_{m}\nabla\varphi(\widehat{\bm{x}}_{m};\bm{\xi}) \rangle   \Big|\\ \nonumber
&\hspace{4mm}+{1\over n(n-1)}\Big| \sum_{i\not=m}y_{i}\langle \bm{z},\nabla\varphi(\bm{x}_{i};\bm{\xi}) \rangle(y_{m}\varphi(\bm{x}_{m};\bm{\xi})-\widehat{y}_{m}\varphi(\widehat{\bm{x}}_{m};\bm{\xi}) ) \Big|\\ \nonumber
&\leq \dfrac{1}{n(n-1)}\sum_{i\not= m}|\varphi(\bm{x}_{i};\bm{\xi})|\cdot  \|\bm{z}\|_{2}\cdot\|y_{m}\nabla \varphi(\bm{x}_{m};\bm{\xi})-\widehat{y}_{m}\nabla \varphi(\widehat{\bm{x}}_{m};\bm{\xi}) \|_{2}\\  \nonumber
&\hspace{4mm}+\dfrac{1}{n(n-1)}\sum_{i\not= m} \|\bm{z}\|_{2}\cdot \|\nabla \varphi(\bm{x}_{i};\bm{\xi}) \|_{2}\cdot |y_{m}\varphi(\bm{x}_{m};\bm{\xi})-\widehat{y}_{m}\varphi(\widehat{\bm{x}}_{m};\bm{\xi})| \\ \label{Eq:Inequaltiy_Sub}
&\leq {4L^{2}\over n},
\end{align}
where the last inequality is due to assumption $(\textbf{A.2})$ and the fact that $\|\bm{z}\|_{2}=1$ for $\bm{z}\in \mathrm{S}^{d-1}$. In particular, to derive Inequality \eqref{Eq:Inequaltiy_Sub}, we employed the following upper bounds
\begin{align*}
|\varphi(\bm{x}_{i};\bm{\xi})|&\leq L,\\
|y_{m}\varphi(\bm{x}_{m};\bm{\xi})-\widehat{y}_{m}\varphi(\widehat{\bm{x}}_{m};\bm{\xi})| &\leq |\varphi(\bm{x}_{m};\bm{\xi})|+|\varphi(\widehat{\bm{x}}_{m};\bm{\xi})|\leq 2L,\\
\|\nabla \varphi(\bm{x}_{i};\bm{\xi})\|_{2}&\leq  L,  \\
\|y_{m}\nabla \varphi(\bm{x}_{m};\bm{\xi})-\widehat{y}_{m}\nabla \varphi(\widehat{\bm{x}}_{m};\bm{\xi}) \|_{2} &\leq \|\nabla \varphi(\bm{x}_{m};\bm{\xi})\|_{2}+\|\nabla \varphi(\widehat{\bm{x}}_{m};\bm{\xi})\|_{2}\leq 2L.
\end{align*}
Using McDiarmid Martingale's inequality \cite{mcdiarmid1989method} then gives us
\begin{align}
\label{Eq:Concentration_Bound_1}
\prob\left(\Big|Q_{\bm{z}}((y_{1},\bm{x}_{1}),\cdots,(y_{n},\bm{x}_{n}))\Big|\geq u \right)\leq 2\exp\left(-\dfrac{nu^{2}}{16L^{4}}\right),
\end{align}
for $x\geq 0$. Now, for every $p\in \integer$, the $2p$-th moment of the random variable $Q_{\bm{z}}$ is given by
\begin{align}
\nonumber
\expect\Big[Q^{2p}_{\bm{z}}((y_{1},\bm{x}_{1}),\cdots,(y_{n},\bm{x}_{n}))\Big]&=\int_{\real_{+}}2pu^{2p-1}\prob(Q_{\bm{z}}((y_{1},\bm{x}_{1}),\cdots,(y_{n},\bm{x}_{n}))\geq u)\mathrm{d}u
\\ \nonumber &\stackrel{\rm{(a)}}{\leq}\int_{\real_{+}} 4pu^{2p-1}\exp\left(-\dfrac{u^{2}}{16nL^{4}}\right)\mathrm{d}u
\\ \label{Eq:Readily_Follows} &= {2(16L^{4}/n)^{2p}p!},
\end{align}
where $\rm{(a)}$ is due to the concentration bound in Eq. \eqref{Eq:Concentration_Bound_1}. Now 
Therefore, 
\begin{align}
\nonumber
\expect\big[\exp\big(Q^{2}_{\bm{z}}((y_{1},\bm{x}_{1}),\cdots,(y_{n},\bm{x}_{n}))/\gamma^{2}\big)\big]&=\sum_{p=0}^{\infty}\dfrac{1}{p!\gamma^{2p}}\expect\Big[\phi^{2p}_{\bm{z}}((y_{1},\bm{x}_{1}),\cdots,(y_{n},\bm{x}_{n}))\Big]\\
\nonumber
&=1+2\sum_{p\in \integer}\left(\dfrac{16L^{4}}{n\gamma}\right)^{2p}\\ \nonumber
&=\dfrac{2}{1-(16L^{4}/n\gamma)^{2}}-1.
\end{align}
For $\gamma=16\sqrt{3}L^{4}/n$, we obtain $\expect\big[\exp\big(Q^{2}_{\bm{z}}((y_{1},\bm{x}_{1}),\cdots,(y_{n},\bm{x}_{n}))/\gamma^{2}\big)\big]\leq 2$. Therefore, $\|Q_{\bm{z}}\|_{\psi_{2}}=\|\langle \bm{z},\nabla E_{n}(\bm{\xi}) \rangle \|_{\psi_{2}}\leq 16\sqrt{3}L^{4}/n$ for all $\bm{z}\in \mathrm{S}^{n-1}$ and $\bm{\xi}\in \real^{D}$. Consequently, by the definition of the sub-Gaussian random vector in Eq. \eqref{Eq:Sub_Gaussian_random_vector} of Definition \ref{Definition:Sub-Gaussian Norm}, we have $\| \nabla E_{n}(\bm{\xi})\|_{\psi_{2}}\leq 16\sqrt{3}L^{4}/n$ for every $\bm{\xi}\in \real^{D}$. We invoke the following lemma proved by the first author in \cite[Lemma 16]{khuzani2017stochastic}:

\begin{lemma}\textsc{(The Orlicz Norm of  the Squared Vector Norms, \cite[Lemma 16]{khuzani2017stochastic})}
	\label{Lemma:The_Orlicz_Norm_of_the_Squared_Vector_Norms}
	Consider the zero-mean random vector $\bm{Z}$ satisfying $\|\bm{Z}\|_{\psi_{\nu}}\leq \beta$ for every $\nu\geq 0$. Then, $\|\|\bm{Z}\|_{2}^{2}\|_{\psi_{{\nu\over 2}}}\leq 2\cdot3^{2\over \nu}\cdot \beta^{2}$.
\end{lemma}

Using Lemma \ref{Lemma:The_Orlicz_Norm_of_the_Squared_Vector_Norms}, we now have that $\| \|\nabla E_{n}(\bm{\xi}) \|_{2}^{2} \|_{\psi_{1}}\leq 4608L^{4}/n^{2}$ for every $\bm{\xi}\in \real^{D}$. Applying the exponential Chebyshev's inequality with $\beta=4608L^{4}/n^{2}$ yields
\begin{align}
\nonumber
&\prob\Bigg(\int_{\real^{D}}\int_{0}^{1} \Big|\| \nabla E_{n}((1-s)\bm{\xi}+s\bm{\zeta}_{\ast}) \|_{2}^{2}-\expect_{\bm{x},y}[\|\nabla E_{n}((1-s)\bm{\xi}+s\bm{\zeta}_{\ast}) \|_{2}^{2}]\Big|\mu_{0}(\mathrm{d}\bm{\xi})\geq \delta \Bigg)  \\ \nonumber
&\leq e^{-{n^{2}\delta\over 4608L^{4}}}\expect_{\bm{x},y}\left[e^{\left({n^{2}\over 4608 L^{4}} \int_{\real^{D}}\int_{0}^{1}\big|\| \nabla E_{n}( (1-s)\bm{\xi}+s\bm{\zeta}_{\ast}) \|_{2}^{2}-\expect_{\bm{x},y}[\|\nabla E_{n}((1-s)\bm{\xi}+s\bm{\zeta}_{\ast}) \|_{2}^{2}]\big|\mathrm{d}s\mu_{0}(\mathrm{d}\bm{\xi}) \right)}\right]\\ \nonumber
&\stackrel{\rm{(a)}}{\leq} e^{-{n^{2}\delta\over 4608L^{4}}} \int_{\real^{D}}\int_{0}^{1}\expect_{\bm{x},y}\Big[e^{{n^{2}\over 4608L^{4}} (|\|\nabla E_{n}((1-s)\bm{\xi}+s\bm{\zeta}_{\ast})\|_{2}^{2}-\expect_{\bm{x},y}[\|\nabla E_{n}((1-s)\bm{\xi}+s\bm{\zeta}_{\ast})\|_{2}^{2}])|}\Big]\mathrm{d}s\mu_{0}(\mathrm{d}\bm{\xi})\\   \nonumber
&\stackrel{\rm{(b)}}{\leq} 2 e^{-{n^{2}\delta\over 4608L^{4}}},
\end{align}
where $\rm{(a)}$ follows by Jensen's inequality, and $\rm{(b)}$ follows from the fact that 
\begin{align}
\expect_{\bm{x},y}\Big[e^{{n^{2}\over 4608L^{4}} (|\|\nabla E_{n}((1-s)\bm{\xi}+s\bm{\zeta}_{\ast})\|_{2}^{2}-\expect_{\bm{x},y}[\|\nabla E_{n}((1-s)\bm{\xi}+s\bm{\zeta}_{\ast})\|_{2}^{2}])|}\Big]\leq 2,
\end{align}
due to Definition \ref{Def:Orlicz}. Therefore,
\begin{align}
\nonumber
&\prob\left(\int_{\real^{D}}\int_{0}^{1} \|\nabla E_{n}((1-s)\bm{\xi}+s\bm{\zeta}_{\ast})\|_{2}^{2}\mathrm{d}s\mu(\mathrm{d}\bm{\xi})\geq \delta \right)\\ \label{Eq:damnation_on_lina}
&\hspace{20mm}\leq 2 \exp\left({- \dfrac{n^{2}(\delta -\int_{0}^{1}\int_{\real^{D}}\expect_{\bm{x},y}[\|\nabla E_{n}((1-s)\bm{\xi}+s\bm{\zeta}_{\ast})\|_{2}^{2}]\mathrm{d}s\mu_{0}(\mathrm{d}\bm{\xi}))}{4608L^{4}}}\right).
\end{align}
It now remains to compute an upper bound on the expectation $\expect_{\bm{x},y}[\|\nabla E_{n}((1-s)\bm{\xi}+s\bm{\zeta}_{\ast})\|_{2}^{2}]$. But this readily follows from Eq. \eqref{Eq:Readily_Follows} by letting $p=1$ and $\bm{z}={{\nabla E_{n}((1-s)\bm{\xi}+s\bm{\zeta}_{\ast})}\over {\| \nabla E_{n}((1-s)\bm{\xi}+s\bm{\zeta}_{\ast})\|_{2}}}$ as follows
\begin{align}
\nonumber
\expect_{\bm{x},y}[\|\nabla E_{n}((1-s)\bm{\xi}+s\bm{\zeta}_{\ast})\|_{2}^{2}]&=\expect_{\bm{x},y}\left[\left\langle{\nabla E_{n}((1-s)\bm{\xi}+s\bm{\zeta}_{\ast})\over \| \nabla E_{n}((1-s)\bm{\xi}+s\bm{\zeta}_{\ast})\|_{2}} ,\nabla E_{n}((1-s)\bm{\xi}+s\bm{\zeta}_{\ast}) \right\rangle^{2} \right]\\ \nonumber
&=\expect_{\bm{x},y}\Big[Q^{2}_{{\nabla E_{n}((1-s)\bm{\xi}+s\bm{\zeta}_{\ast})\over \| \nabla E_{n}((1-s)\bm{\xi}+s\bm{\zeta}_{\ast})\|_{2}}}\Big]\\
\label{Eq:Expectation_Upper_Bound}
&\leq 2^{9} {L^{8}\over n^{2}}.
\end{align}
Plugging the expectation upper bound of Eq. \eqref{Eq:Expectation_Upper_Bound} into Eq. \eqref{Eq:damnation_on_lina} completes the proof of the first part of Lemma \ref{Lemma:Tail Bounds for the Finite Sample Estimation Error}.

The second part of Lemma \ref{Lemma:Tail Bounds for the Finite Sample Estimation Error} follows by a similar approach and we thus omit the proof. 

\hfill $\square$

\subsection{Proof of Lemma \ref{Lemma:A_Basic_Inequality}}
\label{Appendix:Proof_of_a_Basic_Inequality}

Let $\bm{W}_{\ast}\df \arg\min_{\bm{W}\in \mathcal{W}} \Psi(\bm{W})$ and $\bm{W}_{\diamond}\df \arg\min_{\bm{W}\in \mathcal{W}}\Phi(\bm{W})$. Then, since $|\Psi(\bm{W})-\Phi(\bm{W})|\leq \delta$ for all $\bm{W}\in \mathcal{W}$, we have that
\begin{align}
\left|\Psi(\bm{W}_{\ast})-\Phi(\bm{W}_{\ast}) \right|=\left|\min_{\bm{W}\in \mathcal{W}}\Psi(\bm{W})-\Phi(\bm{W}_{\ast})\right|\leq \delta.
\end{align}
Therefore,
\begin{align}
\label{Eq:Min+1}
\min_{\bm{W}\in \mathcal{W}} \Phi(\bm{W})\leq \Phi(\bm{W}_{\ast})\leq \min_{\bm{W}\in \mathcal{W}}\Psi(\bm{W})+\delta.
\end{align}
Similarly, it can be shown that
\begin{align}
\label{Eq:Min+2}
\min_{\bm{W}\in \mathcal{W}} \Psi(\bm{W})\leq \Psi(\bm{W}_{\diamond})\leq \min_{\bm{W}\in \mathcal{W}}\Phi(\bm{W})+\delta.
\end{align}
Combining Eqs. \eqref{Eq:Min+1} and \eqref{Eq:Min+2} yields the desired inequality. $\hfill \square$

\subsection{Proof of Lemma \ref{Lemma:Martingale}}
\label{Proof_of_Lemma:Remainder}
We recall the expression of the remainder term $\{R_{m}^{N}\}_{0\leq m\leq NT}$ from Eq. \eqref{Eq:Remainder_1}. For each $0\leq m\leq NT$, $N\in \integer$, we can bound the absolute value of the remainder term as follows
\begin{align}
\nonumber
\big|R_{m}^{N}\big|&=\dfrac{1}{N}\left| \sum_{k=1}^{N}(\bm{\xi}_{m+1}^{k}-\bm{\xi}_{m}^{k})\nabla^{2}f(\widetilde{\bm{\xi}}^{k})(\bm{\xi}_{m+1}^{k}-\bm{\xi}_{m}^{k})^{T}\right|\\ \nonumber
&\leq \dfrac{1}{N}\sum_{k=1}^{N}\left|(\bm{\xi}_{m+1}^{k}-\bm{\xi}_{m}^{k})\nabla^{2}f(\widetilde{\bm{\xi}}^{k})(\bm{\xi}_{m+1}^{k}-\bm{\xi}_{m}^{k})^{T} \right|\\
&\leq \dfrac{1}{N}\sum_{k=1}^{N}\|\bm{\xi}_{m+1}^{k}-\bm{\xi}_{m}^{k}\|_{2}^{2}\cdot \big\|\nabla^{2}f(\widetilde{\bm{\xi}}^{k})\big\|_{F}.
\end{align}
Next, we characterize a bound on the difference term $\|\bm{\xi}_{m+1}^{k}-\bm{\xi}_{m}^{k}\|_{2}$. To attain this goal, we use the iterations of the particle SGD in Equation \eqref{Eq:SGD}. We have that
\small{\begin{align}
	\nonumber
	&\|\bm{\xi}_{m+1}^{k}-\bm{\xi}_{m}^{k}\|_{2}\\ \nonumber
	&\leq {\eta\over N} \left\|\left(y_{m}\widetilde{y}_{m}- {1\over N\alpha}\sum_{k=1}^{N}\varphi(\bm{x}_{m};\bm{\xi}_{m}^{k})\varphi(\widetilde{\bm{x}}_{m};\bm{\xi}_{m}^{k})\right)\nabla_{\bm{\xi}}\Big(\varphi(\bm{x}_{m};\bm{\xi}_{m}^{k})\varphi(\widetilde{\bm{x}}_{m};\bm{\xi}_{m}^{k})\Big)\right\|_{2}\\ \nonumber
	&\leq {\eta \over N} |y_{m}\widetilde{y}_{m}|\left(|\varphi(\bm{x}_{m};\bm{\xi}_{m}^{k})|\cdot \| \nabla_{\bm{\xi}}\varphi(\widetilde{\bm{x}}_{m};\bm{\xi}_{m}^{k})\|_{2}+|\varphi(\widetilde{\bm{x}}_{m};\bm{\xi}_{m}^{k})|\cdot \| \nabla_{\bm{\xi}}\varphi(\bm{x}_{m};\bm{\xi}_{m}^{k})\|_{2}\right) \\ \nonumber
	&+{\eta\over N}\left({1 \over N\alpha}\sum_{k=1}^{N}\left|\varphi(\bm{x}_{m};\bm{\xi}_{m}^{k})\varphi(\widetilde{\bm{x}}_{m};\bm{\xi}_{m}^{k}) \right|\right) \left(|\varphi(\bm{x}_{m};\bm{\xi}_{m}^{k})| \| \nabla_{\bm{\xi}}\varphi(\widetilde{\bm{x}}_{m};\bm{\xi}_{m}^{k})\|_{2}+|\varphi(\widetilde{\bm{x}}_{m};\bm{\xi}_{m}^{k})| \| \nabla_{\bm{\xi}}\varphi(\bm{x}_{m};\bm{\xi}_{m}^{k})\|_{2}\right)\\ 
	&\stackrel{\rm{(a)}}{\leq} {\eta L^{2} \over N}+ \dfrac{2\eta L^{4}}{N\alpha}, 
	\label{Eq:Last_Inequality}
	\end{align}}\normalsize
where in $\rm{(a)}$, we used the fact that $\|\varphi\|_{\infty}<L$ and $\|\nabla_{\bm{\xi}}\varphi(\bm{x},\bm{\xi})\|_{2}<L$ due to $\textbf{(A.1)}$, and $y_{m},\widetilde{y}_{m}\in \{-1,1\}$. Plugging the last inequality in Eq. \eqref{Eq:Last_Inequality} yields
\begin{align}
\label{Inequality:R_m^N}
|R_{m}^{N}|\leq \dfrac{1}{N^{3}}\left(\eta L^{2}+{2\eta L^{4}\over \alpha}\right)\sum_{k=1}^{N}\|\nabla^{2}f(\widetilde{\bm{\xi}}^{k}) \|_{F}.
\end{align}
We next compute an upper bound on the Frobenious norm of the Hessian matrix $\nabla^{2}f(\widetilde{\bm{\xi}}^{k})$. To this end, we first show that there exists a compact set $\mathcal{C}\subset \real^{p}$ such that $\bm{\xi}_{m}^{k}\in \mathcal{C}$ for all $k=1,2,\cdots,N$ and all $m\in [0,NT]\cap \integer$. For each $k=1,2,\cdots,N$, from Inequality \eqref{Eq:Last_Inequality} we obtain that
\begin{align}
\nonumber
\|\bm{\xi}_{m}^{k}\|_{2}&\leq \|\bm{\xi}_{m-1}^{k}\|_{2}+{\eta L^{2}\over N}+\dfrac{2\eta L^{4}}{N\alpha}\\
\nonumber
&=\|\bm{\xi}_{0}^{k}\|_{2}+{m\eta L^{2}\over N}+\dfrac{2m\eta L^{4}}{N\alpha}\\ \label{Eq:From_Bound_On_x}
&\leq \|\bm{\xi}_{0}^{k}\|_{2}+\eta L^{2}T+2(\eta/\alpha) L^{4}T.
\end{align}
Now, $\|\bm{\xi}_{0}^{k}\|_{2}<c_{0}$ for some constant $c_{0}>0$ since the initial samples $\bm{\xi}_{0}^{1},\cdots,\bm{\xi}_{0}^{N}$ are drawn from the measure $\mu_{0}$ whose support $\mathrm{support}(\mu_{0})=\Xi$ is compact due to $\textbf{(A.3)}$. From upper bound in Eq. \eqref{Eq:From_Bound_On_x}, it thus follows that $\|\bm{\xi}_{m}^{k}\|_{2}<C$ for some constant $C>0$, for all $m\in [0,NT]\cap \integer$. Now, recall that $\widetilde{\bm{\xi}}^{k}=(\widetilde{\xi}^{k}(1),\cdots,\widetilde{\xi}^{k}(p))$, where $\widetilde{\xi}^{k}(i)\in [\xi_{m}^{k}(i),\xi_{m+1}^{k}(i)], i=1,2,\cdots,m+1$, for $i=1,2,\cdots,p$. Therefore, $\widetilde{\bm{\xi}}^{k}\in \mathcal{C}$.  Since all the test function $f\in C_{b}^{3}(\real^{D})$ are three-times continuously differentiable, it follows that there exists a constant $C_{0}\df C_{0}(T)>0$ such that $\sup_{\widetilde{\bm{\xi}}\in \mathcal{C}}\|\nabla^{2}f(\widetilde{\bm{\xi}})\|_{F}<C_{0}$. From Inequality \eqref{Inequality:R_m^N}, it follows that
\begin{align}
\label{Eq:Inequality_1}
|R_{m}^{N}|\leq \dfrac{C_{0}}{N^{2}}\left(\eta L^{2}+{2\eta L^{4}\over \alpha}\right), \quad m\in [0,NT]\cap \integer.
\end{align}
Now, recall the definition of the scaled term $\mathcal{R}_{t}^{N}$ from Eq. \eqref{Eq:Remainder_2}. Using the Inequality \eqref{Eq:Inequality_1} as well as the definition of $\mathcal{R}_{t}^{N}$, we obtain 
\begin{align}
\sup_{0\leq t\leq T}|\mathcal{R}_{t}^{N}| \leq \dfrac{C_{0}T}{N}\left(\eta L^{2}+{2\eta L^{4}\over \alpha}\right).
\end{align}
$\hfill \square$

\subsection{Proof of Lemma \ref{Lemma:Martingale}}
\label{Proof_of_Lemma:Martingale}

Let $\mathcal{F}_{m-1}=\sigma ((\bm{x}_{k},y_{k})_{0\leq k\leq m-1},(\widetilde{\bm{x}}_{k},\widetilde{y}_{k})_{0\leq k\leq m-1})$ denotes the $\sigma$-algebra generated by the samples up to time $m-1$. We define $\mathcal{F}_{-1}\df \emptyset$. Further, define the following random variable
\begin{align}
\Delta_{m}^{N}&\df \left( \langle\varphi(\bm{x}_{m},\bm{\xi})\varphi(\widetilde{\bm{x}}_{m},\bm{\xi}), \widehat{\mu}^{N}_{m} \rangle-\alpha y_{m} \widetilde{y}_{m} \right)\times \langle \nabla f(\bm{\xi})\nabla_{\bm{\xi}}(\varphi(\widetilde{\bm{x}}_{m};\bm{\xi})\varphi(\bm{x}_{m};\bm{\xi})),\widehat{\mu}^{N}_{m}\rangle.
\end{align}
Notice that ${1\over N}\expect[\Delta_{m}^{N}|\mathcal{F}_{m-1}]=D_{m}^{N}$. We now rewrite the martingale term in Eq. \eqref{Eq:M} in term of $\Delta_{m}^{N}$,
\begin{align}
\label{Eq:By_Construction}
M^{N}_{m}\df {\eta\over N\alpha}\sum_{\ell=0}^{m} (\Delta_{\ell}^{N}-\expect[\Delta_{\ell}^{N}|\mathcal{F}_{\ell-1}]),
\end{align}
with $M^{N}_{0}= 0$.

By construction of $M^{N}_{m}$ in Eq. \eqref{Eq:By_Construction}, it is a Martingale $\expect[M^{N}_{m}|\mathcal{F}_{m-1}]=M^{N}_{m-1}$. We now prove that $M^{N}_{m}$ has also bounded difference. To do so, we define the shorthand notations
\begin{align}
a^{N}_{m}&\df  \langle\varphi(\bm{x}_{m},\bm{\xi})\varphi(\widetilde{\bm{x}}_{m},\bm{\xi}), \widehat{\mu}^{N}_{m} \rangle-\alpha y_{m} \widetilde{y}_{m},\\
b^{N}_{m}&\df \langle \nabla f(\bm{\xi})(\nabla_{\bm{\xi}}(\varphi(\widetilde{\bm{x}}_{m};\bm{\xi})\varphi(\bm{x}_{m};\bm{\xi})))^{T},\widehat{\mu}^{N}_{m}\rangle.
\end{align}
Then, we compute
\begin{align}
\nonumber
|M^{N}_{m}-M_{m-1}^{N}|&={\eta \over N\alpha}\left|\Delta_{m}^{N}-\expect[\Delta_{m}^{N}|\mathcal{F}_{m-1}]\right|\\ \nonumber
&\leq {\eta\over N\alpha}|\Delta_{m}^{N}|+{\eta\over N\alpha}\expect[|\Delta_{m}^{N}||\mathcal{F}_{m-1}]\\ \label{Eq:Plug}
&\leq {\eta\over N\alpha}|a_{m}^{N}|\cdot |b_{m}^{N}|+{\eta\over N\alpha}\expect\left[|a_{m}^{N}|\cdot |b_{m}^{N}||\mathcal{F}_{m-1}\right].
\end{align}
For the difference terms, we derive that
\begin{align}
\nonumber
|a_{m}^{N}|&=\Big|\langle\varphi(\bm{x}_{m},\bm{\xi})\varphi(\widetilde{\bm{x}}_{m},\bm{\xi}),\widehat{\mu}_{m}^{N}\rangle-\alpha y_{m}\widetilde{y}_{m}\Big|\\ \nonumber
&\leq {1\over N}\sum_{k=1}^{N}\left|\varphi(\bm{x}_{m},\bm{\xi}_{m}^{k})\varphi(\widetilde{\bm{x}}_{m},\bm{\xi}_{m}^{k})  \right|+\alpha|y_{m}\widetilde{y}_{m}|\\ \label{Eq:Upper_Bound_on_am}
&\leq L^{2}+\alpha,
\end{align}
where the last step follows from the fact that $\|\varphi\|_{\infty}\leq L$ due to $\textbf{(A.1)}$. Similarly, we obtain that
\begin{align}
\nonumber
|b_{m}^{N}|&=|\langle \nabla f(\bm{\xi})(\nabla_{\bm{\xi}}(\varphi(\widetilde{\bm{x}}_{m};\bm{\xi})\varphi(\bm{x}_{m};\bm{\xi})))^{T},\widehat{\mu}^{N}_{m}\rangle|\\ \nonumber
&\leq {1\over N}\sum_{k=1}^{N}|\varphi(\widetilde{\bm{x}}_{m};\bm{\xi}_{m}^{k})|\cdot \left| \nabla f(\bm{\xi}_{m}^{k}) (\nabla_{\bm{\xi}}\varphi(\widetilde{\bm{x}}_{m};\bm{\xi}_{m}^{k}))^{T} \right|\\ \nonumber
&\hspace{4mm}+{1\over N}\sum_{k=1}^{N}|\varphi(\bm{x}_{m};\bm{\xi}_{m}^{k})|\cdot |\nabla f(\bm{\xi}_{m}^{k}) (\nabla_{\bm{\xi}}\varphi(\widetilde{\bm{x}}_{m};\bm{\xi}_{m}^{k}))^{T}|\\ \nonumber
&\stackrel{\rm{(a)}}{\leq} \dfrac{L}{N}\sum_{k=1}^{N} \left| \nabla f(\bm{\xi}_{m}^{k}) (\nabla_{\bm{\xi}}\varphi(\widetilde{\bm{x}}_{m};\bm{\xi}_{m}^{k}))^{T} \right|+\dfrac{L}{N}\sum_{k=1}^{N}|\nabla f(\bm{\xi}_{m}^{k}) (\nabla_{\bm{\xi}}\varphi(\widetilde{\bm{x}}_{m};\bm{\xi}_{m}^{k}))^{T}|\\ \nonumber
&\stackrel{\rm{(b)}}{\leq} \dfrac{L}{N}\sum_{k=1}^{N} \| \nabla f(\bm{\xi}_{m}^{k})\|_{2}\cdot \|\nabla_{\bm{\xi}}\varphi(\widetilde{\bm{x}}_{m};\bm{\xi}_{m}^{k})\|_{2}+\dfrac{L}{N}\sum_{k=1}^{N}\|\nabla f(\bm{\xi}_{m}^{k})\|_{2}\cdot \|\nabla_{\bm{\xi}}\varphi(\widetilde{\bm{x}}_{m};\bm{\xi}_{m}^{k})\|_{2}\\ \label{Eq:Upper_Bound_on_bm_0}
&\stackrel{\rm{(c)}}{\leq} {2L^{2}\over N}\sum_{k=1}^{N}\|\nabla f(\bm{\xi}_{m}^{k})\|_{2},
\end{align}
where $\rm{(a)}$ and $\rm{(c)}$ follows from $\mathbf{(A.1)}$, and $\rm{(b)}$ follows from the Cauchy-Schwarz inequality. From Inequality \eqref{Eq:From_Bound_On_x} and the ensuing disucssion in Appendix \ref{Proof_of_Lemma:Remainder}, we recall that $\|\bm{\xi}_{m}^{k}\|_{2}<C$ for some constant and for all $m\in [0,NT]\cap \integer$, and $k=1,2,\cdots,N$. For the two times continuously test function $f\in C^{3}_{b}(\real^{p})$, it then follows that $|\nabla f(\bm{\xi}_{m}^{k})\|_{2}\leq C_{1}$ for some constant $C_{1}>0$. The following bound can now be computed from Eq. \eqref{Eq:Upper_Bound_on_bm_0},
\begin{align}
\label{Eq:Upper_Bound_on_bm}
|b_{m}^{N}|\leq  {2C_{1}L^{2}\over N}.
\end{align}

Plugging the upper bounds on $|a^{N}_{m}|$ and $|b_{m}^{N}|$ from Eqs. \eqref{Eq:Upper_Bound_on_am}-\eqref{Eq:Upper_Bound_on_bm}  into Eq. \eqref{Eq:Plug} we obtain that
\begin{align}
|M^{N}_{m}-M_{m-1}^{N}|\leq {4\eta C_{1}\over N\alpha}L^{2}(L^{2}+\alpha).
\end{align}
Thus, $(M^{N}_{m})_{m\in [0,NT]\cap \integer}$ is a Martingale process with bounded difference. From the Azuma-Hoeffding inequality it follows that 
\begin{align}
\label{Eq:martingale_concentration_inequality}
\prob(|M^{N}_{m}|\geq \varepsilon)=\prob(|M^{N}_{m}-M_{0}^{N}|\geq \varepsilon)\leq 2\exp\left(-\dfrac{N^{2}\alpha^{2}\varepsilon^{2}}{8mL^{4}\eta^{2}C_{1}^{2}(L^{2}+\alpha)^{2} } \right), \quad \forall m\in [0,NT]\cap \integer.
\end{align}
Therefore, since $\mathcal{M}_{t}^{N}=M_{\lfloor Nt \rfloor}^{N}$, we have
\begin{align}
\label{Eq:Martingale_Inequality}
\prob(|\mathcal{M}_{T}^{N}|\geq \varepsilon)\leq 2\exp\left(-\dfrac{N^{2}\alpha^{2}\varepsilon^{2}}{8 L^{4}\lfloor NT \rfloor\eta^{2}C_{1}^{2}(L^{2}+\alpha)^{2} }\right).
\end{align}
Then, 
\begin{align}
\nonumber
\expect\Big[|\mathcal{M}_{T}^{N}|\Big]&=\int_{0}^{\infty}\prob(|\mathcal{M}_{T}^{N}|\geq \varepsilon)\mathrm{d}\varepsilon\\ \nonumber
&\leq 2\int_{0}^{\infty}\exp\left(-\dfrac{N^{2}\alpha^{2}\varepsilon^{2}}{8 L^{4}\lfloor NT \rfloor\eta^{2}C_{1}^{2}(L^{2}+\alpha)^{2} }\right)\\
&= \dfrac{1}{N\alpha}4\sqrt{2}L^{2}\sqrt{\lfloor NT \rfloor}\eta C_{1}(L^{2}+\alpha)^{2}.
\end{align}
where the inequality follows from \eqref{Eq:Martingale_Inequality}.

By Doob's Martingale inequality \cite{doob1953stochastic}, the following inequality holds
\begin{align}
\prob\left(\sup_{0\leq t\leq T} |\mathcal{M}_{t}^{N}|\geq \varepsilon\right)&\leq \dfrac{\expect[|\mathcal{M}_{T}^{N}|]}{\varepsilon}\\
&\leq \dfrac{1}{N\alpha \varepsilon}4\sqrt{2}L^{2}\sqrt{\lfloor NT \rfloor}\eta C_{1}(L^{2}+\alpha)^{2}.
\end{align}
In particular, with the probability of at least $1-\rho$, we have
\begin{align}
\sup_{0\leq t\leq T}|\mathcal{M}_{t}^{N}|\leq  \dfrac{1}{N\alpha \rho}4\sqrt{2}L^{2}\sqrt{\lfloor NT \rfloor}\eta C_{1}(L^{2}+\alpha)^{2}.
\end{align}
$\hfill$ $\square$

\bibliographystyle{plain}
\bibliography{tell}

\end{document}